\documentclass{article}

\usepackage[round]{natbib}

\usepackage[utf8]{inputenc} % allow utf-8 input
\usepackage[T1]{fontenc}    % use 8-bit T1 fonts
\usepackage{hyperref}       % hyperlinks
\usepackage{url}            % simple URL typesetting
\usepackage{booktabs}       % professional-quality tables
\usepackage{amsfonts}       % blackboard math symbols
\usepackage{nicefrac}       % compact symbols for 1/2, etc.
\usepackage{microtype}      % microtypography
\usepackage{xcolor}         % colors

\usepackage{amsmath,amssymb,amsthm,mathrsfs,mathtools,accents}
\usepackage{algorithm,algorithmicx}
\usepackage{graphicx,subfig}
\usepackage{epstopdf}
\usepackage{float,multirow,enumitem,makecell}
\usepackage{lmodern}

\makeatletter
\def\widebar{\accentset{{\cc@style\underline{\mskip10mu}}}}
\makeatother

\makeatletter
\newsavebox\myboxA
\newsavebox\myboxB
\newlength\mylenA

\newcommand*\xoverline[2][0.5]{%
    \sbox{\myboxA}{$\m@th#2$}%
    \setbox\myboxB\null% Phantom box
    \ht\myboxB=\ht\myboxA%
    \dp\myboxB=\dp\myboxA%
    \wd\myboxB=#1\wd\myboxA% Scale phantom
    \sbox\myboxB{$\m@th\overline{\copy\myboxB}$}%  Overlined phantom
    \setlength\mylenA{\the\wd\myboxA}%   calc width diff
    \addtolength\mylenA{-\the\wd\myboxB}%
    \ifdim\wd\myboxB<\wd\myboxA%
       \rlap{\hskip 0.8\mylenA\usebox\myboxB}{\usebox\myboxA}%
    \else
        \hskip -0.5\mylenA\rlap{\usebox\myboxA}{\hskip 0.5\mylenA\usebox\myboxB}%
    \fi}
\makeatother

\theoremstyle{plain}
\newtheorem{theorem}{Theorem}

\newtheorem{lemma}{Lemma}

\newtheorem{proposition}{Proposition}
\newtheorem{corollary}{Corollary}
\newtheorem{definition}{Definition}

\graphicspath{{Figures/}}

\newcommand{\bzero}{\mathbf{0}}

\newcommand{\C}{\mathbb{C}}

\newcommand{\R}{\mathbb{R}}
\newcommand{\ch}{\mathfrak{h}}
\newcommand{\ce}{\mathfrak{e}}

\newcommand{\bA}{\boldsymbol{A}}

\newcommand{\bB}{\boldsymbol{B}}

\newcommand{\bC}{\boldsymbol{C}}

\newcommand{\be}{\boldsymbol{e}}

\newcommand{\bF}{\boldsymbol{F}}

\newcommand{\bG}{\boldsymbol{G}}
\newcommand{\bh}{\boldsymbol{h}}
\newcommand{\bH}{\boldsymbol{H}}
\newcommand{\bI}{\boldsymbol{I}}

\newcommand{\bL}{\boldsymbol{L}}

\newcommand{\bM}{\boldsymbol{M}}

\newcommand{\bQ}{\boldsymbol{Q}}
\newcommand{\br}{\boldsymbol{r}}
\newcommand{\bR}{\boldsymbol{R}}
\newcommand{\bS}{\boldsymbol{S}}
\newcommand{\bT}{\boldsymbol{T}}

\newcommand{\bU}{\boldsymbol{U}}
\newcommand{\bv}{\boldsymbol{v}}
\newcommand{\bV}{\boldsymbol{V}}

\newcommand{\bW}{\boldsymbol{W}}
\newcommand{\bx}{\boldsymbol{x}}
\newcommand{\bX}{\boldsymbol{X}}
\newcommand{\by}{\boldsymbol{y}}
\newcommand{\bY}{\boldsymbol{Y}}
\newcommand{\bz}{\boldsymbol{z}}
\newcommand{\bZ}{\boldsymbol{Z}}

\newcommand{\cA}{\mathcal{A}}
\newcommand{\cB}{\mathcal{B}}
\newcommand{\cC}{\mathcal{C}}

\newcommand{\cF}{\mathcal{F}}
\newcommand{\cG}{\mathcal{G}}
\newcommand{\cH}{\mathcal{H}}
\newcommand{\cI}{\mathcal{I}}
\newcommand{\cJ}{\mathcal{J}}

\newcommand{\cL}{\mathcal{L}}
\newcommand{\cM}{\mathcal{M}}
\newcommand{\cN}{\mathcal{N}}
\newcommand{\cO}{\mathcal{O}}
\newcommand{\cP}{\mathcal{P}}
\newcommand{\cQ}{\mathcal{Q}}
\newcommand{\cR}{\mathcal{R}}
\newcommand{\cS}{\mathcal{S}}
\newcommand{\cT}{\mathcal{T}}
\newcommand{\cU}{\mathcal{U}}
\newcommand{\cV}{\mathcal{V}}
\newcommand{\cW}{\mathcal{W}}
\newcommand{\cX}{\mathcal{X}}
\newcommand{\cY}{\mathcal{Y}}
\newcommand{\cZ}{\mathcal{Z}}

\newcommand{\bcA}{\boldsymbol{\cA}}
\newcommand{\bcB}{\boldsymbol{\cB}}
\newcommand{\bcC}{\boldsymbol{\cC}}

\newcommand{\bcF}{\boldsymbol{\cF}}
\newcommand{\bcG}{\boldsymbol{\cG}}
\newcommand{\bcH}{\boldsymbol{\cH}}
\newcommand{\bcI}{\boldsymbol{\cI}}
\newcommand{\bcJ}{\boldsymbol{\cJ}}

\newcommand{\bcL}{\boldsymbol{\cL}}
\newcommand{\bcM}{\boldsymbol{\cM}}

\newcommand{\bcP}{\boldsymbol{\cP}}
\newcommand{\bcQ}{\boldsymbol{\cQ}}
\newcommand{\bcR}{\boldsymbol{\cR}}
\newcommand{\bcS}{\boldsymbol{\cS}}

\newcommand{\bcU}{\boldsymbol{\cU}}
\newcommand{\bcV}{\boldsymbol{\cV}}
\newcommand{\bcW}{\boldsymbol{\cW}}
\newcommand{\bcX}{\boldsymbol{\cX}}
\newcommand{\bcY}{\boldsymbol{\cY}}
\newcommand{\bcZ}{\boldsymbol{\cZ}}

\newcommand{\bPhi}{\boldsymbol{\Phi}}

\newcommand{\bOmega}{\boldsymbol{\Omega}}

\DeclareMathOperator*{\argmin}{arg\,min}

\newcommand{\sgn}{\mathrm{sgn}}

\newcommand{\dist}{\mathrm{dist}}

\newcommand{\tr}{\mathrm{tr}}

\newcommand{\supp}{\mathrm{supp}}

\title{Guaranteed Nonconvex Low-Rank Tensor Estimation via Scaled Gradient Descent}

\newcommand{\email}[1]{\href{mailto:#1}{#1}}

\author{
Tong Wu\thanks{Beijing Institute for General Artificial Intelligence, Beijing, China. (\email{wutong@bigai.ai}).}
}

\begin{document}

\maketitle

\vspace{5mm}
\begin{abstract}
Tensors, which give a faithful and effective representation to deliver the intrinsic structure of multi-dimensional data, play a crucial role in an increasing number of signal processing and machine learning problems. However, tensor data are often accompanied by arbitrary signal corruptions, including missing entries and sparse noise. A fundamental challenge is to reliably extract the meaningful information from corrupted tensor data in a statistically and computationally efficient manner. This paper develops a scaled gradient descent (ScaledGD) algorithm to directly estimate the tensor factors with tailored spectral initializations under the tensor-tensor product (t-product) and tensor singular value decomposition (t-SVD) framework. With tailored variants for tensor robust principal component analysis, (robust) tensor completion and tensor regression, we theoretically show that ScaledGD achieves linear convergence at a constant rate that is independent of the condition number of the ground truth low-rank tensor, while maintaining the low per-iteration cost of gradient descent. To the best of our knowledge, ScaledGD is the first algorithm that provably has such properties for low-rank tensor estimation with the t-SVD. Finally, numerical examples are provided to demonstrate the efficacy of ScaledGD in accelerating the convergence rate of ill-conditioned low-rank tensor estimation in a number of applications.
\end{abstract}

\section{Introduction}

With the increasing availability of high-dimensional and multiway datasets, we are witnessing a growing demand for efficient data analysis techniques. Many practical applications collect data that are naturally in the form of a tensor. Instances of tensor data include images, videos, hyperspectral images, and signals generated by magnetic resonance systems. Tensors arise naturally as a powerful model for better capturing the multiway relationships and interactions within data than matrices; examples include image processing \citep{LiuMWY.PAMI2013}, climate forecasting \citep{YuCL.ICML2015}, topic modeling \citep{AnandkumarGHKT.JMLR2014}, and neuroimaging data analysis \citep{AhmedRB.SIMODS2020}. In this work, specifically, we consider the tensor estimation problem, which is the central task across many problems. Mathematically, the goal of tensor estimation is to estimate a $K$-way tensor $\bcX_{\star} \in \R^{n_1 \times \dots \times n_K}$ from a set of observations $\by \in \R^m$ given by
\begin{align}\label{eqn:probsetup}
\by \approx \cA(\bcX_{\star}),
\end{align}
where $\cA(\cdot)$ is a linear map that models the data collection process. For ease of presentation, we consider the case $K = 3$ throughout the paper, while the general case will be investigated in future work. Arguably, the data of interest can be well represented by a much smaller number of latent factors compared to the dimensionality of the ambient space, which suggests exploiting low-dimensional geometric structures for meaningful recovery.

\subsection{Low-rank Tensor Estimation}

One of the most widely adopted low-rank tensor decompositions---which is the focus of this paper---considers low-rank structure under the \emph{tensor singular value decomposition (t-SVD)} \citep{KilmerBHH.SIMAX2013}. Specifically, we assume the ground truth tensor $\bcX_{\star}$ of tensor tubal rank-$r$ that admits the following t-SVD:
\begin{align*}
\bcX_{\star} = \bcU_{\star} \ast_{\bPhi} \bcG_{\star} \ast_{\bPhi} \bcV_{\star}^H,
\end{align*}
where $\ast_{\bPhi}$ denotes the t-product evoked by the transformation matrix $\bPhi$, $\bcU_{\star} \in \R^{n_1 \times r \times n_3}$ and $\bcV_{\star} \in \R^{n_2 \times r \times n_3}$ are orthogonal tensors, and $\bcG_{\star} \in \R^{r \times r \times n_3}$ is an f-diagonal tensor (see Definitions in Section~\ref{ssec:tsvddef}). Compared with other tensor decomposition strategies such as the CANDECOMP/PARAFAC (CP) decomposition \citep{Kiers.2000} and Tucker decomposition \citep{Tucker.1966}, t-SVD can take advantage of structures in both the original and transformed domains. It has been observed that after conducting Discrete Fourier Transform (DFT) along the third dimension of a three-way tensor, the transformed tensor may exhibit low-rank structure in the Fourier domain \citep{ZhangEAHK.CVPR2014}, which can be characterized by the tensor tubal rank and tensor nuclear norm \citep{KilmerBHH.SIMAX2013}. Low-rankness in the spectral domain extends traditional models that require strong low-rankness in the original domain. The t-SVD based models evoked by the tubal rank own the same tight recovery bound as the matrix cases \citep{LuFLY.IJCAI2018}. Another advantage of the t-SVD scheme is its superior capability in capturing the ``spatial-shifting'' correlation that is ubiquitous in real multiway data. Interestingly, the t-SVD can be generalized by replacing DFT with any invertible linear transforms \citep{KernfeldKA.LAA2015}. Driven by these advantages, extensive numerical examples have demonstrated its efficacy in many applications \citep{ZhangA.TSP2017,LuFLY.IJCAI2018,LuPW.CVPR2019,SongNZ.NLAA2020,KilmerHAN.PNAS2021,ZhouLFLY.PAMI2021,KongLL.ML2021,Lu.ICCV2021,QinWZWLH.TIP2022,Wu.ML2026}.

\paragraph{Motivating examples.} We focus on optimization problems that look for low-rank tensors using partial or corrupted observations.
\begin{itemize}
  \item \emph{Tensor RPCA.} We first consider the tensor robust principal component analysis (RPCA) problem, which attempts to find a low-rank tensor that best approximates grossly corrupted observations. Mathematically, imagine that we are given a data tensor
\begin{align}\label{eqn:TRPCAprob}
\bcY = \bcX_{\star} + \bcS_{\star},
\end{align}
where $\bcX_{\star}$ is low-rank and $\bcS_{\star}$ is sparse, and both components are of arbitrary magnitudes. Our goal is to recover $\bcX_{\star}$ and $\bcS_{\star}$ from the corrupted observation $\bcY$ in an efficient manner.
  \item \emph{Tensor completion.} We then study the low-rank tensor completion problem, which aims to recover a low-rank tensor from only a small subset of its revealed observations. Mathematically, we are given entries
\begin{align*}
[\bcX_{\star}]_{i,j,k}, \quad (i,j,k) \in \bOmega,
\end{align*}
where $\bOmega$ denotes the set of the indices of the observed entries. The goal is to recover $\bcX_{\star}$ from the observed entries in $\bOmega$.
  \item \emph{Robust tensor completion (RTC)}. In this problem, we attempt to recover an underlying low-rank tensor by observing a small number of sparsely corrupted entries. Formally, let $\bcY$ be a tensor with a decomposition $\bcY = \bcX_{\star} + \bcS_{\star}$, where $\bcX_{\star}$ is low-rank and $\bcS_{\star}$ is a sparse corruption tensor whose few non-zero entries are arbitrary. The RTC problem is to recover $\bcX_{\star}$ from the observed entries $\{ \bcY_{i,j,k} | (i,j,k) \in \bOmega \}$, where $\bOmega$ denotes the locations of the observed entries in $\bcY$.
  \item \emph{Tensor regression.} We finally consider a linear tensor-structured regression model, where the $i$-th response or observation is given by
  \begin{align*}
  y_i = \langle \bcA_i, \bcX_{\star} \rangle, \quad i = 1, 2, \dots, m,
  \end{align*}
  where $\bcA_i \in \R^{n_1 \times n_2 \times n_3}$ is the $i$-th measurement tensor and $\langle \cdot , \cdot \rangle$ denotes the canonical inner product. The goal is to recover the low-rank tensor $\bcX_{\star}$ from the responses $\by = \{ y_i \}_{i=1}^m$.
\end{itemize}

\subsection{Main Contributions}

In view of the low-rank t-SVD, a natural approach to estimate the low-rank tensor $\bcX_{\star}$ in \eqref{eqn:probsetup} is to parameterize $\bcX = \bcL \ast_{\bPhi} \bcR^H$ by two factor tensors $\bcL \in \R^{n_1 \times r \times n_3}$ and $\bcR \in \R^{n_2 \times r \times n_3}$, and then to estimate the ground truth factors via optimizing the unconstrained least-squares loss function:
\begin{align}
\min_{\bcL \in \R^{n_1 \times r \times n_3}, \bcR \in \R^{n_2 \times r \times n_3}} f(\bcL, \bcR) \coloneq \| \cA(\bcL \ast_{\bPhi} \bcR^H) - \by \|_2^2,
\end{align}
where $\bcR^H$ denotes the conjugate transpose of $\bcR$ (see Definition~\ref{def:conjtrans}) and we use the vector $\ell_2$ norm to quantify the error. Despite the nonconvexity of the objective function, a simple and intuitive approach is to update the tensor factors via gradient descent. While remarkable progress has been made in recent years in the matrix setting \citep{ChiLC.TSP2019}, this line of research still remains relatively unexplored for the tensor setting, especially when it comes to provable sample and computational guarantees.

Motivated by the recent success of scaled gradient descent (ScaledGD) \citep{TongMC.TSP2021,TongMC.JMLR2021,TongMC.JMLR2022,DongTMC.IMA2023} for accelerating ill-conditioned low-rank estimation, we propose to update the tensor factors iteratively along the scaled gradient directions:
\begin{align}\label{eqn:scaledGD}
\bcL_{t+1} & = \bcL_t - \eta \nabla_{\bcL} f(\bcL_t, \bcR_t) \ast_{\bPhi} (\bcR_t^H \ast_{\bPhi} \bcR_t)^{-1},  \nonumber \\
\bcR_{t+1} & = \bcR_t - \eta \nabla_{\bcR} f(\bcL_t, \bcR_t) \ast_{\bPhi} (\bcL_t^H \ast_{\bPhi} \bcL_t)^{-1},
\end{align}
where $\eta > 0$ is the learning rate and $\nabla_{\bcL} f(\bcL_t, \bcR_t)$ (resp., $\nabla_{\bcR} f(\bcL_t, \bcR_t)$) denotes the gradient of $f$ with respect to $\bcL_t$ (resp., $\bcR_t$) at the $t$-th iteration. Compared with vanilla gradient descent, our method scales or preconditions the search directions of $\bcL_t$ and $\bcR_t$ in \eqref{eqn:scaledGD} by $(\bcR_t^H \ast_{\bPhi} \bcR_t)^{-1}$ and $(\bcL_t^H \ast_{\bPhi} \bcL_t)^{-1}$, respectively. With the preconditioners, ScaledGD updates the tensor factors with better search directions and larger step sizes. As the preconditioners are computed by inverting two $r \times r \times n_3$ tensors, whose size is much smaller than the dimension of the tensor factors, each iteration of ScaledGD only adds minimal overhead to the gradient computation, while maintaining a linear rate of convergence regardless of the condition number.

\paragraph{Theoretical guarantees.} We investigate the theoretical properties of ScaledGD, which are notably more challenging than the matrix counterpart. Our model is more generic as it is allowed to use any invertible linear transforms that satisfy certain conditions. We establish that ScaledGD---when initialized with a tailored spectral initialization scheme---can achieve linear convergence at a rate \emph{independent} of the condition number of the ground truth tensor with near-optimal sample complexities. To be concrete, we have the following guarantees:
\begin{itemize}
  \item \emph{Tensor RPCA.} Under the deterministic corruption model \citep{LuFCLLY.PAMI2020}, ScaledGD converges linearly to the true low-rank tensor in both the Frobenius norm and the entrywise $\ell_{\infty}$ norm, as long as the level of corruptions---measured in terms of the fraction of nonzero entries per tube in each mode---does not exceed the order of $\alpha \lesssim \frac{n_3 \sqrt{n_{(2)} n_2}}{\mu s_r^{1.5} \kappa (n_1 + n_2 n_3)} \wedge \frac{n_3}{\mu^2 s_r^2 \kappa^2}$, where $n_{(2)} = \min(n_1, n_2)$, $\mu$ and $\kappa$ are the incoherence parameter and the condition number of the ground truth tensor $\bcX_{\star}$, respectively, and $s_r$ is the summation of all the entries in the multi-rank of $\bcX_{\star}$ (to be formally defined later).
  \item \emph{Tensor completion.} Under the Bernoulli sampling model, ScaledGD combined with a properly designed scaled projection step reaches $\epsilon$-accuracy in $\cO (\log(1/\epsilon))$ iterations, as long as the sample complexity satisfies $p \gtrsim \Big( \frac{\mu s_r (n_1 + n_2) \log( (n_1 \vee n_2) n_3 )}{n_1 n_2 n_3} \vee \frac{\mu^2 s_r^2 \kappa^4 \ell \log( (n_1 \vee n_2) n_3 )}{(n_1 \wedge n_2) n_3^2} \Big)$, where $\ell$ is a constant related to the transformation matrix $\bPhi$.
  \item \emph{Robust tensor completion.} Under the random Bernoulli observation model, ScaledGD succeeds with high probability as long as the fraction of corrupted entries $\alpha \lesssim \frac{p n_3 \sqrt{n_{(2)} n_2}}{\mu s_r^{1.5} \kappa (n_1 + n_2 n_3)} \wedge \frac{p n_3}{\mu^2 s_r^2 \kappa^2}$ and the probability of observation satisfies $p \gtrsim \log( (n_1 \vee n_2) n_3 ) \max \Big( \frac{\mu^{1.5} s_r^{1.5} (n_1 + n_2)}{n_3 \sqrt{n_1 n_2}}, \frac{\mu^2 s_r^2 \ell}{n_3 \sqrt{(n_1 \wedge n_2) n_3}}, \frac{\mu s_r \kappa^2}{\sqrt{(n_1 \wedge n_2) n_3}} \Big)$. Note that the recent work \citet{CaiKLY.PAMI2026} proposed a scalable method for robust matrix completion based on ScaledGD, the recovery guarantee derived in \citet{CaiKLY.PAMI2026} is under the assumption that the matrix is fully observed ($p = 1$). Remarkably, our analysis extends this result to the tensor case and considers the general case where the tensor is only partially observed.
  \item \emph{Tensor regression.} As long as the measurement operator satisfies the tensor restricted isometry property (TRIP) \citep{RauhutSS.LAA2017} with a TRIP constant $\delta_{2r} \lesssim 1/(\sqrt{\frac{s_r}{\ell}} \kappa)$, ScaledGD reaches $\epsilon$-accuracy in $\cO(\log(1/\epsilon))$ iterations when initialized by the spectral method.
\end{itemize}

\subsection{Related Work}

\paragraph{Scaled first-order methods for low-rank matrix recovery.} Variants of the scaled gradient methods have been proposed for low-rank matrix recovery \citep{MishraAS.arXiv2012,MishraS.SIOPT2016,TannerW.ACHA2016,TongMC.JMLR2021} that aim to approximate the low-rank matrix by the product of two factor matrices, where strong statistical guarantees were first established in \citet{TongMC.JMLR2021}. The matrix RPCA method in \citet{CaiLY.NeurIPS2021} extended \citet{TongMC.JMLR2021} by using a threshold-based trimming procedure to identify the sparse matrix. \citet{JiaWPFM.NeurIPS2023} proved the global convergence of ScaledGD and alternating scaled gradient descent (AltScaledGD) for the nonconvex low-rank matrix factorization problem. There have been many other efforts in the literature to solve the low-rank matrix recovery problem with provable nonconvex optimization procedures; see, e.g., \citet{ChenW.arXiv2015,GuWL.AISTATS2016,YiPCC.NIPS2016,ZhengL.arXiv2016,GeJZ.ICML2017,CherapanamjeriGJ.ICML2017,DuHL.NeurIPS2018,ZengS.TSP2018,LiZT.IMA2019,YeD.NeurIPS2021} for an incomplete list.

\paragraph{Low-rank tensor estimation using t-SVD.} Turning to the tensor case, unfolding-based approaches typically lead to performance degradation since they neglect the high-order interactions within tensor data. Recently, motivated by the notion of tensor-tensor product (t-product) \citep{KilmerBHH.SIMAX2013} and t-SVD scheme, a new tensor nuclear norm was proposed and applied in tensor RPCA \citep{LuFCLLY.PAMI2020,GaoZXXGT.PAMI2021,Lu.ICCV2021}, tensor completion \citep{ZhangA.TSP2017,LuFLY.IJCAI2018,LuPW.CVPR2019,SongNZ.NLAA2020,QinWZWLH.TIP2022}, and tensor data clustering \citep{ZhouLFLY.PAMI2021,Wu.AISTATS2024,Wu.ML2026}. To accelerate tensor RPCA, \citet{QiuWTMY.ICML2022} proposed two alternating projection algorithms that exhibit linear convergence behavior. While this paper and \citet{QiuWTMY.ICML2022} share the same tensor RPCA setup, our work is fundamentally different from \citet{QiuWTMY.ICML2022}. The methods proposed in \citet{QiuWTMY.ICML2022} are based on the idea of alternating projection, which can be considered an extension of \citet{NetrapalliNSAJ.NIPS2014,CaiCW.JMLR2019} in the matrix case to the case of tensors. In contrast, our work draws inspiration from \citet{TongMC.JMLR2021,CaiLY.NeurIPS2021} and parameterizes the low-rank term in \eqref{eqn:TRPCAprob} by two low-rank factors. To the best of our knowledge, this paper is the first work that provides rigorous statistical and computational guarantees for scaled gradient methods based on the t-SVD framework. There are many technical novelty in our analysis compared to \citet{TongMC.JMLR2021,CaiLY.NeurIPS2021}. Specifically, the optimization problems in \citet{TongMC.JMLR2021,CaiLY.NeurIPS2021} primarily involve matrix-valued variables, whereas our optimization problem primarily involves tensor-valued variables. Because of this reason, all the inequalities involving matrices in \citet{TongMC.JMLR2021,CaiLY.NeurIPS2021} must be proved using some tensor properties involving t-product in our analysis.

\paragraph{Provable low-rank tensor estimation using other decompositions.} As mentioned earlier, it is not straightforward to generalize low-rank matrix estimation methods to tensors because a tensor can be decomposed in many ways. Beyond the t-SVD, common tensor decompositions include CP \citep{Kiers.2000}, Tucker \citep{Tucker.1966}, HOSVD \citep{LathauwerMV.SIMAX2000}, and tensor-train \citep{Oseledets.SISC2011}. Several works for low-rank tensor estimation relying on these decompositions have been proposed \citep{GoldfarbQ.SIMAX2014,AnandkumarJSN.AISTATS2016,DriggsBB.arXiv2019,XiaY.FoCM2019,YuanZGC.SPIC2019,YangZJMH.AMC2020,CaiLX.JASA2022,QinWZ.JMLR2024}. Recently, the authors in \citet{TongMC.JMLR2022,DongTMC.IMA2023} developed efficient algorithms for tensor RPCA, tensor completion and tensor regression under the Tucker decomposition using scaled gradient descent. In contrast to our work, the low-rank tensor in \citet{TongMC.JMLR2022,DongTMC.IMA2023} is factorized as $(\bU, \bV, \bW) \cdot \bcS$, and four factors are needed to be estimated, leading to a much more complicated nonconvex landscape than our case.

\subsection{Paper Organization}

The rest of this paper is organized as follows. Section~\ref{sec:prelim} presents some notations and preliminaries. Section~\ref{sec:theory} describes the applications of the proposed ScaledGD method to the four problems studied in this paper with theoretical guarantees in terms of both statistical and computational complexities. In Section~\ref{sec:proofoutline}, we outline the proof for our main results. Section~\ref{sec:experiment} illustrates the superior empirical performance of the proposed method. Finally, we conclude in Section~\ref{sec:conclude}. The proofs are deferred to the appendix.

\section{Notations and Preliminaries}
\label{sec:prelim}

In this section, we introduce the notations and provide a concise overview of t-SVD, which establishes the groundwork for the development of our algorithms.

\paragraph{Notations.} We use lowercase, bold lowercase, bold uppercase, and bold calligraphic letters for scalars, vectors, matrices, and tensors, respectively. The real and complex Euclidean spaces are denoted as $\R$ and $\C$, respectively. Superscript $H$ and $T$ denote conjugate transpose and transpose, respectively. For a three-way tensor $\bcA \in \C^{n_1 \times n_2 \times n_3}$, we use $\bcA(i,:,:)$, $\bcA(:,i,:)$ and $\bcA(:,:,i)$ to denote its $i$-th horizontal, lateral and frontal slice, respectively. The $(i,j,k)$-th entry of $\bcA$ is denoted as $\bcA_{i,j,k}$. For brevity, the frontal slice $\bcA(:,:,i)$ and the lateral slice $\bcA(:,i,:)$ are denoted compactly as $\bA^{(i)}$ and $\bcA_{(i)}$, respectively. The $(i,j)$-th mode-1, mode-2 and mode-3 tubes are denoted by $\bcA(:,i,j)$, $\bcA(i,:,j)$ and $\bcA(i,j,:)$, respectively. For a matrix $\bA$, its ($i,j$)-th entry is denoted as $\bA_{i,j}$. The $i$-th row and the $i$-th column of $\bA$ are denoted by $\bA_{i,\cdot}$ and $\bA_{\cdot,i}$, respectively. The $n \times n$ identity matrix is denoted by $\bI_n$. The inner product between two matrices $\bA, \bB \in \C^{n_1 \times n_2}$ is defined as $\langle \bA, \bB \rangle = \tr(\bA^H \bB)$, where $\tr(\cdot)$ denotes the matrix trace. The inner product between two tensors $\bcA, \bcB \in \C^{n_1 \times n_2 \times n_3}$ is defined as $\langle \bcA, \bcB \rangle = \sum_{i=1}^{n_3} \langle \bA^{(i)}, \bB^{(i)} \rangle$. The facewise product between two tensors $\bcA \in \C^{n_1 \times n_2 \times n_3}$ and $\bcB \in \C^{n_2 \times n_4 \times n_3}$, denoted by $\bcA \triangle \bcB$, is a tensor $\bcC \in \C^{n_1 \times n_4 \times n_3}$ such that each frontal slice of $\bcC$ is the matrix multiplication of the corresponding frontal slices of $\bcA$ and $\bcB$, that is, $\bC^{(i)} = \bA^{(i)} \bB^{(i)}$ \citep{KernfeldKA.LAA2015}. Let $a \vee b = \max ( a, b )$, $a \wedge b = \min ( a, b )$, and $[a] \coloneq \{1, 2, \dots, a\}$. Further, $f(n) \gtrsim g(n)$ (resp., $f(n) \lesssim g(n)$) means $|f(n)|/|g(n)| \geq c$ (resp., $|f(n)|/|g(n)| \leq c$) for some constant $c > 0$ when $n$ is sufficiently large. We use the terminology ``with high probability'' to denote the event happens with probability at least $1 - c_1 n^{-c_2}$, where $c_1$ and $c_2$ denote positive universal constants, which are allowed to differ from line to line.

Some norms of vector, matrix and tensor are used. The $\ell_1$, $\ell_{\infty}$, $\ell_{2,\infty}$ and Frobenius norms of $\bcA$ are defined as $\|\bcA\|_1 = \sum_{i,j,k} |\bcA_{i,j,k}|$, $\|\bcA\|_{\infty} = \max_{i,j,k} |\bcA_{i,j,k}|$, $\|\bcA\|_{2,\infty} = \max_i \| \bcA(i,:,:) \|_F$ and $\|\bcA\|_F = \sqrt{\sum_{i,j,k} |\bcA_{i,j,k}|^2}$, respectively. For $\bv \in \C^n$, its $\ell_2$ norm is denoted as $\| \bv \|_2 = \sqrt{\sum_i |v_i|^2}$ and we use $\ell_0$ norm to denote the number of nonzero elements in $\bv$. The spectral norm of a matrix $\bA$ is denoted as $\| \bA \| = \max_i \sigma_i (\bA)$, where $\sigma_i (\bA)$'s are the singular values of $\bA$. The matrix nuclear norm of $\bA$ is $\| \bA \|_{\ast} = \sum_i \sigma_i (\bA)$.

\subsection{Tensor Singular Value Decomposition}
\label{ssec:tsvddef}

The framework of tensor singular value decomposition (t-SVD) is based on the t-product under an invertible linear transform $L$ \citep{KernfeldKA.LAA2015}. In this paper, the transformation matrix $\bPhi$ defining the transform $L$ is restricted to be orthogonal, i.e., $\bPhi \in \C^{n_3 \times n_3}$ satisfying
\begin{align}\label{eqn:phiconstraint}
\bPhi \bPhi^H = \bPhi^H \bPhi = \ell \bI_{n_3},
\end{align}
where $\ell > 0$ is a constant. We define the associated linear transform $L(\cdot)$ with its inverse mapping $L^{-1}(\cdot)$ on any $\bcA \in \R^{n_1 \times n_2 \times n_3}$ as
\begin{align}\label{eqn:mode3prod}
\xoverline{\bcA} \coloneq L(\bcA) = \bcA \times_3 \bPhi \quad \mathrm{and} \quad L^{-1}(\bcA) \coloneq \bcA \times_3 \bPhi^{-1},
\end{align}
where $\times_3$ denotes the mode-3 tensor-matrix product \citep{KoldaB.Rev2009}, i.e., for any $\bcA \in \C^{n_1 \times n_2 \times n_3}$ and $\bM \in \C^{n_4 \times n_3}$, $\bcB = \bcA \times_3 \bM \in \C^{n_1 \times n_2 \times n_4}$ such that
\begin{align*}
\bcB_{i,j,l} = \sum_{k=1}^{n_3} \bcA_{i,j,k} \bM_{l,k}, \quad i \in [n_1], j \in [n_2], l \in [n_4].
\end{align*}
\begin{definition}[t-product \citep{KernfeldKA.LAA2015}]
The t-product of any $\bcA \in \C^{n_1 \times n_2 \times n_3}$ and $\bcB \in \C^{n_2 \times n_4 \times n_3}$ under transform $L$ in \eqref{eqn:mode3prod}, is a unique tensor $\bcC = \bcA \ast_{\bPhi} \bcB \in \C^{n_1 \times n_4 \times n_3}$ such that $L(\bcC) = L(\bcA) \triangle L(\bcB)$.
\end{definition}
We denote $\widebar{\bA} \in \C^{n_1 n_3 \times n_2 n_3}$ as the block diagonal matrix with its $i$-th diagonal block corresponding to the $i$-th frontal slice $\widebar{\bA}^{(i)}$ of $\xoverline{\bcA}$ in \eqref{eqn:mode3prod}, i.e.,
\begin{align}\label{eqn:defbdiag}
\widebar{\bA} = \mathtt{bdiag}(\xoverline{\bcA}) =
\begin{bmatrix}
\widebar{\bA}^{(1)} &  &  \\
 & \ddots &  \\
 & & \widebar{\bA}^{(n_3)}
\end{bmatrix},
\end{align}
where $\mathtt{bdiag}$ is an operator that maps $\xoverline{\bcA}$ to $\widebar{\bA}$. Using \eqref{eqn:phiconstraint}, we have the following properties:
\begin{align}\label{eqn:tensorproperty}
\|\bcA\|_F = \frac{1}{\sqrt{\ell}} \|\xoverline{\bcA}\|_F = \frac{1}{\sqrt{\ell}} \|\widebar{\bA}\|_F \quad \mathrm{and} \quad \langle \bcA, \bcB \rangle = \frac{1}{\ell} \langle \xoverline{\bcA}, \xoverline{\bcB} \rangle = \frac{1}{\ell} \langle \widebar{\bA}, \widebar{\bB} \rangle.
\end{align}
\begin{definition}[Conjugate transpose \citep{SongNZ.NLAA2020}]\label{def:conjtrans}
The conjugate transpose of a tensor $\bcA \in \C^{n_1 \times n_2 \times n_3}$ under $L$ is the tensor $\bcA^H \in \C^{n_2 \times n_1 \times n_3}$ satisfying $L(\bcA^H)^{(i)} = (L(\bcA)^{(i)})^H$, $i = 1, 2, \dots, n_3$.
\end{definition}
\begin{definition}[Identity tensor \citep{SongNZ.NLAA2020}]
The identity tensor $\bcI_n \in \C^{n \times n \times n_3}$ under $L$ is a tensor such that each frontal slice of $L(\bcI_n) = \xoverline{\bcI}_n$ is the identity matrix $\bI_n$. Then $\bcI_n = L^{-1}(\xoverline{\bcI}_n)$ gives the identity tensor under $L$.
\end{definition}
\begin{definition}[Orthogonal tensor \citep{SongNZ.NLAA2020}]
A tensor $\bcQ \in \C^{n \times n \times n_3}$ is orthogonal if it satisfies $\bcQ \ast_{\bPhi} \bcQ^H = \bcQ^H \ast_{\bPhi} \bcQ = \bcI_n$.
\end{definition}
\begin{definition}[Invertible tensor]
A tensor $\bcA \in \R^{r \times r \times n_3}$ is invertible if there exists a tensor $\bcB \in \R^{r \times r \times n_3}$ such that $\bcA \ast_{\bPhi} \bcB = \bcB \ast_{\bPhi} \bcA = \bcI_r$. The set of invertible tensors in $\R^{r \times r \times n_3}$ is denoted by $\mathrm{GL}(r)$.
\end{definition}
\begin{definition}[t-SVD \citep{SongNZ.NLAA2020}]
Let $L$ be any invertible linear transform in \eqref{eqn:mode3prod}. For any $\bcA \in \C^{n_1 \times n_2 \times n_3}$, it can be factorized as $\bcA = \bcU \ast_{\bPhi} \bcG \ast_{\bPhi} \bcV^H$, where $\bcU \in \C^{n_1 \times n_1 \times n_3}$ and $\bcV \in \C^{n_2 \times n_2 \times n_3}$ are orthogonal, and $\bcG \in \C^{n_1 \times n_2 \times n_3}$ is an f-diagonal tensor, whose frontal slices are diagonal matrices.
\end{definition}
For any $\bcA$, we have the following relationship between its t-SVD and the matrix SVD of its block-diagonal matrix \citep{KernfeldKA.LAA2015}:
\begin{align*}
\bcA = \bcU \ast_{\bPhi} \bcG \ast_{\bPhi} \bcV^H \Leftrightarrow \widebar{\bA} = \widebar{\bU} \cdot \widebar{\bG} \cdot \widebar{\bV}^H,
\end{align*}
where $\cdot$ denotes the matrix multiplication. In words, the t-product in the spatial domain corresponds to matrix multiplication of the frontal slices in the transformed domain. Note that when $n_3 = 1$, the operator $\ast_{\bPhi}$ reduces to matrix multiplication.
\begin{definition}[Tensor multi-rank and tubal rank \citep{KilmerHAN.PNAS2021}]
The multi-rank of a tensor $\bcA \in \C^{n_1 \times n_2 \times n_3}$ is a vector $\br \in \R^{n_3}$, with its $i$-th entry being the rank of the $i$-th frontal slice of $\xoverline{\bcA}$, i.e., $r_i = \mathrm{rank} (\widebar{\bA}^{(i)})$. Let $\bcA = \bcU \ast_{\bPhi} \bcG \ast_{\bPhi} \bcV^H$ be the t-SVD of $\bcA$. The tensor tubal rank $\mathrm{rank}_t(\bcA)$ under $L$ is defined as the number of nonzero singular tubes of $\bcG$, i.e.,
\begin{align*}
\mathrm{rank}_t(\bcA) = \sharp \{ i: \bcG(i,i,:) \neq \bzero \} = \max(r_1, \dots, r_{n_3}).
\end{align*}
\end{definition}
We denote
\begin{align}\label{eqn:rankrproj}
\mathbb{P}_r(\bcA) = \argmin_{\widetilde{\bcA} : \mathrm{rank}_t(\widetilde{\bcA}) \leq r} \| \bcA - \widetilde{\bcA} \|_F^2
\end{align}
as the tubal rank-$r$ approximation of $\bcA$, which is given by the top-$r$ t-SVD of $\bcA$ by the tensor Eckart-Young theorem \citep{ZhangEAHK.CVPR2014,KilmerHAN.PNAS2021}.
\begin{definition}[Tensor nuclear norm and spectral norm \citep{LuPW.CVPR2019}]
The tensor nuclear norm of $\bcA \in \C^{n_1 \times n_2 \times n_3}$ under $L$ is defined as $\| \bcA \|_{\ast} = \frac{1}{\ell} \sum_{i=1}^{n_3} \| \widebar{\bA}^{(i)} \|_{\ast} = \frac{1}{\ell} \| \widebar{\bA} \|_{\ast}$. The spectral norm of $\bcA$ is defined as $\| \bcA \| = \| \widebar{\bA} \|$.
\end{definition}

\section{Main Results}
\label{sec:theory}

This section is devoted to introducing ScaledGD and establishing its performance guarantee for various low-rank tensor estimation problems.

\subsection{Models and Assumptions}

Suppose that the ground truth tubal rank-$r$ tensor $\bcX_{\star}$ with multi-rank $\br$ admits the following compact t-SVD $\bcX_{\star} = \bcU_{\star} \ast_{\bPhi} \bcG_{\star} \ast_{\bPhi} \bcV_{\star}^H$, where $\bcU_{\star} \in \R^{n_1 \times r \times n_3}$, $\bcV_{\star} \in \R^{n_2 \times r \times n_3}$, and $\bcG_{\star} \in \R^{r \times r \times n_3}$. Define the ground truth low-rank factors as
\begin{align*}
\bcL_{\star} = \bcU_{\star} \ast_{\bPhi} \bcG_{\star}^{\frac{1}{2}} \quad \mathrm{and} \quad \bcR_{\star} = \bcV_{\star} \ast_{\bPhi} \bcG_{\star}^{\frac{1}{2}}
\end{align*}
so that $\bcX_{\star} = \bcL_{\star} \ast_{\bPhi} \bcR_{\star}^H$. Here, the ``square root'' of a tensor $\bcA$, denoted by $\bcA^{\frac{1}{2}}$, is obtained by setting $\bcA^{\frac{1}{2}} \coloneq \bcB = L^{-1} (\xoverline{\bcB})$, where the $i$-th frontal slice of $\xoverline{\bcB}$ is $\widebar{\bB}^{(i)} = (\widebar{\bA}^{(i)})^{\frac{1}{2}}$. The factored representation is not unique in that for any invertible tensor $\bcQ \in \mathrm{GL}(r)$, one has $\bcL_{\star} \ast_{\bPhi} \bcR_{\star}^H = (\bcL_{\star} \ast_{\bPhi} \bcQ) \ast_{\bPhi} (\bcR_{\star} \ast_{\bPhi} \bcQ^{-H})^H$. We define the following two important singular values of tensor $\bcX_{\star}$ as
\begin{align*}
& \bar{\sigma}_1 (\bcX_{\star}) = \max \{ \xoverline{\bcG}_{i,i,k} | \xoverline{\bcG}_{i,i,k} > 0, i \leq r_k, k \in [n_3] \}  \\
\mathrm{and} \quad & \bar{\sigma}_{s_r} (\bcX_{\star}) = \min \{ \xoverline{\bcG}_{i,i,k} | \xoverline{\bcG}_{i,i,k} > 0, i \leq r_k, k \in [n_3] \}.
\end{align*}
We first introduce the condition number of the tensor $\bcX_{\star}$.
\begin{definition}[Condition number]\label{def:condnumber}
The condition number $\kappa$ of $\bcX_{\star}$ is defined as
\begin{align*}
\kappa \coloneq \bar{\sigma}_1 (\bcX_{\star}) / \bar{\sigma}_{s_r} (\bcX_{\star}).
\end{align*}
\end{definition}
Another parameter is the incoherence parameter, which is crucial in determining the well-posedness of low-rank tensor estimation.
\begin{definition}[Incoherence]\label{def:incoherence}
For a tubal rank-$r$ tensor $\bcX_{\star} \in \R^{n_1 \times n_2 \times n_3}$, assume that the multi-rank of $\bcX_{\star}$ is $\br$ and it has the t-SVD $\bcX_{\star} = \bcU_{\star} \ast_{\bPhi} \bcG_{\star} \ast_{\bPhi} \bcV_{\star}^H$. Then $\bcX_{\star}$ is said to satisfy the tensor incoherence conditions with parameter $\mu$ if
\begin{align}\label{eqn:incoherence}
\| \bcU_{\star} \|_{2,\infty} = \max_{i \in [n_1]} \| \bcU_{\star}^H \ast_{\bPhi} \mathring{\boldsymbol{\ce}}_i \|_F \leq \sqrt{\frac{\mu s_r}{n_1 n_3 \ell}} \quad \mathrm{and} \quad \| \bcV_{\star} \|_{2,\infty} = \max_{j \in [n_2]} \| \bcV_{\star}^H \ast_{\bPhi} \mathring{\boldsymbol{\ce}}_j \|_F \leq \sqrt{\frac{\mu s_r}{n_2 n_3 \ell}}.
\end{align}
Here, $\mathring{\boldsymbol{\ce}}_i$ is a tensor of size $n_1 \times 1 \times n_3$ with the entries of the $(i,1)$-th mode-3 tube of $L(\mathring{\boldsymbol{\ce}}_i)$ equaling 1 and the rest equaling 0, and $s_r = \sum_{k=1}^{n_3} r_k$.
\end{definition}
Notice that when all the $r_k$'s are equal to the tubal rank $r$, i.e., $s_r = n_3 r$, the above conditions will be equivalent to the ones in \citet[Proposition 1]{Lu.ICCV2021}. Roughly speaking, a small incoherence parameter ensures that the tensor columns are not correlated with the standard tensor basis, i.e., the energy of the tensor is evenly distributed across its entries. To track the performance of ScaledGD iterates $\bcF_t = [\bcL_t^H, \bcR_t^H]^H$, one needs a distance metric that properly takes account of the factor ambiguity due to invertible transforms. Motivated by the analysis in \citet{TongMC.JMLR2021}, we consider the following distance metric that resolves the ambiguity in the t-SVD.
\begin{definition}[Distance metric]\label{def:distF}
Let $\bcF = \begin{bmatrix} \bcL \\ \bcR \end{bmatrix} \in \R^{(n_1+n_2) \times r \times n_3}$ and $\bcF_{\star} = \begin{bmatrix} \bcL_{\star} \\ \bcR_{\star} \end{bmatrix} \in \R^{(n_1+n_2) \times r \times n_3}$, denote
\begin{align}
\dist(\bcF, \bcF_{\star}) = \sqrt{\inf_{\bcQ \in \mathrm{GL}(r)} \| (\bcL \ast_{\bPhi} \bcQ - \bcL_{\star}) \ast_{\bPhi} \bcG_{\star}^{\frac{1}{2}} \|_F^2 + \| (\bcR \ast_{\bPhi} \bcQ^{-H} - \bcR_{\star}) \ast_{\bPhi} \bcG_{\star}^{\frac{1}{2}} \|_F^2}.
\end{align}
If the infimum is attained at the argument $\bcQ$, it is called the optimal alignment tensor between $\bcF$ and $\bcF_{\star}$.
\end{definition}
As indicated in Appendix~\ref{sec:techlemmas}, for the ScaledGD iterates $\{ \bcF_t \}$, the optimal alignment tensors $\{ \bcQ_t \}$ always exist and hence are well-defined.

\subsection{Tensor RPCA}

Suppose we have observed a data tensor $\bcY \in \R^{n_1 \times n_2 \times n_3}$ of the form $\bcY = \bcX_{\star} + \bcS_{\star}$, where $\bcX_{\star}$ is a low-rank tensor with tubal rank-$r$ and $\bcS_{\star}$ is a sparse tensor---in which the number of nonzero entries is much smaller than its ambient dimension---modeling corruptions in the observations due to sensor failures, malicious attacks, or other system errors. Given $\bcY$ and $r$, the goal of tensor RPCA is to reliably estimate the two tensors $\bcX_{\star}$ and $\bcS_{\star}$.

Following the matrix case in \citet{ChandrasekaranSPW.SIOPT2011,NetrapalliNSAJ.NIPS2014}, we consider a deterministic sparsity model for $\bcS_{\star}$, in which $\bcS_{\star}$ contains at most $\alpha$-fraction of nonzero entries per tube.
\begin{definition}\label{def:sparsity}
A sparse tensor $\bcS_{\star} \in \R^{n_1 \times n_2 \times n_3}$ is $\alpha$-sparse, i.e., $\bcS_{\star} \in \mathscr{S}_{\alpha}$, where we denote
\begin{align*}
\mathscr{S}_{\alpha} \coloneq \{ \bcS \in \R^{n_1 \times n_2 \times n_3} & : \| \bcS(:,j,k) \|_0 \leq \alpha n_1, \| \bcS(i,:,k) \|_0 \leq \alpha n_2, \| \bcS(i,j,:) \|_0 \leq \alpha n_3,  \\
&\qquad\qquad\qquad\qquad ~\mathrm{for}~ i \in [n_1], j \in [n_2], k \in [n_3] \}.
\end{align*}
\end{definition}

Motivated by the works in \citet{TongMC.JMLR2021,CaiLY.NeurIPS2021}, we parameterize $\bcX = \bcL \ast_{\bPhi} \bcR^H$ by two low-rank factors $\bcL \in \R^{n_1 \times r \times n_3}$ and $\bcR \in \R^{n_2 \times r \times n_3}$ that are more memory-efficient, and instead optimize over the factors by solving the following optimization problem:
\begin{align}\label{eqn:TRPCAprobform}
\min_{ \bcF = [ \bcL^H, \bcR^H ]^H \in \R^{(n_1+n_2) \times r \times n_3}, \bcS \in \R^{n_1 \times n_2 \times n_3} } f(\bcF, \bcS) \coloneq \frac{1}{2} \| \bcL \ast_{\bPhi} \bcR^H + \bcS - \bcY \|_F^2.
\end{align}

Despite the nonconvexity of the objective function, one might be tempted to update the tensor factors via gradient descent, which, however, likely to converge slowly even for moderately ill-conditioned tensors \citep{YiPCC.NIPS2016,TongMC.JMLR2021,HanWZ.AOS2022}. Similar to the algorithms in \citet{TongMC.JMLR2021,DongTMC.IMA2023}, our algorithm alternates between corruption removal and factor refinements, where $(\bcL, \bcR)$ are updated via the proposed ScaledGD algorithm, and $\bcS$ is updated by soft thresholding. Specifically, we use the following soft-shrinkage operator $\cT_{\zeta} (\cdot)$ that sets entries with magnitudes smaller than $\zeta$ to 0, while uniformly shrinking the magnitudes of the other entries by $\zeta$, defined as
\begin{align}\label{eqn:soft-thresholding}
[\cT_{\zeta} (\bcM)]_{i,j,k} \coloneq \sgn([\bcM]_{i,j,k}) \cdot \max ( 0, |[\bcM]_{i,j,k}| - \zeta ).
\end{align}
The sparse outlier tensor is updated via
\begin{align}
\bcS_{t+1} = \cT_{\zeta_{t+1}}(\bcY - \bcL_t \ast_{\bPhi} \bcR_t^H)
\end{align}
with the schedule $\zeta_t$ to be specified shortly.

To complete the algorithm description, we still need to specify how to initialize the algorithm. We start with initializing the sparse tensor by $\bcS_0 = \cT_{\zeta_0} (\bcY)$ to remove the obvious outliers. Next, for the low-rank component, we set $\bcL_0 = \bcU_0 \ast_{\bPhi} \bcG_0^{\frac{1}{2}}$ and $\bcR_0 = \bcV_0 \ast_{\bPhi} \bcG_0^{\frac{1}{2}}$, where $\bcU_0 \ast_{\bPhi} \bcG_0 \ast_{\bPhi} \bcV_0^H$ is the best tubal rank-$r$ approximation of $\bcY - \bcS_0$. Combining all the steps mentioned above, we can now formally present the algorithm in Algorithm~\ref{algo:ScaledGDTRPCA}, which we dub ScaledGD for simplicity. ScaledGD costs only $\cO(n_1 n_2 n_3 r + n_1 n_2 n_3^2 + (n_1 + n_2) n_3 r^2 + n_3 r^3)$ flops per iteration provided $r \ll n_1 \wedge n_2$. For some special transforms, e.g., DFT, the per-iteration complexity is $\cO(n_1 n_2 n_3 r + n_1 n_2 n_3 \log(n_3) + (n_1 + n_2) n_3 r^2 + n_3 r^3)$. The breakdown of ScaledGD's complexity is provided in Appendix~\ref{sec:complexity}.

\begin{algorithm}[t]
\caption{ScaledGD for tensor robust PCA with spectral initialization}
\label{algo:ScaledGDTRPCA}
\textbf{Input:} Observed tensor $\bcY$, the transformation matrix $\bPhi$ associated with the transform $L$, the tubal rank $r$, learning rate $\eta$, maximum number of iterations $T$, and threshold schedule $\{ \zeta_t \}_{t=0}^T$.
\begin{algorithmic}
\State \textbf{Spectral initialization:} Let $\bcU_0 \ast_{\bPhi} \bcG_0 \ast_{\bPhi} \bcV_0^H$ be the top-$r$ t-SVD of $\bcY - \bcS_0$, where $\bcS_0 = \cT_{\zeta_0}(\bcY)$.
\begin{align*}
\mathrm{Set} \quad \bcL_0 = \bcU_0 \ast_{\bPhi} \bcG_0^{\frac{1}{2}} \quad \mathrm{and} \quad \bcR_0 = \bcV_0 \ast_{\bPhi} \bcG_0^{\frac{1}{2}}.
\end{align*}
\State \textbf{Scaled gradient updates:} \textbf{for} $t = 0, 1, \dots, T-1$ \textbf{do}
\begin{align}\label{eqn:TRPCAupdate}
\bcS_{t+1} & = \cT_{\zeta_{t+1}}(\bcY - \bcL_t \ast_{\bPhi} \bcR_t^H)  \nonumber \\
\bcL_{t+1} & = \bcL_t - \eta (\bcL_t \ast_{\bPhi} \bcR_t^H + \bcS_{t+1} - \bcY) \ast_{\bPhi} \bcR_t \ast_{\bPhi} (\bcR_t^H \ast_{\bPhi} \bcR_t)^{-1}  \\
\bcR_{t+1} & = \bcR_t - \eta (\bcL_t \ast_{\bPhi} \bcR_t^H + \bcS_{t+1} - \bcY)^H \ast_{\bPhi} \bcL_t \ast_{\bPhi} (\bcL_t^H \ast_{\bPhi} \bcL_t)^{-1}.  \nonumber
\end{align}
\end{algorithmic}
\textbf{Output:} The recovered low-rank tensor $\bcX_T = \bcL_T \ast_{\bPhi} \bcR_T^H$.
\end{algorithm}

\paragraph{Theoretical guarantees.} The following theorem establishes that ScaledGD algorithm---with proper choices of the tuning parameters, recovers the ground truth tensor $\bcX_{\star}$, as long as the fraction of corruptions is not too large. For convenience, we denote $n_{(1)} = \max(n_1, n_2)$ and $n_{(2)} = \min(n_1, n_2)$ in the following.
\begin{theorem}\label{thm:TRPCA}
Let $\bcY = \bcX_{\star} + \bcS_{\star} \in \R^{n_1 \times n_2 \times n_3}$, where $\bcX_{\star}$ is a multi-rank $\br$ tensor with $\mu$-incoherence and $\bcS_{\star}$ is $\alpha$-sparse. Suppose that the thresholding values $\{ \zeta_t \}_{t=0}^{\infty}$ obey that $\| \bcX_{\star} \|_{\infty} \leq \zeta_0 \leq 2 \| \bcX_{\star} \|_{\infty}$ and $\zeta_{t+1} = \rho \zeta_t$, $t \geq 1$, for some properly tuned $\zeta_1 = 3 \frac{\mu s_r}{n_3 \sqrt{n_1 n_2 \ell}} \bar{\sigma}_{s_r} (\bcX_{\star})$ and $\frac{1}{4} \leq \eta \leq \frac{8}{9}$, where $\rho = 1 - 0.6 \eta$. Then the iterates of ScaledGD with spectral initialization satisfy
\begin{align}
\| \bcL_t \ast_{\bPhi} \bcR_t^H - \bcX_{\star} \|_F & \leq \frac{0.03}{\sqrt{\ell}} \rho^t \bar{\sigma}_{s_r} (\bcX_{\star}),  \nonumber \\
\| \bcL_t \ast_{\bPhi} \bcR_t^H - \bcX_{\star} \|_{\infty} & \leq 3 \rho^t \frac{\mu s_r}{n_3 \sqrt{n_1 n_2 \ell}} \bar{\sigma}_{s_r} (\bcX_{\star}),  \nonumber \\
\| \bcS_t - \bcS_{\star} \|_{\infty} \leq 6 \rho^{t-1} \frac{\mu s_r}{n_3 \sqrt{n_1 n_2 \ell}} \bar{\sigma}_{s_r} (\bcX_{\star}) & \quad \mathrm{and} \quad \supp(\bcS_t) \subseteq \supp (\bcS_{\star})
\end{align}
for all $t \geq 1$, as long as the level of corruptions obeys $\alpha \leq \frac{n_3 \sqrt{n_{(2)} n_2}}{10^4 \mu s_r^{1.5} \kappa (n_1 + n_2 n_3)} \wedge \frac{n_3}{10^5 \mu^2 s_r^2 \kappa^2}$.
\end{theorem}

Note that the value of $\rho$ was used to simplify the proof, and it should not be considered as an optimal convergence rate. Theorem~\ref{thm:TRPCA} implies that upon appropriate choices of parameters, if the level of corruptions satisfies $\alpha \lesssim \frac{n_3 \sqrt{n_{(2)} n_2}}{\mu s_r^{1.5} \kappa (n_1 + n_2 n_3)} \wedge \frac{n_3}{\mu^2 s_r^2 \kappa^2}$, we can ensure that the proposed ScaledGD algorithm---starting from a carefully designed spectral initialization, converges linearly to the ground truth tensor $\bcX_{\star}$ in both the Frobenius norm and the entrywise $\ell_{\infty}$ norm at a constant rate, which is independent of the condition number, even when the gross corruptions are arbitrary. Note that the choice of parameters relies on the knowledge of $\bcX_{\star}$, which is usually unknown in practice. Thus, Theorem~\ref{thm:TRPCA} can be considered as a proof for the existence of the appropriate parameters.

The proof of our convergence theorem follows the route established in \citet{CaiLY.NeurIPS2021}. However, the algorithm in \citet{CaiLY.NeurIPS2021} is designed for matrices, while the extension from matrices to tensors is not trivial as different mathematical tools are required. For example, we need to use the property of t-product to prove some bounds on norms of $\bcS$, and we have to interconvert between the original and transformed domains in the proofs.

\subsection{Tensor Completion}

We now consider the tensor completion problem. Let $\bcX_{\star} \in \R^{n_1 \times n_2 \times n_3}$ be an unknown tensor and it has tubal rank $\mathrm{rank}_t(\bcX_{\star}) = r$. We assume to observe the entries of $\bcX_{\star}$ at locations given by a set $\bOmega = \{ (i,j,k) | \delta_{ijk} = 1 \}$, where $\delta_{ijk}$'s are independent and identically distributed (i.i.d.) Bernoulli variables taking value one with probability $p$ and zero with probability $1-p$. We denote such a Bernoulli sampling by $\bOmega \sim \mathrm{Ber}(p)$. The goal of tensor completion is to recover the tensor $\bcX_{\star}$ from its partial observation $\bcP_{\bOmega}(\bcX_{\star})$, where $\bcP_{\bOmega} : \R^{n_1 \times n_2 \times n_3} \to \R^{n_1 \times n_2 \times n_3}$ is a projection such that
\begin{displaymath}
[\bcP_{\bOmega}(\bcX_{\star})]_{i,j,k} = \left\{ \begin{array}{ll}
[\bcX_{\star}]_{i,j,k}, & \mathrm{if~} (i,j,k) \in \bOmega, \\
0, & \mathrm{otherwise}.
\end{array} \right.
\end{displaymath}
This can be achieved by minimizing the loss function
\begin{align}\label{eqn:TCprobform}
\min_{\bcF = [ \bcL^H, \bcR^H ]^H \in \R^{(n_1+n_2) \times r \times n_3}} f(\bcF) \coloneq \frac{1}{2p} \| \bcP_{\bOmega} (\bcL \ast_{\bPhi} \bcR^H - \bcX_{\star}) \|_F^2.
\end{align}

Similar to tensor RPCA, the underlying low-rank tensor $\bcX_{\star}$ is required to be incoherent (cf. Definition~\ref{def:incoherence}) to avoid ill-posedness. One typical strategy to ensure the incoherence condition in the matrix setting is to trim the rows of the factors after the gradient update \citep{ChenW.arXiv2015}. However, we need to be careful here because we need to preserve the equivariance with respect to invertible transforms. Again, inspired by \citet{TongMC.JMLR2021}, we introduce the following projection operator: for every $\widetilde{\bcF} = \begin{bmatrix} \widetilde{\bcL} \\ \widetilde{\bcR} \end{bmatrix} \in \R^{(n_1+n_2) \times r \times n_3}$,
\begin{align}\label{eqn:scaledproj}
\cP_{\varsigma}(\widetilde{\bcF}) & \coloneq \argmin_{\bcF \in \R^{(n_1+n_2) \times r \times n_3}} \| (\bcL - \widetilde{\bcL}) \ast_{\bPhi} (\widetilde{\bcR}^H \ast_{\bPhi} \widetilde{\bcR})^{\frac{1}{2}} \|_F^2 + \| (\bcR - \widetilde{\bcR}) \ast_{\bPhi} (\widetilde{\bcL}^H \ast_{\bPhi} \widetilde{\bcL})^{\frac{1}{2}} \|_F^2  \nonumber \\
&\quad \mathrm{s.t.} \quad \sqrt{n_1} \| \bcL \ast_{\bPhi} (\widetilde{\bcR}^H \ast_{\bPhi} \widetilde{\bcR})^{\frac{1}{2}} \|_{2,\infty} \vee \sqrt{n_2} \| \bcR \ast_{\bPhi} (\widetilde{\bcL}^H \ast_{\bPhi} \widetilde{\bcL})^{\frac{1}{2}} \|_{2,\infty} \leq \varsigma,
\end{align}
where $\varsigma > 0$ is the projection radius. Fortunately, following the proof in \citet[Proposition 7]{TongMC.JMLR2021}, this problem has a closed-form solution, as stated below.
\begin{proposition}\label{prop:scaledprojsol}
The solution to \eqref{eqn:scaledproj} is given by $\cP_{\varsigma}(\widetilde{\bcF}) = \begin{bmatrix} \bcL \\ \bcR \end{bmatrix}$, where
\begin{align}\label{eqn:scaledprojsol}
\bcL(i,:,:) & = \Big(1 \wedge \frac{\varsigma}{\sqrt{n_1} \| \widetilde{\bcL}(i,:,:) \ast_{\bPhi} \widetilde{\bcR}^H \|_F} \Big) \widetilde{\bcL}(i,:,:), \quad i \in [n_1],  \nonumber \\
\quad \mathrm{and} \quad \bcR(j,:,:) & = \Big(1 \wedge \frac{\varsigma}{\sqrt{n_2} \| \widetilde{\bcR}(j,:,:) \ast_{\bPhi} \widetilde{\bcL}^H \|_F} \Big) \widetilde{\bcR}(j,:,:), \quad j \in [n_2].
\end{align}
\end{proposition}
With this projection operator in hand, we are ready to propose our ScaledGD method with the spectral initialization for solving tensor completion, as described in Algorithm~\ref{algo:ScaledGDTC}. The computational cost at each iteration is $\cO(n_1 n_2 n_3 r + p n_1 n_2 n_3^2 + (n_1 + n_2) n_3 r^2 + n_3 r^3)$ for general linear transforms, with reduction to $\cO(n_1 n_2 n_3 r + p n_1 n_2 n_3 \log(n_3) + (n_1 + n_2) n_3 r^2 + n_3 r^3)$ for DFT, which is much lower compared to the algorithm in \citet{LuPW.CVPR2019}.

\begin{algorithm}[t]
\caption{ScaledGD for tensor completion with spectral initialization}
\label{algo:ScaledGDTC}
\textbf{Input:} Partially observed data tensor $\bcP_{\bOmega}(\bcX_{\star})$, the transformation matrix $\bPhi$ associated with the transform $L$, the tubal rank $r$, learning rate $\eta$, and maximum number of iterations $T$.
\begin{algorithmic}
\State \textbf{Spectral initialization:} Let $\bcU_0 \ast_{\bPhi} \bcG_0 \ast_{\bPhi} \bcV_0^H$ be the top-$r$ t-SVD of $\frac{1}{p} \bcP_{\bOmega}(\bcX_{\star})$, and set
\begin{align}\label{eqn:TCinit}
\begin{bmatrix} \bcL_0 \\ \bcR_0 \end{bmatrix} = \cP_{\varsigma} \left(\begin{bmatrix} \bcU_0 \ast_{\bPhi} \bcG_0^{\frac{1}{2}} \\ \bcV_0 \ast_{\bPhi} \bcG_0^{\frac{1}{2}} \end{bmatrix} \right).
\end{align}
\State \textbf{Scaled projected gradient updates:} \textbf{for} $t = 0, 1, \dots, T-1$ \textbf{do}
\begin{align}\label{eqn:TCupdate}
\begin{bmatrix} \bcL_{t+1} \\ \bcR_{t+1} \end{bmatrix} = \cP_{\varsigma} \left(\begin{bmatrix}
\bcL_t - \frac{\eta}{p} \bcP_{\bOmega}(\bcL_t \ast_{\bPhi} \bcR_t^H - \bcX_{\star}) \ast_{\bPhi} \bcR_t \ast_{\bPhi} (\bcR_t^H \ast_{\bPhi} \bcR_t)^{-1}  \\
\bcR_t - \frac{\eta}{p} \bcP_{\bOmega}(\bcL_t \ast_{\bPhi} \bcR_t^H - \bcX_{\star})^H \ast_{\bPhi} \bcL_t \ast_{\bPhi} (\bcL_t^H \ast_{\bPhi} \bcL_t)^{-1}
\end{bmatrix}\right).
\end{align}
\end{algorithmic}
\textbf{Output:} The recovered low-rank tensor $\bcX_T = \bcL_T \ast_{\bPhi} \bcR_T^H$.
\end{algorithm}

\paragraph{Theoretical guarantees.} Encouragingly, we can guarantee that ScaledGD provably recovers the ground truth tensor, as long as the sample size is sufficiently large.
\begin{theorem}\label{thm:TC}
Suppose that $\bcX_{\star}$ is $\mu$-incoherent, and that $p$ satisfies $p \geq c \Big( \frac{\mu s_r (n_1 + n_2) \log( (n_1 \vee n_2) n_3 )}{n_1 n_2 n_3} \vee \frac{\mu^2 s_r^2 \kappa^4 \ell \log( (n_1 \vee n_2) n_3 )}{(n_1 \wedge n_2) n_3^2} \Big)$ for some sufficiently large constant $c$. Set the projection radius as $\varsigma \geq c_{\varsigma} \sqrt{\frac{\mu s_r}{n_3 \ell}} \bar{\sigma}_1 (\bcX_{\star})$ for some constant $c_{\varsigma} \geq 1.02$. If the step size obeys $0 < \eta \leq \frac{2}{3}$, then with high probability, for all $t \geq 0$, the iterates of ScaledGD in \eqref{eqn:TCupdate} satisfy
\begin{align*}
\dist(\bcF_t, \bcF_{\star}) \leq (1 - 0.6 \eta)^t 0.02 \bar{\sigma}_{s_r} (\bcX_{\star}) \quad \mathrm{and} \quad \| \bcL_t \ast_{\bPhi} \bcR_t^H - \bcX_{\star} \|_F \leq (1 - 0.6 \eta)^t 0.03 \bar{\sigma}_{s_r} (\bcX_{\star}).
\end{align*} 
\end{theorem}

Theorem~\ref{thm:TC} establishes that the distance $\dist(\bcF_t, \bcF_{\star})$ contracts linearly at a constant rate, as long as the probability of observation satisfies $p \gtrsim \Big( \frac{\mu s_r (n_1 + n_2) \log( (n_1 \vee n_2) n_3 )}{n_1 n_2 n_3} \vee \frac{\mu^2 s_r^2 \kappa^4 \ell \log( (n_1 \vee n_2) n_3 )}{(n_1 \wedge n_2) n_3^2} \Big)$. To reach an $\epsilon$-accurate estimate, i.e., $\| \bcL_t \ast_{\bPhi} \bcR_t^H - \bcX_{\star} \|_F \leq \epsilon \bar{\sigma}_{s_r} (\bcX_{\star})$, ScaledGD takes at most $\cO (\log(1/\epsilon))$ iterations, which is independent of the condition number, as long as the sample complexity is large enough.

\subsection{Robust Tensor Completion}
\label{ssec:RTC}

Assume that we partially observe some entries of a data tensor $\bcY \in \R^{n_1 \times n_2 \times n_3}$ of the form $\bcY = \bcX_{\star} + \bcS_{\star}$, where $\bcX_{\star}$ is a tubal rank-$r$ tensor and $\bcS_{\star}$ is a sparse tensor. The task of robust tensor completion is to recover the two tensors $\bcX_{\star}$ and $\bcS_{\star}$ from the observations $\bcP_{\bOmega}(\bcY) = \bcP_{\bOmega}(\bcX_{\star} + \bcS_{\star})$, where the set of observed locations in $\bOmega$ is sampled independently according to the Bernoulli model with probability $p$. We also assume that $\bcP_{\bOmega}(\bcS_{\star})$ is $\alpha p$-sparse, i.e., it has at most $\alpha p$-fraction of nonzero entries per tube for each mode.

To avoid the nonconvex low-tubal-rank constraint on $\bcX$, we rewrite $\bcX = \bcL \ast_{\bPhi} \bcR^H$ and consider the following objective function:
\begin{align}\label{eqn:RTCprobform}
\min_{ \bcF = [ \bcL^H, \bcR^H ]^H \in \R^{(n_1+n_2) \times r \times n_3}, \bcS \in \R^{n_1 \times n_2 \times n_3} } f(\bcF, \bcS) \coloneq \frac{1}{2p} \| \bcP_{\bOmega}(\bcL \ast_{\bPhi} \bcR^H + \bcS - \bcY) \|_F^2.
\end{align}
Similar to Algorithm~\ref{algo:ScaledGDTRPCA}, we first initialize the sparse tensor by $\bcS_0 = \cT_{\zeta_0}(\bcP_{\bOmega}(\bcY))$ and take $\bcL_0 = \bcU_0 \ast_{\bPhi} \bcG_0^{\frac{1}{2}}$ and $\bcR_0 = \bcV_0 \ast_{\bPhi} \bcG_0^{\frac{1}{2}}$, where $\bcU_0 \ast_{\bPhi} \bcG_0 \ast_{\bPhi} \bcV_0^H$ is the best tubal rank-$r$ approximation of $\frac{1}{p} \bcP_{\bOmega}(\bcY - \bcS_0)$. Note that $p^{-1}$ is necessary here to reweight the expectation of sampling operator \citep{Recht.JMLR2011}. In the phase of iterative updates, we alternatively update the sparse and low-rank components in the fashion of ScaledGD, formally stated in Algorithm~\ref{algo:ScaledGDRTC}. For any invertible linear transforms, the per-iteration complexity is $\cO(n_1 n_2 n_3 r + p n_1 n_2 n_3^2 + (n_1 + n_2) n_3 r^2 + n_3 r^3)$. For some special transforms, e.g., DFT, the per-iteration complexity is $\cO(n_1 n_2 n_3 r + p n_1 n_2 n_3 \log(n_3) + (n_1 + n_2) n_3 r^2 + n_3 r^3)$.

\begin{algorithm}[t]
\caption{ScaledGD for robust tensor completion with spectral initialization}
\label{algo:ScaledGDRTC}
\textbf{Input:} Partially observed data tensor $\bcP_{\bOmega}(\bcY)$, the transformation matrix $\bPhi$ associated with the transform $L$, the tubal rank $r$, learning rate $\eta$, maximum number of iterations $T$, and threshold schedule $\{ \zeta_t \}_{t=0}^T$.
\begin{algorithmic}
\State \textbf{Spectral initialization:} Let $\bcU_0 \ast_{\bPhi} \bcG_0 \ast_{\bPhi} \bcV_0^H$ be the top-$r$ t-SVD of $\frac{1}{p} \bcP_{\bOmega}(\bcY - \bcS_0)$, where $\bcS_0 = \cT_{\zeta_0}(\bcP_{\bOmega}(\bcY))$.
\begin{align*}
\mathrm{Set} \quad \bcL_0 = \bcU_0 \ast_{\bPhi} \bcG_0^{\frac{1}{2}} \quad \mathrm{and} \quad \bcR_0 = \bcV_0 \ast_{\bPhi} \bcG_0^{\frac{1}{2}}.
\end{align*}
\State \textbf{Scaled gradient updates:} \textbf{for} $t = 0, 1, \dots, T-1$ \textbf{do}
\begin{align}\label{eqn:RTCupdate}
\bcS_{t+1} & = \cT_{\zeta_{t+1}}(\bcP_{\bOmega}(\bcY - \bcL_t \ast_{\bPhi} \bcR_t^H))  \nonumber \\
\bcL_{t+1} & = \bcL_t - \frac{\eta}{p} \bcP_{\bOmega}(\bcL_t \ast_{\bPhi} \bcR_t^H + \bcS_{t+1} - \bcY) \ast_{\bPhi} \bcR_t \ast_{\bPhi} (\bcR_t^H \ast_{\bPhi} \bcR_t)^{-1}  \\
\bcR_{t+1} & = \bcR_t - \frac{\eta}{p} \bcP_{\bOmega}(\bcL_t \ast_{\bPhi} \bcR_t^H + \bcS_{t+1} - \bcY)^H \ast_{\bPhi} \bcL_t \ast_{\bPhi} (\bcL_t^H \ast_{\bPhi} \bcL_t)^{-1}.  \nonumber
\end{align}
\end{algorithmic}
\textbf{Output:} The recovered low-rank tensor $\bcX_T = \bcL_T \ast_{\bPhi} \bcR_T^H$.
\end{algorithm}

\paragraph{Theoretical guarantees.} The following theorem states that ScaledGD algorithm---with proper choices of the tuning parameters, recovers $\bcX_{\star}$ with high probability, as long as the fraction of corruptions is not too large and the number of observations is large enough.
\begin{theorem}\label{thm:RTC}
Let $\bcP_{\bOmega}(\bcY) = \bcP_{\bOmega}(\bcX_{\star} + \bcS_{\star}) \in \R^{n_1 \times n_2 \times n_3}$, where $\bcX_{\star}$ is a multi-rank $\br$ tensor with $\mu$-incoherence and $\bcP_{\bOmega}(\bcS_{\star})$ is $\alpha p$-sparse. Suppose that the thresholding values $\{ \zeta_t \}_{t=0}^{\infty}$ obey that $\| \bcP_{\bOmega}(\bcX_{\star}) \|_{\infty} \leq \zeta_0 \leq 2 \| \bcP_{\bOmega}(\bcX_{\star}) \|_{\infty}$ and $\zeta_{t+1} = \rho \zeta_t$, $t \geq 1$, for some properly tuned $\zeta_1 = 3 \frac{\mu s_r}{n_3 \sqrt{n_1 n_2 \ell}} \bar{\sigma}_{s_r} (\bcX_{\star})$ and $\frac{1}{5} \leq \eta \leq \frac{1}{2}$, where $\rho = 1 - 0.3 \eta$. Then the iterates of ScaledGD satisfy
\begin{align}
\| \bcL_t \ast_{\bPhi} \bcR_t^H - \bcX_{\star} \|_F & \leq \frac{0.03}{\sqrt{\ell}} \rho^t \bar{\sigma}_{s_r} (\bcX_{\star}),  \nonumber \\
\| \bcL_t \ast_{\bPhi} \bcR_t^H - \bcX_{\star} \|_{\infty} & \leq 3 \rho^t \frac{\mu s_r}{n_3 \sqrt{n_1 n_2 \ell}} \bar{\sigma}_{s_r} (\bcX_{\star}),  \nonumber \\
\| \bcS_t - \bcS_{\star} \|_{\infty} \leq 6 \rho^{t-1} \frac{\mu s_r}{n_3 \sqrt{n_1 n_2 \ell}} \bar{\sigma}_{s_r} (\bcX_{\star}) & \quad \mathrm{and} \quad \supp(\bcS_t) \subseteq \supp (\bcP_{\bOmega}(\bcS_{\star}))
\end{align}
for all $t \geq 1$, as long as the level of corruptions obeys $\alpha \leq \frac{p n_3 \sqrt{n_{(2)} n_2}}{10^4 \mu s_r^{1.5} \kappa (n_1 + n_2 n_3)} \wedge \frac{p n_3}{10^6 \mu^2 s_r^2 \kappa^2}$ and the probability to observe an entry is high enough, i.e.,
\begin{align*}
p \geq c \log( (n_1 \vee n_2) n_3 ) \max \Big( \frac{\mu^{1.5} s_r^{1.5} (n_1 + n_2)}{n_3 \sqrt{n_1 n_2}}, \frac{\mu^2 s_r^2 \ell}{n_3 \sqrt{(n_1 \wedge n_2) n_3}}, \frac{\mu s_r \kappa^2}{\sqrt{(n_1 \wedge n_2) n_3}} \Big)
\end{align*}
for some sufficiently large constant $c$.
\end{theorem}

Theorem~\ref{thm:RTC} establishes that the proposed ScaledGD algorithm (cf. Algorithm~\ref{algo:ScaledGDRTC}) finds the ground truth tensor at a constant linear rate, as long as the fraction of corruptions satisfies $\alpha \lesssim \frac{p n_3 \sqrt{n_{(2)} n_2}}{\mu s_r^{1.5} \kappa (n_1 + n_2 n_3)} \wedge \frac{p n_3}{\mu^2 s_r^2 \kappa^2}$ and the probability of observation satisfies $p \gtrsim \log( (n_1 \vee n_2) n_3 ) \max \Big( \frac{\mu^{1.5} s_r^{1.5} (n_1 + n_2)}{n_3 \sqrt{n_1 n_2}}, \frac{\mu^2 s_r^2 \ell}{n_3 \sqrt{(n_1 \wedge n_2) n_3}}, \frac{\mu s_r \kappa^2}{\sqrt{(n_1 \wedge n_2) n_3}} \Big)$. Notice that in contrast to Theorem~\ref{thm:TRPCA}, the upper bound of $\alpha$ is proportional to the value of $p$, which is aligned with the intuition that as we observe more entries in $\bcY$, it is still possible to recover $\bcX_{\star}$ even for a higher level of corruptions in $\bcX_{\star}$.

\subsection{Tensor Regression}
\label{ssec:TR}

We now move on to the tensor regression problem. Assume that we have collected a set of observations, given by
\begin{align}
\by = \cA(\bcX_{\star}) \in \R^m ~ \mathrm{with} ~ y_i = \langle \bcA_i, \bcX_{\star} \rangle, i = 1, \dots, m,
\end{align}
where each $\bcA_i \in \R^{n_1 \times n_2 \times n_3}$ corresponds to the $i$-th measurement tensor, and $\cA(\bcX) = \{ \langle \bcA_i, \bcX \rangle \}_{i=1}^m$ is a linear map from $\R^{n_1 \times n_2 \times n_3}$ to $\R^m$. The goal of tensor regression is to recover $\bcX_{\star}$ from $\by$, especially when the number of observations $m \ll n_1 n_2 n_3$. The preceding discussion helps us pose this problem in terms of the following optimization problem:
\begin{align}\label{eqn:TRprobform}
\min_{\bcF = [ \bcL^H, \bcR^H ]^H \in \R^{(n_1+n_2) \times r \times n_3}} f(\bcF) \coloneq \frac{1}{2} \| \cA(\bcL \ast_{\bPhi} \bcR^H) - \by \|_F^2.
\end{align}

The proposed ScaledGD algorithm to minimize \eqref{eqn:TRprobform} is summarized in Algorithm~\ref{algo:ScaledGDTR}. The algorithm starts with an initialization step by applying t-SVD on $\cA^{\ast}(\by)$, followed by scaled gradient updates given in \eqref{eqn:TRupdate}, where $\cA^{\ast}(\cdot)$ is the adjoint operator of $\cA(\cdot)$, defined as $\cA^{\ast}(\by) = \sum_{i=1}^m y_i \bcA_i$.

\begin{algorithm}[t]
\caption{ScaledGD for low-rank tensor regression with spectral initialization}
\label{algo:ScaledGDTR}
\textbf{Input:} Linear map $\cA(\cdot)$, observation vector $\by$, the transformation matrix $\bPhi$ associated with the transform $L$, the tubal rank $r$, learning rate $\eta$, and maximum number of iterations $T$.
\begin{algorithmic}
\State \textbf{Spectral initialization:} Let $\bcU_0 \ast_{\bPhi} \bcG_0 \ast_{\bPhi} \bcV_0^H$ be the top-$r$ t-SVD of $\cA^{\ast}(\by)$, and set
\begin{align}\label{eqn:TRinit}
\mathrm{Set} \quad \bcL_0 = \bcU_0 \ast_{\bPhi} \bcG_0^{\frac{1}{2}} \quad \mathrm{and} \quad \bcR_0 = \bcV_0 \ast_{\bPhi} \bcG_0^{\frac{1}{2}}.
\end{align}
\State \textbf{Scaled projected gradient updates:} \textbf{for} $t = 0, 1, \dots, T-1$ \textbf{do}
\begin{align}\label{eqn:TRupdate}
\bcL_{t+1} & = \bcL_t - \eta \cA^{\ast} (\cA(\bcL_t \ast_{\bPhi} \bcR_t^H) - \by) \ast_{\bPhi} \bcR_t \ast_{\bPhi} (\bcR_t^H \ast_{\bPhi} \bcR_t)^{-1}  \\
\bcR_{t+1} & = \bcR_t - \eta \cA^{\ast} (\cA(\bcL_t \ast_{\bPhi} \bcR_t^H) - \by)^H \ast_{\bPhi} \bcL_t \ast_{\bPhi} (\bcL_t^H \ast_{\bPhi} \bcL_t)^{-1}.  \nonumber
\end{align}
\end{algorithmic}
\textbf{Output:} The recovered low-rank tensor $\bcX_T = \bcL_T \ast_{\bPhi} \bcR_T^H$.
\end{algorithm}

\paragraph{Theoretical guarantees.} To understand the performance of ScaledGD for tensor regression, we adopt a standard assumption on the measurement operator $\cA(\cdot)$, namely the tensor restricted isometry property (TRIP).
\begin{definition}[TRIP \citep{RauhutSS.LAA2017}]
The linear map $\cA(\cdot) : \R^{n_1 \times n_2 \times n_3} \to \R^m$ is said to obey the tubal rank-$r$ TRIP with a constant $\delta_r \in [0, 1)$, if for all tensors $\bcX \in \R^{n_1 \times n_2 \times n_3}$ of tubal rank at most $r$, the following condition holds 
\begin{align*}
(1 - \delta_r) \| \bcX \|_F^2 \leq \| \cA(\bcX) \|_2^2 \leq (1 + \delta_r) \| \bcX \|_F^2.
\end{align*}
\end{definition}

If $\cA(\cdot)$ satisfies tubal rank-$2r$ TRIP with $\delta_r \in [0, 1)$, then for any two tensors $\bcX_1, \bcX_2 \in \R^{n_1 \times n_2 \times n_3}$ of tubal rank at most $r$, we have
\begin{align*}
(1 - \delta_{2r}) \| \bcX_1 - \bcX_2 \|_F^2 \leq \| \cA(\bcX_1 - \bcX_2) \|_2^2 \leq (1 + \delta_{2r}) \| \bcX_1 - \bcX_2 \|_F^2.
\end{align*}
The TRIP under the higher-order singular value decomposition (HOSVD), the tensor train (TT) format, and the Tucker decomposition has been investigated extensively \citep{RauhutSS.LAA2017}.
\begin{theorem}\label{thm:TR}
Suppose that $\cA(\cdot)$ obeys the $2r$-TRIP with $\delta_{2r} \leq 0.02/(\sqrt{\frac{s_r}{\ell}} \kappa)$. If the step size obeys $0 < \eta \leq \frac{2}{3}$, then for all $t \geq 0$, the iterates of the ScaledGD method in Algorithm~\ref{algo:ScaledGDTR} satisfy 
\begin{align*}
\dist(\bcF_t, \bcF_{\star}) \leq (1 - 0.6 \eta)^t 0.1 \bar{\sigma}_{s_r} (\bX_{\star}) \quad \mathrm{and} \quad \| \bcL_t \ast_{\bPhi} \bcR_t^H - \bcX_{\star} \|_F \leq (1 - 0.6 \eta)^t 0.15 \bar{\sigma}_{s_r} (\bX_{\star}).
\end{align*}
\end{theorem}

Theorem~\ref{thm:TR} establishes that the distance $\dist(\bcF_t, \bcF_{\star})$ contracts linearly at a constant rate, as long as $\delta_{2r} \lesssim 1/(\sqrt{\frac{s_r}{\ell}} \kappa)$. To reach $\epsilon$-accuracy, i.e., $\| \bcL_t \ast_{\bPhi} \bcR_t^H - \bcX_{\star} \|_F \leq \epsilon \bar{\sigma}_{s_r} (\bX_{\star})$, ScaledGD again takes at most $T = \cO(\log(1/\epsilon))$ iterations, which is independent of the condition number $\kappa$ of $\bcX_{\star}$.

\section{Proof Sketch}
\label{sec:proofoutline}

In this section, we provide some intuitions and sketch the proof of our main theorems.

\subsection{ScaledGD for Tensor Factorization}

To shed light on why ScaledGD is robust to ill-conditioning, we first consider the problem of factorizing a tensor $\bcX_{\star}$ into two low-rank factors: 
\begin{align}\label{eqn:TFloss}
\min_{\bcF = [ \bcL^H, \bcR^H ]^H \in \R^{(n_1+n_2) \times r \times n_3}} f(\bcF) \coloneq \frac{1}{2} \| \bcL \ast_{\bPhi} \bcR^H - \bcX_{\star} \|_F^2.
\end{align}
Recalling the update rule~\eqref{eqn:scaledGD}, ScaledGD proceeds as follows:
\begin{align}\label{eqn:TFupdate}
\bcL_{t+1} & = \bcL_t - \eta (\bcL_t \ast_{\bPhi} \bcR_t^H - \bcX_{\star}) \ast_{\bPhi} \bcR_t \ast_{\bPhi} (\bcR_t^H \ast_{\bPhi} \bcR_t)^{-1},  \nonumber \\
\bcR_{t+1} & = \bcR_t - \eta (\bcL_t \ast_{\bPhi} \bcR_t^H - \bcX_{\star})^H \ast_{\bPhi} \bcL_t \ast_{\bPhi} (\bcL_t^H \ast_{\bPhi} \bcL_t)^{-1},
\end{align}
Due to property \eqref{eqn:tensorproperty}, problem \eqref{eqn:TFloss} can be transformed into the transformed domain by solving the following problem:
\begin{align}\label{eqn:TFlossmatrixform}
\min_{\widetilde{\bF} = [ \widebar{\bL}^H, \widebar{\bR}^H ]^H \in \C^{(n_1+n_2) n_3 \times n_3 r}} f(\widebar{\bL}, \widebar{\bR}) \coloneq \frac{1}{2 \ell} \| \widebar{\bL} \cdot \widebar{\bR}^H - \widebar{\bX}_{\star} \|_F^2,
\end{align}
where $\widebar{\bL} = \mathtt{bdiag}(\xoverline{\bcL})$ as defined in \eqref{eqn:defbdiag} and $\cdot$ again denotes the matrix multiplication. Note that we use $\widetilde{\bF} = \begin{bmatrix} \widebar{\bL} \\ \widebar{\bR} \end{bmatrix}$ here to denote the vertical concatenation of $\widebar{\bL}$ and $\widebar{\bR}$, which differs from the matrix $\widebar{\bF} = \mathtt{bdiag}(\xoverline{\bcF})$ only by a permutation of rows, where $\xoverline{\bcF} = L(\bcF)$ with $\bcF = \begin{bmatrix} \bcL \\ \bcR \end{bmatrix}$. This is now a standard matrix least-squares problem, up to a scalar. The gradients of $f(\widetilde{\bF})$ in \eqref{eqn:TFlossmatrixform} with respect to $\widebar{\bL}$ and $\widebar{\bR}$ are given by
\begin{align*}
\nabla_{\widebar{\bL}} f(\widetilde{\bF}) = \frac{1}{\ell} ( \widebar{\bL} \widebar{\bR}^H - \widebar{\bX}_{\star}) \widebar{\bR} \quad \mathrm{and} \quad \nabla_{\widebar{\bR}} f(\widetilde{\bF}) = \frac{1}{\ell} (\widebar{\bL} \widebar{\bR}^H - \widebar{\bX}_{\star})^H \widebar{\bL},
\end{align*}
which allows for the computation of the Hessian with respect to $\widebar{\bL}$ and $\widebar{\bR}$. When written in terms of the vectorized variables, the Hessians are expressed as
\begin{align*}
\nabla_{\widebar{\bL},\widebar{\bL}}^2 f(\widetilde{\bF}) = \frac{1}{\ell} (\widebar{\bR}^H \widebar{\bR}) \otimes \bI_{n_1 n_3} \quad \mathrm{and} \quad \nabla_{\widebar{\bR},\widebar{\bR}}^2 f(\widetilde{\bF}) = \frac{1}{\ell} (\widebar{\bL}^H \widebar{\bL}) \otimes \bI_{n_2 n_3},
\end{align*}
where $\otimes$ denotes the Kronecker product. Let $\mathtt{vec}(\bA)$ denote the vectorization of a matrix $\bA$. Viewed in the vectorized form, it is not difficult to check that the ScaledGD update \eqref{eqn:TFupdate} is equivalent to approximating the Hessian of the loss function \eqref{eqn:TFlossmatrixform} by only keeping its diagonal blocks, i.e., 
\begin{align*}
\mathtt{vec}(\widetilde{\bF}_{t+1}) = \mathtt{vec}(\widetilde{\bF}_t) - \frac{\eta}{\ell} \begin{bmatrix} \nabla_{\widebar{\bL},\widebar{\bL}}^2 f(\widetilde{\bF}_t) & \bzero \\
\bzero & \nabla_{\widebar{\bR},\widebar{\bR}}^2 f(\widetilde{\bF}_t) \end{bmatrix}^{-1} \mathtt{vec}(\nabla_{\widetilde{\bF}} f(\widetilde{\bF}_t)).
\end{align*}
Under this formulation, the ScaledGD update corresponds to a Newton-type update up to a constant step size, yielding a natural quasi-Newton interpretation where the preconditioner is designed as the inverse of the diagonal approximation of the Hessian. Compared with vanilla gradient descent, ScaledGD exploits additional curvature information through a structured scaling of the gradient. This quasi-Newton interpretation provides a natural explanation for its improved convergence behavior, especially in ill-conditioned regimes.

\paragraph{Theoretical guarantees for tensor factorization.} The following theorem, whose proof can be found in Appendix~\ref{sec:TFproof}, formally establishes that as long as initialization is not too far from the ground truth, $\dist(\bcF_t, \bcF_{\star})$ will contract at a constant linear rate for the tensor factorization problem.

\begin{theorem}\label{thm:TF}
Suppose that the initialization $\bcF_0$ satisfies $\dist(\bcF_0, \bcF_{\star}) \leq \frac{0.1}{\sqrt{\ell}} \bar{\sigma}_{s_r} (\bcX_{\star})$. If the step size obeys $0 < \eta \leq \frac{2}{3}$, then for all $t \geq 0$, the iterates of the ScaledGD method in \eqref{eqn:TFupdate} satisfy
\begin{align*}
\dist(\bcF_t, \bcF_{\star}) \leq (1 - 0.7 \eta)^t \frac{0.1}{\sqrt{\ell}} \bar{\sigma}_{s_r} (\bcX_{\star}) \quad \mathrm{and} \quad \| \bcL_t \ast_{\bPhi} \bcR_t^H - \bcX_{\star} \|_F \leq (1 - 0.7 \eta)^t \frac{0.15}{\sqrt{\ell}} \bar{\sigma}_{s_r} (\bcX_{\star}).
\end{align*}
\end{theorem}

\subsection{Proof Outline for Tensor RPCA}

The proof of Theorem~\ref{thm:TRPCA} is inductive in nature, where we aim to establish the following induction hypothesis at all the iterations:
\begin{align*}
\dist(\bcF_t, \bcF_{\star}) & \leq \frac{\epsilon}{\sqrt{\ell}} \rho^t \bar{\sigma}_{s_r} (\bcX_{\star}), \quad \mathrm{and}  \\
\sqrt{n_1} \| (\bcL_t \ast_{\bPhi} \bcQ_t - \bcL_{\star}) \ast_{\bPhi} \bcG_{\star}^{\frac{1}{2}} \|_{2,\infty} & \vee \sqrt{n_2} \| (\bcR_t \ast_{\bPhi} \bcQ_t^{-H} - \bcR_{\star}) \ast_{\bPhi} \bcG_{\star}^{\frac{1}{2}} \|_{2,\infty} \leq \sqrt{\frac{\mu s_r}{n_3 \ell}} \rho^t \bar{\sigma}_{s_r} (\bcX_{\star}).
\end{align*}
The following two lemmas, establishes the induction hypothesis for both the induction case and the base case. We start by outlining the local contraction of the proposed Algorithm~\ref{algo:ScaledGDTRPCA}.

\begin{lemma}[Local contraction]\label{lemma:TRPCAcontraction}
Let $\bcY = \bcX_{\star} + \bcS_{\star} \in \R^{n_1 \times n_2 \times n_3}$, where $\bcX_{\star} = \bcL_{\star} \ast_{\bPhi} \bcR_{\star}^H$ is a multi-rank $\br$ tensor with $\mu$-incoherence and $\bcS_{\star}$ is $\alpha$-sparse. Let $\bcQ_t$ be the optimal alignment tensor between $\begin{bmatrix} \bcL_t \\ \bcR_t \end{bmatrix}$ and $\begin{bmatrix} \bcL_{\star} \\ \bcR_{\star} \end{bmatrix}$. Under the assumption that $\alpha \leq \frac{n_3 \sqrt{n_{(2)} n_2}}{10^4 \mu s_r^{1.5} (n_1 + n_2 n_3)}$, if the spectral initialization obeys the conditions
\begin{align*}
\dist(\bcF_0, \bcF_{\star}) & \leq \frac{\epsilon}{\sqrt{\ell}} \bar{\sigma}_{s_r} (\bcX_{\star}),  \\
\sqrt{n_1} \| (\bcL_0 \ast_{\bPhi} \bcQ_0 - \bcL_{\star}) \ast_{\bPhi} \bcG_{\star}^{\frac{1}{2}} \|_{2,\infty} & \vee \sqrt{n_2} \| (\bcR_0 \ast_{\bPhi} \bcQ_0^{-H} - \bcR_{\star}) \ast_{\bPhi} \bcG_{\star}^{\frac{1}{2}} \|_{2,\infty} \leq \sqrt{\frac{\mu s_r}{n_3 \ell}} \bar{\sigma}_{s_r} (\bcX_{\star})
\end{align*}
with $\epsilon = 0.02$, then by setting the thresholding values $\zeta_t$ in Theorem~\ref{thm:TRPCA} and the step size $\frac{1}{4} \leq \eta \leq \frac{8}{9}$, the iterates of Algorithm~\ref{algo:ScaledGDTRPCA} satisfy
\begin{align*}
\dist(\bcF_t, \bcF_{\star}) & \leq \frac{\epsilon}{\sqrt{\ell}} \rho^t \bar{\sigma}_{s_r} (\bcX_{\star}),  \\
\sqrt{n_1} \| (\bcL_t \ast_{\bPhi} \bcQ_t - \bcL_{\star}) \ast_{\bPhi} \bcG_{\star}^{\frac{1}{2}} \|_{2,\infty} & \vee \sqrt{n_2} \| (\bcR_t \ast_{\bPhi} \bcQ_t^{-H} - \bcR_{\star}) \ast_{\bPhi} \bcG_{\star}^{\frac{1}{2}} \|_{2,\infty} \leq \sqrt{\frac{\mu s_r}{n_3 \ell}} \rho^t \bar{\sigma}_{s_r} (\bcX_{\star}),
\end{align*}
where the convergence rate $\rho = 1 - 0.6 \eta$.
\end{lemma}
The following lemma ensures that the spectral initialization satisfies the distance and incoherence conditions.
\begin{lemma}[Guaranteed initialization]\label{lemma:TRPCAinitial}
Let $\bcY = \bcX_{\star} + \bcS_{\star} \in \R^{n_1 \times n_2 \times n_3}$, where $\bcX_{\star} = \bcL_{\star} \ast_{\bPhi} \bcR_{\star}^H$ is a multi-rank $\br$ tensor with $\mu$-incoherence and $\bcS_{\star}$ is $\alpha$-sparse. Under the assumption that $\alpha \leq \frac{c_0}{\mu s_r^{1.5} \kappa \frac{n_1 + n_2 n_3}{n_3 \sqrt{n_1 n_2}}} \wedge \frac{c_0^2 n_3}{\mu^2 s_r^2 \kappa^2}$ for some small positive constant $c_0 \leq 0.06$ and the choice of the thresholding value $\| \bcX_{\star} \|_{\infty} \leq \zeta_0 \leq 2 \| \bcX_{\star} \|_{\infty}$, the spectral initialization satisfies
\begin{align*}
\dist(\bcF_0, \bcF_{\star}) & \leq \frac{5 c_0}{\sqrt{\ell}} \bar{\sigma}_{s_r} (\bcX_{\star}) \quad \mathrm{and}  \\
\sqrt{n_1} \| (\bcL_0 \ast_{\bPhi} \bcQ_0 - \bcL_{\star}) \ast_{\bPhi} \bcG_{\star}^{\frac{1}{2}} \|_{2,\infty} & \vee \sqrt{n_2} \| (\bcR_0 \ast_{\bPhi} \bcQ_0^{-H} - \bcR_{\star}) \ast_{\bPhi} \bcG_{\star}^{\frac{1}{2}} \|_{2,\infty} \leq \sqrt{\frac{\mu s_r}{n_3 \ell}} \bar{\sigma}_{s_r} (\bcX_{\star}),
\end{align*}
where $\bcQ_0$ be the optimal alignment tensor between $\begin{bmatrix} \bcL_0 \\ \bcR_0 \end{bmatrix}$ and $\begin{bmatrix} \bcL_{\star} \\ \bcR_{\star} \end{bmatrix}$.
\end{lemma}
The proofs of the above two lemmas are provided in Appendix~\ref{sec:TRPCAproof}. We also present the following lemma that verifies the selection of thresholding value is indeed effective.
\begin{lemma}[\citet{CaiLY.NeurIPS2021}, Lemma 5]\label{lemma:sparity}
At the $(t+1)$-th iteration of Algorithm~\ref{algo:ScaledGDTRPCA}, taking the thresholding value $\zeta_{t+1} = \| \bcX_{\star} - \bcX_t \|_{\infty}$ gives
\begin{align*}
\| \bcS_{\star} - \bcS_{t+1} \|_{\infty} \leq 2 \| \bcX_{\star} - \bcX_t \|_{\infty} \quad \mathrm{and} \quad \supp(\bcS_{t+1}) \subseteq \supp (\bcS_{\star}).
\end{align*}
\end{lemma}

\subsection{Proof Outline for Tensor Completion}

We start with the following lemma that ensures the scaled projection in \eqref{eqn:scaledprojsol} satisfies both non-expansiveness and incoherence under the scaled metric.
\begin{lemma}\label{lemma:scaledproj}
Suppose that $\bcX_{\star}$ is $\mu$-incoherent, and $\dist(\widetilde{\bcF}, \bcF_{\star}) \leq \frac{\epsilon}{\sqrt{\ell}} \bar{\sigma}_{s_r} (\bcX_{\star})$ for some $\epsilon < 1$. Set $\varsigma \geq (1 + \epsilon) \sqrt{\frac{\mu s_r}{n_3 \ell}} \bar{\sigma}_1 (\bcX_{\star})$, then $\cP_{\varsigma}(\widetilde{\bcF})$ satisfies the non-expansiveness
\begin{align*}
\dist(\cP_{\varsigma}(\widetilde{\bcF}), \bcF_{\star}) \le \dist(\widetilde{\bcF}, \bcF_{\star}),
\end{align*}
and the incoherence condition
\begin{align*}
\sqrt{n_1} \| \bcL \ast_{\bPhi} \bcR^H \|_{2,\infty} \vee \sqrt{n_2} \| \bcR \ast_{\bPhi} \bcL^H \|_{2,\infty} \leq \varsigma.
\end{align*}
\end{lemma}
Our next lemma guarantees the fast local convergence of Algorithm~\ref{algo:ScaledGDTC} as long as the sample complexity is large enough and the parameter $\varsigma$ is set properly.
\begin{lemma}\label{lemma:TCcontraction}
Suppose that $\bcX_{\star}$ is $\mu$-incoherent, and $p \geq c \Big( \frac{\mu s_r (n_1 + n_2) \log( (n_1 \vee n_2) n_3 )}{n_1 n_2 n_3} \vee \frac{\mu^2 s_r^2 \kappa^4 \ell \log( (n_1 \vee n_2) n_3 )}{(n_1 \wedge n_2) n_3^2} \Big)$ for some sufficiently large constant $c$. Set the projection radius as $\varsigma \geq c_{\varsigma} \sqrt{\frac{\mu s_r}{n_3 \ell}} \bar{\sigma}_1 (\bcX_{\star})$ for some constant $c_{\varsigma} \geq 1.02$. Under an event $G$ which happens with high probability, if the $t$-th iterate satisfies $\dist(\bcF_t, \bcF_{\star}) \leq \frac{0.02}{\sqrt{\ell}} \bar{\sigma}_{s_r} (\bcX_{\star})$, and the incoherence condition
\begin{align*}
\sqrt{n_1} \| \bcL_t \ast_{\bPhi} \bcR_t^H \|_{2,\infty} \vee \sqrt{n_2} \| \bcR_t \ast_{\bPhi} \bcL_t^H \|_{2,\infty} \leq \varsigma,
\end{align*}
then $\| \bcL_t \ast_{\bPhi} \bcR_t^H - \bcX_{\star} \|_F \leq 1.5 \dist(\bcF_t, \bcF_{\star})$. In addition, if the step size obeys $0 < \eta \leq \frac{2}{3}$, then the $(t+1)$-th iterate $\bcF_{t+1}$ of the ScaledGD method in \eqref{eqn:TCupdate} of Algorithm~\ref{algo:ScaledGDTC} satisfies
\begin{align*}
\dist(\bcF_{t+1}, \bcF_{\star}) \leq (1 - 0.6 \eta) \dist(\bcF_t, \bcF_{\star}),
\end{align*}
and the incoherence condition
\begin{align*}
\sqrt{n_1} \| \bcL_{t+1} \ast_{\bPhi} \bcR_{t+1}^H \|_{2,\infty} \vee \sqrt{n_2} \| \bcR_{t+1} \ast_{\bPhi} \bcL_{t+1}^H \|_{2,\infty} \leq \varsigma.
\end{align*}
\end{lemma}
Therefore, if we can find an initialization that is close to the ground truth and satisfies the incoherence condition, Lemma~\ref{lemma:TCcontraction} guarantees that the iterates of ScaledGD remain incoherent and converge linearly.
\begin{lemma}\label{lemma:TCinitial}
Suppose that $\bcX_{\star}$ is $\mu$-incoherent, then with high probability, the spectral initialization before projection $\widetilde{\bcF}_0 = \begin{bmatrix} \bcU_0 \ast_{\bPhi} \bcG_0^{\frac{1}{2}} \\ \bcV_0 \ast_{\bPhi} \bcG_0^{\frac{1}{2}} \end{bmatrix}$ in \eqref{eqn:TCinit} satisfies
\begin{align*}
\dist(\widetilde{\bcF}_0, \bcF_{\star}) \leq c \Big( \frac{\mu s_r \log( (n_1 \vee n_2) n_3 ) }{p n_3 \sqrt{n_1 n_2}} + \sqrt{\frac{\mu s_r \log( (n_1 \vee n_2) n_3 )}{p (n_1 \wedge n_2) n_3}} \Big) 5 \sqrt{\frac{s_r}{\ell}} \kappa \bar{\sigma}_{s_r} (\bcX_{\star}).
\end{align*}
\end{lemma}
Therefore, as long as $p \geq c \mu s_r^2 \kappa^2 \log( (n_1 \vee n_2) n_3 ) / ( (n_1 \wedge n_2) n_3 )$ for some sufficiently large constant $c$, the initial distance satisfies $\dist(\widetilde{\bcF}_0, \bcF_{\star}) \leq \frac{0.02}{\sqrt{\ell}} \bar{\sigma}_{s_r} (\bcX_{\star})$. We can then invoke Lemma~\ref{lemma:scaledproj} to have that $\bcF_0 = \cP_{\varsigma} (\widetilde{\bcF}_0)$ meets the conditions required in Lemma~\ref{lemma:TCcontraction}, which further enables us to repetitively apply Lemma~\ref{lemma:scaledproj} to finish the proof of Theorem~\ref{thm:TC}. The proofs of the three supporting lemmas can be found in Appendix~\ref{sec:Completionproof}.

\subsection{Proof Outline for Robust Tensor Completion}

We first present the following two lemmas of local linear convergence and guaranteed initialization, whose proofs are deferred to Appendix~\ref{sec:RobustCompletionproof}. Then the claims in Theorem~\ref{thm:RTC} follow immediately in the same fashion as the proof of Theorem~\ref{thm:TRPCA} by setting $c_0 \leq 0.001$.

\begin{lemma}[Local contraction]\label{lemma:RTCcontraction}
Let $\bcP_{\bOmega}(\bcY) = \bcP_{\bOmega}(\bcX_{\star} + \bcS_{\star}) \in \R^{n_1 \times n_2 \times n_3}$, where $\bcX_{\star} = \bcL_{\star} \ast_{\bPhi} \bcR_{\star}^H$ is a multi-rank $\br$ tensor with $\mu$-incoherence and $\bcP_{\bOmega}(\bcS_{\star})$ is $\alpha p$-sparse. Let $\bcQ_t$ be the optimal alignment tensor between $\begin{bmatrix} \bcL_t \\ \bcR_t \end{bmatrix}$ and $\begin{bmatrix} \bcL_{\star} \\ \bcR_{\star} \end{bmatrix}$. Under the assumption that $\alpha \leq \frac{p n_3 \sqrt{n_{(2)} n_2}}{10^4 \mu s_r^{1.5} (n_1 + n_2 n_3)}$, where $p \geq c \Big( \frac{\mu^{1.5} s_r^{1.5} (n_1 + n_2) \log( (n_1 \vee n_2) n_3 )}{n_3 \sqrt{n_1 n_2}} \vee \frac{\mu^2 s_r^2 \ell \log( (n_1 \vee n_2) n_3 )}{n_3 \sqrt{(n_1 \wedge n_2) n_3}} \Big)$ for some sufficiently large constant $c$. If the spectral initialization obeys the conditions
\begin{align*}
\dist(\bcF_0, \bcF_{\star}) & \leq \frac{\epsilon}{\sqrt{\ell}} \bar{\sigma}_{s_r} (\bcX_{\star}),  \\
\sqrt{n_1} \| (\bcL_0 \ast_{\bPhi} \bcQ_0 - \bcL_{\star}) \ast_{\bPhi} \bcG_{\star}^{\frac{1}{2}} \|_{2,\infty} & \vee \sqrt{n_2} \| (\bcR_0 \ast_{\bPhi} \bcQ_0^{-H} - \bcR_{\star}) \ast_{\bPhi} \bcG_{\star}^{\frac{1}{2}} \|_{2,\infty} \leq \sqrt{\frac{\mu s_r}{n_3 \ell}} \bar{\sigma}_{s_r} (\bcX_{\star})
\end{align*}
with $\epsilon = 0.02$, then by setting the thresholding values $\zeta_t$ in Theorem~\ref{thm:RTC} and the step size $\frac{1}{5} \leq \eta \leq \frac{1}{2}$, the iterates of Algorithm~\ref{algo:ScaledGDRTC} satisfy
\begin{align*}
\dist(\bcF_t, \bcF_{\star}) & \leq \frac{\epsilon}{\sqrt{\ell}} \rho^t \bar{\sigma}_{s_r} (\bcX_{\star}),  \\
\sqrt{n_1} \| (\bcL_t \ast_{\bPhi} \bcQ_t - \bcL_{\star}) \ast_{\bPhi} \bcG_{\star}^{\frac{1}{2}} \|_{2,\infty} & \vee \sqrt{n_2} \| (\bcR_t \ast_{\bPhi} \bcQ_t^{-H} - \bcR_{\star}) \ast_{\bPhi} \bcG_{\star}^{\frac{1}{2}} \|_{2,\infty} \leq \sqrt{\frac{\mu s_r}{n_3 \ell}} \rho^t \bar{\sigma}_{s_r} (\bcX_{\star}),
\end{align*}
where the convergence rate $\rho = 1 - 0.3 \eta$.
\end{lemma}
\begin{lemma}[Guaranteed initialization]\label{lemma:RTCinitial}
Let $\bcP_{\bOmega}(\bcY) = \bcP_{\bOmega}(\bcX_{\star} + \bcS_{\star}) \in \R^{n_1 \times n_2 \times n_3}$, where $\bcX_{\star} = \bcL_{\star} \ast_{\bPhi} \bcR_{\star}^H$ is a multi-rank $\br$ tensor with $\mu$-incoherence and $\bcP_{\bOmega}(\bcS_{\star})$ is $\alpha p$-sparse. Under the assumption that $\alpha \leq \frac{c_0 p}{\mu s_r^{1.5} \kappa \frac{n_1 + n_2 n_3}{n_3 \sqrt{n_1 n_2}}} \wedge \frac{c_0^2 p n_3}{\mu^2 s_r^2 \kappa^2}$ for some small positive constant $c_0 \leq 0.05$, where $p \geq c \mu s_r \kappa^2 \log( (n_1 \vee n_2) n_3 ) / \sqrt{(n_1 \wedge n_2) n_3}$ for some sufficiently large constant $c$. Given the choice of the thresholding value $\| \bcP_{\bOmega}(\bcX_{\star}) \|_{\infty} \leq \zeta_0 \leq 2 \| \bcP_{\bOmega}(\bcX_{\star}) \|_{\infty}$, the spectral initialization satisfies
\begin{align*}
\dist(\bcF_0, \bcF_{\star}) & \leq \frac{5 c_0 + 0.01}{\sqrt{\ell}} \bar{\sigma}_{s_r} (\bcX_{\star}) \quad \mathrm{and}  \\
\sqrt{n_1} \| (\bcL_0 \ast_{\bPhi} \bcQ_0 - \bcL_{\star}) \ast_{\bPhi} \bcG_{\star}^{\frac{1}{2}} \|_{2,\infty} & \vee \sqrt{n_2} \| (\bcR_0 \ast_{\bPhi} \bcQ_0^{-H} - \bcR_{\star}) \ast_{\bPhi} \bcG_{\star}^{\frac{1}{2}} \|_{2,\infty} \leq \sqrt{\frac{\mu s_r}{n_3 \ell}} \bar{\sigma}_{s_r} (\bcX_{\star}),
\end{align*}
where $\bcQ_0$ be the optimal alignment tensor between $\begin{bmatrix} \bcL_0 \\ \bcR_0 \end{bmatrix}$ and $\begin{bmatrix} \bcL_{\star} \\ \bcR_{\star} \end{bmatrix}$.
\end{lemma}

\subsection{Proof Outline for Tensor Regression}

Now we turn to the proof outline for tensor regression (cf. Theorem~\ref{thm:TR}). Leveraging the TRIP of $\cA(\cdot)$, we can establish the following local convergence guarantee of ScaledGD, as long as the iterates are close to the ground truth.
\begin{lemma}\label{lemma:TRcontraction}
Suppose that $\cA(\cdot)$ obeys the $2r$-TRIP with $\delta_{2r} \leq 0.02$. If the $t$-th iterate satisfies $\dist(\bcF_t, \bcF_{\star}) \leq 0.1 \bar{\sigma}_{s_r} (\bX_{\star})$, then $\| \bcL_t \ast_{\bPhi} \bcR_t^H - \bcX_{\star} \|_F \leq 1.5 \dist(\bcF_t, \bcF_{\star})$. In addition, if the step size obeys $0 < \eta \leq \frac{2}{3}$, then the $(t+1)$-th iterate $\bcF_{t+1}$ of the ScaledGD method in \eqref{eqn:TRupdate} of Algorithm~\ref{algo:ScaledGDTR} satisfies
\begin{align*}
\dist(\bcF_{t+1}, \bcF_{\star}) \leq (1 - 0.6 \eta) \dist(\bcF_t, \bcF_{\star}).
\end{align*}
\end{lemma}

To establish the induction hypothesis, we still need to check the quality of the spectral initialization, for which we have the following lemma.
\begin{lemma}\label{lemma:TRinitial}
Suppose that $\cA(\cdot)$ obeys the $2r$-TRIP with a constant $\delta_{2r}$. Then the spectral initialization in \eqref{eqn:TRinit} for low-rank tensor regression satisfies
\begin{align*}
\dist(\bcF_{0},\bcF_{\star}) \leq 5 \delta_{2r} \sqrt{\frac{s_r}{\ell}} \kappa \bar{\sigma}_{s_r} (\bcX_{\star}).
\end{align*}
\end{lemma}
As a result, setting $\delta_{2r} \leq 0.02/(\sqrt{\frac{s_r}{\ell}} \kappa)$ as specified in Theorem~\ref{thm:TR}, the initial distance satisfies $\dist(\bcF_0, \bcF_{\star}) \leq 0.1 \bar{\sigma}_{s_r} (\bX_{\star})$, allowing us to invoke Lemma~\ref{lemma:TRcontraction} recursively. The proof of Theorem~\ref{thm:TR} is then complete. The proofs of Lemma~\ref{lemma:TRcontraction} and Lemma~\ref{lemma:TRinitial} can be found in Appendix~\ref{sec:Regressionproof}.

\section{Numerical Experiments}
\label{sec:experiment}

In this section, we present several experimental results demonstrating the effectiveness of our proposed methods. We first provide numerical experiments to corroborate our theoretical findings. Then we evaluate the performance of our algorithms by focusing on video denoising and background initialization tasks. All experiments are performed in Matlab with an AMD Ryzen 9 5950X 3.40GHz CPU and 64GB RAM.

\subsection{Synthetic Data Experiments}

We compare the iteration complexity of ScaledGD with vanilla gradient descent (GD). For fair comparison, both algorithms start from the same spectral initialization, and the update rule of GD is given by
\begin{align*}
\bcL_{t+1} & = \bcL_t - \eta_{\mathrm{GD}} \nabla_{\bcL} f(\bcL_t, \bcR_t),  \\
\bcR_{t+1} & = \bcR_t - \eta_{\mathrm{GD}} \nabla_{\bcR} f(\bcL_t, \bcR_t),
\end{align*}
where $\eta_{\mathrm{GD}} = \eta / \bar{\sigma}_1 (\bcX_{\star})$ stands for the step size for vanilla GD.

\begin{figure}[t]
\centering
\subfloat[\centering Tensor RPCA, DFT \\ $n = 100$, $\alpha = 0.1$]{\includegraphics[height=1.4in]{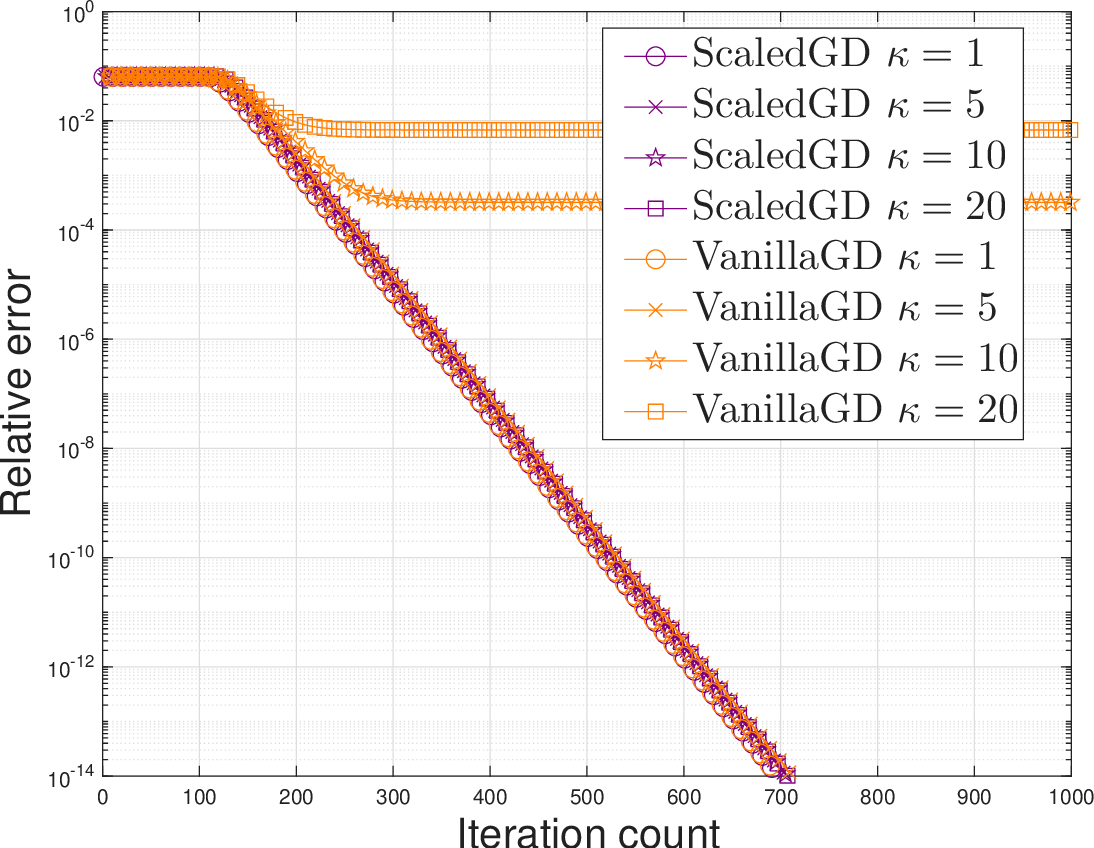} \label{fig:TRPCAFFTnoiseless}}
\quad
\subfloat[\centering Tensor RPCA, DCT \\ $n = 100$, $\alpha = 0.1$]{\includegraphics[height=1.4in]{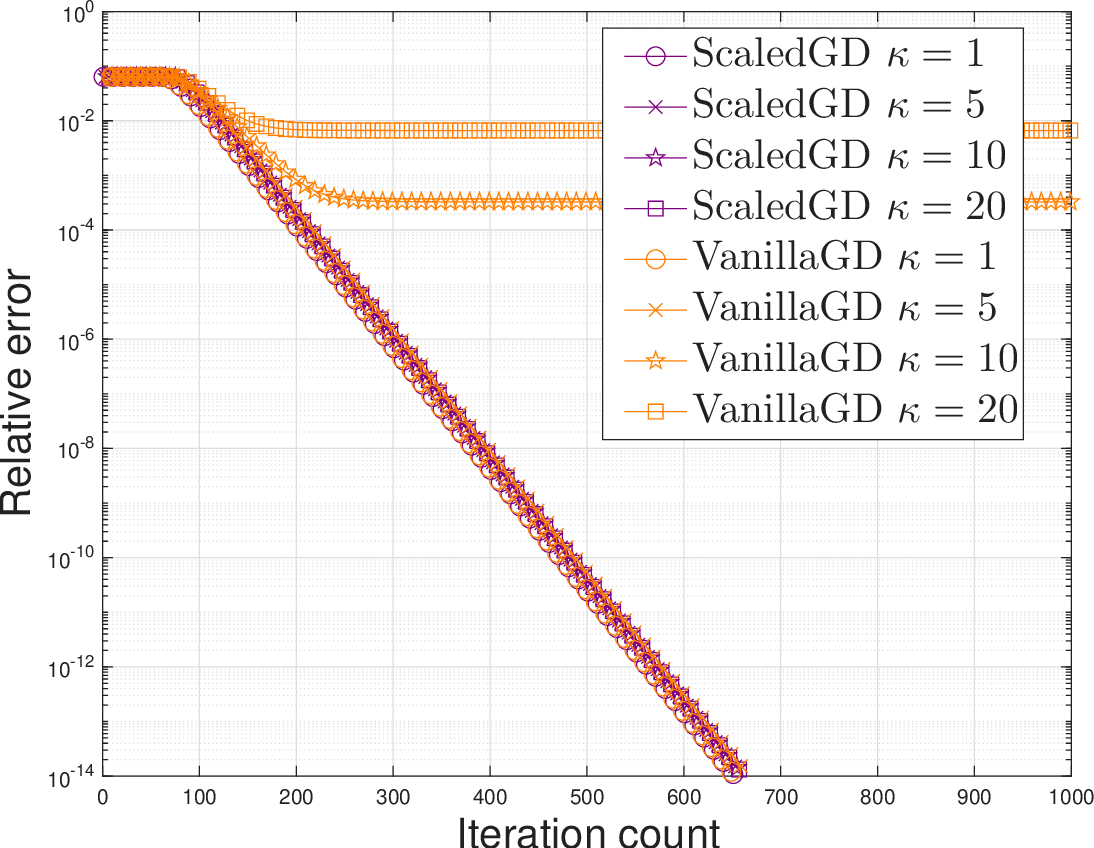} \label{fig:TRPCADCTnoiseless}}
\quad
\subfloat[\centering Tensor Completion, DFT \\ $n = 100$, $p = 0.4$]{\includegraphics[height=1.4in]{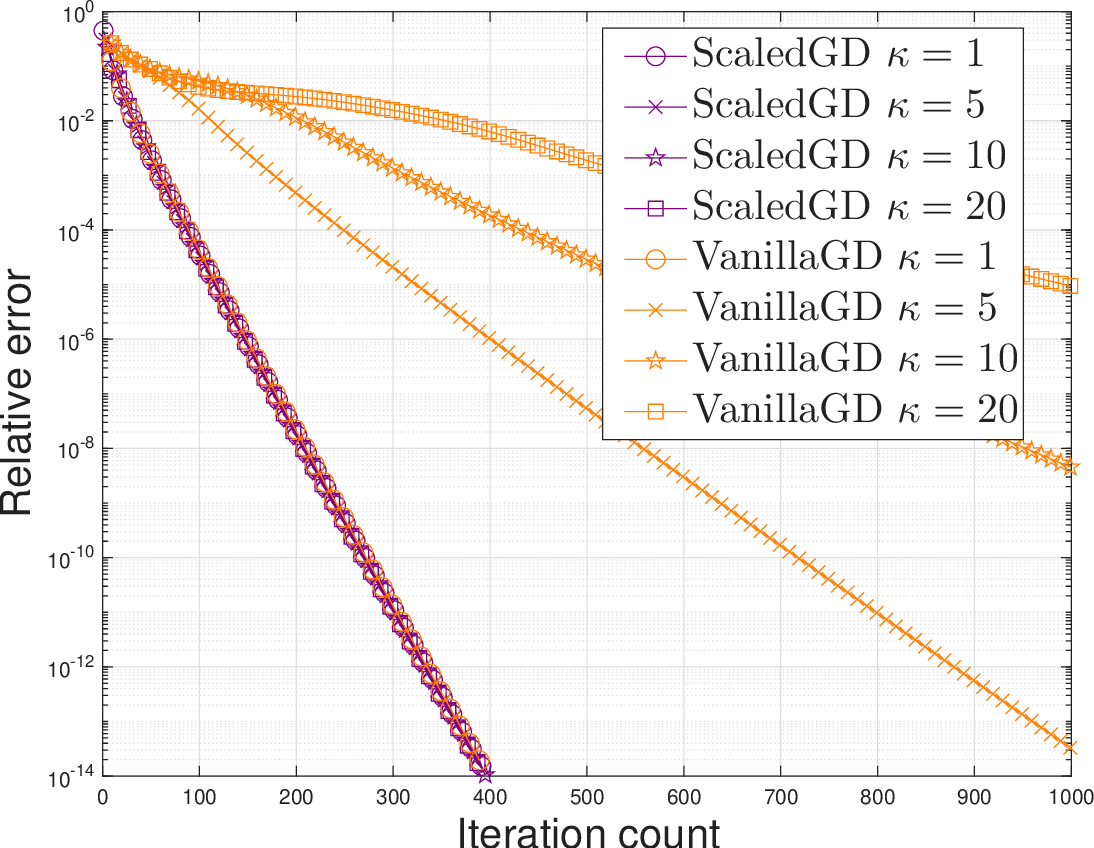} \label{fig:TCFFTnoiseless}}
\quad
\subfloat[\centering Tensor Completion, DCT \\ $n = 100$, $p = 0.4$]{\includegraphics[height=1.4in]{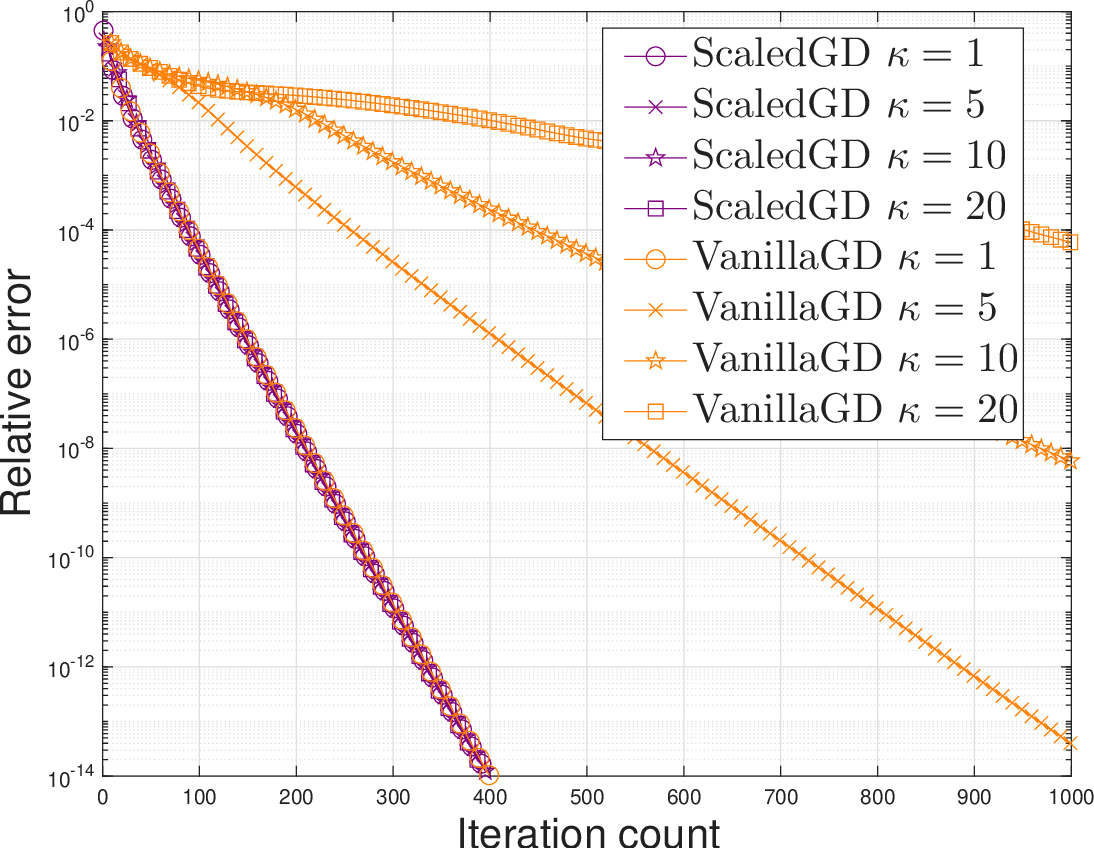} \label{fig:TCDCTnoiseless}}
\quad
\subfloat[\centering RTC, DFT \\ $n = 100$, $\alpha = 0.1$, $p = 0.6$]{\includegraphics[height=1.4in]{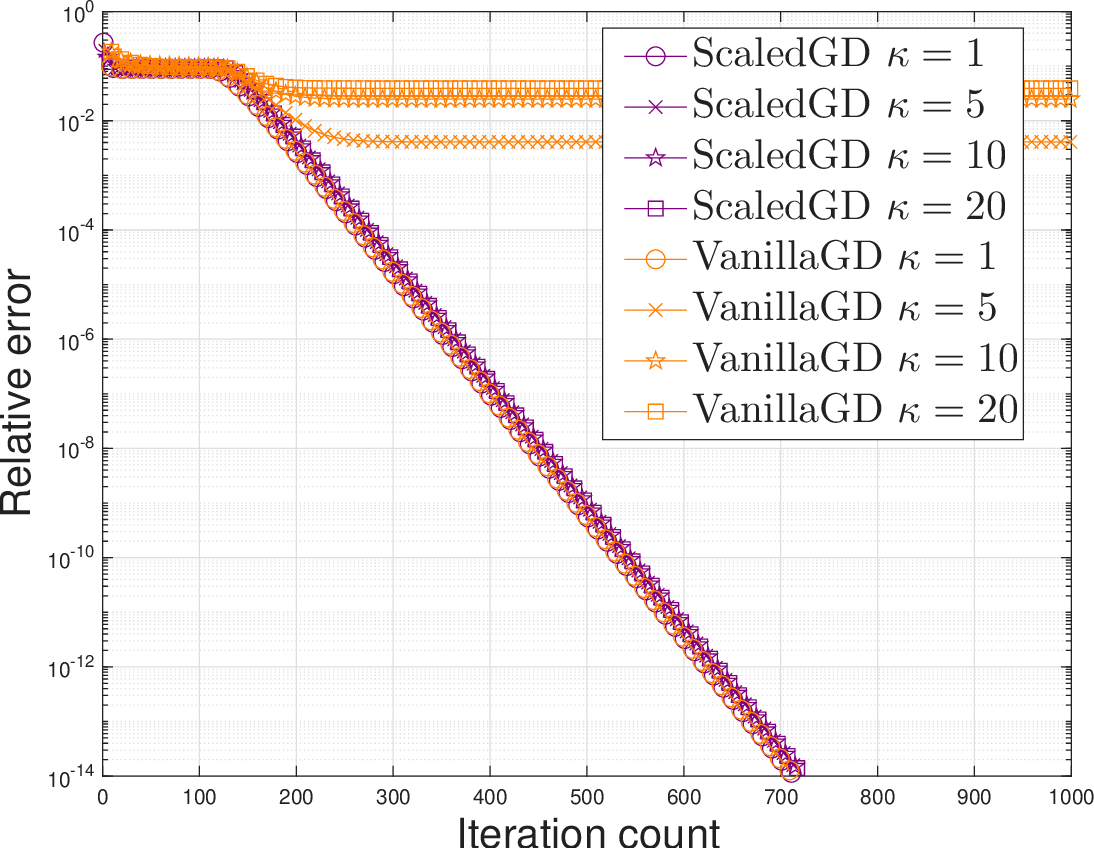} \label{fig:RTCFFTnoiseless}}
\quad
\subfloat[\centering RTC, DCT \\ $n = 100$, $\alpha = 0.1$, $p = 0.6$]{\includegraphics[height=1.4in]{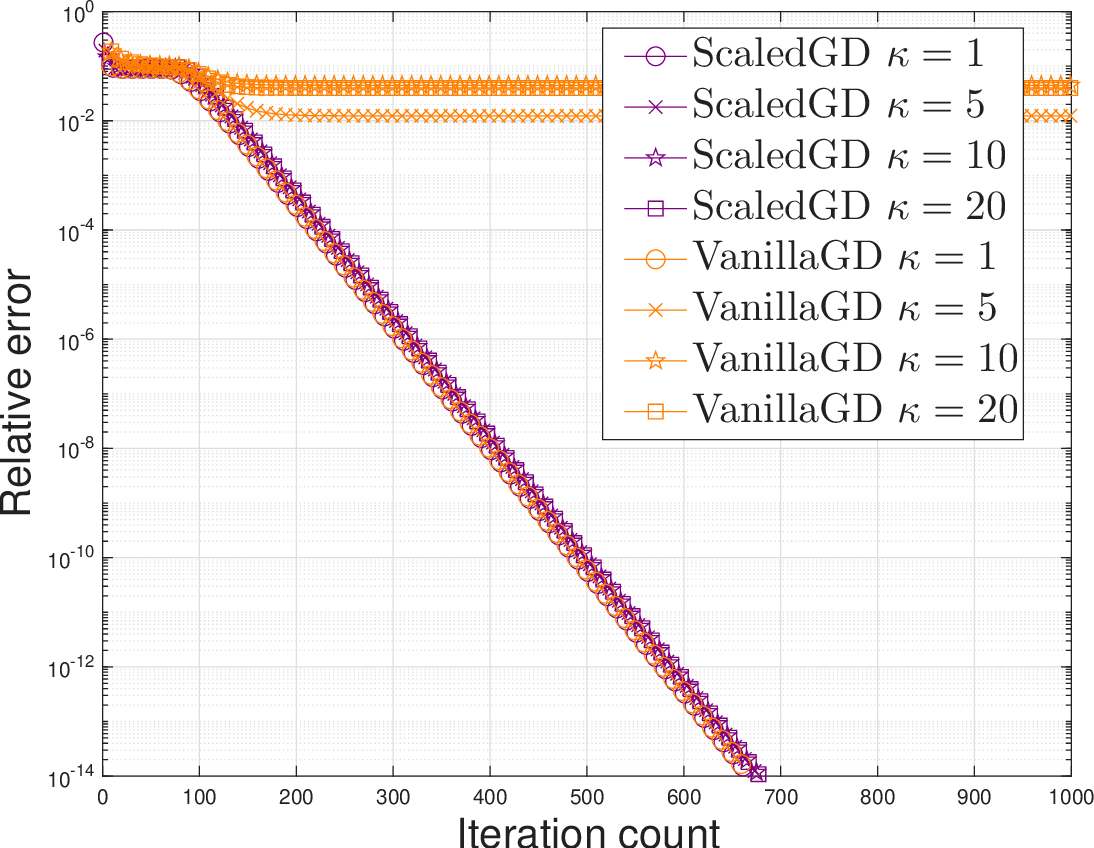} \label{fig:RTCDCTnoiseless}}
\quad
\subfloat[\centering Tensor Regression, DFT \\ $n = 50$, $m = 25000$]{\includegraphics[height=1.4in]{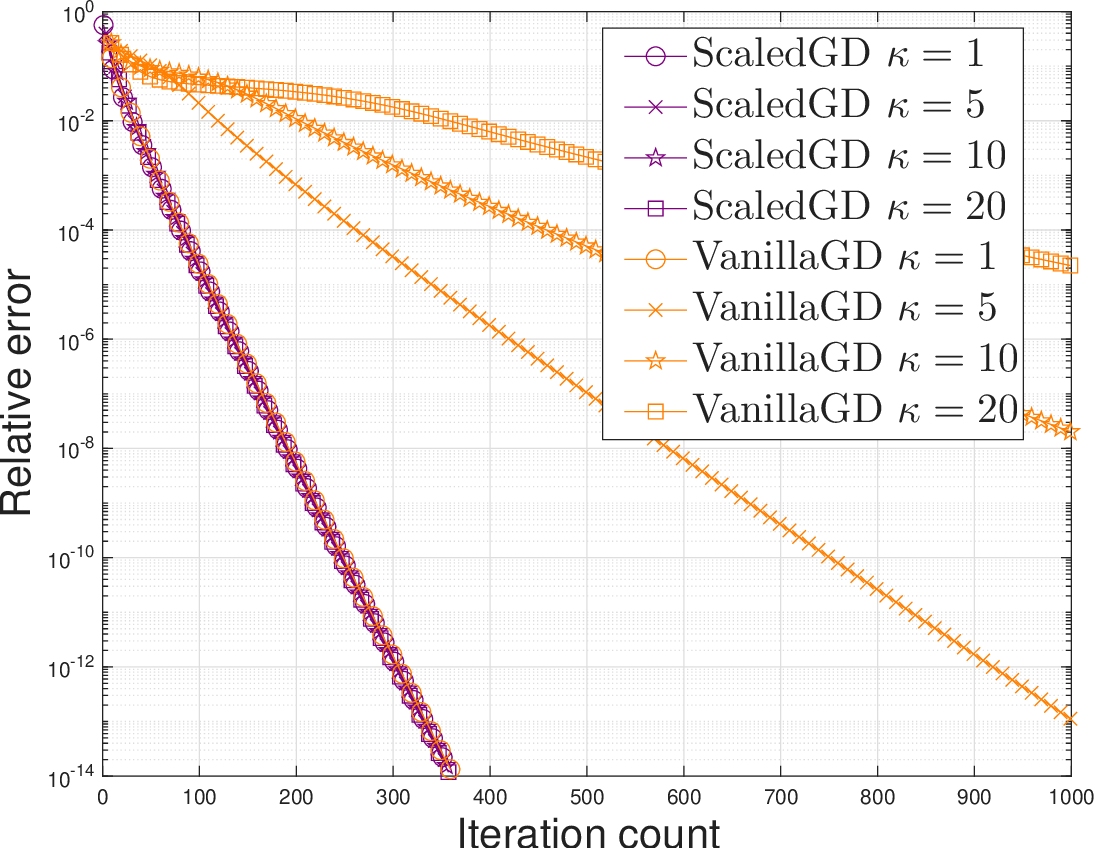} \label{fig:TRFFTnoiseless}}
\quad
\subfloat[\centering Tensor Regression, DCT \\ $n = 50$, $m = 25000$]{\includegraphics[height=1.4in]{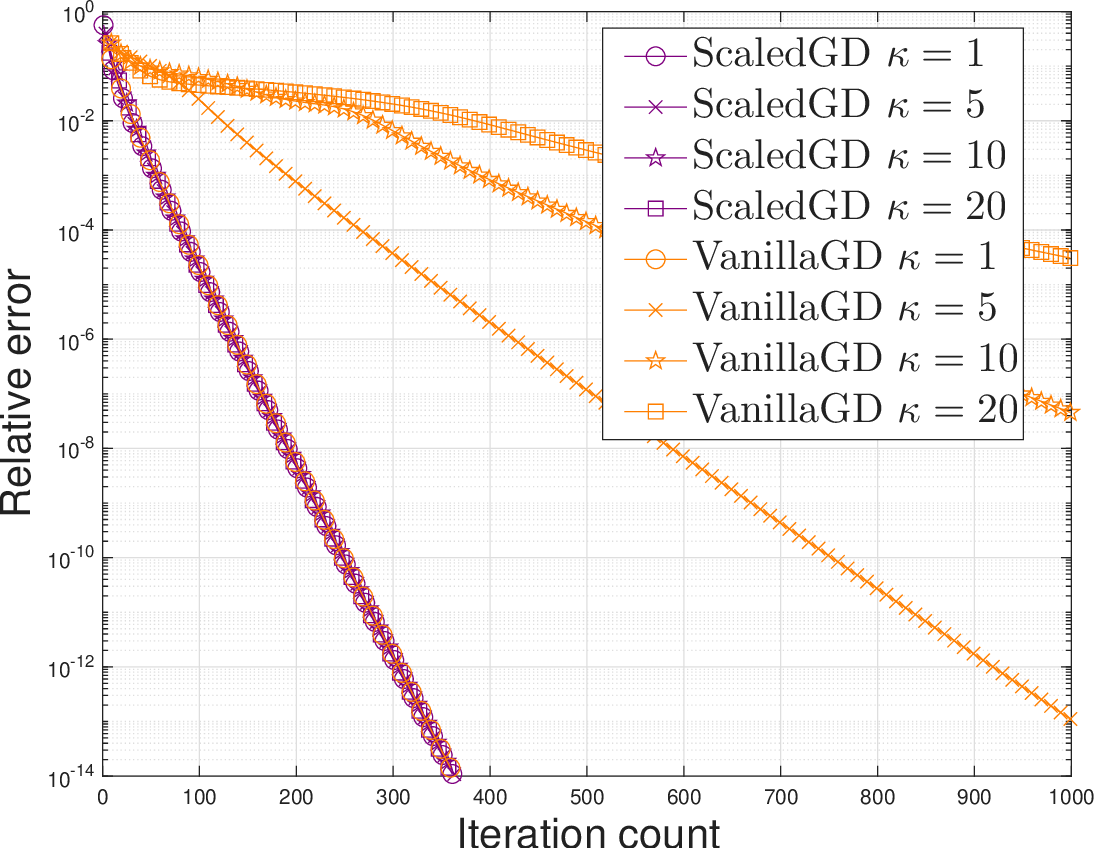} \label{fig:TRDCTnoiseless}}
\caption{The relative errors of ScaledGD and vanilla GD with respect to the iteration count under different condition numbers $\kappa = 1, 5, 10, 20$ for (a/b) tensor RPCA, (c/d) tensor completion, (e/f) robust tensor completion, and (g/h) tensor regression.}
\label{fig:noiselessiteration}
\end{figure}

\begin{figure}[t]
\centering
\subfloat[\centering Tensor RPCA, DFT \\ $n = 100$, $\alpha = 0.1$]{\includegraphics[height=1.4in]{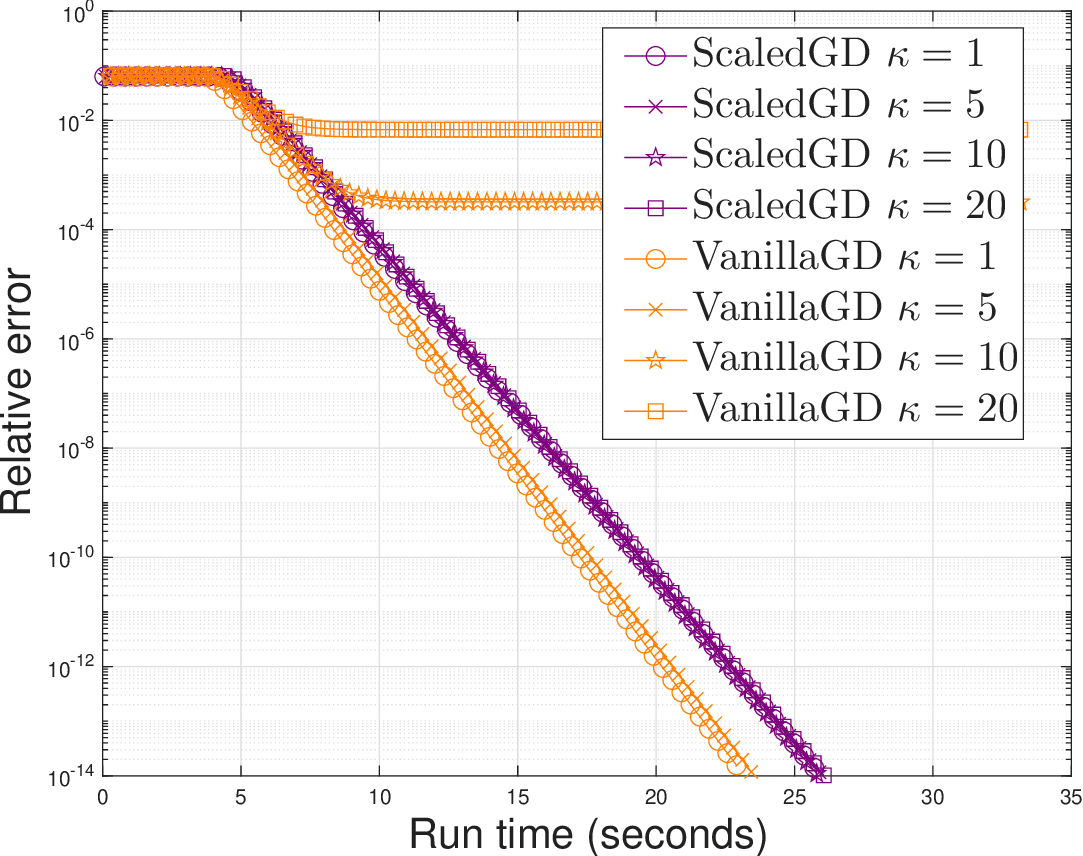} \label{fig:TRPCAFFTtime}}
\quad
\subfloat[\centering Tensor RPCA, DCT \\ $n = 100$, $\alpha = 0.1$]{\includegraphics[height=1.4in]{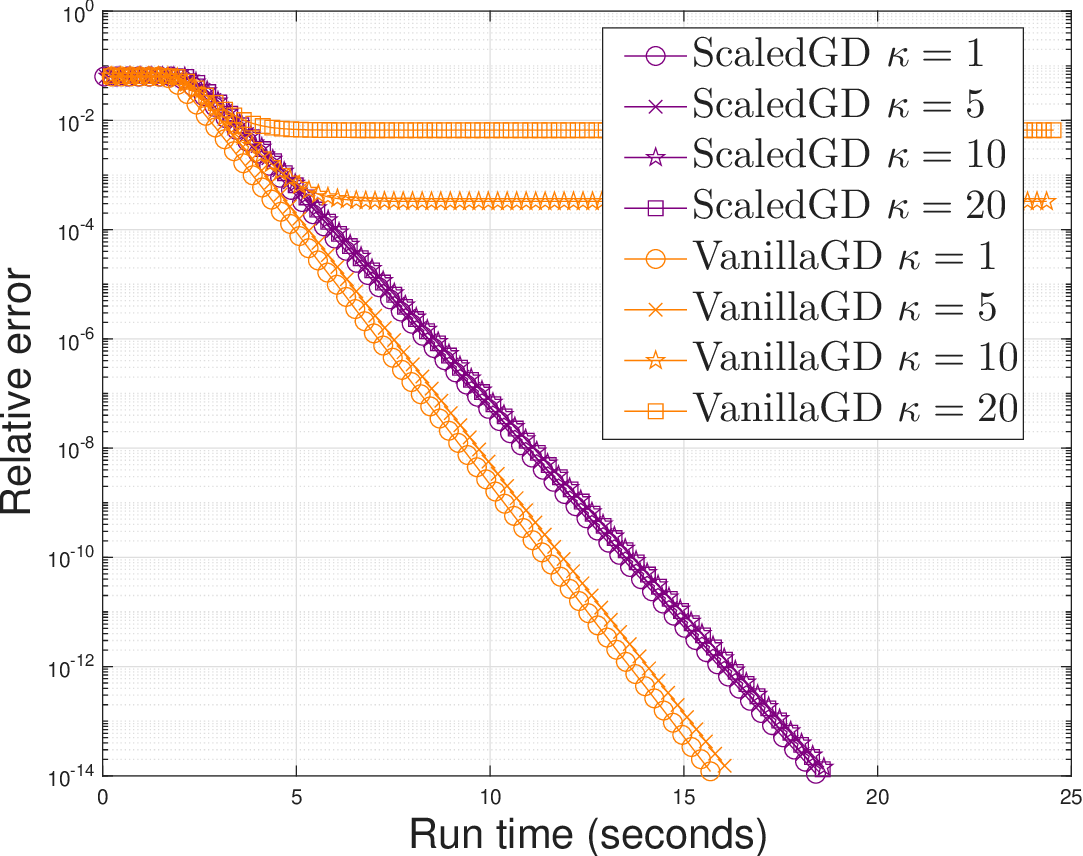} \label{fig:TRPCADCTtime}}
\quad
\subfloat[\centering Tensor Completion, DFT \\ $n = 100$, $p = 0.4$]{\includegraphics[height=1.4in]{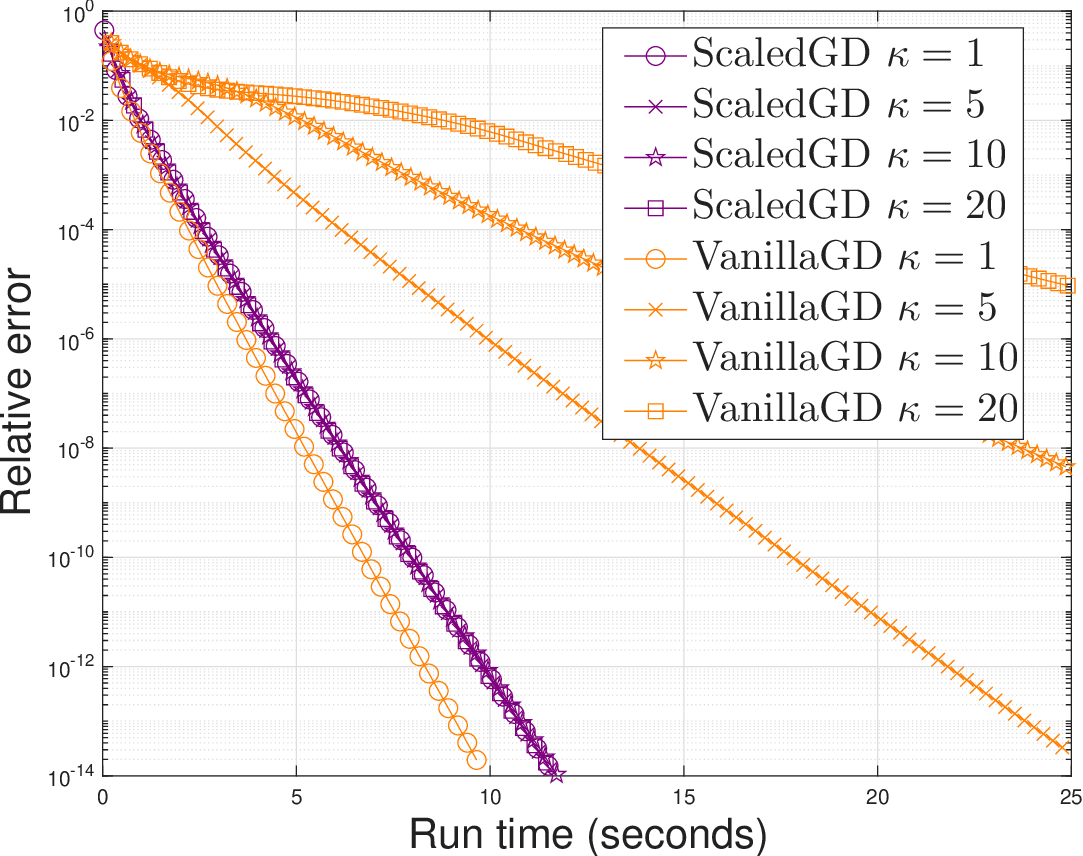} \label{fig:TCFFTtime}}
\quad
\subfloat[\centering Tensor Completion, DCT \\ $n = 100$, $p = 0.4$]{\includegraphics[height=1.4in]{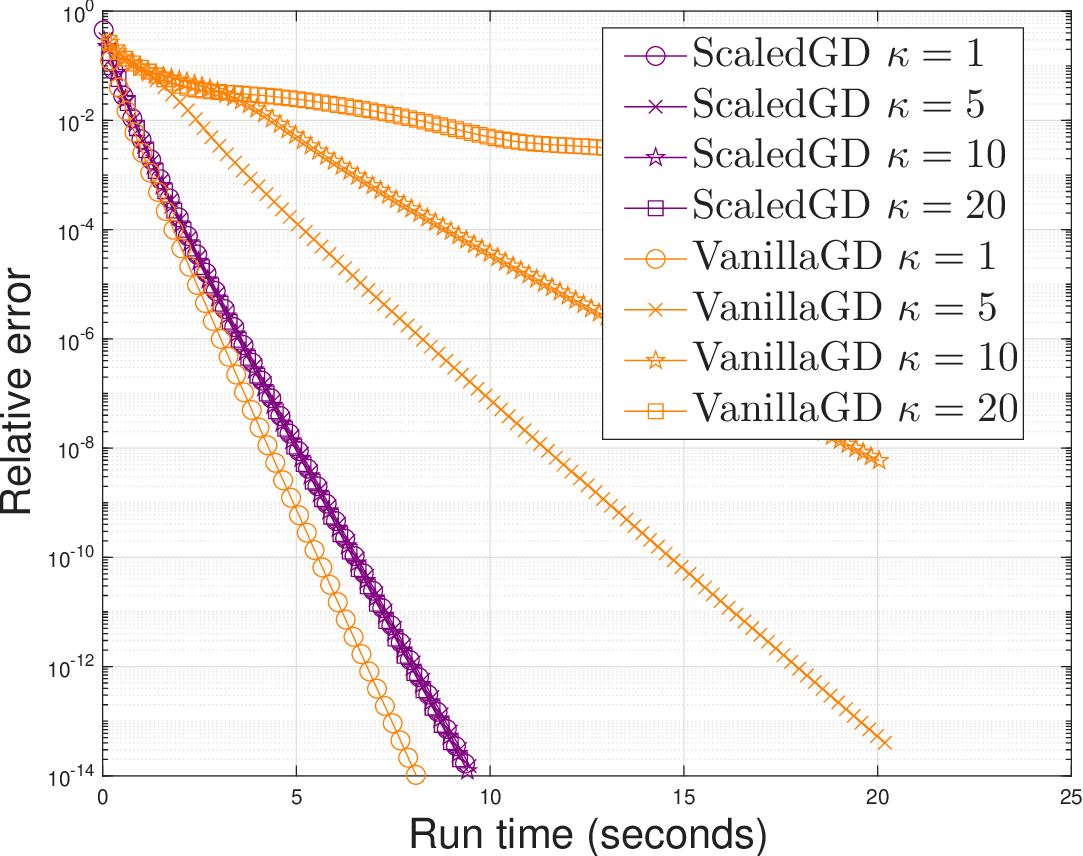} \label{fig:TCDCTtime}}
\quad
\subfloat[\centering RTC, DFT \\ $n = 100$, $\alpha = 0.1$, $p = 0.6$]{\includegraphics[height=1.4in]{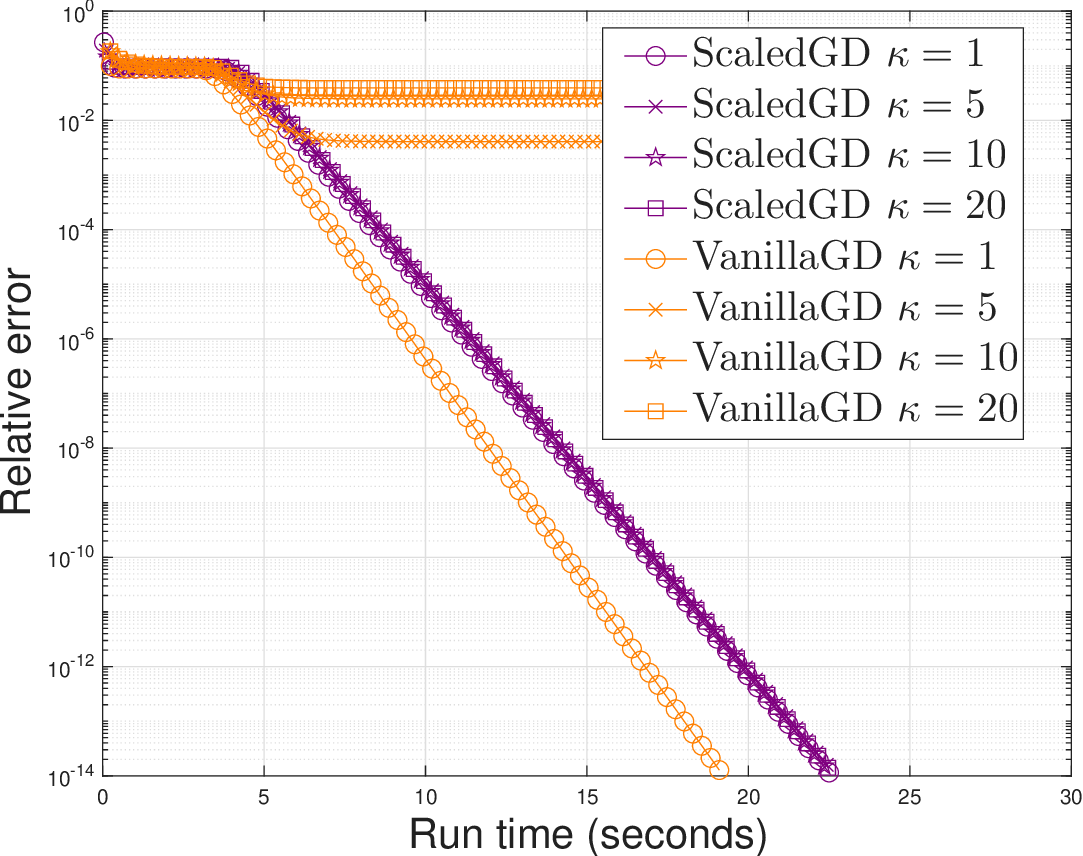} \label{fig:RTCFFTtime}}
\quad
\subfloat[\centering RTC, DCT \\ $n = 100$, $\alpha = 0.1$, $p = 0.6$]{\includegraphics[height=1.4in]{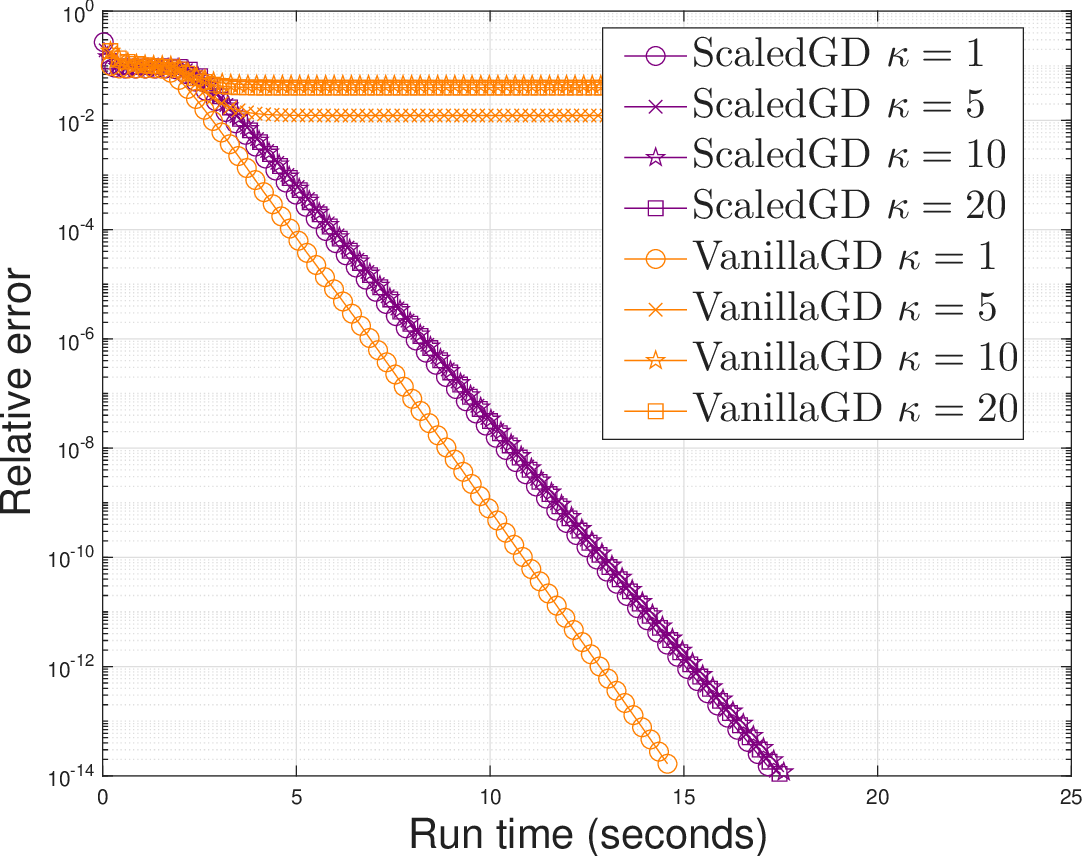} \label{fig:RTCDCTtime}}
\quad
\subfloat[\centering Tensor Regression, DFT \\ $n = 50$, $m = 25000$]{\includegraphics[height=1.4in]{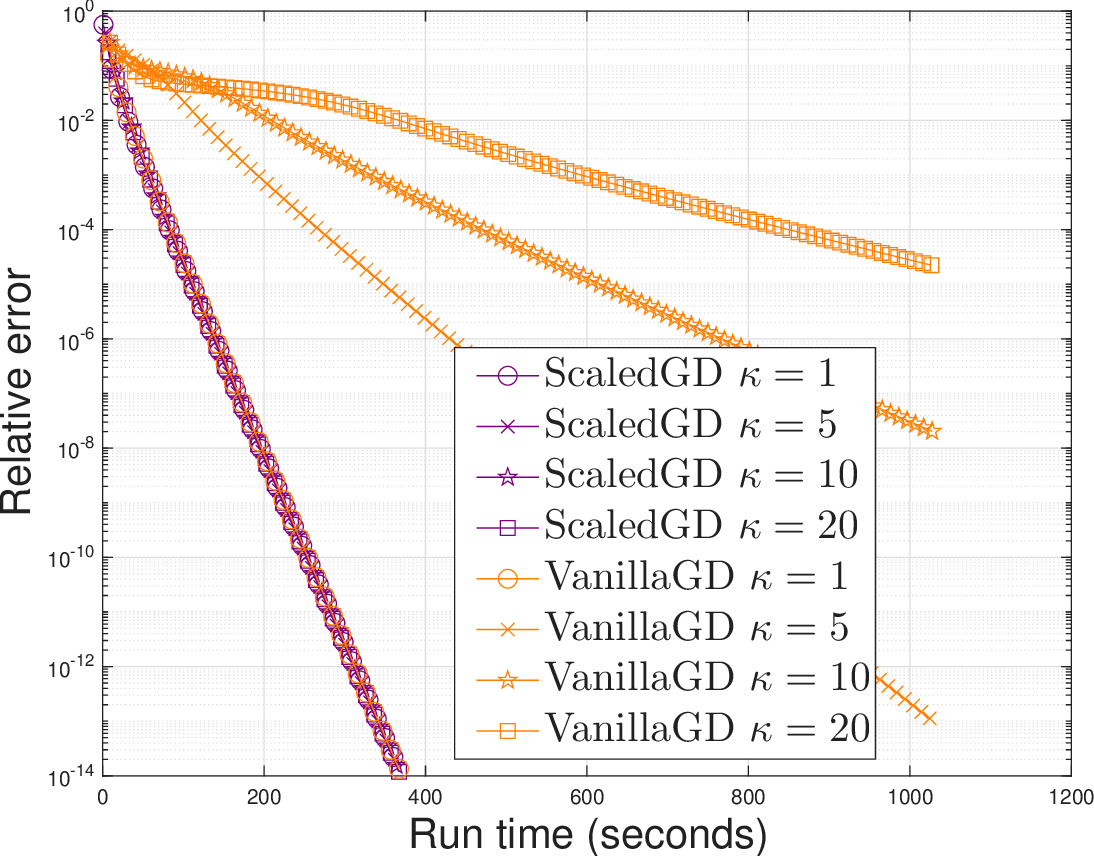} \label{fig:TRFFTtime}}
\quad
\subfloat[\centering Tensor Regression, DCT \\ $n = 50$, $m = 25000$]{\includegraphics[height=1.4in]{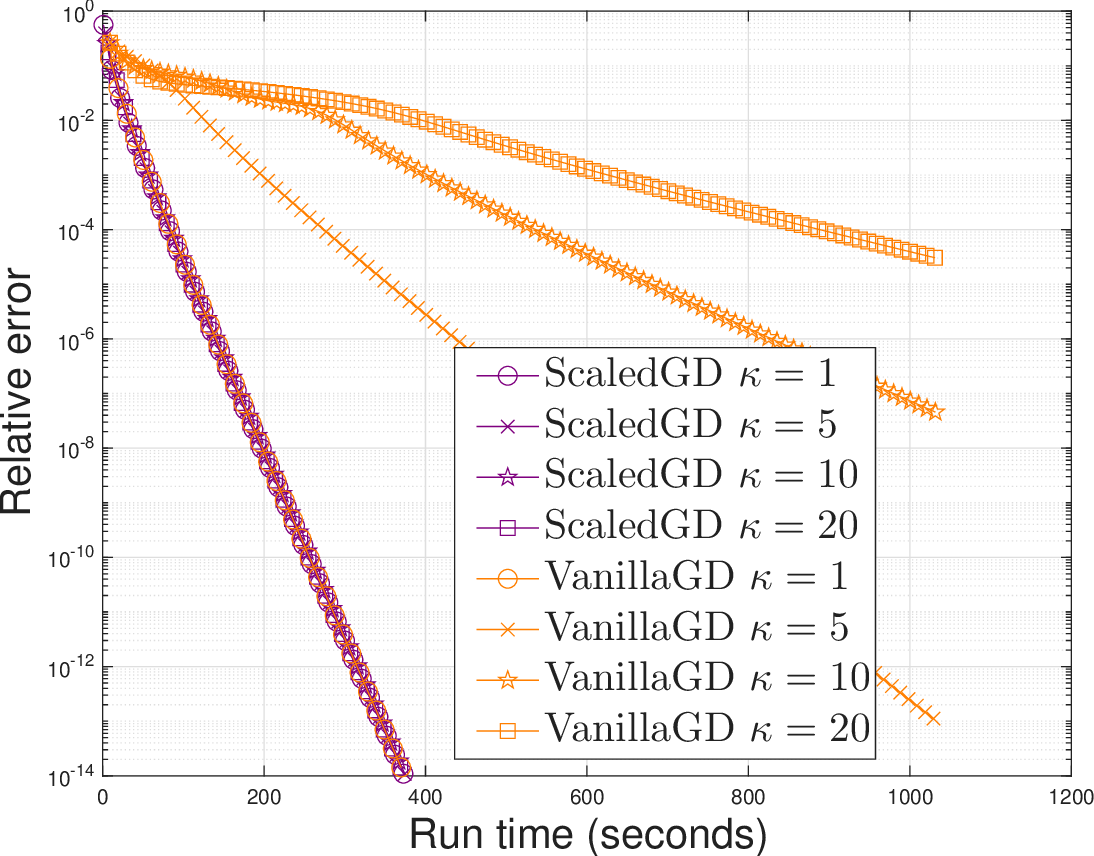} \label{fig:TRDCTtime}}
\caption{The relative errors of ScaledGD and vanilla GD with respect to run time (in seconds) under different condition numbers $\kappa = 1, 5, 10, 20$ for (a/b) tensor RPCA, (c/d) tensor completion, (e/f) robust tensor completion, and (g/h) tensor regression.}
\label{fig:noiselessruntime}
\end{figure}

In this experiment, we adopt two invertible linear transforms: (a) Discrete Fourier Transform (DFT) and (b) Discrete Cosine Transform (DCT). For simplicity, we set $n_1 = n_2 = n$. We generate the ground truth tensor $\bcX_{\star} \in \R^{n \times n \times n_3}$ as follows. We first generate an $n \times r \times n_3$ tensor with i.i.d. random signs, and take its $r$ left singular tensors as $\bcU_{\star}$, and similarly for $\bcV_{\star}$. The diagonal entries in each frontal slice of the f-diagonal tensor $\xoverline{\bcG}_{\star}$ are set to be linearly distributed from 1 to $1/\kappa$. Then the underlying low-rank tensor is generated by $\bcX_{\star} = \bcU_{\star} \ast_{\bPhi} \bcG_{\star} \ast_{\bPhi} \bcV_{\star}^H$, which has the specified condition number $\kappa$ and tubal rank $r$. We consider the following four low-rank tensor estimation tasks:
\begin{itemize}
  \item \emph{Tensor RPCA.} We generate the sparse corruption tensor by uniformly and independently sampling $\alpha$-fraction of the entries as the non-zero locations of $\bcS_{\star}$ with $\alpha = 0.1$. The magnitudes of the non-zero entries of $\bcS_{\star}$ are sampled i.i.d. from the uniform distribution over the interval of $[-\mathbb{E}[|\bcX_{\star}|_{i,j,k}], \mathbb{E}[|\bcX_{\star}|_{i,j,k}]]$. The observation tensor is $\bcY = \bcX_{\star} + \bcS_{\star} + \bcW$, where $\bcW_{i,j,k} \sim \cN(0, \sigma_w^2)$ composed of i.i.d. Gaussian entries.
  \item \emph{Tensor completion.} We assume random Bernoulli observations, where each entry of $\bcX_{\star}$ is observed with probability $p = 0.4$ independently. The observation is $\bcY = \bcP_{\bOmega}(\bcX_{\star} + \bcW)$, where $\bcW_{i,j,k} \sim \cN(0, \sigma_w^2)$ is composed of i.i.d. Gaussian entries. Moreover, we perform the scaled gradient updates without projections.
  \item \emph{Robust tensor completion.} The problem formulation is stated in Section~\ref{ssec:RTC}. We first generate the corruption using a sparse tensor $\bcS_{\star} \in \mathscr{S}_{\alpha}$ with $\alpha = 0.1$. The observation is generated by $\bcY = \bcP_{\bOmega}(\bcX_{\star} + \bcS_{\star})$, where $\bOmega$ is generated according to the Bernoulli model with probability $p = 0.6$.
  \item \emph{Tensor regression.} The problem formulation is detailed in Section~\ref{ssec:TR}. Here, we collect $m = 5 n n_3 r$ measurements in the form of $y_i = \langle \bcA_i, \bcX_{\star} \rangle$, in which the measurement tensors $\bcA_i$ are generated with i.i.d. Gaussian entries with zero mean and variance $1/m$.
\end{itemize}
We simply set $\eta = 0.4$ for the robust tensor completion problem and fix $\eta = 0.5$ for both ScaledGD and vanilla GD for the other three problems (see Figure~\ref{fig:hyperparam} for justifications). The hyperparameters for tensor RPCA and robust tensor completion are set as follows: $\zeta_0 = 0.5$, $\zeta_1 = 0.5$, and $\rho = 0.95$. For simplicity, we set $n_3 = n = 100$ and $r = 10$ for the first three problems and set $n = 50$, $n_3 = 20$ and $r = 5$ for the tensor regression problem.

\begin{figure}[t]
\centering
\subfloat[Tensor RPCA, DFT]{\includegraphics[height=1.7in]{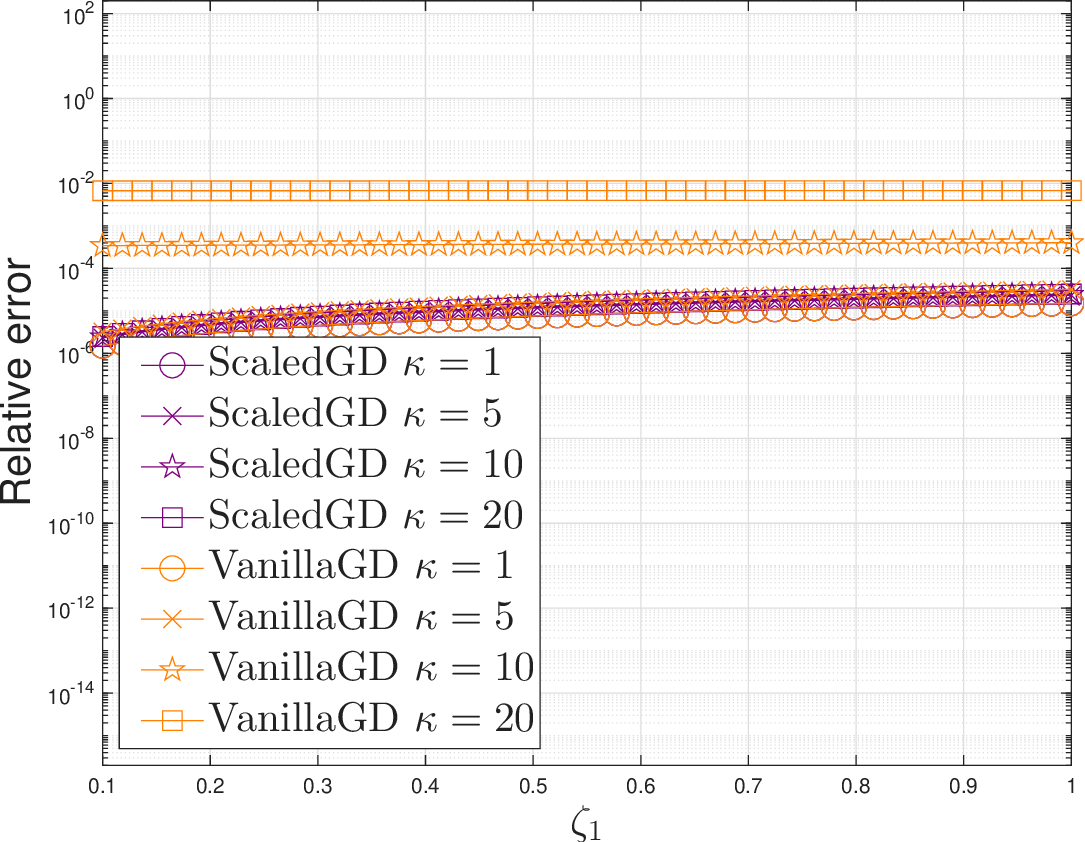} \label{fig:TRPCAFFTzeta1}}
\quad
\subfloat[Tensor RPCA, DCT]{\includegraphics[height=1.7in]{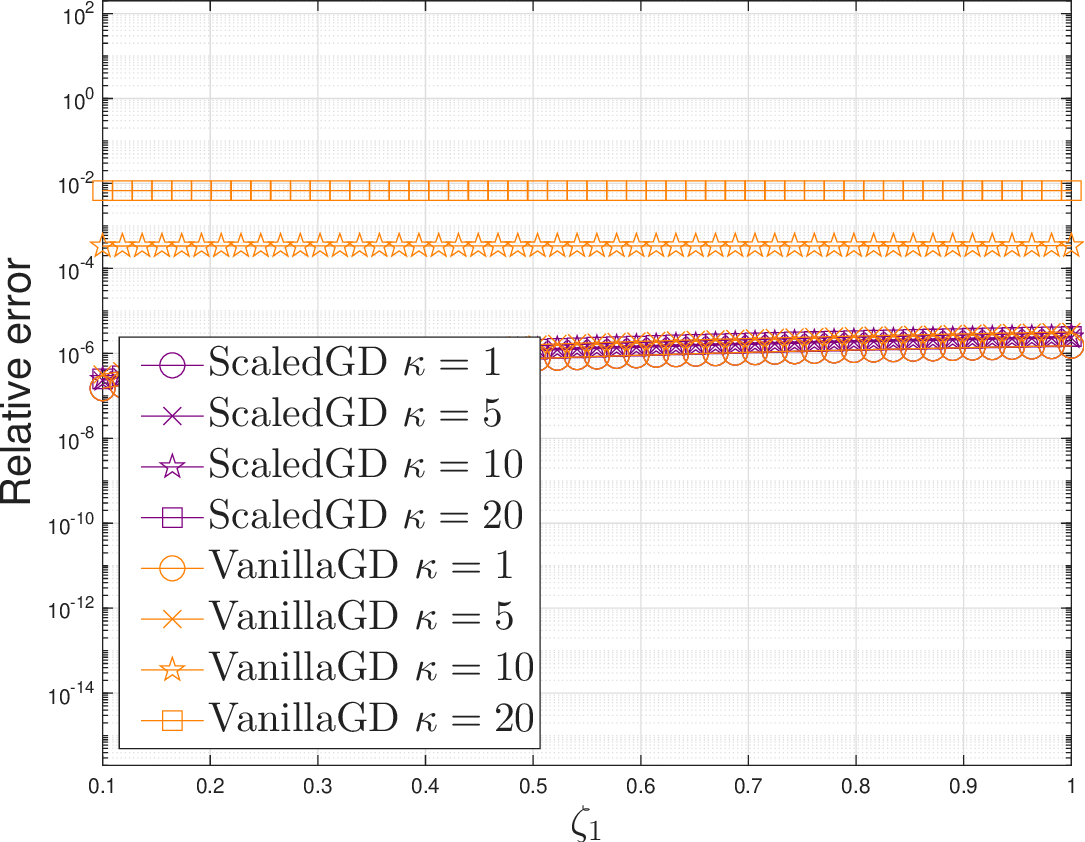} \label{fig:TRPCADCTzeta1}}
\quad
\subfloat[Tensor Completion, DFT]{\includegraphics[height=1.7in]{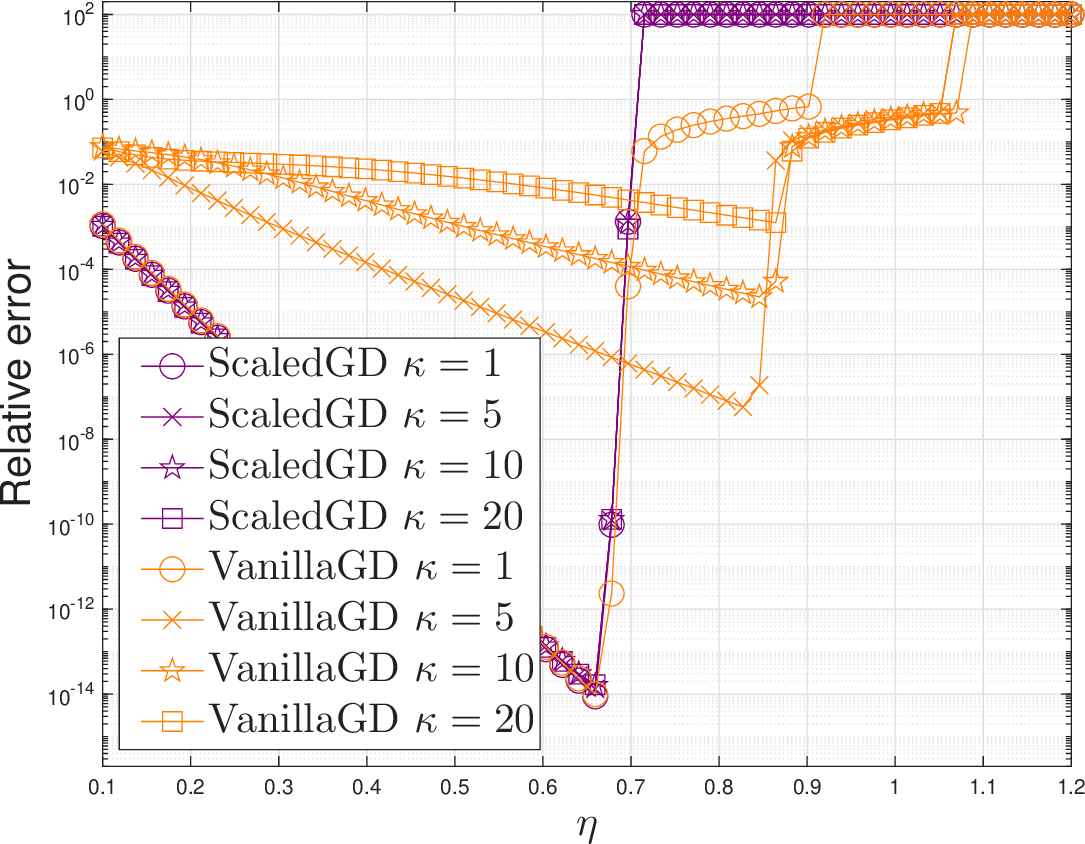} \label{fig:TCFFTstep}}
\quad
\subfloat[Tensor Completion, DCT]{\includegraphics[height=1.7in]{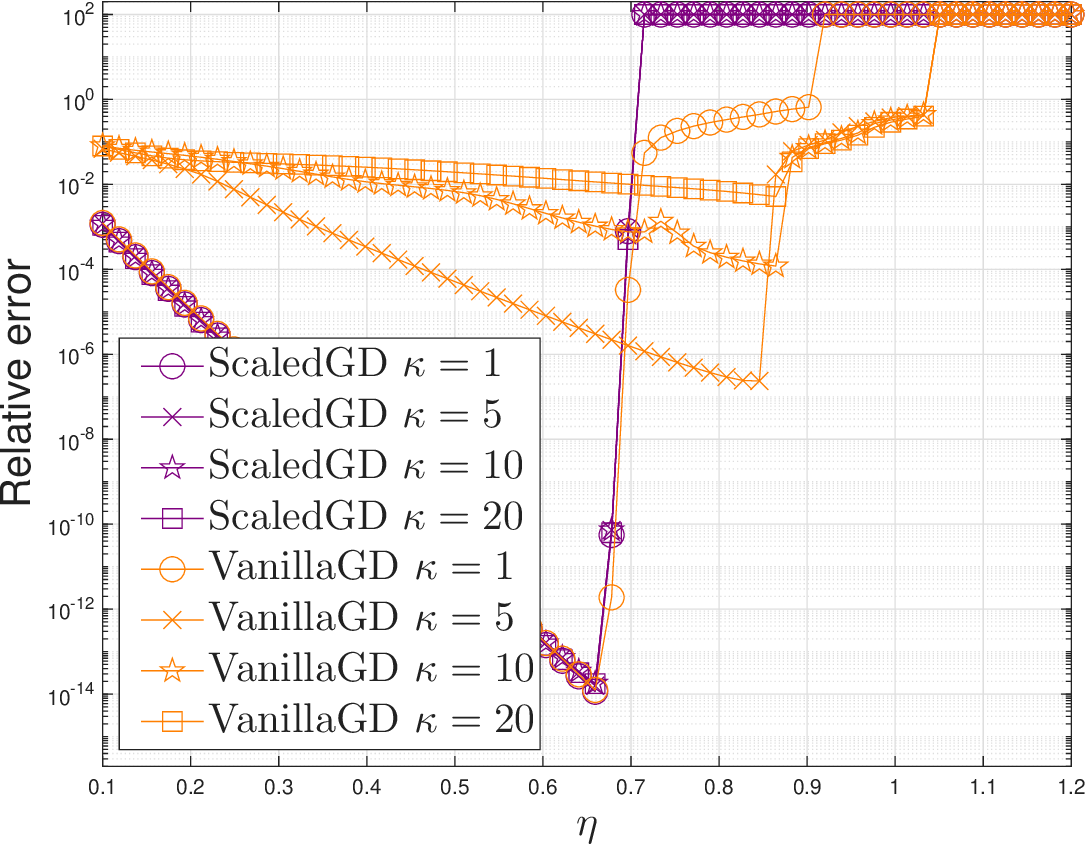} \label{fig:TCDCTstep}}
\caption{(a) and (b) show the relative errors of ScaledGD and vanilla GD after 300 iterations with respect to different $\zeta_1$'s from 0.1 to 1 under different condition numbers $\kappa = 1, 5, 10, 20$ for tensor RPCA with $n = 100$, $r = 10$, and $\alpha = 0.1$. (c) and (d) show the relative errors of ScaledGD and vanilla GD after 300 iterations with respect to different step sizes $\eta$ from 0.1 to 1.2 under different condition numbers $\kappa = 1, 5, 10, 20$ for tensor completion with $n = 100$, $r = 10$, and $p = 0.4$.}
\label{fig:hyperparam}
\end{figure}

We first illustrate the convergence performance under noise-free observations, i.e., $\bcW = \bzero$. Figure~\ref{fig:noiselessiteration} shows the relative reconstruction error $\| \bcX_t - \bcX_{\star} \|_F / \| \bcX_{\star} \|_F$ with respect to the iteration count $t$ for the four problems under different condition numbers $\kappa = 1, 5, 10, 20$ using the two transforms. For the tensor RPCA and robust tensor completion problems, ScaledGD converges at the same rate as vanilla GD under good conditioning (e.g., $\kappa = 1, 5$ for tensor RPCA and $\kappa = 1$ for robust tensor completion); under ill-conditioning, i.e., when $\kappa$ is large, ScaledGD converges linearly with a rate that is independent of $\kappa$, while vanilla GD does not converge because the relative error stays above $10^{-4}$. For the tensor completion and tensor regression problems, we can see that ScaledGD converges rapidly at a rate independent of the condition number, and matches the fastest rate of vanilla GD with perfect conditioning $\kappa = 1$. We plot the relative reconstruction errors of ScaledGD and vanilla GD for the four problems with respect to the algorithm running time (in seconds) under different condition numbers $\kappa = 1, 5, 10, 20$ in Figure~\ref{fig:noiselessruntime}. We again observe the advantage of ScaledGD over vanilla GD for the tensor RPCA and robust tensor completion problems under ill-conditioning because vanilla GD does not converge in this scenario. For the tensor completion problem, although vanilla GD runs slightly faster than ScaledGD when $\kappa = 1$, the convergence rate of vanilla GD deteriorates quickly with the increase of $\kappa$. As such, under ill-conditioning, the computational burdens can be substantially increased for vanilla GD compared to ScaledGD. Finally, Figure~\ref{fig:TRFFTtime} and Figure~\ref{fig:TRDCTtime} suggest that ScaledGD can be as fast as vanilla GD even under perfect conditioning for tensor regression, while vanilla GD turns out to be much slower than ScaledGD when $\kappa$ increases.

Next, we study the sensitivity of the convergence behavior of ScaledGD with respect to the choice of the hyperparameters. We first examine the effect of $\zeta_1$ on the convergence speeds of ScaledGD and vanilla GD for tensor RPCA by fixing $\zeta_0 = 0.5$, $\eta = 0.5$, and $\rho = 0.95$. We run both algorithms for at most 300 iterations (the algorithm is terminated if the relative error exceeds $10^2$). As indicated in Figure~\ref{fig:TRPCAFFTzeta1} and Figure~\ref{fig:TRPCADCTzeta1}, ScaledGD achieves exact recovery over a wide range of $\zeta_1$ values, i.e., the relative error is always below $10^{-5}$ and the recovery performance remains unaffected by the condition number. In contrast, vanilla GD cannot achieve exact recovery when $\kappa$ gets large. Figure~\ref{fig:TCFFTstep} and Figure~\ref{fig:TCDCTstep} illustrate the convergence speeds of ScaledGD and vanilla GD under different step sizes for tensor completion, where we again run both algorithms for at most 300 iterations. It can be seen that ScaledGD outperforms vanilla GD over a large range of step sizes, even when the step size of vanilla GD is optimized for its performance. Hence, our choice of $\eta = 0.5$ in the previous tensor completion experiments for the comparison between ScaledGD and vanilla GD is reasonable.

Finally, we examine the performance of ScaledGD for tensor RPCA and tensor completion under Gaussian noisy observations. We denote the signal-to-noise ratio as $\mathrm{SNR} \coloneq 10 \log_{10} \frac{\| \bcX{\star} \|_F^2}{n^2 n_3 \sigma_w^2}$ in dB. We plot the relative error $\| \bcX_t - \bcX_{\star} \|_F / \| \bcX_{\star} \|_F$ with respect to the iteration count $t$ in Figure~\ref{fig:noisy} under the condition number $\kappa = 10$ and various $\mathrm{SNR} = 40, 60, 80\mathrm{dB}$. For tensor RPCA, ScaledGD achieves smaller error compared to vanilla GD. For tensor completion, both methods achieve the same statistical error eventually, but ScaledGD converges much faster. In addition, the convergence speeds of ScaledGD are irrespective of the noise levels.

\begin{figure}[tbp]
\centering
\subfloat[Tensor RPCA, DFT]{\includegraphics[height=1.7in]{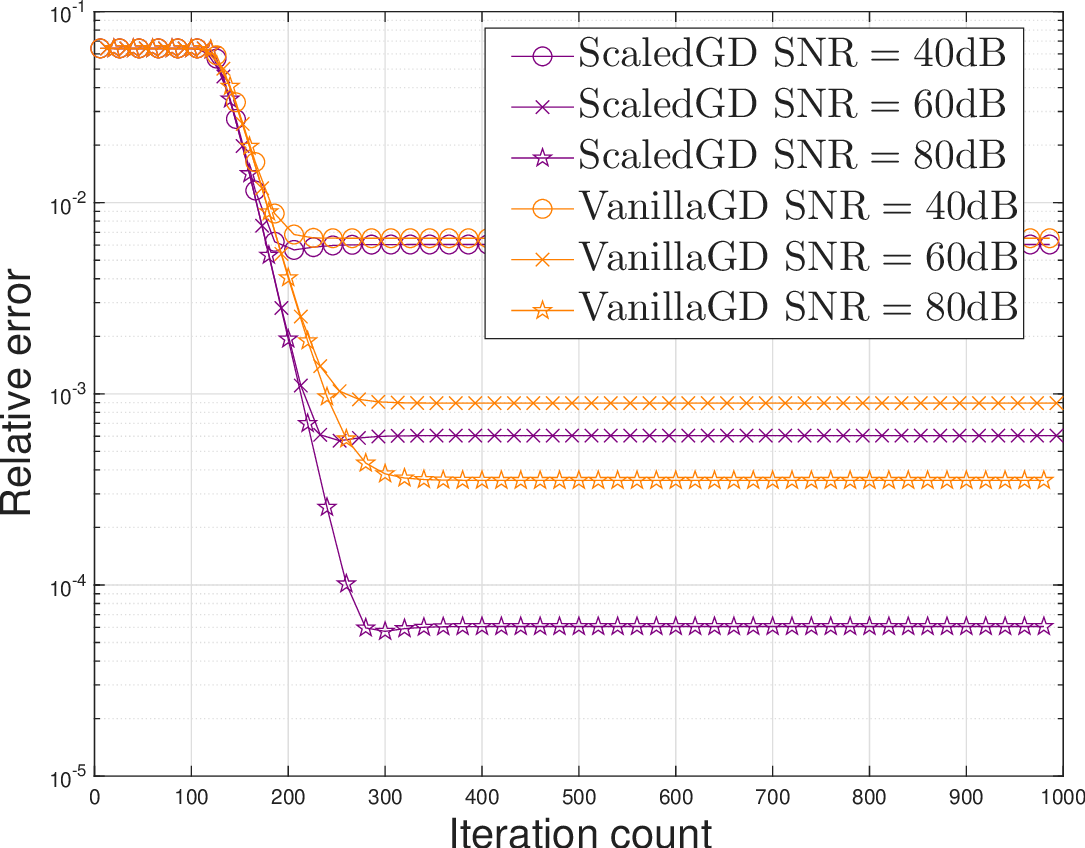} \label{fig:TRPCAFFTnoisy}}
\quad
\subfloat[Tensor RPCA, DCT]{\includegraphics[height=1.7in]{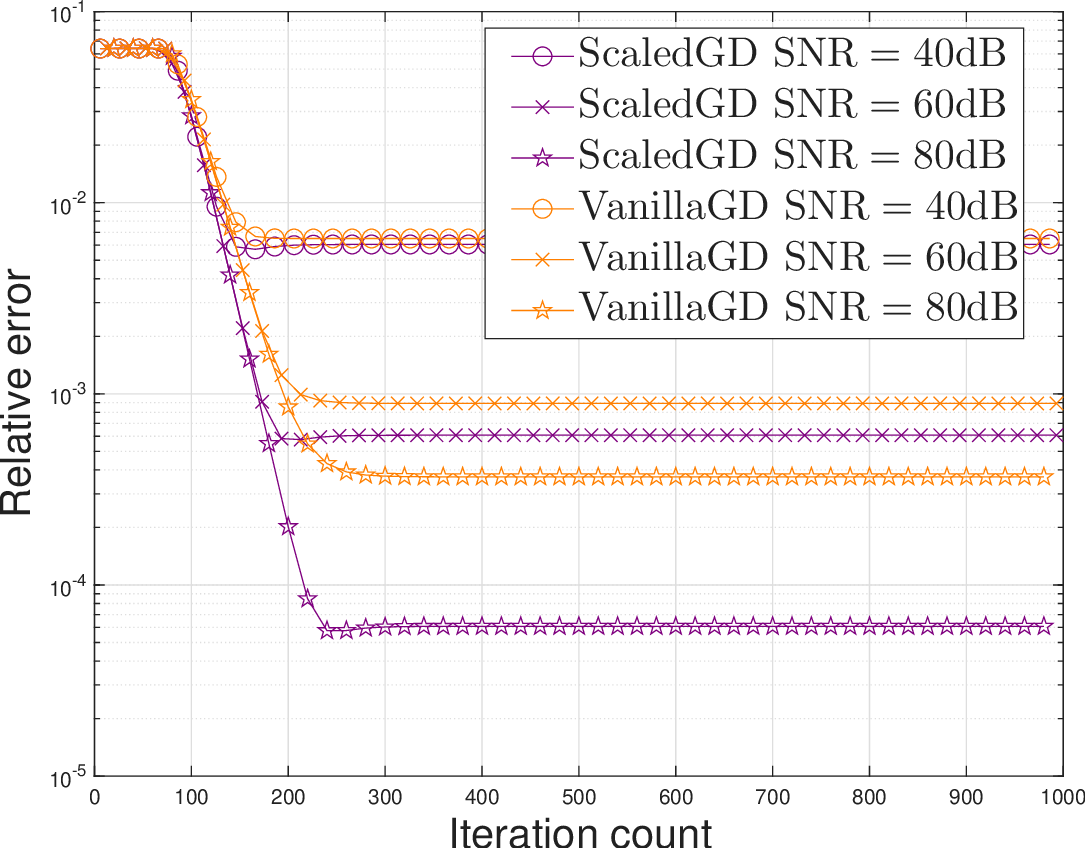} \label{fig:TRPCADCTnoisy}}
\quad
\subfloat[Tensor completion, DFT]{\includegraphics[height=1.7in]{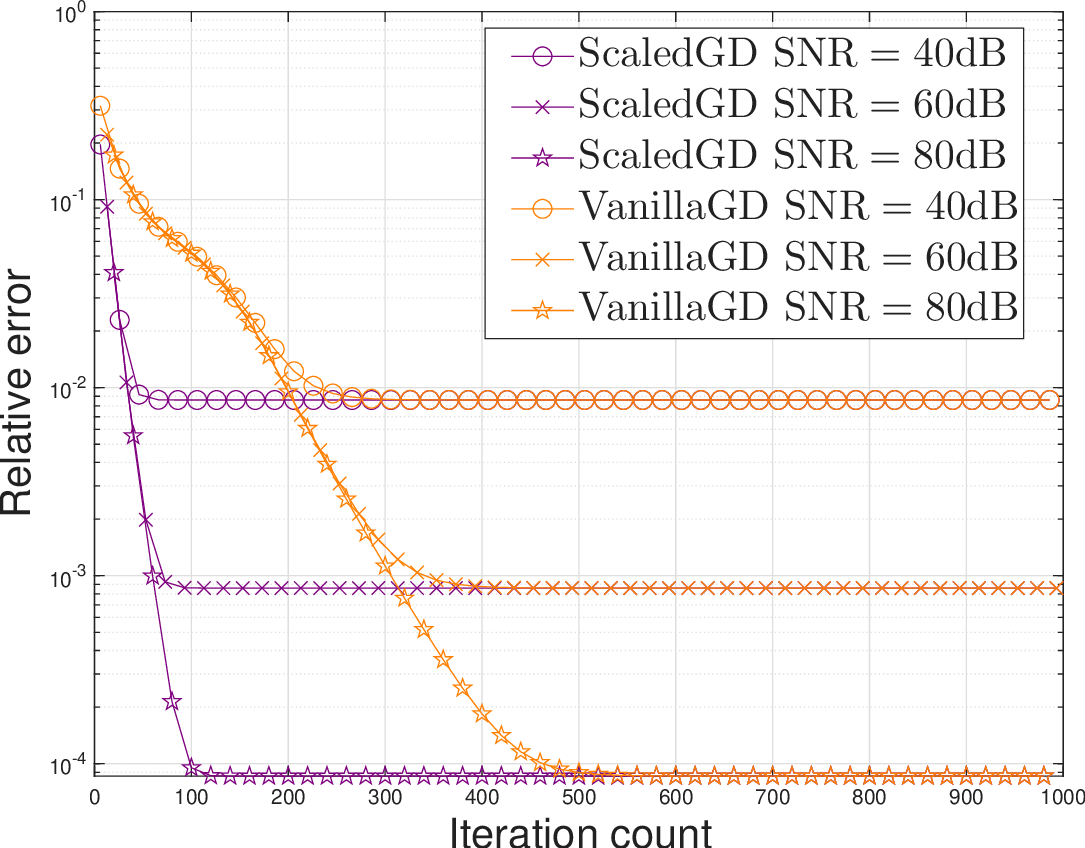} \label{fig:TCFFTnoisy}}
\quad
\subfloat[Tensor completion, DCT]{\includegraphics[height=1.7in]{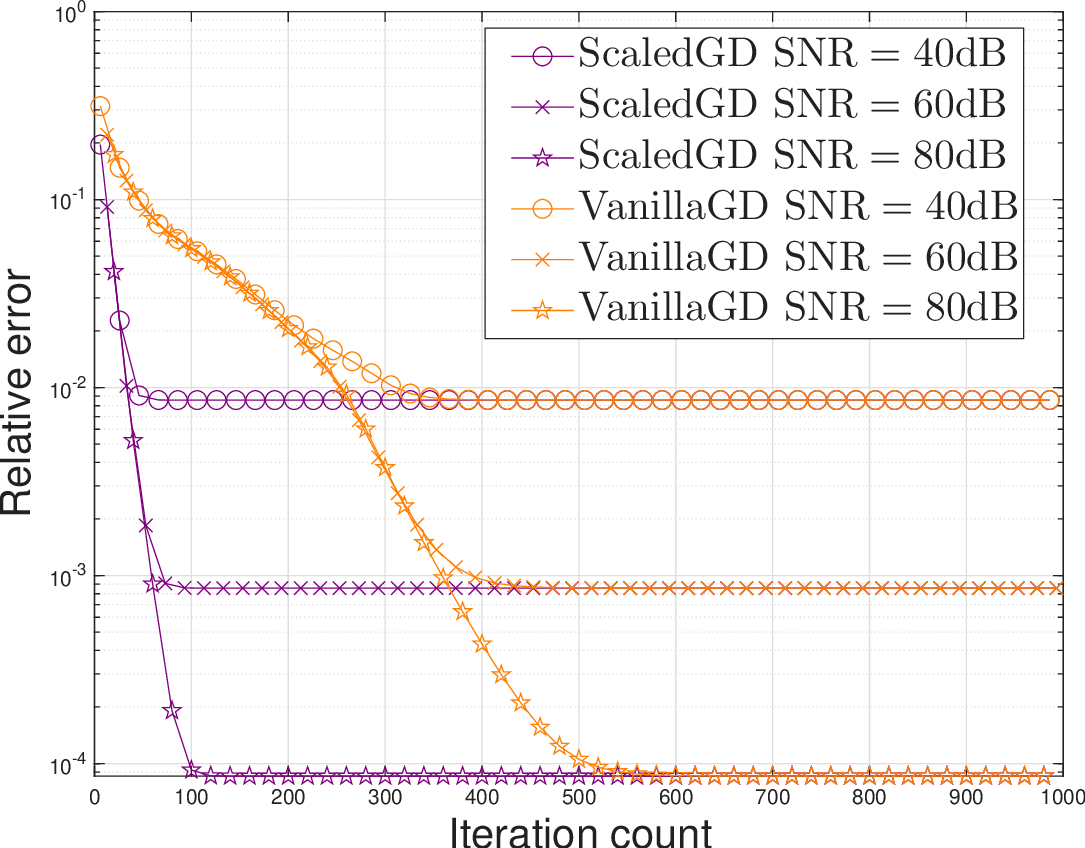} \label{fig:TCDCTnoisy}}
\caption{The relative errors of ScaledGD and vanilla GD with respect to the iteration count under the condition number $\kappa = 10$ and signal-to-noise ratios $\mathrm{SNR} = 40, 60, 80\mathrm{dB}$ for (a/b) tensor RPCA and (c/d) tensor completion.}
\label{fig:noisy}
\end{figure}

\subsection{Real-World Applications}

In this section, we apply ScaledGD for tensor RPCA and compare the performance of ScaledGD with other tensor RPCA algorithms, including TRPCA \citep{LuFCLLY.PAMI2020,Lu.ICCV2021}, ScaledGD using Tucker decomposition (ScaledGD-Tucker) \citep{DongTMC.IMA2023} and EAPT \citep{QiuWTMY.ICML2022} on real datasets. We again use two different linear transforms for ScaledGD, TRPCA and EAPT in all experiments, i.e., DFT and DCT. The corresponding methods of ScaledGD are called ScaledGD-DFT and ScaledGD-DCT for short, and this naming convention is the same for TRPCA and EAPT. We fix $\lambda = 1/\sqrt{\max(n_1, n_2) \ell}$ for TRPCA and tune the parameters of ScaledGD-Tucker and EAPT to achieve their best performance.

\begin{table}[t]
\centering
\caption{Comparison of video data denoising performance with different levels of corruptions. The best result is shown in bold.}
\begin{tabular}{c||ccc|ccc}
\hline
\multirow{2}{*}{Methods} & \multicolumn{3}{c|}{$\alpha = 0.05$} & \multicolumn{3}{c}{$\alpha = 0.1$}  \\
\cline{2-7}
 & Time(sec) & RSE & PSNR & Time(sec) & RSE & PSNR  \\
\hline
\multicolumn{7}{c}{Carphone}  \\
\hline
TRPCA-DFT & 185.88 & 0.0658 & 30.5242 & 185.47 & 0.0678 & 30.2640  \\
TRPCA-DCT & 215.57 & 0.0655 & 30.5618 & 216.82 & 0.0673 & 30.3350  \\
ScaledGD-Tucker & 53.77 & 0.1384 & 24.0724 & 53.57 & 0.1364 & 24.1950  \\
EAPT-DFT & 30.42 & 0.0648 & 30.6570 & 30.28 & 0.0654 & 30.5801  \\
EAPT-DCT & \textbf{25.62} & 0.0679 & 30.2524 & 25.70 & 0.0685 & 30.1802  \\
ScaledGD-DFT & 35.09 & \textbf{0.0603} & \textbf{31.2809} & 33.08 & \textbf{0.0624} & \textbf{30.9828}  \\
ScaledGD-DCT & 26.80 & 0.0632 & 30.8775 & \textbf{25.53} & 0.0655 & 30.5617  \\
\hline
\multicolumn{7}{c}{Coastguard}  \\
\hline
TRPCA-DFT & 165.00 & \textbf{0.0511} & \textbf{31.7435} & 164.15 & \textbf{0.0537} & \textbf{31.3203}  \\
TRPCA-DCT & 196.24 & 0.0570 & 30.8040 & 199.77 & 0.0595 & 30.4273  \\
ScaledGD-Tucker & 39.37 & 0.1442 & 22.7378 & 39.29 & 0.1432 & 22.7957  \\
EAPT-DFT & 24.30 & 0.0611 & 30.2289 & 24.24 & 0.0621 & 30.0861  \\
EAPT-DCT & 23.64 & 0.0734 & 28.6131 & 23.54 & 0.0748 & 28.4520  \\
ScaledGD-DFT & 23.23 & 0.0578 & 30.6929 & 22.96 & 0.0597 & 30.4192  \\
ScaledGD-DCT & \textbf{20.81} & 0.0696 & 29.0768 & \textbf{20.64} & 0.0723 & 28.7491  \\
\hline
\end{tabular}
\label{tab:denoise}
\end{table}

\begin{figure}[tbp]
\centering
\includegraphics[width=6in]{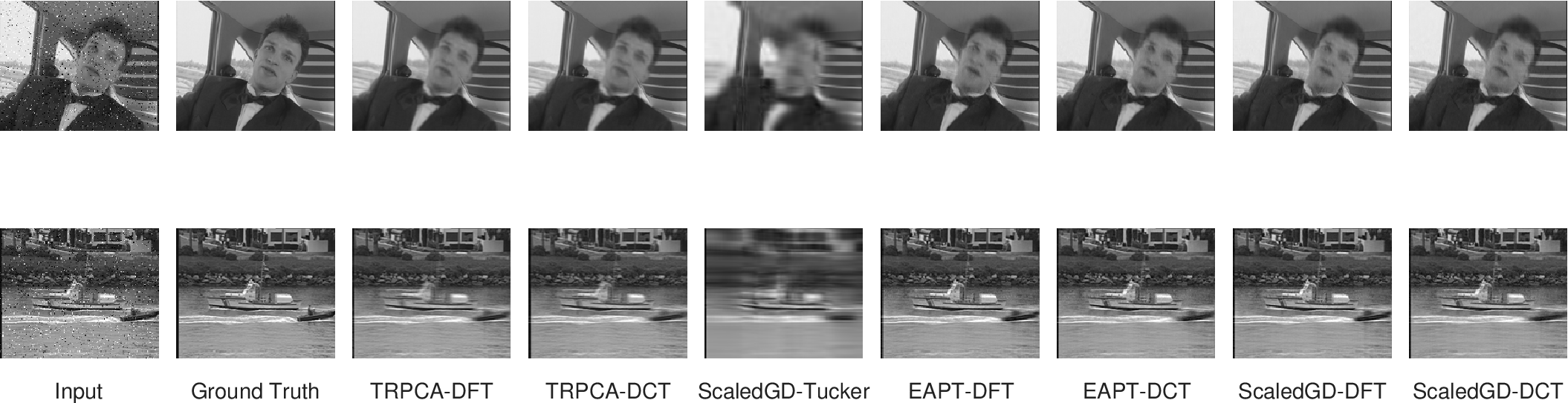}
\caption{Video denoising examples of ``Carphone'' (top) and ``Coastguard'' (bottom) for $\alpha = 0.1$.}
\label{fig:denoiseviz}
\end{figure}

\begin{figure}[htbp]
\centering
\subfloat[DFT]{\includegraphics[height=1.7in]{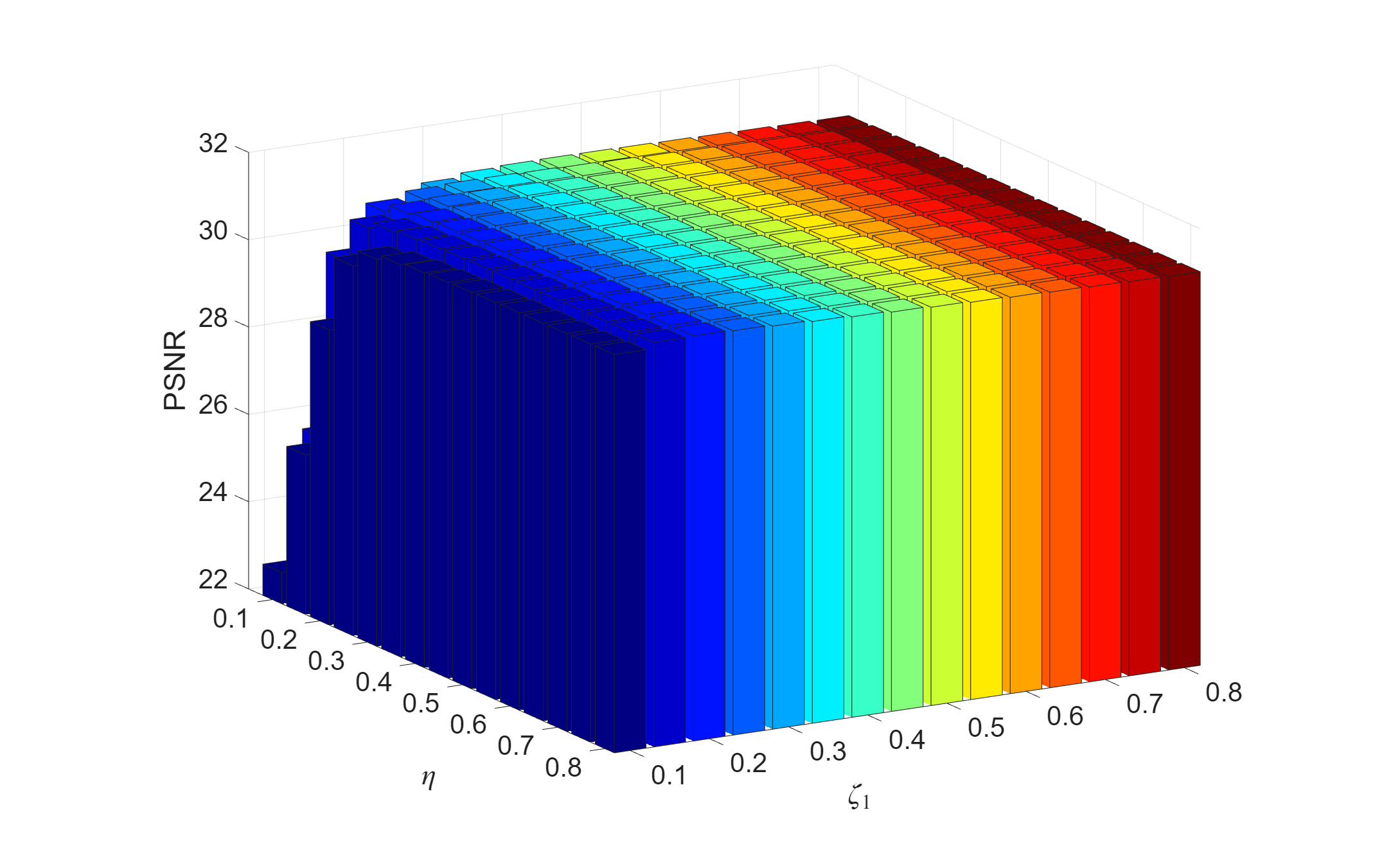} \label{fig:paramsensitivityFFT}}
\quad
\subfloat[DCT]{\includegraphics[height=1.7in]{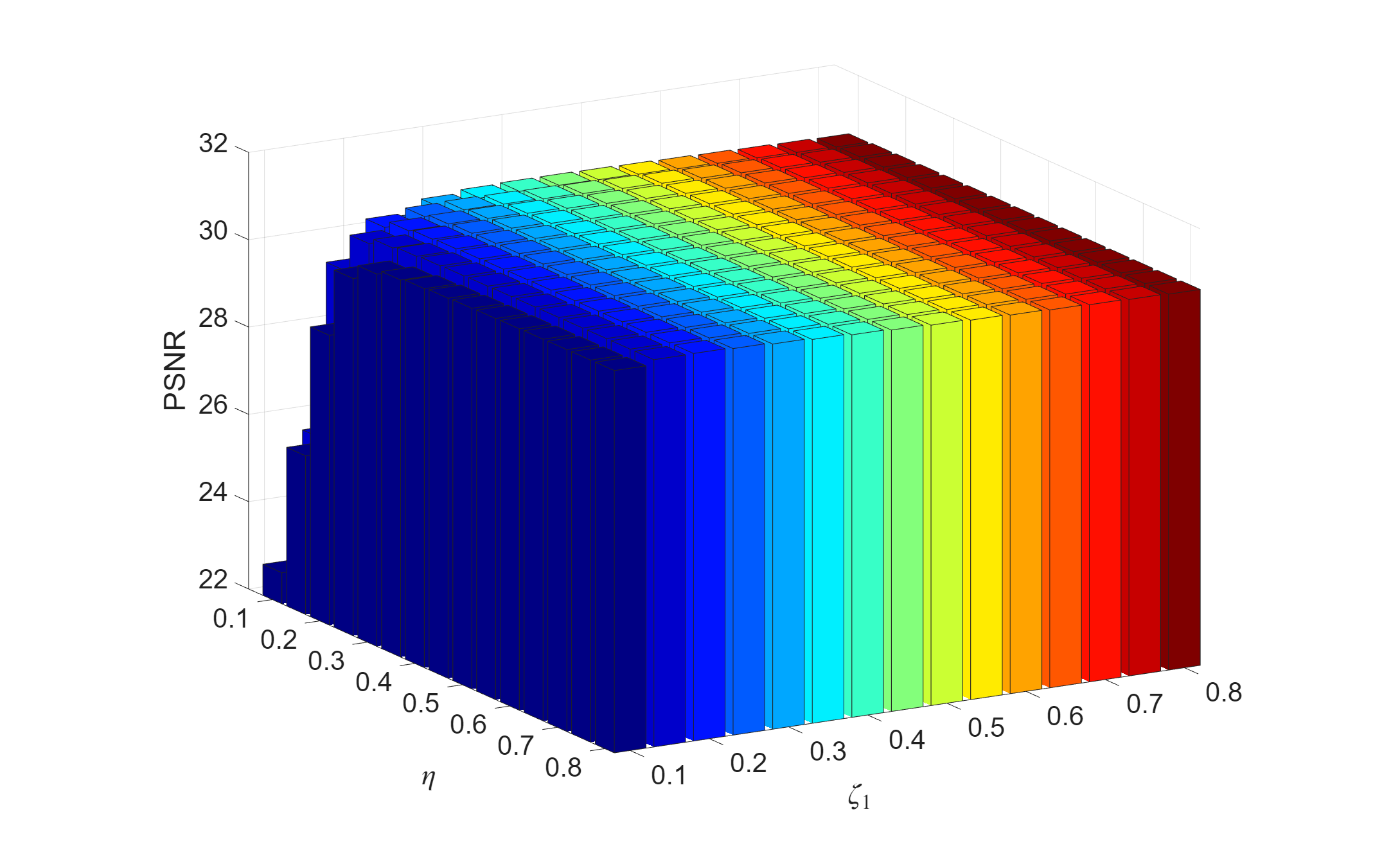} \label{fig:paramsensitivityDCT}}
\caption{Parameter sensitivity analysis with respect to $\zeta_1$ and $\eta$ on the ``Carphone'' sequence.}
\end{figure}

\paragraph{Video denoising.} Here, we compare the performance of different methods on video denoising. Specifically, we select two video sequences ``Carphone'' and ``Coastguard'' from YUV video sequences\footnote{\url{http://trace.eas.asu.edu/yuv/index.html}} for comparison. We randomly select $\alpha$-fraction of pixels in each frame and add white Gaussian noise with variance 0.1 on these pixels. We adopt two metrics, namely, the Peak Signal-to-Noise Ratio (PSNR) and the relative standard error (RSE) for evaluation. For ScaledGD, the parameters are set to be $\zeta_0 = 0.5$, $\zeta_1 = 0.5$, $\eta = 0.5$, $\rho = 0.9$ and $r = 20$. The numerical results in Table~\ref{tab:denoise} demonstrate that t-SVD based methods achieve much better performance than ScaledGD-Tucker in terms of higher PSNR values and lower RSE values, which suggests that t-SVD provides a more realistic modelling scenario compared to Tucker decomposition for such data. ScaledGD-DFT achieves superior recovery accuracy compared to the state-of-the-art methods for ``Carphone'' sequence. While the PSNR value for ScaledGD-DFT is slightly lower than TRPCA-DFT for ``Coastguard'' sequence, the running time of TRPCA-DFT is 7 times of that of ScaledGD-DFT because ScaledGD eliminates the need for full t-SVD computations within TRPCA. Figure~\ref{fig:denoiseviz} provides two visual examples depicting the recovery results.

We also conduct experiments on ``Carphone'' sequence to show the effect of the parameters $\zeta_1$ and $\eta$ on the PSNR value by fixing $\zeta_0 = 0.5$ and $\rho = 0.9$. The PSNR results with different combinations of $\zeta_1$ and $\eta$ for DFT and DCT are given in Figure~\ref{fig:paramsensitivityFFT} and Figure~\ref{fig:paramsensitivityDCT}, respectively. It is evident that competitive performance can be obtained over a wide range of parameters, e.g., the PSNR value can be above 31 when $\zeta_1, \eta \geq 0.2$ for DFT.

\begin{table}[t]
\small
\centering
\caption{Background subtraction results on the BMC dataset. The best result is shown in bold and the second best on F-measure and running time (in seconds) is underlined.}
\setlength\tabcolsep{0.15em}{\begin{tabular}{c|c||c|c|c|c|c|c|c|c|c}
\hline
Methods & Metrics & Video 1 & Video 2 & Video 3 & Video 4 & Video 5 & Video 6 & Video 7 & Video 8 & Video 9  \\
\hline
\multirow{4}{*}{\shortstack{TRPCA \\ -DFT}} & Precision & 0.5618 & \textbf{0.7175} & 0.4657 & \textbf{0.3893} & \textbf{0.6078} & \textbf{0.6378} & 0.6238 & \textbf{0.5967} & 0.3094  \\
 & Recall & 0.6367 & 0.5500 & 0.9147 & 0.7488 & 0.7036 & 0.6368 & 0.6088 & 0.6776 & 0.8428  \\
 & F-measure & 0.5812 & 0.5565 & \textbf{0.6119} & \textbf{0.4741} & 0.5733 & 0.5861 & \underline{0.5650} & 0.6211 & 0.4212  \\
 & Time & 102.86 & 124.17 & 14.83 & 80.36 & 105.59 & 127.40 & 145.10 & 70.87 & 76.64  \\
\hline
\multirow{4}{*}{\shortstack{TRPCA \\ -DCT}} & Precision & 0.5559 & 0.7154 & 0.4646 & 0.3877 & 0.6072 & 0.6374 & \textbf{0.6245} & 0.5912 & 0.3160  \\
 & Recall & 0.6340 & 0.5512 & 0.9143 & 0.7510 & 0.7082 & 0.6401 & 0.6107 & 0.6765 & 0.8434  \\
 & F-measure & 0.5766 & 0.5567 & \underline{0.6107} & \underline{0.4733} & 0.5750 & 0.5881 & \textbf{0.5671} & 0.6174 & 0.4306  \\
 & Time & 137.14 & 165.57 & 14.32 & 106.39 & 142.36 & 168.28 & 189.68 & 93.19 & 101.81  \\
\hline
\multirow{4}{*}{\shortstack{ScaledGD \\ -Tucker}} & Precision & 0.2466 & 0.4249 & 0.0824 & 0.0297 & 0.1972 & 0.2241 & 0.2247 & 0.1751 & 0.0156  \\
 & Recall & \textbf{0.8107} & \textbf{0.8304} & \textbf{0.9503} & \textbf{0.8908} & \textbf{0.9305} & \textbf{0.8277} & \textbf{0.8109} & \textbf{0.8167} & \textbf{0.8887}  \\
 & F-measure & 0.3546 & 0.5356 & 0.1504 & 0.0550 & 0.2952 & 0.3244 & 0.3184 & 0.2755 & 0.0301  \\
 & Time & 18.96 & 26.76 & 5.55 & 15.64 & 21.17 & 23.20 & 27.02 & 14.24 & 17.63  \\
\hline
\multirow{4}{*}{\shortstack{EAPT \\ -DFT}} & Precision & 0.5443 & 0.6271 & 0.0292 & 0.0887 & 0.4789 & 0.5254 & 0.5275 & 0.5138 & 0.1657  \\
 & Recall & 0.7430 & 0.6423 & 0.3698 & 0.7552 & 0.8676 & 0.7200 & 0.6797 & 0.7846 & 0.7693  \\
 & F-measure & 0.6220 & 0.5795 & 0.0537 & 0.1338 & 0.5491 & 0.5601 & 0.5436 & 0.6042 & 0.2518  \\
 & Time & 12.38 & 15.55 & 3.87 & 10.25 & 13.24 & 14.64 & 16.34 & 9.62 & 11.03  \\
\hline
\multirow{4}{*}{\shortstack{EAPT \\ -DCT}} & Precision & 0.5603 & 0.6342 & 0.0233 & 0.0705 & 0.4918 & 0.5311 & 0.4574 & 0.4545 & 0.1096  \\
 & Recall & 0.7464 & 0.6513 & 0.3992 & 0.7647 & 0.8772 & 0.7252 & 0.7210 & 0.7918 & 0.7594  \\
 & F-measure & 0.6338 & \underline{0.5912} & 0.0437 & 0.1127 & 0.5605 & 0.5701 & 0.5154 & 0.5513 & 0.1786  \\
 & Time & \textbf{7.37} & \textbf{9.15} & \textbf{2.30} & \textbf{5.79} & \textbf{7.78} & \textbf{8.61} & \textbf{9.22} & \textbf{5.70} & \textbf{6.51}  \\
\hline
\multirow{4}{*}{\shortstack{ScaledGD \\ -DFT}} & Precision & 0.6181 & 0.6849 & \textbf{0.6467} & 0.3705 & 0.5382 & 0.5831 & 0.5775 & 0.5782 & 0.6085  \\
 & Recall & 0.7031 & 0.6046 & 0.3387 & 0.7209 & 0.8263 & 0.6837 & 0.6551 & 0.7580 & 0.7225  \\
 & F-measure & \underline{0.6483} & 0.5879 & 0.4087 & 0.4220 & \underline{0.5841} & \underline{0.5891} & 0.5636 & \textbf{0.6451} & \underline{0.6363}  \\
 & Time & 10.96 & 14.54 & 2.97 & 10.41 & 12.87 & 13.74 & 16.24 & 9.94 & 10.39  \\
\hline
\multirow{4}{*}{\shortstack{ScaledGD \\ -DCT}} & Precision & \textbf{0.6322} & 0.6916 & 0.6432 & 0.3733 & 0.5674 & 0.5844 & 0.5742 & 0.5722 & \textbf{0.6095}  \\
 & Recall & 0.7101 & 0.6057 & 0.3639 & 0.7193 & 0.8535 & 0.6881 & 0.6546 & 0.7602 & 0.7357  \\
 & F-measure & \textbf{0.6605} & \textbf{0.5924} & 0.4338 & 0.4188 & \textbf{0.6172} & \textbf{0.5935} & 0.5620 & \underline{0.6410} & \textbf{0.6458}  \\
 & Time & \underline{8.80} & \underline{12.18} & \underline{2.38} & \underline{8.32} & \underline{10.39} & \underline{11.67} & \underline{14.08} & \underline{7.54} & \underline{8.01}  \\
\hline
\end{tabular}}
\label{tab:bgs}
\end{table}

\begin{figure}[ht]
\centering
\includegraphics[width=6in]{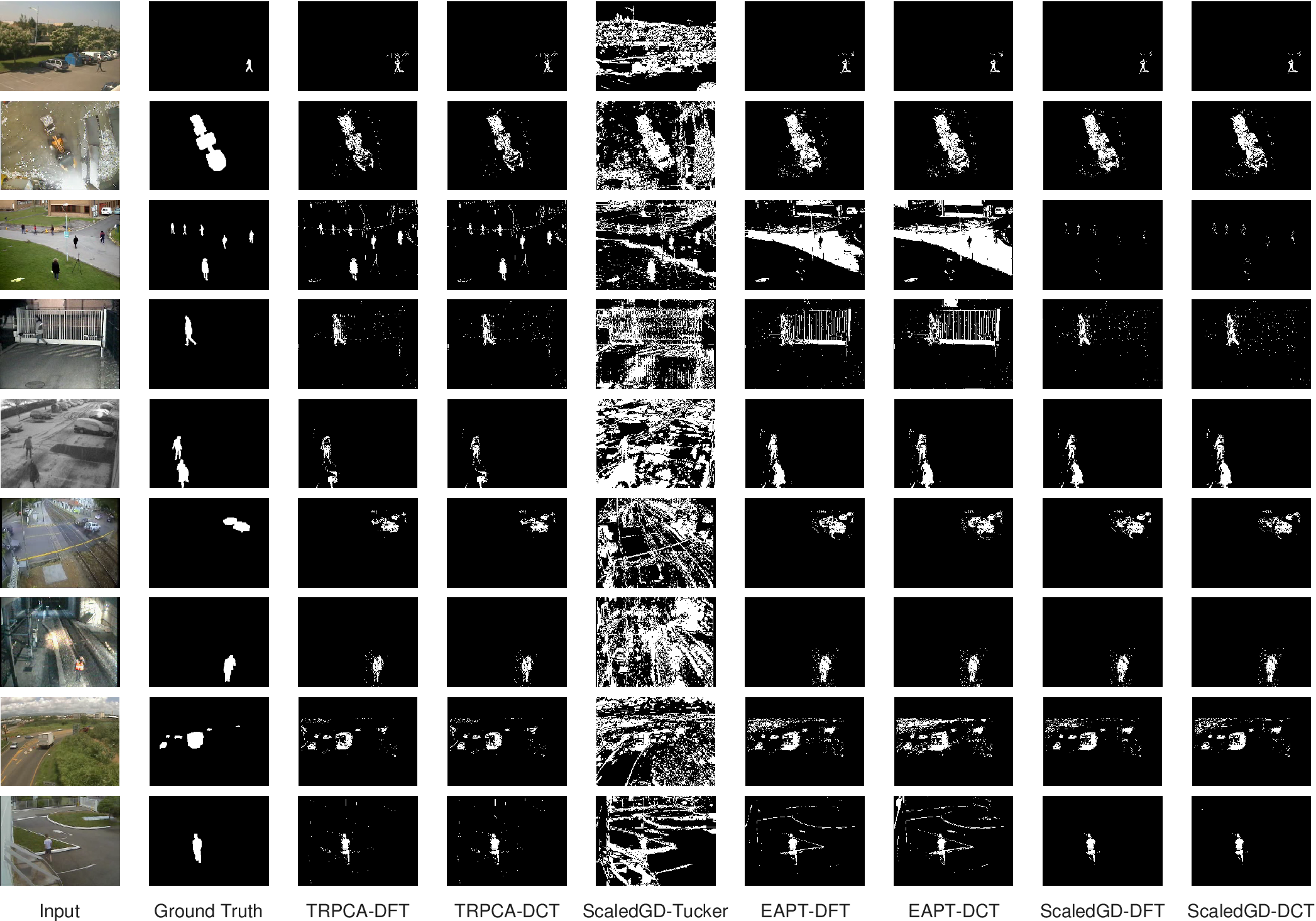}
\caption{Video background subtraction visual results using the BMC dataset \citep{VacavantCWL.ACCVW2012}. From top to bottom are 9 sequences within the dataset.}
\label{fig:bgsviz}
\end{figure}

\paragraph{Video background subtraction.} We show the effectiveness of the proposed ScaledGD on the Background Models Challenge (BMC) dataset\footnote{\url{http://backgroundmodelschallenge.eu/}} \citep{VacavantCWL.ACCVW2012} for video background subtraction. Here, the low-rank tensor corresponds to the relatively static background across frames, while the foreground consisting of moving objects, which are usually sparsely distributed in the video frames and can be viewed as outliers, correspond to the sparse tensor. Thus, it is reasonable to apply tensor RPCA to separate the background and foreground. The dataset contains 9 real video sequences and we use all these sequences for both qualitative and quantitative analysis. To evaluate the performance of ScaledGD, the Precision, Recall, and F-measure, are used as basic evaluation metrics. For both ScaledGD-DFT and ScaledGD-DCT, we fix $\zeta_0 = 0.15$, $\zeta_1 = 0.15$, $\eta = 0.85$, $\rho = 0.95$, $r = 5$. As can be seen from Table~\ref{tab:bgs}, EAPT-DCT takes the least running time and ScaledGD-DCT is only a little slower than EAPT-DCT. While ScaledGD-Tucker tends to give us a high recall value, ScaledGD-DFT and ScaledGD-DCT achieve the highest F-measure scores for 6 videos. For Video 7, the performance of ScaledGD-DFT and ScaledGD-DCT is also comparable to that of TRPCA-DFT and TRPCA-DCT, but the running time of TRPCA is at least 8.5 times of that of ScaledGD. The visual results depicted in Figure~\ref{fig:bgsviz} verify that our method is competent to extract the foreground from these videos.

We conclude by noting that the choice of transformations in these experiments is agnostic to the data. Nonetheless, selection of the optimal transformation and even learning the optimal transform is an open problem for future research.

\section{Conclusions}
\label{sec:conclude}

This paper developed a scaled gradient descent (ScaledGD) algorithm for factored low-rank tensor estimation based on the t-SVD under invertible linear transforms. We rigorously establish that under standard assumptions, ScaledGD only takes $\cO (\log(1/\epsilon))$ iterations to reach $\epsilon$-accuracy, without the dependency on the condition number of the ground truth tensor when initialized via the spectral method. There are several future directions that are worth exploring, which we briefly discuss below.
\begin{itemize}
  \item \emph{Parameter estimation for tensor RPCA.} The proposed algorithm for tensor RPCA involves an iteration-varying threshold operation following a geometric decaying schedule, which contains several hyperparameters that need to be tuned carefully for real data. Our future work includes learning the optimal hyperparameters using deep unfolding and self-supervised learning \citep{CaiLY.NeurIPS2021,DongSDC.ICASSP2023}.
  \item \emph{High-order extension of ScaledGD.} Based on the multilinear algebra for high-order t-SVD, it is of great interest to extend our method for high-order tensors and to unify the understanding of theoretical guarantee for low-rank tensor estimation problems when the invertible linear transforms $L$ satisfy certain conditions.
  \item \emph{Entrywise error control for tensor completion.} In this work, we aimed at minimizing the Frobenius norm of the reconstructed tensor in tensor completion. There exists another work that deals with minimizing the $\ell_{\infty}$ error for matrix completion with statistical guarantees \citep{MaWCC.FoCM2020}. It is therefore interesting to develop similar strong entrywise error guarantees of ScaledGD for tensor completion with t-SVD.
\end{itemize}

\small
\bibliography{TW_bib}

%%%%%%%%%%%%%%%%%%%%%%%%%%%%%%%%%%%%%%%%%%%%%%%%%%%%%%%%%%%%

\newpage
\appendix

\section{Technical Lemmas}
\label{sec:techlemmas}

In this section, we introduce main preliminaries and useful lemmas which will be used in the proofs. We use bold calligraphic letters with arrows on top to denote tensor columns of size $n_1 \times 1 \times n_3$, e.g., $\overrightarrow{\bcA}$. We define the $\ell_{\infty,2}$-norm and $\ell_{2,2,\infty}$-norm of a tensor $\bcA$ as
\begin{align*}
\| \bcA \|_{\infty,2} = \max \{ \max_i \| \bcA(i,:,:) \|_F, \max_j \| \bcA(:,j,:) \|_F \},
\end{align*}
and $\| \bcA \|_{2,2,\infty} = \max_{i,k} \| \bcA(i,:,k) \|_F$, respectively.

\begin{definition}[Standard tensor basis \citep{Lu.ICCV2021}]\label{def:stdtensorbasis}
The tensor \textbf{column basis} with respect to the transform $L$, denoted as $\mathring{\boldsymbol{\ce}}_i$, is a tensor of size $n_1 \times 1 \times n_3$ with the entries of the $(i,1)$-th mode-3 tube of $L(\mathring{\boldsymbol{\ce}}_i)$ equaling 1 and the rest equaling 0. Similarly, the \textbf{row basis} $\mathring{\boldsymbol{\ce}}_j^H$ is of size $1 \times n_2 \times n_3$ with the entries of the $(1,j)$-th mode-3 tube of $L(\mathring{\boldsymbol{\ce}}_j^H)$ equaling to 1 and the rest equaling to 0. The \textbf{tube basis} $\dot{\boldsymbol{\ce}}_k$ is a tensor of size $1 \times 1 \times n_3$ with the $(1,1,k)$-th entry of $L(\dot{\boldsymbol{\ce}}_k)$ equaling 1 and the rest equaling 0.
\end{definition}

Denote $\bar{\boldsymbol{\ce}}_{ijk}$ as a unit tensor with only the $(i,j,k)$-th entry equaling 1 and others equaling 0. Based on Definition~\ref{def:stdtensorbasis}, $\bar{\boldsymbol{\ce}}_{ijk}$ can be expressed as
\begin{align}\label{eqn:defeijk}
\bar{\boldsymbol{\ce}}_{ijk} = L(\mathring{\boldsymbol{\ce}}_i \ast_{\bPhi} \dot{\boldsymbol{\ce}}_k \ast_{\bPhi} \mathring{\boldsymbol{\ce}}_j^H).
\end{align}
Then for any tensor $\bcA \in \R^{n_1 \times n_2 \times n_3}$, we have $\bcA_{i,j,k} = \langle \bcA, \bar{\boldsymbol{\ce}}_{ijk} \rangle$ and
\begin{align*}
\bcA = \sum_{i,j,k} \langle \bcA, \bar{\boldsymbol{\ce}}_{ijk} \rangle \bar{\boldsymbol{\ce}}_{ijk}.
\end{align*}

\begin{definition}
For any $\bcA \in \R^{n_1 \times n_2 \times n_3}$, the projection onto $\bOmega$ is defined as
\begin{align*}
\bcP_{\bOmega}(\bcA) = \sum_{ijk} \delta_{ijk} \bcA_{i,j,k} \bar{\boldsymbol{\ce}}_{ijk},
\end{align*}
where $\delta_{ijk} = 1_{(i,j,k) \in \bOmega}$ and $1_{(\cdot)}$ denotes the indicator function.
\end{definition}

\begin{definition}
Let $\bcM \in \R^{n_1 \times n_2 \times n_3}$ with $\mathrm{rank}_t(\bcM) = r$ and its skinny t-SVD be $\bcM = \bcU \ast_{\bPhi} \bcG \ast_{\bPhi} \bcV^H$. We denote $\bT$ by the set
\begin{align}\label{eqn:defT}
\bT = \{ \bcU \ast_{\bPhi} \bcZ^H + \bcW \ast_{\bPhi} \bcV^H | \bcZ \in \R^{n_2 \times r \times n_3}, \bcW \in \R^{n_1 \times r \times n_3} \},
\end{align}
and by $\bT^{\perp}$ its orthogonal complement.
\end{definition}
The projections onto $\bT$ and its complementary set $\bT^{\perp}$ are respectively denoted as
\begin{align}\label{eqn:defprojT}
\bcP_{\bT}(\bcA) = \bcU \ast_{\bPhi} \bcU^H \ast_{\bPhi} \bcA + \bcA \ast_{\bPhi} \bcV \ast_{\bPhi} \bcV^H - \bcU \ast_{\bPhi} \bcU^H \ast_{\bPhi} \bcA \ast_{\bPhi} \bcV \ast_{\bPhi} \bcV^H,
\end{align}
and
\begin{align*}
\bcP_{\bT^{\perp}}(\bcA) = \bcA - \bcP_{\bT}(\bcA) = (\bcI_{n_1} - \bcU \ast_{\bPhi} \bcU^H) \ast_{\bPhi} \bcA \ast_{\bPhi} (\bcI_{n_2} - \bcV \ast_{\bPhi} \bcV^H).
\end{align*}

In the following, we use $\bcQ_t$ to denote the optimal alignment tensor between $(\bcL_t;\bcR_t)$ and $(\bcL_{\star};\bcR_{\star})$. For notational convenience, we denote $\bcL_{\sharp} \coloneq \bcL_t \ast_{\bPhi} \bcQ_t$, $\bcR_{\sharp} \coloneq \bcR_t \ast_{\bPhi} \bcQ_t^{-H}$, $\bcL_{\triangle} \coloneq \bcL_{\sharp} - \bcL_{\star}$, $\bcR_{\triangle} \coloneq \bcR_{\sharp} - \bcR_{\star}$, $\bcS_{\triangle} \coloneq \bcS_{t+1} - \bcS_{\star}$.

\subsection{Tensor Algebra}

\begin{lemma}
Let $\bcX \in \R^{n_1 \times n_2 \times n_3}$, then
\begin{align*}
\xoverline{\bcX}(:,:,k) = \begin{bmatrix}
\bX_1 & \bX_2 & \cdots & \bX_{n_2}
\end{bmatrix}
\begin{bmatrix}
\bPhi_{k,\cdot}^T & \bzero & \bzero & \bzero  \\
\bzero & \bPhi_{k,\cdot}^T & \bzero & \bzero  \\
\vdots & \vdots & \ddots & \vdots  \\
\bzero & \bzero & \bzero & \bPhi_{k,\cdot}^T
\end{bmatrix}
\coloneq \bX_{(1)} \widetilde{\bPhi}_k,
\end{align*}
where each $\bX_j$ is obtained by ``squeezing'' each sample $\bcX_{(j)} \in \R^{n_1 \times 1 \times n_3}$ into a matrix, i.e., $\bX_j = \mathtt{squeeze} (\bcX_{(j)}) \in \R^{n_1 \times n_3}$.
\end{lemma}

\begin{proof}
The ($i,j,k$)-th entry of $\xoverline{\bcX}$ can be written as
\begin{align*}
\xoverline{\bcX}_{i,j,k} = \sum_{k'=1}^{n_3} \bPhi_{k,k'} \bcX_{i,j,k'} = \bcX(i,j,:) \bPhi_{k,\cdot}^T.
\end{align*}
Then we have
\begin{align*}
& \xoverline{\bcX}(:,:,k) \\
= & \begin{bmatrix}
\mathtt{squeeze}(\bcX_{(1)}) & \mathtt{squeeze}(\bcX_{(2)}) & \cdots & \mathtt{squeeze}(\bcX_{(n_2)})
\end{bmatrix}
\begin{bmatrix}
\bPhi_{k,\cdot}^T & \bzero & \bzero & \bzero  \\
\bzero & \bPhi_{k,\cdot}^T & \bzero & \bzero  \\
\vdots & \vdots & \ddots & \vdots  \\
\bzero & \bzero & \bzero & \bPhi_{k,\cdot}^T
\end{bmatrix} \\
= & \begin{bmatrix}
\bX_1 & \bX_2 & \cdots & \bX_{n_2}
\end{bmatrix}
\begin{bmatrix}
\bPhi_{k,\cdot}^T & \bzero & \bzero & \bzero  \\
\bzero & \bPhi_{k,\cdot}^T & \bzero & \bzero  \\
\vdots & \vdots & \ddots & \vdots  \\
\bzero & \bzero & \bzero & \bPhi_{k,\cdot}^T
\end{bmatrix} \coloneq \bX_{(1)} \widetilde{\bPhi}_k,
\end{align*}
where $\bX_{(1)} \in \R^{n_1 \times n_2 n_3}$ and $\widetilde{\bPhi}_k \in \C^{n_2 n_3 \times n_2}$.
\end{proof}
\begin{lemma}\label{lemma:normbound}
Let $\bcS \in \R^{n_1 \times n_2 \times n_3}$ is $\alpha$-sparse, then
\begin{align*}
\| \bcS \| \leq \frac{\alpha \sqrt{\ell}}{2} (n_1 + n_2 n_3) \| \bcS \|_{\infty} \quad \mathrm{and} \quad \| \bcS \|_{2,\infty} \leq \sqrt{\alpha n_2 n_3} \| \bcS \|_{\infty}.
\end{align*}
\end{lemma}

\begin{proof}
We only need to prove the first claim, since the second claim can be directly followed by the fact that $\bcS$ has at most $\alpha$-fraction nonzero entries along each mode-2 tube of $\bcS$. Let $\bx \in \R^{n_1}$ and $\bz \in \R^{n_2 n_3}$ be unit vectors, using $a b \leq (a^2 + b^2)/2$, we have
\begin{align*}
\| \bcS \| & = \| \widebar{\bS} \| = \max_{k = 1, \dots, n_3} \| \xoverline{\bcS}(:,:,k) \| = \max_{k = 1, \dots, n_3} \| \bS_{(1)} \widetilde{\bPhi}_k \| \leq \max_{k = 1, \dots, n_3} \| \bS_{(1)} \| \| \widetilde{\bPhi}_k \|  \\
& = \max_{k = 1, \dots, n_3} \sqrt{\ell} \| \bS_{(1)} \| = \sqrt{\ell} \sum_{i=1}^{n_1} \sum_{j=1}^{n_2} \sum_{k=1}^{n_3} x_i \bcS_{i,j,k} z_{(j-1)n_3+k}  \\
& \leq \frac{\sqrt{\ell}}{2} \sum_{i=1}^{n_1} \sum_{j=1}^{n_2} \sum_{k=1}^{n_3} (x_i^2 + z_{(j-1)n_3+k}^2) \bcS_{i,j,k}  \\
& \leq \frac{\sqrt{\ell}}{2} (\alpha n_2 n_3 + \alpha n_1) \| \bcS \|_{\infty} = \frac{\alpha \sqrt{\ell}}{2} (n_1 + n_2 n_3) \| \bcS \|_{\infty},
\end{align*}
where the last inequality follows from the assumption that $\bcS \in \mathscr{S}_{\alpha}$.
\end{proof}
\begin{lemma}\label{lemma:infnormbound}
Let $\bcA \in \R^{n_1 \times n_2 \times n_3}$ and $\bcB \in \R^{n_4 \times n_2 \times n_3}$ be two tensors, then
\begin{align*}
\| \bcA \ast_{\bPhi} \bcB^H \|_{\infty} \leq \sqrt{\ell} \| \bcA \|_{2,\infty} \| \bcB \|_{2,\infty}.
\end{align*}
\end{lemma}

\begin{proof}
By the definition of the transform based t-product, we have
\begin{align*}
\| \bcA \ast_{\bPhi} \bcB^H \|_{\infty} & = \max_{i,i'} \| \sum_{j=1}^{n_2} \bcA(i,j,:) \ast_{\bPhi} \bcB(i',j,:) \|_{\infty} \leq \max_{i,i'} \| \sum_{j=1}^{n_2} \bcA(i,j,:) \ast_{\bPhi} \bcB(i',j,:) \|_2  \\
& \leq \max_{i,i'} \sum_{j=1}^{n_2} \| \bcA(i,j,:) \ast_{\bPhi} \bcB(i',j,:) \|_2 = \max_{i,i'} \sum_{j=1}^{n_2} \frac{1}{\sqrt{\ell}} \| \xoverline{\bcA}(i,j,:) \triangle \xoverline{\bcB}(i',j,:) \|_2  \\
& \leq \max_{i,i'} \sum_{j=1}^{n_2} \frac{1}{\sqrt{\ell}} \| \xoverline{\bcA}(i,j,:) \|_2 \| \xoverline{\bcB}(i',j,:) \|_2 = \max_{i,i'} \sum_{j=1}^{n_2} \sqrt{\ell} \| \bcA(i,j,:) \|_2 \| \bcB(i',j,:) \|_2  \\
& \leq \max_{i,i'} \sqrt{\ell} \| \bcA(i,:,:) \|_F \| \bcB(i',:,:) \|_F \leq \sqrt{\ell} \| \bcA \|_{2,\infty} \| \bcB \|_{2,\infty}.
\end{align*}
\end{proof}
\begin{lemma}\label{lemma:2infbound}
Let $\bcA \in \R^{n_1 \times n_2 \times n_3}$ and $\bcB \in \R^{n_2 \times n_4 \times n_3}$ be two tensors. Assume the multi-rank of $\bcB$ is $\br$ and let $s_r = \sum_{k=1}^{n_3} r_k$, then
\begin{align*}
\| \bcA \ast_{\bPhi} \bcB \|_F \geq \| \bcA \|_F \bar{\sigma}_{s_r} (\bcB) \quad \mathrm{and} \quad \| \bcA \ast_{\bPhi} \bcB \|_F \leq \| \bcA \|_F \| \bcB \|.
\end{align*}
Moreover,
\begin{align*}
\| \bcA \ast_{\bPhi} \bcB \|_{2,\infty} \geq \| \bcA \|_{2,\infty} \bar{\sigma}_{s_r} (\bcB) \quad \mathrm{and} \quad \| \bcA \ast_{\bPhi} \bcB \|_{2,\infty} \leq \| \bcA \|_{2,\infty} \| \bcB \|.
\end{align*}
\end{lemma}

\begin{proof}
To prove the first claim, we start with the definition of Frobenius norm
\begin{align*}
\| \bcA \ast_{\bPhi} \bcB \|_F^2 & = \frac{1}{\ell} \| \xoverline{\bcA} \triangle \xoverline{\bcB} \|_F^2 = \frac{1}{\ell} \| \widebar{\bA} \widebar{\bB} \|_F^2 = \frac{1}{\ell} \tr(\widebar{\bA} \widebar{\bB} \widebar{\bB}^H \widebar{\bA}^H) = \frac{1}{\ell} \tr(\widebar{\bB} \widebar{\bB}^H \widebar{\bA}^H \widebar{\bA}).
\end{align*}
By using the inequalities for matrix trace, we have
\begin{align*}
\lambda_{\mathrm{min}} (\widebar{\bB} \widebar{\bB}^H) \tr(\widebar{\bA}^H \widebar{\bA}) \leq \tr(\widebar{\bB} \widebar{\bB}^H \widebar{\bA}^H \widebar{\bA}) \leq \lambda_{\mathrm{max}} (\widebar{\bB} \widebar{\bB}^H) \tr(\widebar{\bA}^H \widebar{\bA}),
\end{align*}
where $\lambda_{\mathrm{min}}$ and $\lambda_{\mathrm{max}}$ denote the minimum and maximum eigenvalue, respectively. This implies
\begin{align*}
\sigma_{\mathrm{min}}^2 (\widebar{\bB}) \tr(\widebar{\bA}^H \widebar{\bA}) \leq \ell \| \bcA \ast_{\bPhi} \bcB \|_F^2 \leq \sigma_{\mathrm{max}}^2 (\widebar{\bB}) \tr(\widebar{\bA}^H \widebar{\bA}),
\end{align*}
where $\sigma_{\mathrm{min}}$ and $\sigma_{\mathrm{max}}$ denote the minimum and maximum singular value, respectively. Hence, we have
\begin{align*}
& \| \bcA \ast_{\bPhi} \bcB \|_F^2 \geq \sigma_{\mathrm{min}}^2 (\widebar{\bB}) \frac{1}{\ell} \tr(\widebar{\bA}^H \widebar{\bA}) = \bar{\sigma}_{s_r}^2 (\bcB) \| \bcA \|_F^2  \\
\mathrm{and} \quad & \| \bcA \ast_{\bPhi} \bcB \|_F^2 \leq \sigma_{\mathrm{max}}^2 (\widebar{\bB}) \frac{1}{\ell} \tr(\widebar{\bA}^H \widebar{\bA}) = \| \bcB \|^2 \| \bcA \|_F^2.
\end{align*}
Taking the square root on both sides to arrive at the first claim. The second conclusion is an easy consequence of the first one as
\begin{align*}
& \| \bcA \ast_{\bPhi} \bcB \|_{2,\infty} = \max_i \| \bcA(i,:,:) \ast_{\bPhi} \bcB \|_F \geq \max_i \| \bcA(i,:,:) \|_F \bar{\sigma}_{s_r} (\bcB) = \| \bcA \|_{2,\infty} \bar{\sigma}_{s_r} (\bcB)  \\
\mathrm{and} \quad & \| \bcA \ast_{\bPhi} \bcB \|_{2,\infty} = \max_i \| \bcA(i,:,:) \ast_{\bPhi} \bcB \|_F \leq \max_i \| \bcA(i,:,:) \|_F \| \bcB \| = \| \bcA \|_{2,\infty} \| \bcB \|.
\end{align*}
\end{proof}
\begin{lemma}\label{lemma:2infprodbound}
Let $\bcA \in \R^{n_1 \times n_2 \times n_3}$ and $\bcB \in \R^{n_2 \times n_4 \times n_3}$ be two tensors, then
\begin{align*}
\| \bcA \ast_{\bPhi} \bcB \|_{2,\infty} \leq \sqrt{n_2 \ell} \| \bcA \|_{2,\infty} \| \bcB \|_{2,\infty}.
\end{align*}
\end{lemma}

\begin{proof}
Based on the definition of the $\ell_{2,\infty}$-norm, we have
\begin{align*}
\| \bcA \ast_{\bPhi} \bcB \|_{2,\infty} = \frac{1}{\sqrt{\ell}} \| \xoverline{\bcA} \triangle \xoverline{\bcB} \|_{2,\infty} = \max_i \frac{1}{\sqrt{\ell}} \| \xoverline{\bcA}(i,:,:) \triangle \xoverline{\bcB} \|_F. 
\end{align*}
Note that
\begin{align*}
\| \xoverline{\bcA}(i,:,:) \triangle \xoverline{\bcB} \|_F & = \sqrt{\sum_{k=1}^{n_3} \| \xoverline{\bcA}(i,:,k) \xoverline{\bcB}(:,:,k) \|_2^2} = \sqrt{\sum_{k=1}^{n_3} \sum_{j=1}^{n_4} (\langle \xoverline{\bcA}(i,:,k), \xoverline{\bcB}(:,j,k) \rangle)^2}  \\
& \leq \sqrt{\sum_{k=1}^{n_3} \sum_{j=1}^{n_4} \| \xoverline{\bcA}(i,:,k) \|_2^2 \| \xoverline{\bcB}(:,j,k) \|_2^2} = \sqrt{\sum_{k=1}^{n_3} \| \xoverline{\bcA}(i,:,k) \|_2^2 \| \xoverline{\bcB}(:,:,k) \|_F^2}  \\
& \leq \| \xoverline{\bcA}(i,:,:) \|_F \| \xoverline{\bcB} \|_F = \ell \| \bcA(i,:,:) \|_F \| \bcB \|_F.
\end{align*}
Thus,
\begin{align*}
\| \bcA \ast_{\bPhi} \bcB \|_{2,\infty} \leq \max_i \sqrt{\ell} \| \bcA(i,:,:) \|_F \| \bcB \|_F = \sqrt{\ell} \| \bcA \|_{2,\infty} \| \bcB \|_F \leq \sqrt{n_2 \ell} \| \bcA \|_{2,\infty} \| \bcB \|_{2,\infty}.
\end{align*}
\end{proof}
\begin{lemma}\label{lemma:eijkbounds}
Let $L$ be any invertible linear transform in \eqref{eqn:mode3prod} and it satisfies \eqref{eqn:phiconstraint}. For any $i \in [n_1]$, $j \in [n_2]$ and $k \in [n_3]$, given $\bar{\boldsymbol{\ce}}_{ijk}$ in \eqref{eqn:defeijk}, we have the following properties
\begin{align}\label{eqn:eijkbound1}
\| \bar{\boldsymbol{\ce}}_{ijk} \| \leq \sqrt{\ell},
\end{align}
\begin{align}\label{eqn:eijkbound2}
\Big\| \sum_{i=1}^{n_1} \sum_{j=1}^{n_2} \sum_{k=1}^{n_3} (\bar{\boldsymbol{\ce}}_{ijk}^H \ast_{\bPhi} \bar{\boldsymbol{\ce}}_{ijk}) \Big\| \leq n_1 \ell,
\end{align}
and
\begin{align}\label{eqn:eijkbound3}
\Big\| \sum_{i=1}^{n_1} \sum_{j=1}^{n_2} \sum_{k=1}^{n_3} (\bar{\boldsymbol{\ce}}_{ijk} \ast_{\bPhi} \bar{\boldsymbol{\ce}}_{ijk}^H) \Big\| \leq n_2 \ell.
\end{align}
\end{lemma}

\begin{proof}
To simplify the notation, we denote $\bcA = \bar{\boldsymbol{\ce}}_{ijk}$ in this proof. By the definition of the tensor spectral norm, we have
\begin{align*}
\| \bcA \| = \| \widebar{\bA} \| = \max_{k' = 1, \dots, n_3} \| \widebar{\bA}^{(k')} \|.
\end{align*}
Since $\bcA$ is the unit tensor, the $(i,j)$-th mode-3 tube of $\xoverline{\bcA} = L(\bar{\boldsymbol{\ce}}_{ijk})$ is the only nonzero mode-3 tube and its entries are the same as the $k$-th column of $\bPhi$, i.e., $\xoverline{\bcA}_{i,j,k'} = \bPhi_{k',k}$. Then $\| \widebar{\bA}^{(k')} \| = |\bPhi_{k',k}|$ and we have
\begin{align*}
\| \bcA \| = \max_{k' = 1, \dots, n_3} \| \widebar{\bA}^{(k')} \| = \max_{k' = 1, \dots, n_3} |\bPhi_{k',k}| \leq \sqrt{\ell},
\end{align*}
where the inequality comes from $|\bPhi_{k',k}| = \sqrt{|\bPhi_{k',k}|^2} \leq \sqrt{\sum_{k'=1}^{n_3} |\bPhi_{k',k}|^2} = \sqrt{\ell}$. Therefore, \eqref{eqn:eijkbound1} is verified. To prove \eqref{eqn:eijkbound2}, let $\bcB = \bcA^H \ast_{\bPhi} \bcA$, then $\xoverline{\bcB} = \xoverline{\bcA}^H \triangle \xoverline{\bcA}$, which means that the $(j,j)$-th mode-3 tube of $\xoverline{\bcB}$ is the only nonzero mode-3 tube and its $k'$-th entry is $\xoverline{\bcB}_{j,j,k'} = |\bPhi_{k',k}|^2$. Thus the $(j,j)$-th mode-3 tube of
\begin{align*}
\sum_{k=1}^{n_3} L(\bar{\boldsymbol{\ce}}_{ijk}^H \ast_{\bPhi} \bar{\boldsymbol{\ce}}_{ijk}) = \sum_{k=1}^{n_3} \xoverline{\bcB}
\end{align*}
is the only nonzero mode-3 tube and all of its entries equal $\ell$. Further,
\begin{align*}
\sum_{j=1}^{n_2} \sum_{k=1}^{n_3} L(\bar{\boldsymbol{\ce}}_{ijk}^H \ast_{\bPhi} \bar{\boldsymbol{\ce}}_{ijk}) = \sum_{j=1}^{n_2} \sum_{k=1}^{n_3} \xoverline{\bcB}
\end{align*}
is an f-diagonal tensor and all the entries on the diagonal equal $\ell$. Finally,
\begin{align*}
\sum_{i=1}^{n_1} \sum_{j=1}^{n_2} \sum_{k=1}^{n_3} L(\bar{\boldsymbol{\ce}}_{ijk}^H \ast_{\bPhi} \bar{\boldsymbol{\ce}}_{ijk}) = \sum_{i=1}^{n_1} \sum_{j=1}^{n_2} \sum_{k=1}^{n_3} \xoverline{\bcB}
\end{align*}
is also f-diagonal and all of its entries on the $(j,j)$-th mode-3 tube equal $n_1 \ell$, $j \in [n_2]$. It is obvious that each frontal slice of such a tensor has the spectral norm $n_1 \ell$. Therefore, we have \eqref{eqn:eijkbound2}. We can prove \eqref{eqn:eijkbound3} in a similar way.
\end{proof}
\begin{lemma}\label{lemma:projeijk}
Given $\bcP_{\bT}$ in \eqref{eqn:defprojT} and assume that the tensor incoherence conditions \eqref{eqn:incoherence} hold. We have
\begin{align*}
\| \bcP_{\bT} (\bar{\boldsymbol{\ce}}_{ijk}) \|_F^2 \leq \frac{\mu s_r (n_1 + n_2) }{n_1 n_2 n_3}.
\end{align*}
\end{lemma}

\begin{proof}
Note that $\bcP_{\bT}$ is self-adjoint. So we have
\begin{align*}
& \| \bcP_{\bT} (\bar{\boldsymbol{\ce}}_{ijk}) \|_F^2  \\
= & \langle \bcP_{\bT} (\bar{\boldsymbol{\ce}}_{ijk}), \bar{\boldsymbol{\ce}}_{ijk} \rangle  \\
= & \langle \bcU \ast_{\bPhi} \bcU^H \ast_{\bPhi} \bar{\boldsymbol{\ce}}_{ijk} + \bar{\boldsymbol{\ce}}_{ijk} \ast_{\bPhi} \bcV \ast_{\bPhi} \bcV^H , \bar{\boldsymbol{\ce}}_{ijk} \rangle + \langle \bcU \ast_{\bPhi} \bcU^H \ast_{\bPhi} \bar{\boldsymbol{\ce}}_{ijk} \ast_{\bPhi} \bcV \ast_{\bPhi} \bcV^H , \bar{\boldsymbol{\ce}}_{ijk} \rangle.
\end{align*}
Note that $\bar{\boldsymbol{\ce}}_{ijk}$ is the unit tensor with the ($i,j,k$)-th entry equaling to 1 and the rest equaling to 0. Hence, the $(i,j)$-th mode-3 tube of $\bar{\boldsymbol{\ce}}_{ijk}$ equals $L(\dot{\boldsymbol{\ce}}_k)$ and the $(i,j)$-th mode-3 tube of $L(\bar{\boldsymbol{\ce}}_{ijk})$ equals $L(L(\dot{\boldsymbol{\ce}}_k))$. Then the only nonzero lateral slice of $L(\bar{\boldsymbol{\ce}}_{ijk})$ is the $j$-th lateral slice and it is equal to $L(\mathring{\boldsymbol{\ce}}_i) \triangle L(L(\dot{\boldsymbol{\ce}}_k))$. This implies that
\begin{align*}
\langle \bcU \ast_{\bPhi} \bcU^H \ast_{\bPhi} \bar{\boldsymbol{\ce}}_{ijk}, \bar{\boldsymbol{\ce}}_{ijk} \rangle = & \frac{1}{\ell} \| \bar{\bcU}^H \triangle L(\bar{\boldsymbol{\ce}}_{ijk}) \|_F^2 = \frac{1}{\ell} \| \bar{\bcU}^H \triangle L(\mathring{\boldsymbol{\ce}}_i) \triangle L(L(\dot{\boldsymbol{\ce}}_k)) \|_F^2  \\
= & \| \bcU^H \ast_{\bPhi} \mathring{\boldsymbol{\ce}}_i \ast_{\bPhi} L(\dot{\boldsymbol{\ce}}_k) \|_F^2.
\end{align*}
In the meanwhile, we have
\begin{align*}
\| \bcU^H \ast_{\bPhi} \mathring{\boldsymbol{\ce}}_i \ast_{\bPhi} L(\dot{\boldsymbol{\ce}}_k) \|_F & = \frac{1}{\sqrt{\ell}} \| L(\bcU^H) \triangle L(\mathring{\boldsymbol{\ce}}_i) \triangle L(L(\dot{\boldsymbol{\ce}}_k)) \|_F \leq \frac{1}{\sqrt{\ell}} \| L(\bcU^H) \triangle L(\mathring{\boldsymbol{\ce}}_i) \|_F \| L(L(\dot{\boldsymbol{\ce}}_k)) \|_F  \\
& = \sqrt{\ell} \| \bcU^H \ast_{\bPhi} \mathring{\boldsymbol{\ce}}_i \|_F \leq \sqrt{\frac{\mu s_r}{n_1 n_3}},
\end{align*}
where we use $\| L(L(\dot{\boldsymbol{\ce}}_k)) \|_F = \sqrt{\ell} \| L(\dot{\boldsymbol{\ce}}_k) \|_F = \sqrt{\ell}$. Similarly, we can also have $\| \bcV^H \ast_{\bPhi} \mathring{\boldsymbol{\ce}}_j \ast_{\bPhi} L(\dot{\boldsymbol{\ce}}_k) \|_F \leq \sqrt{\frac{\mu s_r}{n_2 n_3}}$. Therefore, we have
\begin{align*}
\| \bcP_{\bT} (\bar{\boldsymbol{\ce}}_{ijk}) \|_F^2 = & \| \bcU^H \ast_{\bPhi} \mathring{\boldsymbol{\ce}}_i \ast_{\bPhi} L(\dot{\boldsymbol{\ce}}_k) \|_F^2 + \| \bcV^H \ast_{\bPhi} \mathring{\boldsymbol{\ce}}_j \ast_{\bPhi} L(\dot{\boldsymbol{\ce}}_k) \|_F^2 - \| \bcU^H \ast_{\bPhi} \bar{\boldsymbol{\ce}}_{ijk} \ast_{\bPhi} \bcV \|_F^2  \\
\leq & \| \bcU^H \ast_{\bPhi} \mathring{\boldsymbol{\ce}}_i \ast_{\bPhi} L(\dot{\boldsymbol{\ce}}_k) \|_F^2 + \| \bcV^H \ast_{\bPhi} \mathring{\boldsymbol{\ce}}_j \ast_{\bPhi} L(\dot{\boldsymbol{\ce}}_k) \|_F^2  \\
\leq & \frac{\mu s_r (n_1 + n_2) }{n_1 n_2 n_3}.
\end{align*}
The proof is completed.
\end{proof}
Using the same proof technique in \citet[Lemma 4.2]{LuFCLLY.PAMI2020}, we have the following result.
\begin{lemma}\label{lemma:projerrorbound}
Suppose $\bOmega \sim \mathrm{Ber}(p)$, and $\bT$ is defined in \eqref{eqn:defT}. Then with high probability,
\begin{align*}
\| \bcP_{\bT} - \frac{1}{p} \bcP_{\bT} \bcP_{\bOmega} \bcP_{\bT} \| \leq \epsilon,
\end{align*}
provided that $p \geq c \epsilon^{-2} \mu s_r (n_1 + n_2) \log( (n_1 \vee n_2) n_3 ) / (n_1 n_2 n_3)$ for some numerical constant $c > 0$.
\end{lemma}

\subsection{Distance Metric}

\begin{lemma}\label{lemma:Qexistence}
Fix any factor tensor $\bcF = \begin{bmatrix} \bcL \\ \bcR \end{bmatrix} \in \R^{(n_1+n_2) \times r \times n_3}$. Suppose that
\begin{align}\label{eqn:Qexistencecondition}
\dist(\bcF, \bcF_{\star}) < \frac{1}{\sqrt{\ell}} \bar{\sigma}_{s_r} (\bcX_{\star}),
\end{align}
then the minimizer of the above minimization problem is attained at some $\bcQ \in \mathrm{GL}(r)$, i.e., the optimal alignment tensor $\bcQ$ between $\bcF$ and $\bcF_{\star}$ exists. 
\end{lemma}

\begin{proof}
Based on the definition of infimum and condition \eqref{eqn:Qexistencecondition}, there must exist a tensor $\widetilde{\bcQ} \in \mathrm{GL}(r)$ such that
\begin{align*}
& \sqrt{\| (\widebar{\bL} \widebar{\widetilde{\bQ}} - \widebar{\bL}_{\star}) \widebar{\bG}_{\star}^{-\frac{1}{2}} \widebar{\bG}_{\star} \|_F^2 + \| (\widebar{\bR} \widebar{\widetilde{\bQ}}^{-H} - \widebar{\bR}_{\star}) \widebar{\bG}_{\star}^{-\frac{1}{2}} \widebar{\bG}_{\star} \|_F^2}  \\
= & \sqrt{\ell} \sqrt{\| (\bcL \ast_{\bPhi} \widetilde{\bcQ} - \bcL_{\star}) \ast_{\bPhi} \bcG_{\star}^{\frac{1}{2}} \|_F^2 + \| (\bcR \ast_{\bPhi} \widetilde{\bcQ}^{-H} - \bcR_{\star}) \ast_{\bPhi} \bcG_{\star}^{\frac{1}{2}} \|_F^2}  \\
\leq & \epsilon \bar{\sigma}_{s_r} (\bcX_{\star})
\end{align*}
together with the relation $\| \bA \bB \|_F \geq \| \bA \|_F \sigma_{\mathrm{min}} (\bB)$ tells that
\begin{align*}
\sqrt{\| (\widebar{\bL} \widebar{\widetilde{\bQ}} - \widebar{\bL}_{\star}) \widebar{\bG}_{\star}^{-\frac{1}{2}} \|_F^2 + \| (\widebar{\bR} \widebar{\widetilde{\bQ}}^{-H} - \widebar{\bR}_{\star}) \widebar{\bG}_{\star}^{-\frac{1}{2}} \|_F^2} \bar{\sigma}_{s_r} (\bcX_{\star}) \leq \epsilon \bar{\sigma}_{s_r} (\bcX_{\star})
\end{align*}
for some $\epsilon$ obeying $0 < \epsilon < 1$. It further implies that
\begin{align*}
\| (\widebar{\bL} \widebar{\widetilde{\bQ}} - \widebar{\bL}_{\star}) \widebar{\bG}_{\star}^{-\frac{1}{2}} \| \vee \| (\widebar{\bR} \widebar{\widetilde{\bQ}}^{-H} - \widebar{\bR}_{\star}) \widebar{\bG}_{\star}^{-\frac{1}{2}} \| \leq \epsilon.
\end{align*}
The rest of the proof is the same as the one in \citet[Lemma 22]{TongMC.JMLR2021}.
\end{proof}
\begin{lemma}\label{lemma:Qcriterion}
For any factor tensor $\bcF = \begin{bmatrix} \bcL \\ \bcR \end{bmatrix} \in \R^{(n_1+n_2) \times r \times n_3}$, suppose that the optimal alignment tensor
\begin{align}\label{eqn:Qfunc}
\bcQ = \argmin_{\bcQ \in \mathrm{GL}(r)} \| (\bcL \ast_{\bPhi} \bcQ - \bcL_{\star}) \ast_{\bPhi} \bcG_{\star}^{\frac{1}{2}} \|_F^2 + \| (\bcR \ast_{\bPhi} \bcQ^{-H} - \bcR_{\star}) \ast_{\bPhi} \bcG_{\star}^{\frac{1}{2}} \|_F^2
\end{align}
between $\bcF$ and $\bcF_{\star}$ exists, then $\bcQ$ obeys
\begin{align}\label{eqn:Qcriterion}
(\bcL \ast_{\bPhi} \bcQ)^H \ast_{\bPhi} (\bcL \ast_{\bPhi} \bcQ - \bcL_{\star}) \ast_{\bPhi} \bcG_{\star} = \bcG_{\star} \ast_{\bPhi} (\bcR \ast_{\bPhi} \bcQ^{-H} - \bcR_{\star})^H \ast_{\bPhi} \bcR \ast_{\bPhi} \bcQ^{-H}.
\end{align}
\end{lemma}

\begin{proof}
Check the gradient of the objective function in \eqref{eqn:Qfunc} with respect to $\bcQ$ and set it to zero yields
\begin{align*}
2 \bcL^H \ast_{\bPhi} (\bcL \ast_{\bPhi} \bcQ - \bcL_{\star}) \ast_{\bPhi} \bcG_{\star} - 2 \bcQ^{-H} \ast_{\bPhi} \bcG_{\star} \ast_{\bPhi} (\bcR \ast_{\bPhi} \bcQ^{-H} - \bcR_{\star})^H \ast_{\bPhi} \bcR \ast_{\bPhi} \bcQ^{-H} = \bzero,
\end{align*}
which implies the optimal alignment criterion \eqref{eqn:Qcriterion}.
\end{proof}

Lastly, following the proof in \citet[Lemma 24]{TongMC.JMLR2021}, we connect the proposed distance to the Frobenius norm in Lemma~\ref{lemma:Procrustes}.
\begin{lemma}\label{lemma:Procrustes}
For any factor tensor $\bcF = \begin{bmatrix} \bcL \\ \bcR \end{bmatrix} \in \R^{(n_1+n_2) \times r \times n_3}$, the distance between $\bcF$ and $\bcF_{\star}$ satisfies
\begin{align*}
\dist(\bcF, \bcF_{\star}) \leq \left(\sqrt{2}+1 \right)^{\frac{1}{2}} \| \bcL \ast_{\bPhi} \bcR^H - \bcX_{\star} \|_F.
\end{align*}
\end{lemma}

\subsection{Tensor Perturbation Bounds}

\begin{lemma}\label{lemma:Weyl}
For any $\bcL_{\sharp} \in \R^{n_1 \times r \times n_3}$, $\bcR_{\sharp} \in \R^{n_2 \times r \times n_3}$, denote $\bcL_{\triangle} \coloneq \bcL_{\sharp} - \bcL_{\star}$ and $\bcR_{\triangle} \coloneq \bcR_{\sharp} - \bcR_{\star}$. Suppose that $\| \bcL_{\triangle} \ast_{\bPhi} \bcG_{\star}^{-\frac{1}{2}} \| \vee \| \bcR_{\triangle} \ast_{\bPhi} \bcG_{\star}^{-\frac{1}{2}} \| < 1$, then
\begin{align}
\| \bcL_{\sharp} \ast_{\bPhi} (\bcL_{\sharp}^H \ast_{\bPhi} \bcL_{\sharp})^{-1} \ast_{\bPhi} \bcG_{\star}^{\frac{1}{2}} \| & \leq \frac{1}{1 - \| \bcL_{\triangle} \ast_{\bPhi} \bcG_{\star}^{-\frac{1}{2}} \|};  \label{eqn:Weyl-1L} \\
\| \bcR_{\sharp} \ast_{\bPhi} (\bcR_{\sharp}^H \ast_{\bPhi} \bcR_{\sharp})^{-1} \ast_{\bPhi} \bcG_{\star}^{\frac{1}{2}} \| & \leq \frac{1}{1 - \| \bcR_{\triangle} \ast_{\bPhi} \bcG_{\star}^{-\frac{1}{2}} \|};  \label{eqn:Weyl-1R} \\
\| \bcL_{\sharp} \ast_{\bPhi} (\bcL_{\sharp}^H \ast_{\bPhi} \bcL_{\sharp})^{-1} \ast_{\bPhi} \bcG_{\star}^{\frac{1}{2}} - \bcU_{\star} \| & \leq \frac{\sqrt{2}\| \bcL_{\triangle} \ast_{\bPhi} \bcG_{\star}^{-\frac{1}{2}} \|}{1 - \| \bcL_{\triangle} \ast_{\bPhi} \bcG_{\star}^{-\frac{1}{2}} \|};  \label{eqn:Weyl-2L} \\
\| \bcR_{\sharp} \ast_{\bPhi} (\bcR_{\sharp}^H \ast_{\bPhi} \bcR_{\sharp})^{-1} \ast_{\bPhi} \bcG_{\star}^{\frac{1}{2}} - \bcV_{\star} \| & \leq \frac{\sqrt{2}\| \bcR_{\triangle} \ast_{\bPhi} \bcG_{\star}^{-\frac{1}{2}} \|}{1 - \| \bcR_{\triangle} \ast_{\bPhi} \bcG_{\star}^{-\frac{1}{2}} \|}.  \label{eqn:Weyl-2R}
\end{align}
\end{lemma}

\begin{proof}
We only prove \eqref{eqn:Weyl-1L} and \eqref{eqn:Weyl-2L}, since the claims \eqref{eqn:Weyl-1R} and \eqref{eqn:Weyl-2R} on the factor $\bcR$ can be proved in a similar way. We first notice that 
\begin{align*}
\| \bcL_{\sharp} \ast_{\bPhi} (\bcL_{\sharp}^H \ast_{\bPhi} \bcL_{\sharp})^{-1} \ast_{\bPhi} \bcG_{\star}^{\frac{1}{2}} \| = \| \widebar{\bL}_{\sharp} (\widebar{\bL}_{\sharp}^H \widebar{\bL}_{\sharp})^{-1} \widebar{\bG}_{\star}^{\frac{1}{2}} \| = \frac{1}{\bar{\sigma}_{s_r}(\widebar{\bL}_{\sharp} \widebar{\bG}_{\star}^{-\frac{1}{2}})}.
\end{align*}
We invoke Weyl's inequality $| \bar{\sigma}_{s_r}(\bcA) - \bar{\sigma}_{s_r}(\bcB) | = | \sigma_{s_r}(\widebar{\bA}) - \sigma_{s_r}(\widebar{\bB}) | \leq \| \widebar{\bA} - \widebar{\bB} \|$, and use the fact that $\widebar{\bU}_{\star} = \widebar{\bL}_{\star} \widebar{\bG}_{\star}^{-\frac{1}{2}}$ satisfies $\sigma_{s_r}(\widebar{\bU}_{\star}) = 1$ to obtain
\begin{align*}
\sigma_{s_r}(\widebar{\bL}_{\sharp} \widebar{\bG}_{\star}^{-\frac{1}{2}}) \geq \sigma_{s_r}(\widebar{\bL}_{\star} \widebar{\bG}_{\star}^{-\frac{1}{2}}) - \| \widebar{\bL}_{\triangle} \widebar{\bG}_{\star}^{-\frac{1}{2}} \| = 1 - \| \widebar{\bL}_{\triangle} \widebar{\bG}_{\star}^{-\frac{1}{2}} \| = 1 - \| \bcL_{\triangle} \ast_{\bPhi} \bcG_{\star}^{-\frac{1}{2}} \|.
\end{align*}
Then \eqref{eqn:Weyl-1L} follows immediately by combining the preceding two relations. In order to prove \eqref{eqn:Weyl-2L}, using the facts that $\bcL_{\star}^H \ast_{\bPhi} \bcU_{\star} = \bcG_{\star}^{\frac{1}{2}}$ and $(\bcI_{n_1} - \bcL_{\sharp} \ast_{\bPhi} (\bcL_{\sharp}^H \ast_{\bPhi} \bcL_{\sharp})^{-1} \ast_{\bPhi} \bcL_{\sharp}^H) \ast_{\bPhi} \bcL_{\sharp} = \bzero$, we have the following decomposition
\begin{align*}
& \bcL_{\sharp} \ast_{\bPhi} (\bcL_{\sharp}^H \ast_{\bPhi} \bcL_{\sharp})^{-1} \ast_{\bPhi} \bcG_{\star}^{\frac{1}{2}} - \bcU_{\star}  \\
= & \bcL_{\sharp} \ast_{\bPhi} (\bcL_{\sharp}^H \ast_{\bPhi} \bcL_{\sharp})^{-1} \ast_{\bPhi} \bcL_{\star}^H \ast_{\bPhi} \bcU_{\star} - \bcL_{\star} \ast_{\bPhi} \bcG_{\star}^{-\frac{1}{2}}  \\
= & - \bcL_{\sharp} \ast_{\bPhi} (\bcL_{\sharp}^H \ast_{\bPhi} \bcL_{\sharp})^{-1} \ast_{\bPhi} (\bcL_{\sharp} - \bcL_{\star})^H \ast_{\bPhi} \bcU_{\star}  \\
&\qquad\qquad\qquad\qquad - (\bcI_{n_1} - \bcL_{\sharp} \ast_{\bPhi} (\bcL_{\sharp}^H \ast_{\bPhi} \bcL_{\sharp})^{-1} \ast_{\bPhi} \bcL_{\sharp}^H) \ast_{\bPhi} \bcL_{\star} \ast_{\bPhi} \bcG_{\star}^{-\frac{1}{2}}  \\
= & - \bcL_{\sharp} \ast_{\bPhi} (\bcL_{\sharp}^H \ast_{\bPhi} \bcL_{\sharp})^{-1} \ast_{\bPhi} \bcL_{\triangle}^H \ast_{\bPhi} \bcU_{\star}  \\
&\qquad\qquad\qquad\qquad + (\bcI_{n_1} - \bcL_{\sharp} \ast_{\bPhi} (\bcL_{\sharp}^H \ast_{\bPhi} \bcL_{\sharp})^{-1} \ast_{\bPhi} \bcL_{\sharp}^H) \ast_{\bPhi} (\bcL_{\sharp} - \bcL_{\star}) \ast_{\bPhi} \bcG_{\star}^{-\frac{1}{2}}  \\
= & - \bcL_{\sharp} \ast_{\bPhi} (\bcL_{\sharp}^H \ast_{\bPhi} \bcL_{\sharp})^{-1} \ast_{\bPhi} \bcL_{\triangle}^H \ast_{\bPhi} \bcU_{\star}  \\
&\qquad\qquad\qquad\qquad + (\bcI_{n_1} - \bcL_{\sharp} \ast_{\bPhi} (\bcL_{\sharp}^H \ast_{\bPhi} \bcL_{\sharp})^{-1} \ast_{\bPhi} \bcL_{\sharp}^H) \ast_{\bPhi} \bcL_{\triangle} \ast_{\bPhi} \bcG_{\star}^{-\frac{1}{2}}.
\end{align*}
In the meanwhile, we can verify that $\bcL_{\sharp} \ast_{\bPhi} (\bcL_{\sharp}^H \ast_{\bPhi} \bcL_{\sharp})^{-1} \ast_{\bPhi} \bcL_{\triangle}^H \ast_{\bPhi} \bcU_{\star}$ and $(\bcI_{n_1} - \bcL_{\sharp} \ast_{\bPhi} (\bcL_{\sharp}^H \ast_{\bPhi} \bcL_{\sharp})^{-1} \ast_{\bPhi} \bcL_{\sharp}^H) \ast_{\bPhi} \bcL_{\triangle} \ast_{\bPhi} \bcG_{\star}^{-\frac{1}{2}}$ are orthogonal, thus
\begin{align*}
& \| \bcL_{\sharp} \ast_{\bPhi} (\bcL_{\sharp}^H \ast_{\bPhi} \bcL_{\sharp})^{-1} \ast_{\bPhi} \bcG_{\star}^{\frac{1}{2}} - \bcU_{\star} \|^2  \\
\leq & \| \bcL_{\sharp} \ast_{\bPhi} (\bcL_{\sharp}^H \ast_{\bPhi} \bcL_{\sharp})^{-1} \ast_{\bPhi} \bcL_{\triangle}^H \ast_{\bPhi} \bcU_{\star} \|^2  \\
&\qquad\qquad\qquad\qquad + \| (\bcI_{n_1} - \bcL_{\sharp} \ast_{\bPhi} (\bcL_{\sharp}^H \ast_{\bPhi} \bcL_{\sharp})^{-1} \ast_{\bPhi} \bcL_{\sharp}^H) \ast_{\bPhi} \bcL_{\triangle} \ast_{\bPhi} \bcG_{\star}^{-\frac{1}{2}} \|^2  \\
\leq & \| \bcL_{\sharp} \ast_{\bPhi} (\bcL_{\sharp}^H \ast_{\bPhi} \bcL_{\sharp})^{-1} \ast_{\bPhi} \bcG_{\star}^{\frac{1}{2}} \|^2 \| \bcL_{\triangle} \ast_{\bPhi} \bcG_{\star}^{-\frac{1}{2}} \|^2  \\
&\qquad\qquad\qquad\qquad + \| \bcI_{n_1} - \bcL_{\sharp} \ast_{\bPhi} (\bcL_{\sharp}^H \ast_{\bPhi} \bcL_{\sharp})^{-1} \ast_{\bPhi} \bcL_{\sharp}^H \|^2 \| \bcL_{\triangle} \ast_{\bPhi} \bcG_{\star}^{-\frac{1}{2}} \|^2  \\
\leq & \frac{\| \bcL_{\triangle} \ast_{\bPhi} \bcG_{\star}^{-\frac{1}{2}} \|^2}{(1 - \| \bcL_{\triangle} \ast_{\bPhi} \bcG_{\star}^{-\frac{1}{2}} \|)^2} + \| \bcL_{\triangle} \ast_{\bPhi} \bcG_{\star}^{-\frac{1}{2}} \|^2  \\
\leq & \frac{2 \| \bcL_{\triangle} \ast_{\bPhi} \bcG_{\star}^{-\frac{1}{2}} \|^2}{(1 - \| \bcL_{\triangle} \ast_{\bPhi} \bcG_{\star}^{-\frac{1}{2}} \|)^2},
\end{align*}
where we have used \eqref{eqn:Weyl-1L} and the fact that $\| \bcI_{n_1} - \bcL_{\sharp} \ast_{\bPhi} (\bcL_{\sharp}^H \ast_{\bPhi} \bcL_{\sharp})^{-1} \ast_{\bPhi} \bcL_{\sharp}^H \| \leq 1$.
\end{proof}
\begin{lemma}\label{lemma:tensor2factor}
For any $\bcL_{\sharp} \in \R^{n_1 \times r \times n_3}$, $\bcR_{\sharp} \in \R^{n_2 \times r \times n_3}$, denote $\bcL_{\triangle} \coloneq \bcL_{\sharp} - \bcL_{\star}$ and $\bcR_{\triangle} \coloneq \bcR_{\sharp} - \bcR_{\star}$, then
\begin{align*}
\| \bcL_{\sharp} \ast_{\bPhi} \bcR_{\sharp}^H - \bcX_{\star} \|_F & \leq \| \bcL_{\triangle} \ast_{\bPhi} \bcR_{\star}^H \|_F + \| \bcL_{\star} \ast_{\bPhi} \bcR_{\triangle}^H \|_F + \| \bcL_{\triangle} \ast_{\bPhi} \bcR_{\triangle}^H \|_F \\
& \leq \Big( 1 + \frac{1}{2}(\| \bcL_{\triangle} \ast_{\bPhi} \bcG_{\star}^{-\frac{1}{2}} \| \vee \| \bcR_{\triangle} \ast_{\bPhi} \bcG_{\star}^{-\frac{1}{2}} \|) \Big)  \\
&\qquad\qquad\qquad\qquad \Big( \| \bcL_{\triangle} \ast_{\bPhi} \bcG_{\star}^{\frac{1}{2}} \|_F + \| \bcR_{\triangle} \ast_{\bPhi} \bcG_{\star}^{\frac{1}{2}} \|_F \Big).
\end{align*}
\end{lemma}

\begin{proof}
Using the decomposition that $\bcL_{\sharp} \ast_{\bPhi} \bcR_{\sharp}^H - \bcX_{\star} = \bcL_{\triangle} \ast_{\bPhi} \bcR_{\star}^H + \bcL_{\star} \ast_{\bPhi} \bcR_{\triangle}^H + \bcL_{\triangle} \ast_{\bPhi} \bcR_{\triangle}^H$ and the facts that
\begin{align*}
&\| \bcL_{\triangle} \ast_{\bPhi} \bcR_{\star}^H \|_F = \| \bcL_{\triangle} \ast_{\bPhi} \bcG_{\star}^{\frac{1}{2}} \ast_{\bPhi} \bcV_{\star}^H \|_F = \| \bcL_{\triangle} \ast_{\bPhi} \bcG_{\star}^{\frac{1}{2}} \|_F,  \\
\mathrm{and} \quad & \| \bcL_{\star} \ast_{\bPhi} \bcR_{\triangle}^H \|_F = \| \bcU_{\star} \ast_{\bPhi} \bcG_{\star}^{\frac{1}{2}} \ast_{\bPhi} \bcR_{\triangle}^H \|_F = \| \bcR_{\triangle} \ast_{\bPhi} \bcG_{\star}^{\frac{1}{2}} \|_F,
\end{align*}
as well as the triangle inequality, we have
\begin{align*}
\| \bcL_{\sharp} \ast_{\bPhi} \bcR_{\sharp}^H - \bcX_{\star} \|_F & \leq \| \bcL_{\triangle} \ast_{\bPhi} \bcR_{\star}^H \|_F + \| \bcL_{\star} \ast_{\bPhi} \bcR_{\triangle}^H \|_F + \| \bcL_{\triangle} \ast_{\bPhi} \bcR_{\triangle}^H \|_F  \\
& = \| \bcL_{\triangle} \ast_{\bPhi} \bcG_{\star}^{\frac{1}{2}} \|_F + \| \bcR_{\triangle} \ast_{\bPhi} \bcG_{\star}^{\frac{1}{2}} \|_F + \| \bcL_{\triangle} \ast_{\bPhi} \bcR_{\triangle}^H \|_F.
\end{align*}
This together with the following upper bound 
\begin{align*}
\| \bcL_{\triangle} \ast_{\bPhi} \bcR_{\triangle}^H \|_F & = \frac{1}{2} \| \bcL_{\triangle} \ast_{\bPhi} \bcG_{\star}^{\frac{1}{2}} \ast_{\bPhi} (\bcR_{\triangle} \ast_{\bPhi} \bcG_{\star}^{-\frac{1}{2}})^H \|_F + \frac{1}{2} \| \bcL_{\triangle} \ast_{\bPhi} \bcG_{\star}^{-\frac{1}{2}} \ast_{\bPhi} (\bcR_{\triangle} \ast_{\bPhi} \bcG_{\star}^{\frac{1}{2}})^H \|_F  \\
& = \frac{1}{2 \sqrt{\ell}} \| \widebar{\bL}_{\triangle} \widebar{\bG}_{\star}^{\frac{1}{2}} (\widebar{\bR}_{\triangle} \widebar{\bG}_{\star}^{-\frac{1}{2}})^H \|_F + \frac{1}{2 \sqrt{\ell}} \| \widebar{\bL}_{\triangle} \widebar{\bG}_{\star}^{-\frac{1}{2}} (\widebar{\bR}_{\triangle} \widebar{\bG}_{\star}^{\frac{1}{2}})^H \|_F  \\
& \leq \frac{1}{2 \sqrt{\ell}} \| \widebar{\bL}_{\triangle} \widebar{\bG}_{\star}^{\frac{1}{2}} \|_F \| \widebar{\bR}_{\triangle} \widebar{\bG}_{\star}^{-\frac{1}{2}} \| + \frac{1}{2 \sqrt{\ell}} \| \widebar{\bL}_{\triangle} \widebar{\bG}_{\star}^{-\frac{1}{2}} \| \| \widebar{\bR}_{\triangle} \widebar{\bG}_{\star}^{\frac{1}{2}} \|_F  \\
& \leq \frac{1}{2 \sqrt{\ell}} (\| \widebar{\bL}_{\triangle} \widebar{\bG}_{\star}^{-\frac{1}{2}} \| \vee \| \widebar{\bR}_{\triangle} \widebar{\bG}_{\star}^{-\frac{1}{2}} \|) \Big( \| \widebar{\bL}_{\triangle} \widebar{\bG}_{\star}^{\frac{1}{2}} \|_F + \| \widebar{\bR}_{\triangle} \widebar{\bG}_{\star}^{\frac{1}{2}} \|_F \Big)  \\
& = \frac{1}{2} (\| \bcL_{\triangle} \ast_{\bPhi} \bcG_{\star}^{-\frac{1}{2}} \| \vee \| \bcR_{\triangle} \ast_{\bPhi} \bcG_{\star}^{-\frac{1}{2}} \|) \Big( \| \bcL_{\triangle} \ast_{\bPhi} \bcG_{\star}^{\frac{1}{2}} \|_F + \| \bcR_{\triangle} \ast_{\bPhi} \bcG_{\star}^{\frac{1}{2}} \|_F \Big)
\end{align*}
finishes the proof.
\end{proof}

\subsection{Partial Frobenius Norm}

We introduce the partial Frobenius norm 
\begin{align}\label{eqn:parFnormdef}
\| \bcX \|_{F,r} \coloneq \frac{1}{\sqrt{\ell}} \sqrt{\sum_{k=1}^{n_3} \sum_{i=1}^r \xoverline{\bcG}_{i,i,k}^2} = \| \mathbb{P}_r(\bcX) \|_F
\end{align}
as the Frobenius norm of the tubal rank-$r$ approximation $\mathbb{P}_r(\bcX)$ defined in \eqref{eqn:rankrproj}. According to Lemma 28 in \citet{TongMC.JMLR2021}, it is easy to verify that $\|\cdot\|_{F,r}$ is a norm, and we have the following lemma that provides useful characterizations of the partial Frobenius norm on tensors.
\begin{lemma}\label{lemma:parFnormvariation}
For any $\bcX \in \R^{n_1 \times n_2 \times n_3}$, we have
\begin{align}
\| \bcX \|_{F,r} & = \max_{\widetilde{\bcV} \in \R^{n_2 \times r \times n_3} : \widetilde{\bcV}^H \ast_{\bPhi} \widetilde{\bcV} = \bcI_r} \| \bcX \ast_{\bPhi} \widetilde{\bcV} \|_F  \\
& = \max_{\widetilde{\bcX} \in \R^{n_1 \times n_2 \times n_3} : \| \widetilde{\bcX} \|_F \leq 1, \mathrm{rank}_t(\widetilde{\bcX}) \leq r} |\langle \bcX , \widetilde{\bcX} \rangle|  \\
& = \max_{\widetilde{\bcR} \in \R^{n_2 \times r \times n_3} : \| \widetilde{\bcR} \| \leq 1} \| \bcX \ast_{\bPhi} \widetilde{\bcR} \|_F.
\end{align}
\end{lemma}
Recall that $\mathbb{P}_r(\bcX)$ denotes the best rank-$r$ approximation of $\bX$ under the Frobenius norm, see \citet[Theorem 2.3.1]{ZhangEAHK.CVPR2014} and \citet[Theorem 3.7]{KilmerHAN.PNAS2021}. Following the proof of \citet[Lemma 30]{TongMC.JMLR2021}, we can prove that $\mathbb{P}_r(\bcX)$ is also the best tubal rank-$r$ approximation of $\bcX$ under the partial Frobenius norm $\|\cdot\|_{F,r}$, as stated below.
\begin{lemma}\label{lemma:Fnorm-Eckart-Yang}
Fix any $\bcX \in \R^{n_1 \times n_2 \times n_3}$, we have
\begin{align*}
\mathbb{P}_r(\bcX) = \argmin_{\widetilde{\bcX} \in \R^{n_1 \times n_2 \times n_3} : \mathrm{rank}_t(\widetilde{\bcX}) \leq r} \| \bcX - \widetilde{\bcX} \|_{F,r}.
\end{align*}
\end{lemma}

\section{Proof for Low-rank Tensor Factorization}
\label{sec:TFproof}

\subsection{Proof of Theorem~\ref{thm:TF}}

We prove Theorem~\ref{thm:TF} by induction. Specifically, we show that for all $t \geq 0$, ($i$) $\dist(\bcF_t, \bcF_{\star}) \leq (1 - 0.7 \eta)^t \dist(\bcF_0, \bcF_{\star}) \leq \frac{0.1}{\sqrt{\ell}} (1 - 0.7 \eta)^t \bar{\sigma}_{s_r} (\bcX_{\star})$ and ($ii$) the optimal alignment tensor $\bcQ_t$ between $\bcF_t$ and $\bcF_{\star}$ exists.

For the base case $t = 0$, the first induction hypothesis on the distance metric trivially holds, and the second one also holds because of Lemma~\ref{lemma:Qexistence} and the assumption that $\dist(\bcF_0, \bcF_{\star}) \leq \frac{0.1}{\sqrt{\ell}} \bar{\sigma}_{s_r} (\bcX_{\star})$. Suppose that the $t$-th iterate $\bcF_t$ obeys the aforementioned induction hypotheses. Our goal is to show that $\bcF_{t+1}$ also satisfies these conditions. Let $\epsilon \coloneq 0.1$. By the definition of $\dist^2(\bcF_{t+1},\bcF_{\star})$, we have
\begin{align*}
\dist^2(\bcF_{t+1},\bcF_{\star}) \leq \| (\bcL_{t+1} \ast_{\bPhi} \bcQ_t - \bcL_{\star}) \ast_{\bPhi} \bcG_{\star}^{\frac{1}{2}} \|_F^2 + \| (\bcR_{t+1} \ast_{\bPhi} \bcQ_t^{-H} - \bcR_{\star}) \ast_{\bPhi} \bcG_{\star}^{\frac{1}{2}} \|_F^2.
\end{align*}
According to the update rule \eqref{eqn:TFupdate}, we have
\begin{align}\label{eqn:TFexpand}
& (\bcL_{t+1} \ast_{\bPhi} \bcQ_t - \bcL_{\star}) \ast_{\bPhi} \bcG_{\star}^{\frac{1}{2}}  \nonumber \\
= & \Big( \bcL_{\sharp} - \eta (\bcL_{\sharp} \ast_{\bPhi} \bcR_{\sharp}^H - \bcX_{\star}) \ast_{\bPhi} \bcR_{\sharp} \ast_{\bPhi} (\bcR_{\sharp}^H \ast_{\bPhi} \bcR_{\sharp})^{-1} - \bcL_{\star} \Big) \ast_{\bPhi} \bcG_{\star}^{\frac{1}{2}}  \nonumber \\
= & \Big( \bcL_{\triangle} - \eta (\bcL_{\triangle} \ast_{\bPhi} \bcR_{\sharp}^H + \bcL_{\star} \ast_{\bPhi} \bcR_{\triangle}^H) \ast_{\bPhi} \bcR_{\sharp} \ast_{\bPhi} (\bcR_{\sharp}^H \ast_{\bPhi} \bcR_{\sharp})^{-1} \Big) \ast_{\bPhi} \bcG_{\star}^{\frac{1}{2}}  \nonumber \\
= & (1 - \eta) \bcL_{\triangle} \ast_{\bPhi} \bcG_{\star}^{\frac{1}{2}} - \eta \bcL_{\star} \ast_{\bPhi} \bcR_{\triangle}^H \ast_{\bPhi} \bcR_{\sharp} \ast_{\bPhi} (\bcR_{\sharp}^H \ast_{\bPhi} \bcR_{\sharp})^{-1} \ast_{\bPhi} \bcG_{\star}^{\frac{1}{2}}.
\end{align}
As a result, we can write \eqref{eqn:TFexpand} as $\| (\bcL_{t+1} \ast_{\bPhi} \bcQ_t - \bcL_{\star}) \ast_{\bPhi} \bcG_{\star}^{\frac{1}{2}} \|_F^2 = \frac{1}{\ell} \| (\widebar{\bL}_{t+1} \widebar{\bQ}_t - \widebar{\bL}_{\star}) \widebar{\bG}_{\star}^{\frac{1}{2}} \|_F^2$, where
\begin{align}\label{eqn:Q1}
\mathfrak{Q}_1 & \coloneq \| (\widebar{\bL}_{t+1} \widebar{\bQ}_t - \widebar{\bL}_{\star}) \widebar{\bG}_{\star}^{\frac{1}{2}} \|_F^2  \nonumber \\
& = (1 - \eta)^2 \tr( \widebar{\bL}_{\triangle} \widebar{\bG}_{\star} \widebar{\bL}_{\triangle}^H ) - 2 \eta (1 - \eta) \tr( \widebar{\bL}_{\star} \widebar{\bR}_{\triangle}^H \widebar{\bR}_{\sharp} (\widebar{\bR}_{\sharp}^H \widebar{\bR}_{\sharp})^{-1} \widebar{\bG}_{\star} \widebar{\bL}_{\triangle}^H )  \nonumber \\
&\qquad\qquad + \eta^2 \tr( \widebar{\bR}_{\sharp} (\widebar{\bR}_{\sharp}^H \widebar{\bR}_{\sharp})^{-1} \widebar{\bG}_{\star} (\widebar{\bR}_{\sharp}^H \widebar{\bR}_{\sharp})^{-1} \widebar{\bR}_{\sharp}^H \widebar{\bR}_{\triangle} \widebar{\bG}_{\star} \widebar{\bR}_{\triangle}^H ),
\end{align}
where we have used the fact $\widebar{\bL}_{\star}^H \widebar{\bL}_{\star} = \widebar{\bG}_{\star}$. Since $\bcL_{\sharp}$ and $\bcR_{\sharp}$ are aligned with $\bcL_{\star}$ and $\bcR_{\star}$, Lemma~\ref{lemma:Qcriterion} tells that $\bcG_{\star} \ast_{\bPhi} \bcL_{\triangle}^H \ast_{\bPhi} \bcL_{\sharp} = \bcR_{\sharp}^H \ast_{\bPhi} \bcR_{\triangle} \ast_{\bPhi} \bcG_{\star}$, i.e., $\widebar{\bG}_{\star} \widebar{\bL}_{\triangle}^H \widebar{\bL}_{\sharp} = \widebar{\bR}_{\sharp}^H \widebar{\bR}_{\triangle} \widebar{\bG}_{\star}$. The second term in \eqref{eqn:Q1} can be written as
\begin{align*}
& 2 \eta (1 - \eta) \tr( \widebar{\bL}_{\star} \widebar{\bR}_{\triangle}^H \widebar{\bR}_{\sharp} (\widebar{\bR}_{\sharp}^H \widebar{\bR}_{\sharp})^{-1} \widebar{\bG}_{\star} \widebar{\bL}_{\triangle}^H )  \\
= & 2 \eta (1 - \eta) \tr( \widebar{\bR}_{\sharp} (\widebar{\bR}_{\sharp}^H \widebar{\bR}_{\sharp})^{-1} \widebar{\bG}_{\star} \widebar{\bL}_{\triangle}^H \widebar{\bL}_{\star} \widebar{\bR}_{\triangle}^H )  \\
= & 2 \eta (1 - \eta) \tr( \widebar{\bR}_{\sharp} (\widebar{\bR}_{\sharp}^H \widebar{\bR}_{\sharp})^{-1} \widebar{\bG}_{\star} \widebar{\bL}_{\triangle}^H \widebar{\bL}_{\sharp} \widebar{\bR}_{\triangle}^H ) - 2 \eta (1 - \eta) \tr( \widebar{\bR}_{\sharp} (\widebar{\bR}_{\sharp}^H \widebar{\bR}_{\sharp})^{-1} \widebar{\bG}_{\star} \widebar{\bL}_{\triangle}^H \widebar{\bL}_{\triangle} \widebar{\bR}_{\triangle}^H )  \\
= & 2 \eta (1 - \eta) \tr( \widebar{\bR}_{\sharp} (\widebar{\bR}_{\sharp}^H \widebar{\bR}_{\sharp})^{-1} \widebar{\bR}_{\sharp}^H \widebar{\bR}_{\triangle} \widebar{\bG}_{\star} \widebar{\bR}_{\triangle}^H ) - 2 \eta (1 - \eta) \tr( \widebar{\bR}_{\sharp} (\widebar{\bR}_{\sharp}^H \widebar{\bR}_{\sharp})^{-1} \widebar{\bG}_{\star} \widebar{\bL}_{\triangle}^H \widebar{\bL}_{\triangle} \widebar{\bR}_{\triangle}^H ).
\end{align*}
The third term in \eqref{eqn:Q1} can be written as
\begin{align*}
& \eta^2 \tr( \widebar{\bR}_{\sharp} (\widebar{\bR}_{\sharp}^H \widebar{\bR}_{\sharp})^{-1} \widebar{\bG}_{\star} (\widebar{\bR}_{\sharp}^H \widebar{\bR}_{\sharp})^{-1} \widebar{\bR}_{\sharp}^H \widebar{\bR}_{\triangle} \widebar{\bG}_{\star} \widebar{\bR}_{\triangle}^H )  \\
= & \eta^2 \tr( \widebar{\bR}_{\sharp} (\widebar{\bR}_{\sharp}^H \widebar{\bR}_{\sharp})^{-1} \widebar{\bR}_{\sharp}^H \widebar{\bR}_{\triangle} \widebar{\bG}_{\star} \widebar{\bR}_{\triangle}^H )  \\
&\qquad - \eta^2 \tr( \widebar{\bR}_{\sharp} (\widebar{\bR}_{\sharp}^H \widebar{\bR}_{\sharp})^{-1} (\widebar{\bR}_{\sharp}^H \widebar{\bR}_{\sharp} - \widebar{\bG}_{\star}) (\widebar{\bR}_{\sharp}^H \widebar{\bR}_{\sharp})^{-1} \widebar{\bR}_{\sharp}^H \widebar{\bR}_{\triangle} \widebar{\bG}_{\star} \widebar{\bR}_{\triangle}^H ).
\end{align*}
Thus, \eqref{eqn:Q1} can be written in another form as
\begin{align*}
\mathfrak{Q}_1 & = (1 - \eta)^2 \tr( \widebar{\bL}_{\triangle} \widebar{\bG}_{\star} \widebar{\bL}_{\triangle}^H ) - \eta (2 - 3 \eta) \tr( \widebar{\bR}_{\sharp} (\widebar{\bR}_{\sharp}^H \widebar{\bR}_{\sharp})^{-1} \widebar{\bR}_{\sharp}^H \widebar{\bR}_{\triangle} \widebar{\bG}_{\star} \widebar{\bR}_{\triangle}^H )  \\
&\qquad\qquad + 2 \eta (1 - \eta) \tr( \widebar{\bR}_{\sharp} (\widebar{\bR}_{\sharp}^H \widebar{\bR}_{\sharp})^{-1} \widebar{\bG}_{\star} \widebar{\bL}_{\triangle}^H \widebar{\bL}_{\triangle} \widebar{\bR}_{\triangle}^H )  \\
&\qquad\qquad - \eta^2 \tr( \widebar{\bR}_{\sharp} (\widebar{\bR}_{\sharp}^H \widebar{\bR}_{\sharp})^{-1} (\widebar{\bR}_{\sharp}^H \widebar{\bR}_{\sharp} - \widebar{\bG}_{\star}) (\widebar{\bR}_{\sharp}^H \widebar{\bR}_{\sharp})^{-1} \widebar{\bR}_{\sharp}^H \widebar{\bR}_{\triangle} \widebar{\bG}_{\star} \widebar{\bR}_{\triangle}^H ).
\end{align*}
Following the proof in \citet[Section B.2]{TongMC.JMLR2021}, we can bound this term by
\begin{align*}
\mathfrak{Q}_1 & \leq \Big( (1 - \eta)^2 + \frac{2 \epsilon}{1 - \epsilon} \eta (1 - \eta) \Big) \tr( \widebar{\bL}_{\triangle} \widebar{\bG}_{\star} \widebar{\bL}_{\triangle}^H ) + \frac{2 \epsilon + \epsilon^2}{(1 - \epsilon)^2} \eta^2 \tr ( \widebar{\bR}_{\triangle} \widebar{\bG}_{\star} \widebar{\bR}_{\triangle}^H )  \\
& = \Big( (1 - \eta)^2 + \frac{2 \epsilon}{1 - \epsilon} \eta (1 - \eta) \Big) \| \widebar{\bL}_{\triangle} \widebar{\bG}_{\star}^{\frac{1}{2}} \|_F^2 + \frac{2 \epsilon + \epsilon^2}{(1 - \epsilon)^2} \eta^2 \| \widebar{\bR}_{\triangle} \widebar{\bG}_{\star}^{\frac{1}{2}} \|_F^2.
\end{align*}
Hence,
\begin{align*}
\| (\bcL_{t+1} \ast_{\bPhi} \bcQ_t - \bcL_{\star}) \ast_{\bPhi} \bcG_{\star}^{\frac{1}{2}} \|_F^2 \leq \Big( (1 - \eta)^2 + \frac{2 \epsilon}{1 - \epsilon} \eta (1 - \eta) \Big) \| \bcL_{\triangle} \ast_{\bPhi} \bcG_{\star}^{\frac{1}{2}} \|_F^2 + \frac{2 \epsilon + \epsilon^2}{(1 - \epsilon)^2} \eta^2 \| \bcR_{\triangle} \ast_{\bPhi} \bcG_{\star}^{\frac{1}{2}} \|_F^2.
\end{align*}
One can have a similar bound for the second term $\| (\bcR_{t+1} \ast_{\bPhi} \bcQ_t^{-H} - \bcR_{\star}) \ast_{\bPhi} \bcG_{\star}^{\frac{1}{2}} \|_F^2$. Therefore we obtain
\begin{align*}
\| (\bcL_{t+1} \ast_{\bPhi} \bcQ_t - \bcL_{\star}) \ast_{\bPhi} \bcG_{\star}^{\frac{1}{2}} \|_F^2 + \| (\bcR_{t+1} \ast_{\bPhi} \bcQ_t^{-H} - \bcR_{\star}) \ast_{\bPhi} \bcG_{\star}^{\frac{1}{2}} \|_F^2 \leq \hbar^2 (\eta;\epsilon) \dist^2(\bcF_t, \bcF_{\star}),
\end{align*}
where the contraction rate is given by
\begin{align*}
\hbar^2 (\eta;\epsilon) \coloneq (1 - \eta)^2 + \frac{2 \epsilon}{1 - \epsilon} \eta (1 - \eta) + \frac{2 \epsilon + \epsilon^2}{(1 - \epsilon)^2} \eta^2.
\end{align*}
With $\epsilon = 0.1$ and $0 < \eta \leq \frac{2}{3}$, we have $\hbar (\eta;\epsilon) \leq 1 - 0.7 \eta$. Thus we conclude that 
\begin{align*}
\dist(\bcF_{t+1}, \bcF_{\star}) & \leq \sqrt{\| (\bcL_{t+1} \ast_{\bPhi} \bcQ_t - \bcL_{\star}) \ast_{\bPhi} \bcG_{\star}^{\frac{1}{2}} \|_F^2 + \| (\bcR_{t+1} \ast_{\bPhi} \bcQ_t^{-H} - \bcR_{\star}) \ast_{\bPhi} \bcG_{\star}^{\frac{1}{2}} \|_F^2}  \\
& \leq (1 - 0.7 \eta) \dist(\bcF_t, \bcF_{\star})  \\
& \leq (1 - 0.7 \eta)^{t+1} \dist(\bcF_0, \bcF_{\star}) \leq (1 - 0.7 \eta)^{t+1} \frac{0.1}{\sqrt{\ell}} \bar{\sigma}_{s_r} (\bcX_{\star}).
\end{align*}
This proves the first induction hypothesis. The existence of the optimal alignment tensor $\bcQ_{t+1}$ between $\bcF_{t+1}$ and $\bcF_{\star}$ is assured by Lemma~\ref{lemma:Qexistence}, which finishes the proof for the second hypothesis.

The second conclusion is an easy consequence of Lemma~\ref{lemma:tensor2factor} as
\begin{align}\label{eqn:disttensor}
\| \bcL_t \ast_{\bPhi} \bcR_t^H - \bcX_{\star} \|_F & \leq (1 + \frac{\epsilon}{2}) ( \| \bcL_{\triangle} \ast_{\bPhi} \bcG_{\star}^{\frac{1}{2}} \|_F + \| \bcR_{\triangle} \ast_{\bPhi} \bcG_{\star}^{\frac{1}{2}} \|_F )  \nonumber \\
& \leq (1 + \frac{\epsilon}{2}) \sqrt{2} \dist(\bcF_t, \bcF_{\star})  \nonumber \\
& \leq 1.5 \dist(\bcF_t, \bcF_{\star}),
\end{align}
where we used $a + b \leq \sqrt{2 (a^2 + b^2)}$ in the second line. The proof is now completed.

\section{Proof for Tensor RPCA}
\label{sec:TRPCAproof}

Before presenting the proof of Lemma~\ref{lemma:TRPCAcontraction}, we present several auxiliary lemmas.
\begin{lemma}\label{lemma:Deltanorm}
If
\begin{align*} 
\dist(\bcF_t, \bcF_{\star}) & \leq \frac{\epsilon}{\sqrt{\ell}} \rho^t \bar{\sigma}_{s_r} (\bcX_{\star})
\end{align*}
for some $0 < \epsilon < 1$, then the following inequalities hold
\begin{align*}
\| \bcL_{\triangle} \ast_{\bPhi} \bcG_{\star}^{\frac{1}{2}} \|_F & \vee \| \bcR_{\triangle} \ast_{\bPhi} \bcG_{\star}^{\frac{1}{2}} \|_F \leq \frac{\epsilon}{\sqrt{\ell}} \rho^t \bar{\sigma}_{s_r} (\bcX_{\star});  \\
\| \bcL_{\triangle} \ast_{\bPhi} \bcG_{\star}^{\frac{1}{2}} \| & \vee \| \bcR_{\triangle} \ast_{\bPhi} \bcG_{\star}^{\frac{1}{2}} \| \leq \epsilon \rho^t \bar{\sigma}_{s_r} (\bcX_{\star});  \\
\| \bcL_{\sharp} \ast_{\bPhi} (\bcL_{\sharp}^H \ast_{\bPhi} \bcL_{\sharp})^{-1} \ast_{\bPhi} \bcG_{\star}^{\frac{1}{2}} \| & \vee \| \bcR_{\sharp} \ast_{\bPhi} (\bcR_{\sharp}^H \ast_{\bPhi} \bcR_{\sharp})^{-1} \ast_{\bPhi} \bcG_{\star}^{\frac{1}{2}} \| \leq \frac{1}{1 - \epsilon}.
\end{align*}
\end{lemma}

\begin{proof}
Notice that $\dist(\bcF_t, \bcF_{\star}) \geq \sqrt{\| \bcL_{\triangle} \ast_{\bPhi} \bcG_{\star}^{\frac{1}{2}} \|_F^2 \vee \| \bcR_{\triangle} \ast_{\bPhi} \bcG_{\star}^{\frac{1}{2}} \|_F^2}$. The first claim is directly followed by the definition of $\dist(\cdot, \cdot)$. By the fact that $\| \bcA \| = \| \widebar{\bA} \| \leq \| \widebar{\bA} \|_F = \| \xoverline{\bcA} \|_F = \sqrt{\ell} \| \bcA \|_F$ for any tensor, the second claim can be deduced from the first claim. As a consequence, we have $\| \bcL_{\triangle} \ast_{\bPhi} \bcG_{\star}^{-\frac{1}{2}} \| \vee \| \bcR_{\triangle} \ast_{\bPhi} \bcG_{\star}^{-\frac{1}{2}} \| \leq \epsilon \rho^t \leq \epsilon$, given $\rho = 1 - 0.6 \eta < 1$, then the third claim can be obtained by applying \eqref{eqn:Weyl-1L} and \eqref{eqn:Weyl-1R}.
\end{proof}
\begin{lemma}\label{lemma:Qperturbation}
If 
\begin{align*}
\| (\bcL_{t+1} \ast_{\bPhi} \bcQ_t - \bcL_{\star}) \ast_{\bPhi} \bcG_{\star}^{\frac{1}{2}} \| \vee \| (\bcR_{t+1} \ast_{\bPhi} \bcQ_t^{-H} - \bcR_{\star}) \ast_{\bPhi} \bcG_{\star}^{\frac{1}{2}} \| \leq \epsilon \rho^{t+1} \bar{\sigma}_{s_r} (\bcX_{\star}),
\end{align*}
then
\begin{align*}
& \| \bcG_{\star}^{\frac{1}{2}} \ast_{\bPhi} \bcQ_t^{-1} \ast_{\bPhi} (\bcQ_{t+1} - \bcQ_t) \ast_{\bPhi} \bcG_{\star}^{\frac{1}{2}} \| \vee \| \bcG_{\star}^{\frac{1}{2}} \ast_{\bPhi} \bcQ_t^H \ast_{\bPhi} (\bcQ_{t+1} - \bcQ_t)^{-H} \ast_{\bPhi} \bcG_{\star}^{\frac{1}{2}} \|  \\
& \leq \frac{2 \epsilon}{1 - \epsilon} \rho^{t+1} \bar{\sigma}_{s_r} (\bcX_{\star}).
\end{align*}
\end{lemma}

\begin{proof}
First, following the proof of \citet[Lemma 27]{TongMC.JMLR2021}, we have the following inequalities:
\begin{align*}
\| \bcG_{\star}^{\frac{1}{2}} \ast_{\bPhi} \widetilde{\bcQ}^{-1} \ast_{\bPhi} \bcQ \ast_{\bPhi} \bcG_{\star}^{\frac{1}{2}} - \bcG_{\star} \| & \leq \frac{\| \bcR_{\sharp} \ast_{\bPhi} (\widetilde{\bcQ}^{-H} - \bcQ^{-H}) \ast_{\bPhi} \bcG_{\star}^{\frac{1}{2}} \|}{1 - \| (\bcR_{\sharp} \ast_{\bPhi} \bcQ^{-H} - \bcR_{\star}) \ast_{\bPhi} \bcG_{\star}^{-\frac{1}{2}} \|},  \\
\| \bcG_{\star}^{\frac{1}{2}} \ast_{\bPhi} \widetilde{\bcQ}^H \ast_{\bPhi} \bcQ^{-H} \ast_{\bPhi} \bcG_{\star}^{\frac{1}{2}} - \bcG_{\star} \| & \leq \frac{\| \bcL_{\sharp} \ast_{\bPhi} (\widetilde{\bcQ} - \bcQ) \ast_{\bPhi} \bcG_{\star}^{\frac{1}{2}} \|}{1 - \| (\bcL_{\sharp} \ast_{\bPhi} \bcQ - \bcL_{\star}) \ast_{\bPhi} \bcG_{\star}^{-\frac{1}{2}} \|}
\end{align*}
for any $\bcL_{\sharp} \in \R^{n_1 \times r \times n_3}$, $\bcR_{\sharp} \in \R^{n_2 \times r \times n_3}$ and any invertible tensors $\bcQ, \widetilde{\bcQ} \in \mathrm{GL}(r)$, as long as $\| (\bcL_{\sharp} \ast_{\bPhi} \bcQ - \bcL_{\star}) \ast_{\bPhi} \bcG_{\star}^{-\frac{1}{2}} \| \vee \| (\bcR_{\sharp} \ast_{\bPhi} \bcQ^{-H} - \bcR_{\star}) \ast_{\bPhi} \bcG_{\star}^{-\frac{1}{2}} \| < 1$.

Next, notice that $\| (\bcR_{t+1} \ast_{\bPhi} \bcQ_{t+1}^{-H} - \bcR_{\star}) \ast_{\bPhi} \bcG_{\star}^{\frac{1}{2}} \| \leq \epsilon \rho^{t+1} \bar{\sigma}_{s_r} (\bcX_{\star})$, i.e., $\| (\bcR_{t+1} \ast_{\bPhi} \bcQ_{t+1}^{-H} - \bcR_{\star}) \ast_{\bPhi} \bcG_{\star}^{-\frac{1}{2}} \| \leq \epsilon \rho^{t+1}$. Thus, by taking $\bcR_{\sharp} = \bcR_{t+1}$, $\bcQ = \bcQ_{t+1}$, and $\widetilde{\bcQ} = \bcQ_t$, we obtain
\begin{align*}
&\| \bcG_{\star}^{\frac{1}{2}} \ast_{\bPhi} \bcQ_t^{-1} \ast_{\bPhi} (\bcQ_{t+1} - \bcQ_t) \ast_{\bPhi} \bcG_{\star}^{\frac{1}{2}} \|  \\
= & \| \bcG_{\star}^{\frac{1}{2}} \ast_{\bPhi} \bcQ_t^{-1} \ast_{\bPhi} \bcQ_{t+1} \ast_{\bPhi} \bcG_{\star}^{\frac{1}{2}} - \bcG_{\star} \|  \\
\leq & \frac{\| \bcR_{t+1} \ast_{\bPhi} (\bcQ_t^{-H} - \bcQ_{t+1}^{-H}) \ast_{\bPhi} \bcG_{\star}^{\frac{1}{2}} \|}{1 - \| (\bcR_{t+1} \ast_{\bPhi} \bcQ_{t+1}^{-H} - \bcR_{\star}) \ast_{\bPhi} \bcG_{\star}^{-\frac{1}{2}} \|}  \\
\leq & \frac{\| (\bcR_{t+1} \ast_{\bPhi} \bcQ_t^{-H} - \bcR_{\star}) \ast_{\bPhi} \bcG_{\star}^{\frac{1}{2}} \| + \| (\bcR_{t+1} \ast_{\bPhi} \bcQ_{t+1}^{-H} - \bcR_{\star}) \ast_{\bPhi} \bcG_{\star}^{\frac{1}{2}} \|}{1 - \| (\bcR_{t+1} \ast_{\bPhi} \bcQ_{t+1}^{-H} - \bcR_{\star}) \ast_{\bPhi} \bcG_{\star}^{-\frac{1}{2}} \|}  \\
\leq & \frac{2 \epsilon \rho^{t+1}}{1 - \epsilon \rho^{t+1}} \bar{\sigma}_{s_r} (\bcX_{\star})  \\
\leq & \frac{2 \epsilon}{1 - \epsilon} \rho^{t+1} \bar{\sigma}_{s_r} (\bcX_{\star}),
\end{align*}
provided $\rho = 1 - 0.6 \eta < 1$. Similarly, one can see
\begin{align*}
\| \bcG_{\star}^{\frac{1}{2}} \ast_{\bPhi} \bcQ_t^H \ast_{\bPhi} (\bcQ_{t+1} - \bcQ_t)^{-H} \ast_{\bPhi} \bcG_{\star}^{\frac{1}{2}} \| \leq \frac{2 \epsilon}{1 - \epsilon} \rho^{t+1} \bar{\sigma}_{s_r} (\bcX_{\star}).
\end{align*}
This finishes the proof.
\end{proof}
\begin{lemma}\label{lemma:Xinfnorm}
If
\begin{align*}
\dist(\bcF_t, \bcF_{\star}) & \leq \frac{\epsilon}{\sqrt{\ell}} \rho^t \bar{\sigma}_{s_r} (\bcX_{\star}) \quad \mathrm{and}  \\
\sqrt{n_1} \| \bcL_{\triangle} \ast_{\bPhi} \bcG_{\star}^{\frac{1}{2}} \|_{2,\infty} & \vee \sqrt{n_2} \| \bcR_{\triangle} \ast_{\bPhi} \bcG_{\star}^{\frac{1}{2}} \|_{2,\infty} \leq \sqrt{\frac{\mu s_r}{n_3 \ell}} \rho^t \bar{\sigma}_{s_r} (\bcX_{\star}),
\end{align*}
then
\begin{align*}
\| \bcX_{\star} - \bcX_t \|_{\infty} \leq 3 \frac{\mu s_r}{n_3 \sqrt{n_1 n_2 \ell}} \rho^t \bar{\sigma}_{s_r} (\bcX_{\star}).
\end{align*}
\end{lemma}

\begin{proof}
First, based on Definition~\ref{def:incoherence} and the assumption of this lemma, we have
\begin{align*}
\| \bcR_{\sharp} \ast_{\bPhi} \bcG_{\star}^{-\frac{1}{2}} \|_{2,\infty} & \leq \| \bcR_{\triangle} \ast_{\bPhi} \bcG_{\star}^{\frac{1}{2}} \|_{2,\infty} \| \bcG_{\star}^{-1} \| + \| \bcR_{\star} \ast_{\bPhi} \bcG_{\star}^{-\frac{1}{2}} \|_{2,\infty}  \\
& \leq (\rho^t + 1) \sqrt{\frac{\mu s_r}{n_2 n_3 \ell}} \leq 2 \sqrt{\frac{\mu s_r}{n_2 n_3 \ell}}.
\end{align*}
Furthermore, using Lemma~\ref{lemma:infnormbound}, we obtain
\begin{align*}
\| \bcX_{\star} - \bcX_t \|_{\infty} & = \| \bcL_{\triangle} \ast_{\bPhi} \bcR_{\sharp}^H + \bcL_{\star} \ast_{\bPhi} \bcR_{\triangle}^H \|_{\infty} \leq \| \bcL_{\triangle} \ast_{\bPhi} \bcR_{\sharp}^H \|_{\infty} + \| \bcL_{\star} \ast_{\bPhi} \bcR_{\triangle}^H \|_{\infty}  \\
& \leq \sqrt{\ell} \| \bcL_{\triangle} \ast_{\bPhi} \bcG_{\star}^{\frac{1}{2}} \|_{2,\infty} \| \bcR_{\sharp} \ast_{\bPhi} \bcG_{\star}^{-\frac{1}{2}} \|_{2,\infty} + \sqrt{\ell} \| \bcL_{\star} \ast_{\bPhi} \bcG_{\star}^{-\frac{1}{2}} \|_{2,\infty} \| \bcR_{\triangle} \ast_{\bPhi} \bcG_{\star}^{\frac{1}{2}} \|_{2,\infty}  \\
& \leq \sqrt{\ell} (2 \sqrt{\frac{\mu s_r}{n_2 n_3 \ell}} + \sqrt{\frac{\mu s_r}{n_2 n_3 \ell}} ) \sqrt{\frac{\mu s_r}{n_1 n_3 \ell}} \rho^t \bar{\sigma}_{s_r} (\bcX_{\star})  \\
& = 3 \frac{\mu s_r}{n_3 \sqrt{n_1 n_2 \ell}} \rho^t \bar{\sigma}_{s_r} (\bcX_{\star}).
\end{align*}
This finishes the proof.
\end{proof}

\subsection{Proof of Lemma~\ref{lemma:TRPCAcontraction}}

This proof is done by induction.

\paragraph{Base case.} Since $\rho^0 = 1$, the assumed initial conditions satisfy the base case at $t = 0$.

\paragraph{Induction step.} At the $t$-th iteration, we assume the conditions
\begin{align*}
\dist(\bcF_t, \bcF_{\star}) & \leq \frac{\epsilon}{\sqrt{\ell}} \rho^t \bar{\sigma}_{s_r} (\bcX_{\star}),  \\
\sqrt{n_1} \| (\bcL_t \ast_{\bPhi} \bcQ_t - \bcL_{\star}) \ast_{\bPhi} \bcG_{\star}^{\frac{1}{2}} \|_{2,\infty} & \vee \sqrt{n_2} \| (\bcR_t \ast_{\bPhi} \bcQ_t^{-H} - \bcR_{\star}) \ast_{\bPhi} \bcG_{\star}^{\frac{1}{2}} \|_{2,\infty} \leq \sqrt{\frac{\mu s_r}{n_3 \ell}} \rho^t \bar{\sigma}_{s_r} (\bcX_{\star})
\end{align*}
hold. In view of the condition $\dist(\bcF_t, \bcF_{\star}) \leq \frac{0.02}{\sqrt{\ell}} \rho^t \bar{\sigma}_{s_r} (\bcX_{\star})$ and Lemma~\ref{lemma:Qexistence}, one knows that $\bcQ_t$, the optimal alignment tensor between $\bcF_t$ and $\bcF_{\star}$ exists, and $\epsilon \coloneq 0.02$. In what follows, we shall prove the distance contraction and the incoherence condition separately.

\emph{Distance contraction:} By the definition of $\dist^2(\bcF_{t+1},\bcF_{\star})$, we have
\begin{align*}
\dist^2(\bcF_{t+1},\bcF_{\star}) \leq \| (\bcL_{t+1} \ast_{\bPhi} \bcQ_t - \bcL_{\star}) \ast_{\bPhi} \bcG_{\star}^{\frac{1}{2}} \|_F^2 + \| (\bcR_{t+1} \ast_{\bPhi} \bcQ_t^{-H} - \bcR_{\star}) \ast_{\bPhi} \bcG_{\star}^{\frac{1}{2}} \|_F^2.
\end{align*}
According to the update rule \eqref{eqn:TRPCAupdate}, we have
\begin{align}\label{eqn:TRPCAexpand}
& (\bcL_{t+1} \ast_{\bPhi} \bcQ_t - \bcL_{\star}) \ast_{\bPhi} \bcG_{\star}^{\frac{1}{2}}  \nonumber \\
= & \Big( \bcL_{\sharp} - \eta (\bcL_{\sharp} \ast_{\bPhi} \bcR_{\sharp}^H + \bcS_{t+1} - \bcX_{\star} - \bcS_{\star}) \ast_{\bPhi} \bcR_{\sharp} \ast_{\bPhi} (\bcR_{\sharp}^H \ast_{\bPhi} \bcR_{\sharp})^{-1} - \bcL_{\star} \Big) \ast_{\bPhi} \bcG_{\star}^{\frac{1}{2}}  \nonumber \\
= & \Big( \bcL_{\triangle} - \eta (\bcL_{\sharp} \ast_{\bPhi} \bcR_{\sharp}^H - \bcX_{\star}) \ast_{\bPhi} \bcR_{\sharp} \ast_{\bPhi} (\bcR_{\sharp}^H \ast_{\bPhi} \bcR_{\sharp})^{-1}  \nonumber \\
&\qquad\qquad\qquad\qquad\qquad\qquad - \eta (\bcS_{t+1} - \bcS_{\star}) \ast_{\bPhi} \bcR_{\sharp} \ast_{\bPhi} (\bcR_{\sharp}^H \ast_{\bPhi} \bcR_{\sharp})^{-1} \Big) \ast_{\bPhi} \bcG_{\star}^{\frac{1}{2}}  \nonumber \\
= & \Big( \bcL_{\triangle} - \eta (\bcL_{\triangle} \ast_{\bPhi} \bcR_{\sharp}^H + \bcL_{\star} \ast_{\bPhi} \bcR_{\triangle}^H) \ast_{\bPhi} \bcR_{\sharp} \ast_{\bPhi} (\bcR_{\sharp}^H \ast_{\bPhi} \bcR_{\sharp})^{-1}  \nonumber \\
&\qquad\qquad\qquad\qquad\qquad\qquad - \eta \bcS_{\triangle} \ast_{\bPhi} \bcR_{\sharp} \ast_{\bPhi} (\bcR_{\sharp}^H \ast_{\bPhi} \bcR_{\sharp})^{-1} \Big) \ast_{\bPhi} \bcG_{\star}^{\frac{1}{2}}  \nonumber \\
= & (1 - \eta) \bcL_{\triangle} \ast_{\bPhi} \bcG_{\star}^{\frac{1}{2}} - \eta \bcL_{\star} \ast_{\bPhi} \bcR_{\triangle}^H \ast_{\bPhi} \bcR_{\sharp} \ast_{\bPhi} (\bcR_{\sharp}^H \ast_{\bPhi} \bcR_{\sharp})^{-1} \ast_{\bPhi} \bcG_{\star}^{\frac{1}{2}}  \nonumber \\
&\qquad\qquad\qquad\qquad\qquad\qquad - \eta \bcS_{\triangle} \ast_{\bPhi} \bcR_{\sharp} \ast_{\bPhi} (\bcR_{\sharp}^H \ast_{\bPhi} \bcR_{\sharp})^{-1} \ast_{\bPhi} \bcG_{\star}^{\frac{1}{2}}.
\end{align}
Taking the squared Frobenius norm of both sides of \eqref{eqn:TRPCAexpand} to obtain
\begin{align*}
\| (\bcL_{t+1} \ast_{\bPhi} \bcQ_t - \bcL_{\star}) \ast_{\bPhi} \bcG_{\star}^{\frac{1}{2}} \|_F^2 = & \frac{1}{\ell} \| (\widebar{\bL}_{t+1} \widebar{\bQ}_t - \widebar{\bL}_{\star}) \widebar{\bG}_{\star}^{\frac{1}{2}} \|_F^2  \\
= & \frac{1}{\ell} \Big( \| (1 - \eta) \widebar{\bL}_{\triangle} \widebar{\bG}_{\star}^{\frac{1}{2}} - \eta \widebar{\bL}_{\star} \widebar{\bR}_{\triangle}^H \widebar{\bR}_{\sharp} (\widebar{\bR}_{\sharp}^H \widebar{\bR}_{\sharp})^{-1} \widebar{\bG}_{\star}^{\frac{1}{2}} \|_F^2  \\
&\quad - 2 \eta (1 - \eta) \tr \Big( \widebar{\bS}_{\triangle} \widebar{\bR}_{\sharp} (\widebar{\bR}_{\sharp}^H \widebar{\bR}_{\sharp})^{-1} \widebar{\bG}_{\star} \widebar{\bL}_{\triangle}^H \Big)  \\
&\quad + 2 \eta^2 \tr \Big( \widebar{\bS}_{\triangle} \widebar{\bR}_{\sharp} (\widebar{\bR}_{\sharp}^H \widebar{\bR}_{\sharp})^{-1} \widebar{\bG}_{\star} (\widebar{\bR}_{\sharp}^H \widebar{\bR}_{\sharp})^{-1} \widebar{\bR}_{\sharp}^H \widebar{\bR}_{\triangle} \widebar{\bL}_{\star}^H \Big)  \\
&\quad + \eta^2 \| \widebar{\bS}_{\triangle} \widebar{\bR}_{\sharp} (\widebar{\bR}_{\sharp}^H \widebar{\bR}_{\sharp})^{-1} \widebar{\bG}_{\star}^{\frac{1}{2}} \|_F^2 \Big)  \\
\coloneq & \frac{1}{\ell} (\mathfrak{R}_1 - \mathfrak{R}_2 + \mathfrak{R}_3 + \mathfrak{R}_4).
\end{align*}

\paragraph{Bound of $\mathfrak{R}_1$.} The component $\mathfrak{R}_1$ is identical to \eqref{eqn:Q1}, and the bound of this term was shown therein. We will clear this bound further by applying Lemma~\ref{lemma:Deltanorm}, that is
\begin{align}\label{eqn:TF-Lt-bound}
\mathfrak{R}_1 & \leq \Big( (1 - \eta)^2 + \frac{2 \epsilon}{1 - \epsilon} \eta (1 - \eta) \Big) \tr( \widebar{\bL}_{\triangle} \widebar{\bG}_{\star} \widebar{\bL}_{\triangle}^H ) + \frac{2 \epsilon + \epsilon^2}{(1 - \epsilon)^2} \eta^2 \tr ( \widebar{\bR}_{\triangle} \widebar{\bG}_{\star} \widebar{\bR}_{\triangle}^H )  \nonumber \\
& = (1 - \eta)^2 \| \widebar{\bL}_{\triangle} \widebar{\bG}_{\star}^{\frac{1}{2}} \|_F^2 + \frac{2 \epsilon}{1 - \epsilon} \eta (1 - \eta) \| \widebar{\bL}_{\triangle} \widebar{\bG}_{\star}^{\frac{1}{2}} \|_F^2 + \frac{2 \epsilon + \epsilon^2}{(1 - \epsilon)^2} \eta^2 \| \widebar{\bR}_{\triangle} \widebar{\bG}_{\star}^{\frac{1}{2}} \|_F^2  \nonumber \\
& = (1 - \eta)^2 \| \widebar{\bL}_{\triangle} \widebar{\bG}_{\star}^{\frac{1}{2}} \|_F^2 + \Big( (1 - \eta) \frac{2 \epsilon^3}{1 - \epsilon} + \eta \frac{2 \epsilon^3 + \epsilon^4}{(1 - \epsilon)^2} \Big) \eta \rho^{2t} \bar{\sigma}_{s_r}^2 (\bcX_{\star}).
\end{align}

\paragraph{Bound of $\mathfrak{R}_2$.} According to Lemma~\ref{lemma:normbound}, $\| \widebar{\bS}_{\triangle} \| = \| \bcS_{\triangle} \| \leq \frac{\alpha \sqrt{\ell}}{2} (n_1 + n_2 n_3) \| \bcS_{\triangle} \|_{\infty}$. Lemma~\ref{lemma:sparity} implies $\bcS_{\triangle}$ is an $\alpha$-sparse tensor and our threshold $\zeta_{t+1}$ suggests that $\| \bcS_{\triangle} \|_{\infty} \leq 6 \frac{\mu s_r}{n_3 \sqrt{n_1 n_2 \ell}} \rho^t \bar{\sigma}_{s_r} (\bcX_{\star})$, we further have
\begin{align*}
& \Big| \tr \Big( \widebar{\bS}_{\triangle} \widebar{\bR}_{\sharp} (\widebar{\bR}_{\sharp}^H \widebar{\bR}_{\sharp})^{-1} \widebar{\bG}_{\star} \widebar{\bL}_{\triangle}^H \Big) \Big|  \\
\leq & \| \widebar{\bS}_{\triangle} \| \| \widebar{\bR}_{\sharp} (\widebar{\bR}_{\sharp}^H \widebar{\bR}_{\sharp})^{-1} \widebar{\bG}_{\star} \widebar{\bL}_{\triangle}^H \|_{\ast}  \\
\leq & \frac{\alpha \sqrt{\ell}}{2} (n_1 + n_2 n_3) \| \bcS_{\triangle} \|_{\infty} \sqrt{s_r} \| \widebar{\bR}_{\sharp} (\widebar{\bR}_{\sharp}^H \widebar{\bR}_{\sharp})^{-1} \widebar{\bG}_{\star} \widebar{\bL}_{\triangle}^H \|_F  \\
\leq & \frac{\alpha \sqrt{\ell}}{2} (n_1 + n_2 n_3) \| \bcS_{\triangle} \|_{\infty} \sqrt{s_r} \| \widebar{\bR}_{\sharp} (\widebar{\bR}_{\sharp}^H \widebar{\bR}_{\sharp})^{-1} \widebar{\bG}_{\star}^{\frac{1}{2}} \| \| \widebar{\bL}_{\triangle} \widebar{\bG}_{\star}^{\frac{1}{2}} \|_F  \\
\leq & 3 \alpha \mu s_r^{1.5} \frac{n_1 + n_2 n_3}{n_3 \sqrt{n_1 n_2}} \frac{\epsilon}{1 - \epsilon} \rho^{2t} \bar{\sigma}_{s_r}^2 (\bcX_{\star}).
\end{align*}
Hence,
\begin{align*}
|\mathfrak{R}_2| \leq 6 \eta (1 - \eta) \alpha \mu s_r^{1.5} \frac{n_1 + n_2 n_3}{n_3 \sqrt{n_1 n_2}} \frac{\epsilon}{1 - \epsilon} \rho^{2t} \bar{\sigma}_{s_r}^2 (\bcX_{\star}).
\end{align*}

\paragraph{Bound of $\mathfrak{R}_3$.} Similar to $\mathfrak{R}_2$, we have
\begin{align*}
& \Big| \tr \Big( \widebar{\bS}_{\triangle} \widebar{\bR}_{\sharp} (\widebar{\bR}_{\sharp}^H \widebar{\bR}_{\sharp})^{-1} \widebar{\bG}_{\star} (\widebar{\bR}_{\sharp}^H \widebar{\bR}_{\sharp})^{-1} \widebar{\bR}_{\sharp}^H \widebar{\bR}_{\triangle} \widebar{\bL}_{\star}^H \Big) \Big|  \\
\leq & \| \widebar{\bS}_{\triangle} \| \| \widebar{\bR}_{\sharp} (\widebar{\bR}_{\sharp}^H \widebar{\bR}_{\sharp})^{-1} \widebar{\bG}_{\star} (\widebar{\bR}_{\sharp}^H \widebar{\bR}_{\sharp})^{-1} \widebar{\bR}_{\sharp}^H \widebar{\bR}_{\triangle} \widebar{\bL}_{\star}^H \|_{\ast}  \\
\leq & \frac{\alpha \sqrt{\ell}}{2} (n_1 + n_2 n_3) \| \bcS_{\triangle} \|_{\infty} \sqrt{s_r} \| \widebar{\bR}_{\sharp} (\widebar{\bR}_{\sharp}^H \widebar{\bR}_{\sharp})^{-1} \widebar{\bG}_{\star} (\widebar{\bR}_{\sharp}^H \widebar{\bR}_{\sharp})^{-1} \widebar{\bR}_{\sharp}^H \widebar{\bR}_{\triangle} \widebar{\bL}_{\star}^H \|_F  \\
\leq & \frac{\alpha \sqrt{\ell}}{2} (n_1 + n_2 n_3) \| \bcS_{\triangle} \|_{\infty} \sqrt{s_r} \| \widebar{\bR}_{\sharp} (\widebar{\bR}_{\sharp}^H \widebar{\bR}_{\sharp})^{-1} \widebar{\bG}_{\star}^{\frac{1}{2}} \|^2 \| \widebar{\bR}_{\triangle} \widebar{\bL}_{\star}^H \|_F  \\
\leq & \frac{\alpha \sqrt{\ell}}{2} (n_1 + n_2 n_3) \| \bcS_{\triangle} \|_{\infty} \sqrt{s_r} \| \widebar{\bR}_{\sharp} (\widebar{\bR}_{\sharp}^H \widebar{\bR}_{\sharp})^{-1} \widebar{\bG}_{\star}^{\frac{1}{2}} \|^2 \| \widebar{\bR}_{\triangle} \widebar{\bG}_{\star}^{\frac{1}{2}} \|_F \| \widebar{\bU}_{\star} \|  \\
\leq & 3 \alpha \mu s_r^{1.5} \frac{n_1 + n_2 n_3}{n_3 \sqrt{n_1 n_2}} \frac{\epsilon}{(1 - \epsilon)^2} \rho^{2t} \bar{\sigma}_{s_r}^2 (\bcX_{\star}).
\end{align*}
Hence,
\begin{align*}
|\mathfrak{R}_3| \leq 6 \eta^2 \alpha \mu s_r^{1.5} \frac{n_1 + n_2 n_3}{n_3 \sqrt{n_1 n_2}} \frac{\epsilon}{(1 - \epsilon)^2} \rho^{2t} \bar{\sigma}_{s_r}^2 (\bcX_{\star}).
\end{align*}

\paragraph{Bound of $\mathfrak{R}_4$.}
\begin{align*}
\| \widebar{\bS}_{\triangle} \widebar{\bR}_{\sharp} (\widebar{\bR}_{\sharp}^H \widebar{\bR}_{\sharp})^{-1} \widebar{\bG}_{\star}^{\frac{1}{2}} \|_F^2 & \leq s_r \| \widebar{\bS}_{\triangle} \widebar{\bR}_{\sharp} (\widebar{\bR}_{\sharp}^H \widebar{\bR}_{\sharp})^{-1} \widebar{\bG}_{\star}^{\frac{1}{2}} \|^2  \\
& \leq s_r \| \widebar{\bS}_{\triangle} \|^2 \| \widebar{\bR}_{\sharp} (\widebar{\bR}_{\sharp}^H \widebar{\bR}_{\sharp})^{-1} \widebar{\bG}_{\star}^{\frac{1}{2}} \|^2  \\
& \leq \frac{\alpha^2 \ell s_r}{4} (n_1 + n_2 n_3)^2 \| \bcS_{\triangle} \|_{\infty}^2 \| \widebar{\bR}_{\sharp} (\widebar{\bR}_{\sharp}^H \widebar{\bR}_{\sharp})^{-1} \widebar{\bG}_{\star}^{\frac{1}{2}} \|^2  \\
& \leq 9 \alpha^2 \mu^2 s_r^3 \frac{(n_1 + n_2 n_3)^2}{n_1 n_2 n_3^2} \frac{1}{(1 - \epsilon)^2} \rho^{2t} \bar{\sigma}_{s_r}^2 (\bcX_{\star}).
\end{align*}
Hence,
\begin{align*}
|\mathfrak{R}_4| \leq 9 \eta^2 \alpha^2 \mu^2 s_r^3 \frac{(n_1 + n_2 n_3)^2}{n_1 n_2 n_3^2} \frac{1}{(1 - \epsilon)^2} \rho^{2t} \bar{\sigma}_{s_r}^2 (\bcX_{\star}).
\end{align*}
Combining all the bounds together, we have
\begin{align*}
& \| (\bcL_{t+1} \ast_{\bPhi} \bcQ_t - \bcL_{\star}) \ast_{\bPhi} \bcG_{\star}^{\frac{1}{2}} \|_F^2  \\ 
\leq & \frac{1}{\ell} \Big( (1 - \eta)^2 \| \widebar{\bL}_{\triangle} \widebar{\bG}_{\star}^{\frac{1}{2}} \|_F^2 + \Big( (1 - \eta) \frac{2 \epsilon^3}{1 - \epsilon} + \eta \frac{2 \epsilon^3 + \epsilon^4}{(1 - \epsilon)^2} \Big) \eta \rho^{2t} \bar{\sigma}_{s_r}^2 (\bcX_{\star})  \\
&\quad + 6 \eta (1 - \eta) \alpha \mu s_r^{1.5} \frac{n_1 + n_2 n_3}{n_3 \sqrt{n_1 n_2}} \frac{\epsilon}{1 - \epsilon} \rho^{2t} \bar{\sigma}_{s_r}^2 (\bcX_{\star})  \\
&\quad + 6 \eta^2 \alpha \mu s_r^{1.5} \frac{n_1 + n_2 n_3}{n_3 \sqrt{n_1 n_2}} \frac{\epsilon}{(1 - \epsilon)^2} \rho^{2t} \bar{\sigma}_{s_r}^2 (\bcX_{\star})  \\
&\quad + 9 \eta^2 \alpha^2 \mu^2 s_r^3 \frac{(n_1 + n_2 n_3)^2}{n_1 n_2 n_3^2} \frac{1}{(1 - \epsilon)^2} \rho^{2t} \bar{\sigma}_{s_r}^2 (\bcX_{\star}) \Big).
\end{align*}
We can obtain a similar bound for $\| (\bcR_{t+1} \ast_{\bPhi} \bcQ_t^{-H} - \bcR_{\star}) \ast_{\bPhi} \bcG_{\star}^{\frac{1}{2}} \|_F^2$. Thus, we have
\begin{align}\label{eqn:distbound}
& \dist^2(\bcF_{t+1}, \bcF_{\star})  \nonumber \\
\leq & \frac{1}{\ell} \Big( (1 - \eta)^2 \Big( \| \widebar{\bL}_{\triangle} \widebar{\bG}_{\star}^{\frac{1}{2}} \|_F^2 + \| \widebar{\bR}_{\triangle} \widebar{\bG}_{\star}^{\frac{1}{2}} \|_F^2 \Big) + 2 \Big( (1 - \eta) \frac{2 \epsilon^3}{1 - \epsilon} + \eta \frac{2 \epsilon^3 + \epsilon^4}{(1 - \epsilon)^2} \Big) \eta \rho^{2t} \bar{\sigma}_{s_r}^2 (\bcX_{\star})  \nonumber \\
&\quad + 12 \eta (1 - \eta) \alpha \mu s_r^{1.5} \frac{n_1 + n_2 n_3}{n_3 \sqrt{n_1 n_2}} \frac{\epsilon}{1 - \epsilon} \rho^{2t} \bar{\sigma}_{s_r}^2 (\bcX_{\star}) + 12 \eta^2 \alpha \mu s_r^{1.5} \frac{n_1 + n_2 n_3}{n_3 \sqrt{n_1 n_2}} \frac{\epsilon}{(1 - \epsilon)^2} \rho^{2t} \bar{\sigma}_{s_r}^2 (\bcX_{\star})  \nonumber \\
&\quad + 18 \eta^2 \alpha^2 \mu^2 s_r^3 \frac{(n_1 + n_2 n_3)^2}{n_1 n_2 n_3^2} \frac{1}{(1 - \epsilon)^2} \rho^{2t} \bar{\sigma}_{s_r}^2 (\bcX_{\star}) \Big)  \nonumber \\
\leq & \Big( (1 - \eta)^2 + 2 \Big( (1 - \eta) \frac{2 \epsilon}{1 - \epsilon} + \eta \frac{2 \epsilon + \epsilon^2}{(1 - \epsilon)^2} \Big) \eta + 12 \eta (1 - \eta) \alpha \mu s_r^{1.5} \frac{n_1 + n_2 n_3}{n_3 \sqrt{n_1 n_2}} \frac{1}{\epsilon (1 - \epsilon)}  \nonumber \\
&\quad + 12 \eta^2 \alpha \mu s_r^{1.5} \frac{n_1 + n_2 n_3}{n_3 \sqrt{n_1 n_2}} \frac{1}{\epsilon (1 - \epsilon)^2} + 18 \eta^2 \alpha^2 \mu^2 s_r^3 \frac{(n_1 + n_2 n_3)^2}{n_1 n_2 n_3^2} \frac{1}{\epsilon^2 (1 - \epsilon)^2} \Big) \frac{1}{\ell} \epsilon^2 \rho^{2t} \bar{\sigma}_{s_r}^2 (\bcX_{\star})  \nonumber \\
\leq & (1 - 0.6 \eta)^2 \frac{1}{\ell} \epsilon^2 \rho^{2t} \bar{\sigma}_{s_r}^2 (\bcX_{\star}),
\end{align}
where we use the fact $\| \widebar{\bL}_{\triangle} \widebar{\bG}_{\star}^{\frac{1}{2}} \|_F^2 + \| \widebar{\bR}_{\triangle} \widebar{\bG}_{\star}^{\frac{1}{2}} \|_F^2 = \ell \dist^2(\bcF_t,\bcF_{\star}) \leq \epsilon^2 \rho^{2t} \bar{\sigma}_{s_r}^2 (\bcX_{\star})$ in the second step, and the last step holds with $\epsilon = 0.02$, $\alpha \mu s_r^{1.5} \frac{n_1 + n_2 n_3}{n_3 \sqrt{n_1 n_2}} \leq \alpha \mu s_r^{1.5} \frac{n_1 + n_2 n_3}{n_3 \sqrt{n_{(2)} n_2}} \leq 10^{-4}$, and $\frac{1}{4} < \eta \leq \frac{8}{9}$. Thus we conclude that
\begin{align*}
\dist(\bcF_{t+1}, \bcF_{\star}) & \leq \frac{\epsilon}{\sqrt{\ell}} \rho^{t+1} \bar{\sigma}_{s_r} (\bcX_{\star})
\end{align*}
by setting $\rho = 1 - 0.6 \eta$.

\emph{Incoherence condition:} We first use \eqref{eqn:TRPCAexpand} again to obtain
\begin{align*}
& \| (\bcL_{t+1} \ast_{\bPhi} \bcQ_t - \bcL_{\star}) \ast_{\bPhi} \bcG_{\star}^{\frac{1}{2}} \|_{2,\infty}  \\
\leq & (1 - \eta) \| \bcL_{\triangle} \ast_{\bPhi} \bcG_{\star}^{\frac{1}{2}} \|_{2,\infty} + \eta \| \bcL_{\star} \ast_{\bPhi} \bcR_{\triangle}^H \ast_{\bPhi} \bcR_{\sharp} \ast_{\bPhi} (\bcR_{\sharp}^H \ast_{\bPhi} \bcR_{\sharp})^{-1} \ast_{\bPhi} \bcG_{\star}^{\frac{1}{2}} \|_{2,\infty}  \\
&\quad + \eta \| \bcS_{\triangle} \ast_{\bPhi} \bcR_{\sharp} \ast_{\bPhi} (\bcR_{\sharp}^H \ast_{\bPhi} \bcR_{\sharp})^{-1} \ast_{\bPhi} \bcG_{\star}^{\frac{1}{2}} \|_{2,\infty}  \\
\coloneq & \mathfrak{B}_1 + \mathfrak{B}_2 + \mathfrak{B}_3.
\end{align*}

\paragraph{Bound of $\mathfrak{B}_1$.} Followed by the assumption in this step, we can easily have $\mathfrak{B}_1 \leq (1 - \eta) \sqrt{\frac{\mu s_r}{n_1 n_3 \ell}} \rho^t \bar{\sigma}_{s_r} (\bcX_{\star})$.

\paragraph{Bound of $\mathfrak{B}_2$.} Definition~\ref{def:incoherence} implies $\| \bcL_{\star} \ast_{\bPhi} \bcG_{\star}^{-\frac{1}{2}} \|_{2,\infty} \leq \sqrt{\frac{\mu s_r}{n_1 n_3 \ell}}$. Thus, using Lemma~\ref{lemma:2infbound} and Lemma~\ref{lemma:Deltanorm}, we have
\begin{align*}
\mathfrak{B}_2 & \leq \eta \| \bcL_{\star} \ast_{\bPhi} \bcG_{\star}^{-\frac{1}{2}} \|_{2,\infty} \| \bcR_{\triangle} \ast_{\bPhi} \bcG_{\star}^{\frac{1}{2}} \| \| \bcR_{\sharp} \ast_{\bPhi} (\bcR_{\sharp}^H \ast_{\bPhi} \bcR_{\sharp})^{-1} \ast_{\bPhi} \bcG_{\star}^{\frac{1}{2}} \|  \\
& \leq \eta \frac{\epsilon}{1 - \epsilon} \sqrt{\frac{\mu s_r}{n_1 n_3 \ell}} \rho^t \bar{\sigma}_{s_r} (\bcX_{\star}).
\end{align*}

\paragraph{Bound of $\mathfrak{B}_3$.} By Lemma~\ref{lemma:sparity}, $\supp(\bcS_{\triangle}) \subseteq \supp(\bcS_{\star})$, which implies that $\bcS_{\triangle}$ is an $\alpha$-sparse tensor. Thus, using Lemmas~\ref{lemma:normbound},~\ref{lemma:sparity} and~\ref{lemma:Xinfnorm}, we have
\begin{align*}
\mathfrak{B}_3 & \leq \eta \| \bcS_{\triangle} \|_{2,\infty} \| \bcR_{\sharp} \ast_{\bPhi} (\bcR_{\sharp}^H \ast_{\bPhi} \bcR_{\sharp})^{-1} \ast_{\bPhi} \bcG_{\star}^{\frac{1}{2}} \| \leq \eta \frac{\sqrt{\alpha n_2 n_3}}{1 - \epsilon} \| \bcS_{\triangle} \|_{\infty}  \\
& \leq 6 \eta \frac{\sqrt{\alpha \mu s_r}}{1 - \epsilon} \sqrt{\frac{\mu s_r}{n_1 n_3 \ell}} \rho^t \bar{\sigma}_{s_r} (\bcX_{\star}).
\end{align*}
Putting all the bounds together, we obtain
\begin{align*}
\| (\bcL_{t+1} \ast_{\bPhi} \bcQ_t - \bcL_{\star}) \ast_{\bPhi} \bcG_{\star}^{\frac{1}{2}} \|_{2,\infty} \leq \Big( (1 - \eta) + \eta \frac{\epsilon}{1 - \epsilon} + 6 \eta \frac{\sqrt{\alpha \mu s_r}}{1 - \epsilon} \Big) \sqrt{\frac{\mu s_r}{n_1 n_3 \ell}} \rho^t \bar{\sigma}_{s_r} (\bcX_{\star}).
\end{align*}
We can then have
\begin{align*}
\| (\bcL_{t+1} \ast_{\bPhi} \bcQ_t - \bcL_{\star}) \ast_{\bPhi} \bcG_{\star}^{-\frac{1}{2}} \|_{2,\infty} \leq \Big( (1 - \eta) + \eta \frac{\epsilon}{1 - \epsilon} + 6 \eta \frac{\sqrt{\alpha \mu s_r}}{1 - \epsilon} \Big) \sqrt{\frac{\mu s_r}{n_1 n_3 \ell}} \rho^t.
\end{align*}
The last step is to switch the alignment tensor from $\bcQ_t$ to $\bcQ_{t+1}$. Note that \eqref{eqn:distbound} together with Lemma~\ref{lemma:Qexistence} confirms the existence of $\bcQ_{t+1}$. Applying the triangle inequality and Lemma~\ref{lemma:Qperturbation}, we have
\begin{align*}
& \| (\bcL_{t+1} \ast_{\bPhi} \bcQ_{t+1} - \bcL_{\star}) \ast_{\bPhi} \bcG_{\star}^{\frac{1}{2}} \|_{2,\infty}  \\
\leq & \| (\bcL_{t+1} \ast_{\bPhi} \bcQ_t - \bcL_{\star}) \ast_{\bPhi} \bcG_{\star}^{\frac{1}{2}} \|_{2,\infty} + \| (\bcL_{t+1} \ast_{\bPhi} (\bcQ_{t+1} - \bcQ_t) \ast_{\bPhi} \bcG_{\star}^{\frac{1}{2}} \|_{2,\infty}  \\
= & \| (\bcL_{t+1} \ast_{\bPhi} \bcQ_t - \bcL_{\star}) \ast_{\bPhi} \bcG_{\star}^{\frac{1}{2}} \|_{2,\infty}  \\
&\qquad\qquad + \| (\bcL_{t+1} \ast_{\bPhi} \bcQ_t \ast_{\bPhi} \bcG_{\star}^{-\frac{1}{2}} \ast_{\bPhi} \bcG_{\star}^{\frac{1}{2}} \ast_{\bPhi} \bcQ_t^{-1} \ast_{\bPhi} (\bcQ_{t+1} - \bcQ_t) \ast_{\bPhi} \bcG_{\star}^{\frac{1}{2}} \|_{2,\infty}  \\
\leq & \| (\bcL_{t+1} \ast_{\bPhi} \bcQ_t - \bcL_{\star}) \ast_{\bPhi} \bcG_{\star}^{\frac{1}{2}} \|_{2,\infty}  \\
&\qquad\qquad + \| (\bcL_{t+1} \ast_{\bPhi} \bcQ_t \ast_{\bPhi} \bcG_{\star}^{-\frac{1}{2}} \|_{2,\infty} \| \bcG_{\star}^{\frac{1}{2}} \ast_{\bPhi} \bcQ_t^{-1} \ast_{\bPhi} (\bcQ_{t+1} - \bcQ_t) \ast_{\bPhi} \bcG_{\star}^{\frac{1}{2}} \|  \\
\leq & \| (\bcL_{t+1} \ast_{\bPhi} \bcQ_t - \bcL_{\star}) \ast_{\bPhi} \bcG_{\star}^{\frac{1}{2}} \|_{2,\infty} + \Big( \| (\bcL_{t+1} \ast_{\bPhi} \bcQ_t - \bcL_{\star}) \ast_{\bPhi} \bcG_{\star}^{-\frac{1}{2}} \|_{2,\infty}  \\
&\qquad\qquad + \| \bcL_{\star} \ast_{\bPhi} \bcG_{\star}^{-\frac{1}{2}} \|_{2,\infty} \Big) \| \bcG_{\star}^{\frac{1}{2}} \ast_{\bPhi} \bcQ_t^{-1} \ast_{\bPhi} (\bcQ_{t+1} - \bcQ_t) \ast_{\bPhi} \bcG_{\star}^{\frac{1}{2}} \|  \\
\leq & \Big( (1 - \eta) + \eta \frac{\epsilon}{1 - \epsilon} + 6 \eta \frac{\sqrt{\alpha \mu s_r}}{1 - \epsilon} \Big) \sqrt{\frac{\mu s_r}{n_1 n_3 \ell}} \rho^t \bar{\sigma}_{s_r} (\bcX_{\star})  \\
&\qquad\qquad + \Big( \big( (1 - \eta) + \eta \frac{\epsilon}{1 - \epsilon} + 6 \eta \frac{\sqrt{\alpha \mu s_r}}{1 - \epsilon} \big) \sqrt{\frac{\mu s_r}{n_1 n_3 \ell}} \rho^t + \sqrt{\frac{\mu s_r}{n_1 n_3 \ell}} \Big) \frac{2 \epsilon}{1 - \epsilon} \rho^{t+1} \bar{\sigma}_{s_r} (\bcX_{\star}) \\
\leq & \Big( (1 - \eta) + \eta \frac{\epsilon}{1 - \epsilon} + 6 \eta \frac{\sqrt{\alpha \mu s_r}}{1 - \epsilon} + \frac{2 \epsilon}{1 - \epsilon} \big( (2 - \eta) + \eta \frac{\epsilon}{1 - \epsilon} + 6 \eta \frac{\sqrt{\alpha \mu s_r}}{1 - \epsilon} \big) \Big) \sqrt{\frac{\mu s_r}{n_1 n_3 \ell}} \rho^t \bar{\sigma}_{s_r} (\bcX_{\star})  \\
\leq & (1 - 0.6 \eta) \sqrt{\frac{\mu s_r}{n_1 n_3 \ell}} \rho^t \bar{\sigma}_{s_r} (\bcX_{\star}),
\end{align*}
where we use $\epsilon = 0.02$, $\alpha \mu s_r < \alpha \mu s_r^{1.5} \frac{n_1 + n_2 n_3}{n_3 \sqrt{n_{(2)} n_2}} \leq 10^{-4}$, and $\frac{1}{4} \leq \eta \leq \frac{8}{9}$ in the last step. A similar result can be computed for $\| (\bcR_{t+1} \ast_{\bPhi} \bcQ_{t+1}^{-H} - \bcR_{\star}) \ast_{\bPhi} \bcG_{\star}^{\frac{1}{2}} \|_{2,\infty}$. The conclusion that
\begin{align*}
& \sqrt{n_1} \| (\bcL_{t+1} \ast_{\bPhi} \bcQ_{t+1} - \bcL_{\star}) \ast_{\bPhi} \bcG_{\star}^{\frac{1}{2}} \|_{2,\infty} \vee \sqrt{n_2} \| (\bcR_{t+1} \ast_{\bPhi} \bcQ_{t+1}^{-H} - \bcR_{\star}) \ast_{\bPhi} \bcG_{\star}^{\frac{1}{2}} \|_{2,\infty}  \\
& \leq \sqrt{\frac{\mu s_r}{n_3 \ell}} \rho^{t+1} \bar{\sigma}_{s_r} (\bcX_{\star})
\end{align*}
can be achieved by setting $\rho = 1 - 0.6 \eta$. This finishes the proof.

\subsection{Proof of Lemma~\ref{lemma:TRPCAinitial}}

First, notice that the $(i,j)$-th mode-3 tube of $\bar{\boldsymbol{\ce}}_{ijk}$, which is equal to $L(\dot{\boldsymbol{\ce}}_k)$, is the only nonzero mode-3 tube of $\bar{\boldsymbol{\ce}}_{ijk}$. Thus the only nonzero mode-3 tube of $L(\bar{\boldsymbol{\ce}}_{ijk})$ is its $(i,j)$-th mode-3 tube, and it is equal to $L(L(\dot{\boldsymbol{\ce}}_k))$. Then $L(\bar{\boldsymbol{\ce}}_{ijk})$ can be written as $L(\bar{\boldsymbol{\ce}}_{ijk}) = L(\mathring{\boldsymbol{\ce}}_i) \triangle L(L(\dot{\boldsymbol{\ce}}_k)) \triangle L(\mathring{\boldsymbol{\ce}}_j^H)$. By Definition~\ref{def:incoherence}, we have
\begin{align*}
\| \bcX_{\star} \|_{\infty} & = \max_{i,j,k} | \langle \bcU_{\star} \ast_{\bPhi} \bcG_{\star} \ast_{\bPhi} \bcV_{\star}^H, \bar{\boldsymbol{\ce}}_{ijk} \rangle | = \frac{1}{\ell} \max_{i,j,k} | \langle \xoverline{\bcU}_{\star} \triangle \xoverline{\bcG}_{\star} \triangle \xoverline{\bcV}_{\star}^H, L(\bar{\boldsymbol{\ce}}_{ijk}) \rangle |  \\
& = \frac{1}{\ell} \max_{i,j,k} | \langle \xoverline{\bcU}_{\star} \triangle \xoverline{\bcG}_{\star} \triangle \xoverline{\bcV}_{\star}^H, L(\mathring{\boldsymbol{\ce}}_i) \triangle L(L(\dot{\boldsymbol{\ce}}_k)) \triangle L(\mathring{\boldsymbol{\ce}}_j^H) \rangle |  \\
& = \frac{1}{\ell} \max_{i,j,k} | \langle \widebar{\bU}_{\star} \widebar{\bG}_{\star} \widebar{\bV}_{\star}^H, \widebar{\mathring{\be}}_i \widebar{\bh}_k \widebar{\mathring{\be}}_j^H \rangle | = \frac{1}{\ell} \max_{i,j,k} | \langle \widebar{\mathring{\be}}_i^H \widebar{\bU}_{\star} \widebar{\bG}_{\star} \widebar{\bV}_{\star}^H \widebar{\mathring{\be}}_j, \widebar{\bh}_k \rangle |  \\
& \leq \frac{1}{\ell} \max_{i,j,k} \| \widebar{\mathring{\be}}_i^H \widebar{\bU}_{\star} \widebar{\bG}_{\star} \widebar{\bV}_{\star}^H \widebar{\mathring{\be}}_j \|_F \| \widebar{\bh}_k \|_F  \\
& \leq \frac{1}{\ell} \max_{i,j,k} \| \widebar{\mathring{\be}}_i^H \widebar{\bU}_{\star} \|_F \| \widebar{\bG}_{\star} \| \| \widebar{\bV}_{\star}^H \widebar{\mathring{\be}}_j \|_F \| \widebar{\bh}_k \|_F  \\
& \leq \max_{i,j,k} \| \mathring{\boldsymbol{\ce}}_i^H \ast_{\bPhi} \bcU_{\star} \|_F \| \widebar{\bG}_{\star} \| \| \bcV_{\star}^H \ast_{\bPhi} \mathring{\boldsymbol{\ce}}_j \|_F \| \widebar{\bh}_k \|_F  \\
& \leq \frac{\mu s_r}{n_3 \sqrt{n_1 n_2 \ell}} \bar{\sigma}_1 (\bcX_{\star}),
\end{align*}
where $\widebar{\mathring{\be}}_i = \mathtt{bdiag} (L(\mathring{\boldsymbol{\ce}}_i))$ and $\boldsymbol{\ch}_k = L(L(\dot{\boldsymbol{\ce}}_k))$. Invoking Lemma~\ref{lemma:sparity} with $\bcX_{-1} = \bzero$, we have
\begin{align*}
\| \bcS_{\star} - \bcS_0 \|_{\infty} \leq 2 \frac{\mu s_r}{n_3 \sqrt{n_1 n_2 \ell}} \bar{\sigma}_1 (\bcX_{\star}) \quad \mathrm{and} \quad \supp(\bcS_0) \subseteq \supp(\bcS_{\star}),
\end{align*}
which implies $\bcS_{\star} - \bcS_0$ is an $\alpha$-sparse tensor. Applying Lemma~\ref{lemma:normbound}, we have
\begin{align*}
\| \bcS_{\star} - \bcS_0 \| \leq \frac{\alpha \sqrt{\ell}}{2} (n_1 + n_2 n_3) \| \bcS_{\star} - \bcS_0 \|_{\infty} \leq \alpha \mu s_r \frac{n_1 + n_2 n_3}{n_3 \sqrt{n_1 n_2}} \bar{\sigma}_1 (\bcX_{\star}) = \alpha \mu s_r \kappa \frac{n_1 + n_2 n_3}{n_3 \sqrt{n_1 n_2}} \bar{\sigma}_{s_r} (\bcX_{\star}).
\end{align*}
Since $\bcX_0 = \bcL_0 \ast_{\bPhi} \bcR_0^H$ is the best approximation of $\bcY - \bcS_0$ with tubal rank $r$, we obtain
\begin{align*}
\| \bcX_{\star} - \bcX_0 \| = \| \widebar{\bX}_{\star} - \widebar{\bX}_0 \| & \leq \| \widebar{\bX}_{\star} - (\widebar{\bY} - \widebar{\bS}_0) \| + \| (\widebar{\bY} - \widebar{\bS}_0) - \widebar{\bX}_0 \|  \\
& \leq 2 \| \widebar{\bX}_{\star} - (\widebar{\bY} - \widebar{\bS}_0) \|  \\
& = 2 \| \widebar{\bS}_{\star} - \widebar{\bS}_0 \|  \\
& \leq 2 \alpha \mu s_r \kappa \frac{n_1 + n_2 n_3}{n_3 \sqrt{n_1 n_2}} \bar{\sigma}_{s_r} (\bcX_{\star}),
\end{align*}
where we use the definition $\widebar{\bY} = \widebar{\bX}_{\star} + \widebar{\bS}_{\star}$ in the equality. Using Lemma~\ref{lemma:Procrustes}, we obtain
\begin{align*}
\dist(\bcF_0, \bcF_{\star}) & \leq (\sqrt{2}+1)^{\frac{1}{2}} \frac{1}{\sqrt{\ell}} \| \widebar{\bX}_{\star} - \widebar{\bX}_0 \|_F \leq \sqrt{\frac{2(\sqrt{2}+1) s_r}{\ell}} \| \widebar{\bX}_{\star} - \widebar{\bX}_0 \|  \\
& \leq 5 \frac{1}{\sqrt{\ell}} \alpha \mu s_r^{1.5} \kappa \frac{n_1 + n_2 n_3}{n_3 \sqrt{n_1 n_2}} \bar{\sigma}_{s_r} (\bcX_{\star}),
\end{align*}
where we use the fact that $\widebar{\bX}_{\star} - \widebar{\bX}_0$ has at most rank-$2 s_r$. Given $\epsilon = 5 c_0$ and $\alpha \leq \frac{c_0}{\mu s_r^{1.5} \kappa \frac{n_1 + n_2 n_3}{n_3 \sqrt{n_1 n_2}}}$, our first claim
\begin{align}\label{eqn:initdist}
\dist(\bcF_0, \bcF_{\star}) \leq \frac{5 c_0}{\sqrt{\ell}} \bar{\sigma}_{s_r} (\bcX_{\star})
\end{align}
is proved.

Next, we proceed to prove the second claim. For the ease of presentation, we define $\bcL_{\sharp} \coloneq \bcL_0 \ast_{\bPhi} \bcQ_0$, $\bcR_{\sharp} \coloneq \bcR_0 \ast_{\bPhi} \bcQ_0^{-H}$, $\bcL_{\triangle} \coloneq \bcL_{\sharp} - \bcL_{\star}$, $\bcR_{\triangle} \coloneq \bcR_{\sharp} - \bcR_{\star}$, $\bcS_{\triangle} \coloneq \bcS_0 - \bcS_{\star}$. We first notice that $\bcU_0 \ast_{\bPhi} \bcG_0 \ast_{\bPhi} \bcV_0^H = \text{t-SVD}_r (\bcY - \bcS_0) = \text{t-SVD}_r (\bcX_{\star} - \bcS_{\triangle})$, thus 
\begin{align*}
\bcL_0 & = \bcU_0 \ast_{\bPhi} \bcG_0^{\frac{1}{2}} = (\bcX_{\star} - \bcS_{\triangle}) \ast_{\bPhi} \bcV_0 \ast_{\bPhi} \bcG_0^{-\frac{1}{2}} = (\bcX_{\star} - \bcS_{\triangle}) \ast_{\bPhi} \bcR_0 \ast_{\bPhi} \bcG_0^{-1}  \\
& = (\bcX_{\star} - \bcS_{\triangle}) \ast_{\bPhi} \bcR_0 \ast_{\bPhi} (\bcR_0^H \ast_{\bPhi} \bcR_0)^{-1}.
\end{align*}
Multiplying $\bcQ_0 \ast_{\bPhi} \bcG_{\star}^{\frac{1}{2}}$ on both sides using transformed t-product, we have
\begin{align*}
\bcL_{\sharp} \ast_{\bPhi} \bcG_{\star}^{\frac{1}{2}} & = \bcL_0 \ast_{\bPhi} \bcQ_0 \ast_{\bPhi} \bcG_{\star}^{\frac{1}{2}} = (\bcX_{\star} - \bcS_{\triangle}) \ast_{\bPhi} \bcR_0 \ast_{\bPhi} (\bcR_0^H \ast_{\bPhi} \bcR_0)^{-1} \ast_{\bPhi} \bcQ_0 \ast_{\bPhi} \bcG_{\star}^{\frac{1}{2}}  \\
& = (\bcX_{\star} - \bcS_{\triangle}) \ast_{\bPhi} \bcR_{\sharp} \ast_{\bPhi} (\bcR_{\sharp}^H \ast_{\bPhi} \bcR_{\sharp})^{-1} \ast_{\bPhi} \bcG_{\star}^{\frac{1}{2}}.
\end{align*}
Subtracting $\bcX_{\star} \ast_{\bPhi} \bcR_{\sharp} \ast_{\bPhi} (\bcR_{\sharp}^H \ast_{\bPhi} \bcR_{\sharp})^{-1} \ast_{\bPhi} \bcG_{\star}^{\frac{1}{2}}$ on both sides and using the fact that $\bcL_{\sharp} \ast_{\bPhi} \bcG_{\star}^{\frac{1}{2}} = \bcL_{\sharp} \ast_{\bPhi} \bcR_{\sharp}^H \ast_{\bPhi} \bcR_{\sharp} \ast_{\bPhi} (\bcR_{\sharp}^H \ast_{\bPhi} \bcR_{\sharp})^{-1} \ast_{\bPhi} \bcG_{\star}^{\frac{1}{2}}$, along with the decomposition that $\bcL_{\sharp} \ast_{\bPhi} \bcR_{\sharp}^H - \bcX_{\star} = \bcL_{\triangle} \ast_{\bPhi} \bcR_{\sharp}^H + \bcL_{\star} \ast_{\bPhi} \bcR_{\triangle}^H$, we have
\begin{align*}
\bcL_{\triangle} \ast_{\bPhi} \bcG_{\star}^{\frac{1}{2}} + \bcL_{\star} \ast_{\bPhi} \bcR_{\triangle}^H \ast_{\bPhi} \bcR_{\sharp} \ast_{\bPhi} (\bcR_{\sharp}^H \ast_{\bPhi} \bcR_{\sharp})^{-1} \ast_{\bPhi} \bcG_{\star}^{\frac{1}{2}} = - \bcS_{\triangle} \ast_{\bPhi} \bcR_{\sharp} \ast_{\bPhi} (\bcR_{\sharp}^H \ast_{\bPhi} \bcR_{\sharp})^{-1} \ast_{\bPhi} \bcG_{\star}^{\frac{1}{2}}.
\end{align*}
Thus, 
\begin{align*}
\| \bcL_{\triangle} \ast_{\bPhi} \bcG_{\star}^{\frac{1}{2}} \|_{2,\infty} & \leq \| \bcL_{\star} \ast_{\bPhi} \bcR_{\triangle}^H \ast_{\bPhi} \bcR_{\sharp} \ast_{\bPhi} (\bcR_{\sharp}^H \ast_{\bPhi} \bcR_{\sharp})^{-1} \ast_{\bPhi} \bcG_{\star}^{\frac{1}{2}} \|_{2,\infty}  \\
&\qquad\qquad + \| \bcS_{\triangle} \ast_{\bPhi} \bcR_{\sharp} \ast_{\bPhi} (\bcR_{\sharp}^H \ast_{\bPhi} \bcR_{\sharp})^{-1} \ast_{\bPhi} \bcG_{\star}^{\frac{1}{2}} \|_{2,\infty}  \\
& \coloneq \mathfrak{J}_1 + \mathfrak{J}_2.
\end{align*}

\paragraph{Bound of $\mathfrak{J}_1$.} Since \eqref{eqn:initdist} holds, Lemma~\ref{lemma:Deltanorm} implies $\| \bcR_{\triangle} \ast_{\bPhi} \bcG_{\star}^{\frac{1}{2}} \| \leq \epsilon \bar{\sigma}_{s_r} (\bcX_{\star})$ and $\| \bcR_{\sharp} \ast_{\bPhi} (\bcR_{\sharp}^H \ast_{\bPhi} \bcR_{\sharp})^{-1} \ast_{\bPhi} \bcG_{\star}^{\frac{1}{2}} \| \leq \frac{1}{1 - \epsilon}$. By Definition~\ref{def:incoherence} and Lemma~\ref{lemma:2infbound}, we obtain
\begin{align}\label{eqn:J1bound}
\mathfrak{J}_1 & \leq \| \bcL_{\star} \ast_{\bPhi} \bcG_{\star}^{-\frac{1}{2}} \|_{2,\infty} \| \bcR_{\triangle} \ast_{\bPhi} \bcG_{\star}^{\frac{1}{2}} \| \| \bcR_{\sharp} \ast_{\bPhi} (\bcR_{\sharp}^H \ast_{\bPhi} \bcR_{\sharp})^{-1} \ast_{\bPhi} \bcG_{\star}^{\frac{1}{2}} \|  \nonumber \\
& \leq \sqrt{\frac{\mu s_r}{n_1 n_3 \ell}} \frac{\epsilon}{1 - \epsilon} \bar{\sigma}_{s_r} (\bcX_{\star}).
\end{align}

\paragraph{Bound of $\mathfrak{J}_2$.} By Lemmas~\ref{lemma:normbound}, ~\ref{lemma:2infbound} and~\ref{lemma:2infprodbound}, we have
\begin{align*}
\mathfrak{J}_2 & \leq \sqrt{n_2 \ell} \| \bcS_{\triangle} \|_{2,\infty} \| \bcR_{\sharp} \ast_{\bPhi} (\bcR_{\sharp}^H \ast_{\bPhi} \bcR_{\sharp})^{-1} \ast_{\bPhi} \bcG_{\star}^{\frac{1}{2}} \|_{2,\infty}  \\
& \leq \sqrt{n_2 \ell} \| \bcS_{\triangle} \|_{2,\infty} \| \bcR_{\sharp} \ast_{\bPhi} \bcG_{\star}^{-\frac{1}{2}} \|_{2,\infty} \| \bcG_{\star}^{\frac{1}{2}} \ast_{\bPhi} (\bcR_{\sharp}^H \ast_{\bPhi} \bcR_{\sharp})^{-1} \ast_{\bPhi} \bcG_{\star}^{\frac{1}{2}} \|  \\
& \leq \sqrt{n_2 \ell} \sqrt{\alpha n_2 n_3} \| \bcS_{\triangle} \|_{\infty} \| \bcR_{\sharp} \ast_{\bPhi} \bcG_{\star}^{-\frac{1}{2}} \|_{2,\infty} \| \bcR_{\sharp} \ast_{\bPhi} (\bcR_{\sharp}^H \ast_{\bPhi} \bcR_{\sharp})^{-1} \ast_{\bPhi} \bcG_{\star}^{\frac{1}{2}} \|^2  \\
& \leq 2 n_2 \sqrt{\alpha n_3 \ell} \frac{\mu s_r}{n_3 \sqrt{n_1 n_2 \ell}} \bar{\sigma}_1 (\bcX_{\star}) \| \bcR_{\sharp} \ast_{\bPhi} \bcG_{\star}^{-\frac{1}{2}} \|_{2,\infty} \| \bcR_{\sharp} \ast_{\bPhi} (\bcR_{\sharp}^H \ast_{\bPhi} \bcR_{\sharp})^{-1} \ast_{\bPhi} \bcG_{\star}^{\frac{1}{2}} \|^2  \\
& \leq 2 \frac{\mu s_r \kappa}{(1 - \epsilon)^2} \sqrt{\frac{\alpha n_2}{n_1 n_3}} \bar{\sigma}_{s_r} (\bcX_{\star}) \| \bcR_{\sharp} \ast_{\bPhi} \bcG_{\star}^{-\frac{1}{2}} \|_{2,\infty}  \\
& \leq 2 \frac{\mu s_r \kappa}{(1 - \epsilon)^2} \sqrt{\frac{\alpha n_2}{n_1 n_3}} \Big( \sqrt{\frac{\mu s_r}{n_2 n_3 \ell}} + \| \bcR_{\triangle} \ast_{\bPhi} \bcG_{\star}^{-\frac{1}{2}} \|_{2,\infty} \Big) \bar{\sigma}_{s_r} (\bcX_{\star}).
\end{align*}
Note that $\| \bcR_{\triangle} \ast_{\bPhi} \bcG_{\star}^{-\frac{1}{2}} \|_{2,\infty} \leq \frac{\| \bcR_{\triangle} \ast_{\bPhi} \bcG_{\star}^{\frac{1}{2}} \|_{2,\infty}}{\bar{\sigma}_{s_r} (\bcX_{\star})}$. Thus,
\begin{align*}
\sqrt{n_1} \| \bcL_{\triangle} \ast_{\bPhi} \bcG_{\star}^{\frac{1}{2}} \|_{2,\infty} & \leq \sqrt{\frac{\mu s_r}{n_3 \ell}} \Big( \frac{\epsilon}{1 - \epsilon} + 2 \frac{\mu s_r \kappa}{(1 - \epsilon)^2} \sqrt{\frac{\alpha}{n_3}} \Big) \bar{\sigma}_{s_r} (\bcX_{\star})  \\
&\qquad\qquad + 2 \frac{\mu s_r \kappa}{(1 - \epsilon)^2} \sqrt{\frac{\alpha}{n_3}} \sqrt{n_2} \| \bcR_{\triangle} \ast_{\bPhi} \bcG_{\star}^{\frac{1}{2}} \|_{2,\infty},
\end{align*}
and similarly one can see
\begin{align*}
\sqrt{n_2} \| \bcR_{\triangle} \ast_{\bPhi} \bcG_{\star}^{\frac{1}{2}} \|_{2,\infty} & \leq \sqrt{\frac{\mu s_r}{n_3 \ell}} \Big( \frac{\epsilon}{1 - \epsilon} + 2 \frac{\mu s_r \kappa}{(1 - \epsilon)^2} \sqrt{\frac{\alpha}{n_3}} \Big) \bar{\sigma}_{s_r} (\bcX_{\star})  \\
&\qquad\qquad + 2 \frac{\mu s_r \kappa}{(1 - \epsilon)^2} \sqrt{\frac{\alpha}{n_3}} \sqrt{n_1} \| \bcL_{\triangle} \ast_{\bPhi} \bcG_{\star}^{\frac{1}{2}} \|_{2,\infty}.
\end{align*}
Therefore,
\begin{align*}
\sqrt{n_1} \| \bcL_{\triangle} \ast_{\bPhi} \bcG_{\star}^{\frac{1}{2}} \|_{2,\infty} \vee \sqrt{n_2} \| \bcR_{\triangle} \ast_{\bPhi} \bcG_{\star}^{\frac{1}{2}} \|_{2,\infty} & \leq \frac{\frac{\epsilon}{1 - \epsilon} + 2 \frac{\mu s_r \kappa}{(1 - \epsilon)^2} \sqrt{\frac{\alpha}{n_3}}}{1 - 2 \frac{\mu s_r \kappa}{(1 - \epsilon)^2} \sqrt{\frac{\alpha}{n_3}}} \sqrt{\frac{\mu s_r}{n_3 \ell}} \bar{\sigma}_{s_r} (\bcX_{\star})  \\
& \leq \frac{\frac{\epsilon}{1 - \epsilon} + 2 \frac{c_0}{(1 - \epsilon)^2}}{1 - 2 \frac{c_0}{(1 - \epsilon)^2}} \sqrt{\frac{\mu s_r}{n_3 \ell}} \bar{\sigma}_{s_r} (\bcX_{\star})  \\
& = \frac{\frac{5 c_0}{1 - 5 c_0} + 2 \frac{c_0}{(1 - 5 c_0)^2}}{1 - 2 \frac{c_0}{(1 - 5 c_0)^2}} \sqrt{\frac{\mu s_r}{n_3 \ell}} \bar{\sigma}_{s_r} (\bcX_{\star})  \\
& \leq \sqrt{\frac{\mu s_r}{n_3 \ell}} \bar{\sigma}_{s_r} (\bcX_{\star})
\end{align*}
as long as $c_0 \leq 0.06$. This finishes the proof.

\subsection{Proof of Theorem~\ref{thm:TRPCA}}

\begin{proof}
We set $c_0 \leq 0.004$ in Lemma~\ref{lemma:TRPCAinitial}, then the results of Lemma~\ref{lemma:TRPCAinitial} satisfy the condition of Lemma~\ref{lemma:TRPCAcontraction} and give
\begin{align*}
\dist(\bcF_t, \bcF_{\star}) & \leq \frac{0.02}{\sqrt{\ell}} (1 - 0.6 \eta)^t \bar{\sigma}_{s_r} (\bcX_{\star}),  \\
\sqrt{n_1} \| (\bcL_t \ast_{\bPhi} \bcQ_t - \bcL_{\star}) \ast_{\bPhi} \bcG_{\star}^{\frac{1}{2}} \|_{2,\infty} & \vee \sqrt{n_2} \| (\bcR_t \ast_{\bPhi} \bcQ_t^{-H} - \bcR_{\star}) \ast_{\bPhi} \bcG_{\star}^{\frac{1}{2}} \|_{2,\infty} \leq \sqrt{\frac{\mu s_r}{n_3 \ell}} (1 - 0.6 \eta)^t \bar{\sigma}_{s_r} (\bcX_{\star})
\end{align*}
for all $t \geq 0$. According to Lemma~\ref{lemma:Deltanorm}, we have $\| \bcL_{\triangle} \ast_{\bPhi} \bcG_{\star}^{-\frac{1}{2}} \| \vee \| \bcR_{\triangle} \ast_{\bPhi} \bcG_{\star}^{-\frac{1}{2}} \| \leq \epsilon (1 - 0.6 \eta)^t \leq \epsilon$. According to \eqref{eqn:disttensor}, $\| \bcL_t \ast_{\bPhi} \bcR_t^H - \bcX_{\star} \|_F \leq 1.5 \dist(\bcF_t, \bcF_{\star})$. Hence, the first claim is proved. The second and third claims are followed by Lemma~\ref{lemma:Xinfnorm} and Lemma~\ref{lemma:sparity}, respectively.
\end{proof}

\section{Proof for Tensor Completion}
\label{sec:Completionproof}

This section is devoted to the proofs of claims related to tensor completion.

\subsection{Proof of Lemma~\ref{lemma:scaledproj}}

\begin{lemma}[\citet{TongMC.JMLR2021}, Claim 5]\label{lemma:nonexpansive}
For tensor columns $\overrightarrow{\bcA}, \overrightarrow{\bcA}_{\star} \in \R^{n_1 \times 1 \times n_3}$ and $\lambda \geq \| \overrightarrow{\bcA}_{\star} \|_F / \| \overrightarrow{\bcA} \|_F$, it holds that
\begin{align*}
\| (1 \wedge \lambda) \overrightarrow{\bcA} - \overrightarrow{\bcA}_{\star} \|_F \leq \| \overrightarrow{\bcA} - \overrightarrow{\bcA}_{\star} \|_F.
\end{align*}
\end{lemma}
We first prove the non-expansiveness property. Denote the optimal alignment tensor between $\widetilde{\bcF}$ and $\bcF_{\star}$ as $\widetilde{\bcQ}$, whose existence is guaranteed by Lemma~\ref{lemma:Qexistence}. Let $\cP_{\varsigma}(\widetilde{\bcF}) = \begin{bmatrix} \widetilde{\bcL} \\ \widetilde{\bcR} \end{bmatrix}$, by the definition of $\dist(\cP_{\varsigma}(\widetilde{\bcF}),\bcF_{\star})$, we have 
\begin{align}\label{eqn:disttrailer}
\dist^{2}(\cP_{\varsigma}(\widetilde{\bcF}), \bcF_{\star}) & \leq \sum_{i=1}^{n_1} \| \bcL(i,:,:) \ast_{\bPhi} \widetilde{\bcQ} \ast_{\bPhi} \bcG_{\star}^{\frac{1}{2}} - ( \bcL_{\star} \ast_{\bPhi} \bcG_{\star}^{\frac{1}{2}} )(i,:,:) \|_F^2  \nonumber \\
&\qquad + \sum_{j=1}^{n_2} \| \bcR(j,:,:) \ast_{\bPhi} \widetilde{\bcQ}^{-H} \ast_{\bPhi} \bcG_{\star}^{\frac{1}{2}} - ( \bcR_{\star} \ast_{\bPhi} \bcG_{\star}^{\frac{1}{2}} )(j,:,:) \|_F^2.
\end{align}
Note that the condition $\dist(\widetilde{\bcF}, \bcF_{\star}) \leq \frac{\epsilon}{\sqrt{\ell}} \bar{\sigma}_{s_r} (\bcX_{\star})$ implies 
\begin{align*}
\| (\widetilde{\bcL} \ast_{\bPhi} \widetilde{\bcQ} - \bcL_{\star}) \ast_{\bPhi} \bcG_{\star}^{-\frac{1}{2}} \| \vee \| (\widetilde{\bcR} \ast_{\bPhi} \widetilde{\bcQ}^{-H} - \bcR_{\star}) \ast_{\bPhi} \bcG_{\star}^{-\frac{1}{2}} \| \leq \epsilon.
\end{align*}
Utilizing the fact that $\bcR_{\star} \ast_{\bPhi} \bcG_{\star}^{-\frac{1}{2}} = \bcV_{\star}$, we arrive at
\begin{align*}
\| \widetilde{\bcL}(i,:,:) \ast_{\bPhi} \widetilde{\bcR}^H \|_F & \leq \| \widetilde{\bcL}(i,:,:) \ast_{\bPhi} \widetilde{\bcQ} \ast_{\bPhi} \bcG_{\star}^{\frac{1}{2}} \|_F \| \widetilde{\bcR} \ast_{\bPhi} \widetilde{\bcQ}^{-H} \ast_{\bPhi} \bcG_{\star}^{-\frac{1}{2}} \|  \\
&\leq \| \widetilde{\bcL}(i,:,:) \ast_{\bPhi} \widetilde{\bcQ} \ast_{\bPhi} \bcG_{\star}^{\frac{1}{2}} \|_F \Big( \| \bcV_{\star} \| + \| (\widetilde{\bcR} \ast_{\bPhi} \widetilde{\bcQ}^{-H} - \bcR_{\star}) \ast_{\bPhi} \bcG_{\star}^{-\frac{1}{2}} \| \Big)  \\
&\leq (1+\epsilon) \| \widetilde{\bcL}(i,:,:) \ast_{\bPhi} \widetilde{\bcQ} \ast_{\bPhi} \bcG_{\star}^{\frac{1}{2}} \|_F.
\end{align*}
In addition, the $\mu$-incoherence of $\bcX_{\star}$ yields
\begin{align*}
\sqrt{n_1} \| ( \bcL_{\star} \ast_{\bPhi} \bcG_{\star}^{\frac{1}{2}} )(i,:,:) \|_F & \leq \sqrt{n_1} \| \bcU_{\star} \|_{2,\infty} \| \bcG_{\star} \| \leq \sqrt{\frac{\mu s_r}{n_3 \ell}} \bar{\sigma}_1 (\bcX_{\star}) \leq \frac{\varsigma}{1+\epsilon},
\end{align*}
where the last inequality follows from the choice of $\varsigma$. Take the above two inequalities to reach 
\begin{align*}
\frac{\varsigma}{\sqrt{n_1} \| \widetilde{\bcL}(i,:,:) \ast_{\bPhi} \widetilde{\bcR}^H \|_F} \geq \frac{\| ( \bcL_{\star} \ast_{\bPhi} \bcG_{\star}^{\frac{1}{2}} )(i,:,:) \|_F}{\| \widetilde{\bcL}(i,:,:) \ast_{\bPhi} \widetilde{\bcQ} \ast_{\bPhi} \bcG_{\star}^{\frac{1}{2}} \|_F}.
\end{align*}
We can then apply Lemma~\ref{lemma:nonexpansive} with $\overrightarrow{\bcA} = \widetilde{\bcL}(i,:,:) \ast_{\bPhi} \widetilde{\bcQ} \ast_{\bPhi} \bcG_{\star}^{\frac{1}{2}}$, $\overrightarrow{\bcA}_{\star} = ( \bcL_{\star} \ast_{\bPhi} \bcG_{\star}^{\frac{1}{2}} )(i,:,:)$, and $\lambda = \frac{\varsigma}{\sqrt{n_1} \| \widetilde{\bcL}(i,:,:) \ast_{\bPhi} \widetilde{\bcR}^H \|_F}$ to obtain
\begin{align*}
& \| \bcL(i,:,:) \ast_{\bPhi} \widetilde{\bcQ} \ast_{\bPhi} \bcG_{\star}^{\frac{1}{2}} - ( \bcL_{\star} \ast_{\bPhi} \bcG_{\star}^{\frac{1}{2}} )(i,:,:) \|_F^2  \\
= & \| \Big(1 \wedge \frac{\varsigma}{\sqrt{n_1} \| \widetilde{\bcL}(i,:,:) \ast_{\bPhi} \widetilde{\bcR}^H \|_F} \Big) \widetilde{\bcL}(i,:,:) \ast_{\bPhi} \widetilde{\bcQ} \ast_{\bPhi} \bcG_{\star}^{\frac{1}{2}} - ( \bcL_{\star} \ast_{\bPhi} \bcG_{\star}^{\frac{1}{2}} )(i,:,:) \|_F^2  \\
\leq & \| \widetilde{\bcL}(i,:,:) \ast_{\bPhi} \widetilde{\bcQ} \ast_{\bPhi} \bcG_{\star}^{\frac{1}{2}} - ( \bcL_{\star} \ast_{\bPhi} \bcG_{\star}^{\frac{1}{2}} )(i,:,:) \|_F^2.
\end{align*}
Following a similar argument for $\bcR$, we conclude that
\begin{align*}
\dist^{2}(\cP_{\varsigma}(\widetilde{\bcF}), \bcF_{\star}) & \leq \sum_{i=1}^{n_1} \| \widetilde{\bcL}(i,:,:) \ast_{\bPhi} \widetilde{\bcQ} \ast_{\bPhi} \bcG_{\star}^{\frac{1}{2}} - ( \bcL_{\star} \ast_{\bPhi} \bcG_{\star}^{\frac{1}{2}} )(i,:,:) \|_F^2  \\
&\qquad + \sum_{j=1}^{n_2} \| \widetilde{\bcR}(j,:,:) \ast_{\bPhi} \widetilde{\bcQ}^{-H} \ast_{\bPhi} \bcG_{\star}^{\frac{1}{2}} - ( \bcR_{\star} \ast_{\bPhi} \bcG_{\star}^{\frac{1}{2}} )(j,:,:) \|_F^2 \leq \dist^{2}(\widetilde{\bcF}, \bcF_{\star}).
\end{align*}

Next, we prove the incoherence condition. For any $i \in [n_1]$, one has
\begin{align*}
& \| \bcL(i,:,:) \ast_{\bPhi} \bcR^H \|_F^2  \\
= & \sum_{j=1}^{n_2} \| \bcL(i,:,:) \ast_{\bPhi} \bcR(j,:,:)^H \|_F^2  \\
= & \sum_{j=1}^{n_2} \Big(1 \wedge \frac{\varsigma}{\sqrt{n_1} \| \widetilde{\bcL}(i,:,:) \ast_{\bPhi} \widetilde{\bcR}^H \|_F} \Big)^2 \| \widetilde{\bcL}(i,:,:) \ast_{\bPhi} \widetilde{\bcR}(j,:,:)^H \|_F^2 \Big(1 \wedge \frac{\varsigma}{\sqrt{n_2} \| \widetilde{\bcR}(j,:,:) \ast_{\bPhi} \widetilde{\bcL}^H \|_F} \Big)^2  \\
\overset{\mathrm{(a)}}{\leq} & \Big(1 \wedge \frac{\varsigma}{\sqrt{n_1} \| \widetilde{\bcL}(i,:,:) \ast_{\bPhi} \widetilde{\bcR}^H \|_F} \Big)^2 \sum_{j=1}^{n_2} \| \widetilde{\bcL}(i,:,:) \ast_{\bPhi} \widetilde{\bcR}(j,:,:)^H \|_F^2  \\
= & \Big(1 \wedge \frac{\varsigma}{\sqrt{n_1} \| \widetilde{\bcL}(i,:,:) \ast_{\bPhi} \widetilde{\bcR}^H \|_F} \Big)^2 \| \widetilde{\bcL}(i,:,:) \ast_{\bPhi} \widetilde{\bcR}^H \|_F^2  \\
\overset{\mathrm{(b)}}{\leq} & \frac{\varsigma^2}{n_1},
\end{align*}
where $\mathrm{(a)}$ follows from $1 \wedge \frac{\varsigma}{\sqrt{n_2} \| \widetilde{\bcR}(j,:,:) \ast_{\bPhi} \widetilde{\bcL}^H \|_F} \leq 1$, and $\mathrm{(b)}$ follows from $1 \wedge \frac{\varsigma}{\sqrt{n_1} \| \widetilde{\bcL}(i,:,:) \ast_{\bPhi} \widetilde{\bcR}^H \|_F} \leq \frac{\varsigma}{\sqrt{n_1} \| \widetilde{\bcL}(i,:,:) \ast_{\bPhi} \widetilde{\bcR}^H \|_F}$. Similarly, one can also have $\| \bcR(j,:,:) \ast_{\bPhi} \bcL^H \|_F^2 \leq \frac{\varsigma^2}{n_2}$. Combining these two bounds completes the proof.

\subsection{Proof of Lemma~\ref{lemma:TCcontraction}}

We gather several useful inequalities regarding the operator $\bcP_{\bOmega} (\cdot)$ for the Bernoulli observation model.
\begin{lemma}[\citet{Tropp.FoCM2012}]\label{lemma:Bernstein}
Consider a finite sequence $\{\bZ_k\}$ of independent, random $d_1 \times d_2$ matrices that satisfy the assumption $\mathbb{E}[\bZ_k] = \bzero$ and $\| \bZ_k \| \leq R$ almost surely. Let $\sigma^2 = \max \{ \| \sum_k \mathbb{E}[\bZ_k \bZ_k^H] \|, \| \sum_k \mathbb{E}[\bZ_k^H \bZ_k] \| \}$. Then, for any $t \geq 0$, we have
\begin{align*}
\mathbb{P} \Big[ \| \sum_k \bZ_k \| \geq t \Big] & \leq (d_1 + d_2) \exp \Big( - \frac{t^2}{2 \sigma^2 + \frac{2}{3} R t} \Big)  \\
& \leq (d_1 + d_2) \exp \Big( - \frac{3 t^2}{8 \sigma^2} \Big), \quad \mathrm{for} ~ t \leq \sigma^2 / R.
\end{align*}
Or, for any $c > 0$, we have
\begin{align*}
\| \sum_k \bZ_k \| \leq 2 \sqrt{c \sigma^2 \log(d_1 + d_2)} + c R \log(d_1 + d_2)
\end{align*}
with probability at least $1 - (d_1 + d_2)^{1 - c}$.
\end{lemma}
\begin{lemma}\label{lemma:POmegaspectral}
Suppose that $\bcZ \in \R^{n_1 \times n_2 \times n_3}$ is fixed, and $\bOmega \sim \mathrm{Ber}(p)$. Then with high probability, 
\begin{align*}
\| (p^{-1} \bcP_{\bOmega} - \bcI_{n_1}) (\bcZ) \| \leq c \Big( \frac{\sqrt{\ell} \log( (n_1 \vee n_2) n_3 )}{p} \| \bcZ \|_{\infty} + \sqrt{\frac{\ell \log( (n_1 \vee n_2) n_3 )}{p}}\| \bcZ \|_{\infty,2} \Big),
\end{align*}
for some numerical constant $c > 0$.
\end{lemma}
The following lemma establishes restricted strong convexity and smoothness of the observation operator for tensors in $\bT$, which can be considered as an extension of \citet[Lemma 10]{ZhengL.arXiv2016}.
\begin{lemma}\label{lemma:POmegaFnorm}
Suppose that $\bcA, \bcB \in \bT$ are fixed tensors and $\bOmega \sim \mathrm{Ber}(p)$. Then with high probability,
\begin{align}\label{eqn:POmegaFnorm1}
p(1 - \epsilon) \| \bcA \|_F^2 \leq \| \bcP_{\bOmega}(\bcA) \|_F^2 \leq p(1 + \epsilon) \| \bcA \|_F^2.
\end{align}
Consequently,
\begin{align}\label{eqn:POmegaFnorm2}
| p^{-1} \langle \bcP_{\bOmega}(\bcA), \bcP_{\bOmega}(\bcB) \rangle - \langle \bcA, \bcB \rangle | \leq \epsilon \| \bcA \|_F \| \bcB \|_F,
\end{align}
provided that $p \geq c \epsilon^{-2} \mu s_r (n_1 + n_2) \log( (n_1 \vee n_2) n_3 ) / (n_1 n_2 n_3)$ for some numerical constant $c > 0$.
\end{lemma}
We then have the following simple corollary.
\begin{corollary}\label{coro:POmegatangent}
Suppose that $\bcX_{\star}$ is $\mu$-incoherent, and $p \gtrsim \mu s_r (n_1 + n_2) \log( (n_1 \vee n_2) n_3 ) / (n_1 n_2 n_3)$. Then with high probability,
\begin{align*}
& | \langle (p^{-1} \bcP_{\bOmega} - \bcI_{n_1})(\bcL_{\star} \ast_{\bPhi} \bcR_A^H + \bcL_A \ast_{\bPhi} \bcR_{\star}^H), \bcL_{\star} \ast_{\bPhi} \bcR_B^H + \bcL_B \ast_{\bPhi} \bcR_{\star}^H \rangle |  \\
& \qquad\qquad \leq c \sqrt{\frac{\mu s_r (n_1 + n_2) \log( (n_1 \vee n_2) n_3 )}{p n_1 n_2 n_3} } \| \bcL_{\star} \ast_{\bPhi} \bcR_A^H + \bcL_A \ast_{\bPhi} \bcR_{\star}^H \|_F \| \bcL_{\star} \ast_{\bPhi} \bcR_B^H + \bcL_B \ast_{\bPhi} \bcR_{\star}^H \|_F,
\end{align*}
simultaneously for all $\bcL_A, \bcL_B \in \R^{n_1 \times r \times n_3}$ and $\bcR_A, \bcR_B \in \R^{n_2 \times r \times n_3}$, where $c > 0$ is some numerical constant.
\end{corollary}
\begin{lemma}\label{lemma:POmeganonconvex}
Suppose that $p \gtrsim \log( (n_1 \vee n_2) n_3 ) / (n_1 \wedge n_2)$. Then with high probability,
\begin{align*}
& | \langle (p^{-1} \bcP_{\bOmega} - \bcI_{n_1})(\bcL_A \ast_{\bPhi} \bcR_A^H), \bcL_B \ast_{\bPhi} \bcR_B^H \rangle |  \\
&\quad \leq c \ell^{\frac{3}{2}} \sqrt{\frac{(n_1 \vee n_2) \log( (n_1 \vee n_2) n_3 )}{p} }  \\
&\qquad \Big( \|\bcL_A\|_{2,2,\infty} \|\bcL_B\|_F \wedge \|\bcL_A\|_F \|\bcL_B\|_{2,2,\infty} \Big) \Big( \|\bcR_A\|_{2,2,\infty} \|\bcR_B\|_F \wedge \|\bcR_A\|_F \|\bcR_B\|_{2,2,\infty} \Big),
\end{align*}
simultaneously for all $\bcL_A, \bcL_B \in \R^{n_1 \times r \times n_3}$ and $\bcR_A, \bcR_B \in \R^{n_2 \times r \times n_3}$, where $c > 0$ is some universal constant.
\end{lemma}
\begin{lemma}\label{lemma:TC-cond}
Under conditions $\dist(\bcF_t, \bcF_{\star}) \leq \frac{\epsilon}{\sqrt{\ell}} \bar{\sigma}_{s_r} (\bcX_{\star})$ and $\sqrt{n_1} \| \bcL_{\sharp} \ast_{\bPhi} \bcR_{\sharp}^H \|_{2,\infty} \vee \sqrt{n_2} \| \bcR_{\sharp} \ast_{\bPhi} \bcL_{\sharp}^H \|_{2,\infty} \leq c_{\varsigma} \sqrt{\frac{\mu s_r}{n_3 \ell}} \bar{\sigma}_1 (\bcX_{\star})$, we have
\begin{subequations}
\begin{align}
\| \bcL_{\triangle} \ast_{\bPhi} \bcG_{\star}^{-\frac{1}{2}} \| \vee \| \bcR_{\triangle} \ast_{\bPhi} \bcG_{\star}^{-\frac{1}{2}} \| & \leq \epsilon; \label{eqn:TC-cond-spec}  \\
\| \bcR_{\sharp} \ast_{\bPhi} (\bcR_{\sharp}^H \ast_{\bPhi} \bcR_{\sharp})^{-1} \ast_{\bPhi} \bcG_{\star}^{\frac{1}{2}} \| & \leq \frac{1}{1 - \epsilon}; \label{eqn:TC-consequences-1}  \\
\| \bcG_{\star}^{\frac{1}{2}} \ast_{\bPhi} (\bcR_{\sharp}^H \ast_{\bPhi} \bcR_{\sharp})^{-1} \ast_{\bPhi} \bcG_{\star}^{\frac{1}{2}} \| & \leq \frac{1}{(1 - \epsilon)^2}; \label{eqn:TC-consequences-2}  \\
\sqrt{n_1} \| \bcL_{\sharp} \ast_{\bPhi} \bcG_{\star}^{\frac{1}{2}} \|_{2,\infty} \vee \sqrt{n_2} \| \bcR_{\sharp} \ast_{\bPhi} \bcG_{\star}^{\frac{1}{2}} \|_{2,\infty} & \leq \frac{c_{\varsigma}}{1 - \epsilon} \sqrt{\frac{\mu s_r}{n_3 \ell}} \bar{\sigma}_1 (\bcX_{\star}); \label{eqn:TC-cond-2inf-1p}  \\
\sqrt{n_1} \| \bcL_{\sharp} \ast_{\bPhi} \bcG_{\star}^{-\frac{1}{2}} \|_{2,\infty} \vee \sqrt{n_2} \| \bcR_{\sharp} \ast_{\bPhi} \bcG_{\star}^{-\frac{1}{2}} \|_{2,\infty} & \leq \frac{c_{\varsigma} \kappa}{1 - \epsilon} \sqrt{\frac{\mu s_r}{n_3 \ell}}; \label{eqn:TC-cond-2inf-1n}  \\
\sqrt{n_1} \| \bcL_{\triangle} \ast_{\bPhi} \bcG_{\star}^{\frac{1}{2}} \|_{2,\infty} \vee \sqrt{n_2} \| \bcR_{\triangle} \ast_{\bPhi} \bcG_{\star}^{\frac{1}{2}} \|_{2,\infty} & \leq (1 + \frac{c_{\varsigma}}{1 - \epsilon}) \sqrt{\frac{\mu s_r}{n_3 \ell}} \bar{\sigma}_1 (\bcX_{\star}). \label{eqn:TC-cond-2inf-2p}
\end{align}
\end{subequations}
\end{lemma}

\begin{proof}
First, repeating the derivation for Lemma~\ref{lemma:Qexistence} obtains \eqref{eqn:TC-cond-spec}. Second, taking the condition \eqref{eqn:TC-cond-spec} and Lemma~\ref{lemma:Weyl} together to obtain \eqref{eqn:TC-consequences-1} and \eqref{eqn:TC-consequences-2}. Third, taking the incoherence condition $\sqrt{n_1} \| \bcL_{\sharp} \ast_{\bPhi} \bcR_{\sharp}^H \|_{2,\infty} \vee \sqrt{n_2} \| \bcR_{\sharp} \ast_{\bPhi} \bcL_{\sharp}^H \|_{2,\infty} \leq c_{\varsigma} \sqrt{\frac{\mu s_r}{n_3 \ell}} \bar{\sigma}_1 (\bcX_{\star})$, and $\| \bcA \ast_{\bPhi} \bcB \|_{2,\infty} \geq \| \bcA \|_{2,\infty} \bar{\sigma}_{s_r} (\bcB)$ from Lemma~\ref{lemma:2infbound}, together with the relations
\begin{align*}
\| \bcL_{\sharp} \ast_{\bPhi} \bcR_{\sharp}^H \|_{2,\infty} & \geq \bar{\sigma}_{s_r} (\bcR_{\sharp} \ast_{\bPhi} \bcG_{\star}^{-\frac{1}{2}}) \| \bcL_{\sharp} \ast_{\bPhi} \bcG_{\star}^{\frac{1}{2}} \|_{2,\infty}  \\
& \geq (\bar{\sigma}_{s_r} (\bcR_{\star} \ast_{\bPhi} \bcG_{\star}^{-\frac{1}{2}}) - \| \bcR_{\triangle} \ast_{\bPhi} \bcG_{\star}^{-\frac{1}{2}} \| ) \| \bcL_{\sharp} \ast_{\bPhi} \bcG_{\star}^{\frac{1}{2}} \|_{2,\infty}  \\
& \geq (1 - \epsilon) \| \bcL_{\sharp} \ast_{\bPhi} \bcG_{\star}^{\frac{1}{2}} \|_{2,\infty};  \\
\| \bcR_{\sharp} \ast_{\bPhi} \bcL_{\sharp}^H \|_{2,\infty} & \geq \bar{\sigma}_{s_r} (\bcL_{\sharp} \ast_{\bPhi} \bcG_{\star}^{-\frac{1}{2}}) \| \bcR_{\sharp} \ast_{\bPhi} \bcG_{\star}^{\frac{1}{2}} \|_{2,\infty}  \\
& \geq (\bar{\sigma}_{s_r} (\bcL_{\star} \ast_{\bPhi} \bcG_{\star}^{-\frac{1}{2}}) - \| \bcL_{\triangle} \ast_{\bPhi} \bcG_{\star}^{-\frac{1}{2}} \| ) \| \bcR_{\sharp} \ast_{\bPhi} \bcG_{\star}^{\frac{1}{2}} \|_{2,\infty}  \\
& \geq (1 - \epsilon) \| \bcR_{\sharp} \ast_{\bPhi} \bcG_{\star}^{\frac{1}{2}} \|_{2,\infty}
\end{align*}
to obtain \eqref{eqn:TC-cond-2inf-1p} and \eqref{eqn:TC-cond-2inf-1n}. Finally, \eqref{eqn:TC-cond-2inf-2p} can be obtained by applying the triangle inequality together with incoherence assumption.
\end{proof}

Now we prove Lemma~\ref{lemma:TCcontraction}. First, we define the event $G$ as the intersection of the events that the bounds in Corollary~\ref{coro:POmegatangent} and Lemma~\ref{lemma:POmeganonconvex} hold. The rest of the proof is under the assumption that $G$ holds, which happens with high probability. By the condition $\dist(\bcF_t, \bcF_{\star}) \leq \frac{0.02}{\sqrt{\ell}} \bar{\sigma}_{s_r} (\bcX_{\star})$ and Lemma~\ref{lemma:Qexistence}, one knows that $\bcQ_t$, the optimal alignment tensor between $\bcF_t$ and $\bcF_{\star}$ exists, and $\epsilon \coloneq 0.02$. In addition, denote $\widetilde{\bcF}_{t+1}$ as the update before projection as
\begin{align*}
\widetilde{\bcF}_{t+1} = \begin{bmatrix} \widetilde{\bcL}_{t+1} \\ \widetilde{\bcR}_{t+1} \end{bmatrix} = \begin{bmatrix} \bcL_t - \frac{\eta}{p} \bcP_{\bOmega}(\bcL_t \ast_{\bPhi} \bcR_t^H - \bcX_{\star}) \ast_{\bPhi} \bcR_t \ast_{\bPhi} (\bcR_t^H \ast_{\bPhi} \bcR_t)^{-1}  \\
\bcR_t - \frac{\eta}{p} \bcP_{\bOmega}(\bcL_t \ast_{\bPhi} \bcR_t^H - \bcX_{\star})^H \ast_{\bPhi} \bcL_t \ast_{\bPhi} (\bcL_t^H \ast_{\bPhi} \bcL_t)^{-1} \end{bmatrix},
\end{align*}
and therefore $\bcF_{t+1} = \cP_{\varsigma} (\widetilde{\bcF}_{t+1})$. Note that it suffices to prove the following relation
\begin{align}\label{eqn:TCgoal}
\dist(\widetilde{\bcF}_{t+1}, \bcF_{\star}) \leq (1 - 0.6 \eta) \dist(\bcF_t, \bcF_{\star}),
\end{align}
since the conclusion $\| \bcL_t \ast_{\bPhi} \bcR_t^H - \bcX_{\star} \|_F \leq 1.5 \dist(\bcF_t, \bcF_{\star})$ is a simple consequence of Lemma~\ref{lemma:tensor2factor}; see \eqref{eqn:disttensor} for details. In the following, we focus on proving \eqref{eqn:TCgoal}. By the definition of $\dist(\widetilde{\bcF}_{t+1}, \bcF_{\star})$, we have
\begin{align}\label{eqn:TCexpand}
\dist^2(\widetilde{\bcF}_{t+1}, \bcF_{\star}) \leq \| (\widetilde{\bcL}_{t+1} \ast_{\bPhi} \bcQ_t - \bcL_{\star}) \ast_{\bPhi} \bcG_{\star}^{\frac{1}{2}} \|_F^2 + \| (\widetilde{\bcR}_{t+1} \ast_{\bPhi} \bcQ_t^{-H} - \bcR_{\star}) \ast_{\bPhi} \bcG_{\star}^{\frac{1}{2}} \|_F^2.
\end{align}
Plugging in the update rule \eqref{eqn:TCupdate} and the decomposition $\bcL_{\sharp} \ast_{\bPhi} \bcR_{\sharp}^H - \bcX_{\star} = \bcL_{\triangle} \ast_{\bPhi} \bcR_{\sharp}^H + \bcL_{\star} \ast_{\bPhi} \bcR_{\triangle}^H$ to obtain
\begin{align*}
& (\widetilde{\bcL}_{t+1} \ast_{\bPhi} \bcQ_t - \bcL_{\star}) \ast_{\bPhi} \bcG_{\star}^{\frac{1}{2}}  \\
= & \Big(\bcL_{\sharp} - \eta p^{-1} \bcP_{\bOmega} (\bcL_{\sharp} \ast_{\bPhi} \bcR_{\sharp}^H - \bcX_{\star}) \ast_{\bPhi} \bcR_{\sharp} \ast_{\bPhi} (\bcR_{\sharp}^H \ast_{\bPhi} \bcR_{\sharp})^{-1} - \bcL_{\star} \Big) \ast_{\bPhi} \bcG_{\star}^{\frac{1}{2}}  \\
= & \bcL_{\triangle} \ast_{\bPhi} \bcG_{\star}^{\frac{1}{2}} - \eta (\bcL_{\sharp} \ast_{\bPhi} \bcR_{\sharp}^H - \bcX_{\star}) \ast_{\bPhi} \bcR_{\sharp} \ast_{\bPhi} (\bcR_{\sharp}^H \ast_{\bPhi} \bcR_{\sharp})^{-1} \ast_{\bPhi} \bcG_{\star}^{\frac{1}{2}}  \\
&\qquad\qquad - \eta (p^{-1} \bcP_{\bOmega} - \bcI_{n_1})(\bcL_{\sharp} \ast_{\bPhi} \bcR_{\sharp}^H - \bcX_{\star}) \ast_{\bPhi} \bcR_{\sharp} \ast_{\bPhi} (\bcR_{\sharp}^H \ast_{\bPhi} \bcR_{\sharp})^{-1} \ast_{\bPhi} \bcG_{\star}^{\frac{1}{2}}  \\
= & (1 - \eta) \bcL_{\triangle} \ast_{\bPhi} \bcG_{\star}^{\frac{1}{2}} - \eta \bcL_{\star} \ast_{\bPhi} \bcR_{\triangle}^H \ast_{\bPhi} \bcR_{\sharp} \ast_{\bPhi} (\bcR_{\sharp}^H \ast_{\bPhi} \bcR_{\sharp})^{-1} \ast_{\bPhi} \bcG_{\star}^{\frac{1}{2}}  \\
&\qquad\qquad - \eta (p^{-1} \bcP_{\bOmega} - \bcI_{n_1})(\bcL_{\sharp} \ast_{\bPhi} \bcR_{\sharp}^H - \bcX_{\star}) \ast_{\bPhi} \bcR_{\sharp} \ast_{\bPhi} (\bcR_{\sharp}^H \ast_{\bPhi} \bcR_{\sharp})^{-1} \ast_{\bPhi} \bcG_{\star}^{\frac{1}{2}}.
\end{align*}
This allows us to expand the square of the first term in \eqref{eqn:TCexpand} as
\begin{align*}
& \| (\widetilde{\bcL}_{t+1} \ast_{\bPhi} \bcQ_t - \bcL_{\star}) \ast_{\bPhi} \bcG_{\star}^{\frac{1}{2}} \|_F^2  \\
= & \| (1 - \eta) \bcL_{\triangle} \ast_{\bPhi} \bcG_{\star}^{\frac{1}{2}} - \eta \bcL_{\star} \ast_{\bPhi} \bcR_{\triangle}^H \ast_{\bPhi} \bcR_{\sharp} \ast_{\bPhi} (\bcR_{\sharp}^H \ast_{\bPhi} \bcR_{\sharp})^{-1} \ast_{\bPhi} \bcG_{\star}^{\frac{1}{2}} \|_F^2  \\
&\quad - 2 \eta (1 - \eta) \langle \bcL_{\triangle} \ast_{\bPhi} \bcG_{\star}^{\frac{1}{2}} , (p^{-1} \bcP_{\bOmega} - \bcI_{n_1})(\bcL_{\sharp} \ast_{\bPhi} \bcR_{\sharp}^H - \bcX_{\star}) \ast_{\bPhi} \bcR_{\sharp} \ast_{\bPhi} (\bcR_{\sharp}^H \ast_{\bPhi} \bcR_{\sharp})^{-1} \ast_{\bPhi} \bcG_{\star}^{\frac{1}{2}} \rangle  \\
&\quad + 2 \eta^2 \langle \bcL_{\star} \ast_{\bPhi} \bcR_{\triangle}^H \ast_{\bPhi} \bcR_{\sharp} \ast_{\bPhi} (\bcR_{\sharp}^H \ast_{\bPhi} \bcR_{\sharp})^{-1} \ast_{\bPhi} \bcG_{\star}^{\frac{1}{2}} ,  \\
&\qquad\qquad\qquad (p^{-1} \bcP_{\bOmega} - \bcI_{n_1})(\bcL_{\sharp} \ast_{\bPhi} \bcR_{\sharp}^H - \bcX_{\star}) \ast_{\bPhi} \bcR_{\sharp} \ast_{\bPhi} (\bcR_{\sharp}^H \ast_{\bPhi} \bcR_{\sharp})^{-1} \ast_{\bPhi} \bcG_{\star}^{\frac{1}{2}} \rangle  \\
&\quad + \eta^2 \| (p^{-1} \bcP_{\bOmega} - \bcI_{n_1})(\bcL_{\sharp} \ast_{\bPhi} \bcR_{\sharp}^H - \bcX_{\star}) \ast_{\bPhi} \bcR_{\sharp} \ast_{\bPhi} (\bcR_{\sharp}^H \ast_{\bPhi} \bcR_{\sharp})^{-1} \ast_{\bPhi} \bcG_{\star}^{\frac{1}{2}} \|_F^2  \\
\coloneq & \mathfrak{P}_1 - \mathfrak{P}_2 + \mathfrak{P}_3 + \mathfrak{P}_4.
\end{align*}

\paragraph{Bound of $\mathfrak{P}_1$.} The first term $\mathfrak{P}_1$ has already been controlled in \eqref{eqn:TF-Lt-bound} as follows.
\begin{align*}
\mathfrak{P}_1 \leq \Big( (1 - \eta)^2 + \frac{2 \epsilon}{1 - \epsilon} \eta (1 - \eta) \Big) \| \bcL_{\triangle} \ast_{\bPhi} \bcG_{\star}^{\frac{1}{2}} \|_F^2 + \frac{2 \epsilon + \epsilon^2}{(1 - \epsilon)^2} \eta^2 \| \bcR_{\triangle} \ast_{\bPhi} \bcG_{\star}^{\frac{1}{2}} \|_F^2.
\end{align*}

\paragraph{Bound of $\mathfrak{P}_2$.} Using the decomposition $\bcL_{\sharp} \ast_{\bPhi} \bcR_{\sharp}^H - \bcX_{\star} = \bcL_{\triangle} \ast_{\bPhi} \bcR_{\star}^H + \bcL_{\sharp} \ast_{\bPhi} \bcR_{\triangle}^H$ and applying the triangle inequality to obtain
\begin{align*}
|\mathfrak{P}_2| & = 2 \eta (1 - \eta) \Big| \langle \bcL_{\triangle} \ast_{\bPhi} \bcG_{\star}^{\frac{1}{2}} , (p^{-1} \bcP_{\bOmega} - \bcI_{n_1})(\bcL_{\sharp} \ast_{\bPhi} \bcR_{\sharp}^H - \bcX_{\star}) \ast_{\bPhi} \bcR_{\sharp} \ast_{\bPhi} (\bcR_{\sharp}^H \ast_{\bPhi} \bcR_{\sharp})^{-1} \ast_{\bPhi} \bcG_{\star}^{\frac{1}{2}} \rangle \Big|  \\
& \leq 2 \eta (1 - \eta) \Big( \Big| \langle \bcL_{\triangle} \ast_{\bPhi} \bcG_{\star}^{\frac{1}{2}} , (p^{-1} \bcP_{\bOmega} - \bcI_{n_1})(\bcL_{\triangle} \ast_{\bPhi} \bcR_{\star}^H) \ast_{\bPhi} \bcR_{\star} \ast_{\bPhi} (\bcR_{\sharp}^H \ast_{\bPhi} \bcR_{\sharp})^{-1} \ast_{\bPhi} \bcG_{\star}^{\frac{1}{2}} \rangle \Big|  \\
&\quad + \Big| \langle \bcL_{\triangle} \ast_{\bPhi} \bcG_{\star}^{\frac{1}{2}} , (p^{-1} \bcP_{\bOmega} - \bcI_{n_1})(\bcL_{\triangle} \ast_{\bPhi} \bcR_{\star}^H) \ast_{\bPhi} \bcR_{\triangle} \ast_{\bPhi} (\bcR_{\sharp}^H \ast_{\bPhi} \bcR_{\sharp})^{-1} \ast_{\bPhi} \bcG_{\star}^{\frac{1}{2}} \rangle \Big|  \\
&\quad + \Big| \langle \bcL_{\triangle} \ast_{\bPhi} \bcG_{\star}^{\frac{1}{2}} , (p^{-1} \bcP_{\bOmega} - \bcI_{n_1})(\bcL_{\sharp} \ast_{\bPhi} \bcR_{\triangle}^H) \ast_{\bPhi} \bcR_{\sharp} \ast_{\bPhi} (\bcR_{\sharp}^H \ast_{\bPhi} \bcR_{\sharp})^{-1} \ast_{\bPhi} \bcG_{\star}^{\frac{1}{2}} \rangle \Big| \Big)  \\
& \coloneq 2 \eta (1 - \eta) (\mathfrak{P}_{2,1} + \mathfrak{P}_{2,2} + \mathfrak{P}_{2,3}).
\end{align*}
For the first term $\mathfrak{P}_{2,1}$, we can invoke Corollary~\ref{coro:POmegatangent} to obtain
\begin{align*}
\mathfrak{P}_{2,1} & \leq c_1 \sqrt{\frac{\mu s_r (n_1 + n_2) \log( (n_1 \vee n_2) n_3 )}{p n_1 n_2 n_3} } \| \bcL_{\triangle} \ast_{\bPhi} \bcR_{\star}^H \|_F \| \bcL_{\triangle} \ast_{\bPhi} \bcG_{\star} \ast_{\bPhi} (\bcR_{\sharp}^H \ast_{\bPhi} \bcR_{\sharp})^{-1} \ast_{\bPhi} \bcR_{\star}^H \|_F  \\
& = c_1 \sqrt{\frac{\mu s_r (n_1 + n_2) \log( (n_1 \vee n_2) n_3 )}{p n_1 n_2 n_3} } \| \bcL_{\triangle} \ast_{\bPhi} \bcG_{\star}^{\frac{1}{2}} \ast_{\bPhi} \bcG_{\star}^{-\frac{1}{2}} \ast_{\bPhi} \bcR_{\star}^H \|_F  \\
&\qquad\qquad\qquad \| \bcL_{\triangle} \ast_{\bPhi} \bcG_{\star}^{\frac{1}{2}} \ast_{\bPhi} \bcG_{\star}^{\frac{1}{2}} \ast_{\bPhi} (\bcR_{\sharp}^H \ast_{\bPhi} \bcR_{\sharp})^{-1} \ast_{\bPhi} \bcG_{\star}^{\frac{1}{2}} \ast_{\bPhi} \bcG_{\star}^{-\frac{1}{2}} \ast_{\bPhi} \bcR_{\star}^H \|_F  \\
& \leq c_1 \sqrt{\frac{\mu s_r (n_1 + n_2) \log( (n_1 \vee n_2) n_3 )}{p n_1 n_2 n_3} } \| \bcL_{\triangle} \ast_{\bPhi} \bcG_{\star}^{\frac{1}{2}} \|_F \| \bcG_{\star}^{-\frac{1}{2}} \ast_{\bPhi} \bcR_{\star}^H \|  \\
&\qquad\qquad\qquad \| \bcL_{\triangle} \ast_{\bPhi} \bcG_{\star}^{\frac{1}{2}} \|_F \| \bcG_{\star}^{\frac{1}{2}} \ast_{\bPhi} (\bcR_{\sharp}^H \ast_{\bPhi} \bcR_{\sharp})^{-1} \ast_{\bPhi} \bcG_{\star}^{\frac{1}{2}} \| \| \bcG_{\star}^{-\frac{1}{2}} \ast_{\bPhi} \bcR_{\star}^H \|  \\
& = c_1 \sqrt{\frac{\mu s_r (n_1 + n_2) \log( (n_1 \vee n_2) n_3 )}{p n_1 n_2 n_3} } \| \bcL_{\triangle} \ast_{\bPhi} \bcG_{\star}^{\frac{1}{2}} \|_F^2 \| \bcG_{\star}^{\frac{1}{2}} \ast_{\bPhi} (\bcR_{\sharp}^H \ast_{\bPhi} \bcR_{\sharp})^{-1} \ast_{\bPhi} \bcG_{\star}^{\frac{1}{2}} \|  \\
& \leq \frac{c_1}{(1 - \epsilon)^2} \sqrt{\frac{\mu s_r (n_1 + n_2) \log( (n_1 \vee n_2) n_3 )}{p n_1 n_2 n_3} } \| \bcL_{\triangle} \ast_{\bPhi} \bcG_{\star}^{\frac{1}{2}} \|_F^2,
\end{align*}
where the last inequality uses \eqref{eqn:TC-consequences-2}. For the term $\mathfrak{P}_{2,2}$, we can invoke Lemma~\ref{lemma:POmeganonconvex} with $\bcL_A \coloneq \bcL_{\triangle} \ast_{\bPhi} \bcG_{\star}^{\frac{1}{2}}$, $\bcR_A \coloneq \bcR_{\star} \ast_{\bPhi} \bcG_{\star}^{-\frac{1}{2}}$, $\bcL_B \coloneq \bcL_{\triangle} \ast_{\bPhi} \bcG_{\star}^{\frac{1}{2}}$, $\bcR_B \coloneq \bcR_{\triangle} \ast_{\bPhi} (\bcR_{\sharp}^H \ast_{\bPhi} \bcR_{\sharp})^{-1} \ast_{\bPhi} \bcG_{\star}^{\frac{1}{2}}$ to obtain
\begin{align*}
\mathfrak{P}_{2,2} & \leq c_2 \ell^{\frac{3}{2}} \sqrt{\frac{(n_1 \vee n_2) \log( (n_1 \vee n_2) n_3 )}{p} } \| \bcL_{\triangle} \ast_{\bPhi} \bcG_{\star}^{\frac{1}{2}} \|_{2,2,\infty} \| \bcL_{\triangle} \ast_{\bPhi} \bcG_{\star}^{\frac{1}{2}} \|_F  \\
&\qquad\qquad\qquad \| \bcR_{\star} \ast_{\bPhi} \bcG_{\star}^{-\frac{1}{2}} \|_{2,2,\infty} \| \bcR_{\triangle} \ast_{\bPhi} (\bcR_{\sharp}^H \ast_{\bPhi} \bcR_{\sharp})^{-1} \ast_{\bPhi} \bcG_{\star}^{\frac{1}{2}} \|_F  \\
& \leq c_2 \ell^{\frac{3}{2}} \sqrt{\frac{(n_1 \vee n_2) \log( (n_1 \vee n_2) n_3 )}{p} } \| \bcL_{\triangle} \ast_{\bPhi} \bcG_{\star}^{\frac{1}{2}} \|_{2,\infty} \| \bcL_{\triangle} \ast_{\bPhi} \bcG_{\star}^{\frac{1}{2}} \|_F  \\
&\qquad\qquad\qquad \| \bcR_{\star} \ast_{\bPhi} \bcG_{\star}^{-\frac{1}{2}} \|_{2,\infty} \| \bcR_{\triangle} \ast_{\bPhi} (\bcR_{\sharp}^H \ast_{\bPhi} \bcR_{\sharp})^{-1} \ast_{\bPhi} \bcG_{\star}^{\frac{1}{2}} \|_F  \\
& \leq c_2 \ell^{\frac{3}{2}} \sqrt{\frac{(n_1 \vee n_2) \log( (n_1 \vee n_2) n_3 )}{p} } \| \bcL_{\triangle} \ast_{\bPhi} \bcG_{\star}^{\frac{1}{2}} \|_{2,\infty} \| \bcL_{\triangle} \ast_{\bPhi} \bcG_{\star}^{\frac{1}{2}} \|_F  \\
&\qquad\qquad\qquad \| \bcR_{\star} \ast_{\bPhi} \bcG_{\star}^{-\frac{1}{2}} \|_{2,\infty} \| \bcR_{\triangle} \ast_{\bPhi} \bcG_{\star}^{-\frac{1}{2}} \|_F \| \bcG_{\star}^{\frac{1}{2}} \ast_{\bPhi} (\bcR_{\sharp}^H \ast_{\bPhi} \bcR_{\sharp})^{-1} \ast_{\bPhi} \bcG_{\star}^{\frac{1}{2}} \|.
\end{align*}
Similarly, we can bound $\mathfrak{P}_{2,3}$ as
\begin{align*}
\mathfrak{P}_{2,3} & \leq c_2 \ell^{\frac{3}{2}} \sqrt{\frac{(n_1 \vee n_2) \log( (n_1 \vee n_2) n_3 )}{p} } \| \bcL_{\sharp} \ast_{\bPhi} \bcG_{\star}^{-\frac{1}{2}} \|_{2,\infty} \| \bcL_{\triangle} \ast_{\bPhi} \bcG_{\star}^{\frac{1}{2}} \|_F  \\
&\qquad\qquad\qquad \| \bcR_{\triangle} \ast_{\bPhi} \bcG_{\star}^{\frac{1}{2}} \|_F \| \bcR_{\sharp} \ast_{\bPhi} \bcG_{\star}^{-\frac{1}{2}} \|_{2,\infty} \| \bcG_{\star}^{\frac{1}{2}} \ast_{\bPhi} (\bcR_{\sharp}^H \ast_{\bPhi} \bcR_{\sharp})^{-1} \ast_{\bPhi} \bcG_{\star}^{\frac{1}{2}} \|.
\end{align*}
Utilizing the consequences in Lemma~\ref{lemma:TC-cond}, we have
\begin{align*}
\mathfrak{P}_{2,2} & \leq \frac{c_2}{(1 - \epsilon)^2} \Big(1 + \frac{c_{\varsigma}}{1 - \epsilon} \Big) \sqrt{\frac{\ell \log( (n_1 \vee n_2) n_3 )}{p (n_1 \wedge n_2)} } \frac{\mu s_r \kappa}{n_3} \| \bcL_{\triangle} \ast_{\bPhi} \bcG_{\star}^{\frac{1}{2}} \|_F \| \bcR_{\triangle} \ast_{\bPhi} \bcG_{\star}^{\frac{1}{2}} \|_F;  \\
\mathfrak{P}_{2,3} & \leq \frac{c_2}{(1 - \epsilon)^4} \sqrt{\frac{\ell \log( (n_1 \vee n_2) n_3 )}{p (n_1 \wedge n_2)} } \frac{\mu s_r c_{\varsigma}^2 \kappa^2}{n_3} \| \bcL_{\triangle} \ast_{\bPhi} \bcG_{\star}^{\frac{1}{2}} \|_F \| \bcR_{\triangle} \ast_{\bPhi} \bcG_{\star}^{\frac{1}{2}} \|_F.
\end{align*}
We then combine the bounds for $\mathfrak{P}_{2,1}$, $\mathfrak{P}_{2,2}$ and $\mathfrak{P}_{2,3}$ to arrive at
\begin{align*}
|\mathfrak{P}_2| & \leq 2 \eta (1 - \eta) \Big( \frac{c_1}{(1 - \epsilon)^2} \sqrt{\frac{\mu s_r (n_1 + n_2) \log( (n_1 \vee n_2) n_3 )}{p n_1 n_2 n_3} } \| \bcL_{\triangle} \ast_{\bPhi} \bcG_{\star}^{\frac{1}{2}} \|_F^2  \\
&\qquad + \frac{c_2}{(1 - \epsilon)^2} \Big(1 + \frac{c_{\varsigma}}{1 - \epsilon} + \frac{c_{\varsigma}^2 \kappa}{(1 - \epsilon)^2} \Big) \sqrt{\frac{\ell \log( (n_1 \vee n_2) n_3 )}{p (n_1 \wedge n_2)} } \frac{\mu s_r \kappa}{n_3} \| \bcL_{\triangle} \ast_{\bPhi} \bcG_{\star}^{\frac{1}{2}} \|_F \| \bcR_{\triangle} \ast_{\bPhi} \bcG_{\star}^{\frac{1}{2}} \|_F \Big)  \\
& = 2 \eta (1 - \eta) \Big( \nu_1 \| \bcL_{\triangle} \ast_{\bPhi} \bcG_{\star}^{\frac{1}{2}} \|_F^2 + \nu_2 \| \bcL_{\triangle} \ast_{\bPhi} \bcG_{\star}^{\frac{1}{2}} \|_F \| \bcR_{\triangle} \ast_{\bPhi} \bcG_{\star}^{\frac{1}{2}} \|_F \Big)  \\
& \leq \eta (1 - \eta) \Big( (2 \nu_1 + \nu_2) \| \bcL_{\triangle} \ast_{\bPhi} \bcG_{\star}^{\frac{1}{2}} \|_F^2 + \nu_2 \| \bcR_{\triangle} \ast_{\bPhi} \bcG_{\star}^{\frac{1}{2}} \|_F^2 \Big),
\end{align*}
where we denote
\begin{align}\label{eqn:defnu1nu2}
\nu_1 & \coloneq \frac{c_1}{(1 - \epsilon)^2} \sqrt{\frac{\mu s_r (n_1 + n_2) \log( (n_1 \vee n_2) n_3 )}{p n_1 n_2 n_3} }  \nonumber \\
\mathrm{and} \quad \nu_2 & \coloneq \frac{c_2}{(1 - \epsilon)^2} \Big(1 + \frac{c_{\varsigma}}{1 - \epsilon} + \frac{c_{\varsigma}^2 \kappa}{(1 - \epsilon)^2} \Big) \sqrt{\frac{\ell \log( (n_1 \vee n_2) n_3 )}{p (n_1 \wedge n_2)} } \frac{\mu s_r \kappa}{n_3}.
\end{align}

\paragraph{Bound of $\mathfrak{P}_3$.} For the term $\mathfrak{P}_3$, we first have
\begin{align*}
|\mathfrak{P}_3| & \leq 2 \eta^2 \Big( \Big| \langle \bcL_{\star} \ast_{\bPhi} \bcR_{\triangle}^H \ast_{\bPhi} \bcR_{\sharp} \ast_{\bPhi} (\bcR_{\sharp}^H \ast_{\bPhi} \bcR_{\sharp})^{-1} \ast_{\bPhi} \bcG_{\star}^{\frac{1}{2}} ,  \\
&\qquad\qquad\qquad (p^{-1} \bcP_{\bOmega} - \bcI_{n_1})(\bcL_{\triangle} \ast_{\bPhi} \bcR_{\sharp}^H) \ast_{\bPhi} \bcR_{\sharp} \ast_{\bPhi} (\bcR_{\sharp}^H \ast_{\bPhi} \bcR_{\sharp})^{-1} \ast_{\bPhi} \bcG_{\star}^{\frac{1}{2}} \rangle \Big|  \\
&\qquad + \Big| \langle \bcL_{\star} \ast_{\bPhi} \bcR_{\triangle}^H \ast_{\bPhi} \bcR_{\sharp} \ast_{\bPhi} (\bcR_{\sharp}^H \ast_{\bPhi} \bcR_{\sharp})^{-1} \ast_{\bPhi} \bcG_{\star}^{\frac{1}{2}} ,  \\
&\qquad\qquad\qquad (p^{-1} \bcP_{\bOmega} - \bcI_{n_1})(\bcL_{\star} \ast_{\bPhi} \bcR_{\triangle}^H) \ast_{\bPhi} \bcR_{\sharp} \ast_{\bPhi} (\bcR_{\sharp}^H \ast_{\bPhi} \bcR_{\sharp})^{-1} \ast_{\bPhi} \bcG_{\star}^{\frac{1}{2}} \rangle \Big| \Big)  \\
& \coloneq 2 \eta^2 (\mathfrak{P}_{3,1} + \mathfrak{P}_{3,2}).
\end{align*}
For $\mathfrak{P}_{3,1}$, we invoke Lemma~\ref{lemma:POmeganonconvex} with $\bcL_A \coloneq \bcL_{\triangle} \ast_{\bPhi} \bcG_{\star}^{\frac{1}{2}}$, $\bcR_A \coloneq \bcR_{\sharp} \ast_{\bPhi} \bcG_{\star}^{-\frac{1}{2}}$, $\bcL_B \coloneq \bcL_{\star} \ast_{\bPhi} \bcG_{\star}^{-\frac{1}{2}}$, $\bcR_B \coloneq \bcR_{\sharp} \ast_{\bPhi} (\bcR_{\sharp}^H \ast_{\bPhi} \bcR_{\sharp})^{-1} \ast_{\bPhi} \bcG_{\star} \ast_{\bPhi} (\bcR_{\sharp}^H \ast_{\bPhi} \bcR_{\sharp})^{-1} \ast_{\bPhi} \bcR_{\sharp}^H \ast_{\bPhi} \bcR_{\triangle} \ast_{\bPhi} \bcG_{\star}^{\frac{1}{2}}$ and use the consequences in Lemma~\ref{lemma:TC-cond} to obtain
\begin{align*}
\mathfrak{P}_{3,1} & \leq c_2 \ell^{\frac{3}{2}} \sqrt{\frac{(n_1 \vee n_2) \log( (n_1 \vee n_2) n_3 )}{p} } \| \bcL_{\triangle} \ast_{\bPhi} \bcG_{\star}^{\frac{1}{2}} \|_F \| \bcL_{\star} \ast_{\bPhi} \bcG_{\star}^{-\frac{1}{2}} \|_{2,2,\infty}  \\
&\qquad \| \bcR_{\sharp} \ast_{\bPhi} \bcG_{\star}^{-\frac{1}{2}} \|_{2,2,\infty} \| \bcR_{\sharp} \ast_{\bPhi} (\bcR_{\sharp}^H \ast_{\bPhi} \bcR_{\sharp})^{-1} \ast_{\bPhi} \bcG_{\star} \ast_{\bPhi} (\bcR_{\sharp}^H \ast_{\bPhi} \bcR_{\sharp})^{-1} \ast_{\bPhi} \bcR_{\sharp}^H \ast_{\bPhi} \bcR_{\triangle} \ast_{\bPhi} \bcG_{\star}^{\frac{1}{2}} \|_F  \\
& \leq c_2 \ell^{\frac{3}{2}} \sqrt{\frac{(n_1 \vee n_2) \log( (n_1 \vee n_2) n_3 )}{p} } \| \bcL_{\triangle} \ast_{\bPhi} \bcG_{\star}^{\frac{1}{2}} \|_F \| \bcL_{\star} \ast_{\bPhi} \bcG_{\star}^{-\frac{1}{2}} \|_{2,\infty}  \\
&\qquad \| \bcR_{\sharp} \ast_{\bPhi} \bcG_{\star}^{-\frac{1}{2}} \|_{2,\infty} \| \bcR_{\sharp} \ast_{\bPhi} (\bcR_{\sharp}^H \ast_{\bPhi} \bcR_{\sharp})^{-1} \ast_{\bPhi} \bcG_{\star} \ast_{\bPhi} (\bcR_{\sharp}^H \ast_{\bPhi} \bcR_{\sharp})^{-1} \ast_{\bPhi} \bcR_{\sharp}^H \ast_{\bPhi} \bcR_{\triangle} \ast_{\bPhi} \bcG_{\star}^{\frac{1}{2}} \|_F  \\
& \leq c_2 \ell^{\frac{3}{2}} \sqrt{\frac{(n_1 \vee n_2) \log( (n_1 \vee n_2) n_3 )}{p} } \| \bcL_{\triangle} \ast_{\bPhi} \bcG_{\star}^{\frac{1}{2}} \|_F \| \bcL_{\star} \ast_{\bPhi} \bcG_{\star}^{-\frac{1}{2}} \|_{2,\infty}  \\
&\qquad \| \bcR_{\sharp} \ast_{\bPhi} \bcG_{\star}^{-\frac{1}{2}} \|_{2,\infty} \| \bcR_{\sharp} \ast_{\bPhi} (\bcR_{\sharp}^H \ast_{\bPhi} \bcR_{\sharp})^{-1} \ast_{\bPhi} \bcG_{\star}^{\frac{1}{2}} \|^2 \| \bcR_{\triangle} \ast_{\bPhi} \bcG_{\star}^{\frac{1}{2}} \|_F  \\
& \leq \frac{c_2}{(1 - \epsilon)^3} \sqrt{\frac{\ell \log( (n_1 \vee n_2) n_3 )}{p (n_1 \wedge n_2)} } \frac{\mu s_r c_{\varsigma} \kappa}{n_3} \| \bcL_{\triangle} \ast_{\bPhi} \bcG_{\star}^{\frac{1}{2}} \|_F \| \bcR_{\triangle} \ast_{\bPhi} \bcG_{\star}^{\frac{1}{2}} \|_F.
\end{align*}
For $\mathfrak{P}_{3,2}$, we again invoke Corollary~\ref{coro:POmegatangent} to obtain
\begin{align*}
\mathfrak{P}_{3,2} & \leq c_1 \sqrt{\frac{\mu s_r (n_1 + n_2) \log( (n_1 \vee n_2) n_3 )}{p n_1 n_2 n_3} } \| \bcL_{\star} \ast_{\bPhi} \bcR_{\triangle}^H \|_F  \\
&\qquad \| \bcL_{\star} \ast_{\bPhi} \bcR_{\triangle}^H \ast_{\bPhi} \bcR_{\sharp} \ast_{\bPhi} (\bcR_{\sharp}^H \ast_{\bPhi} \bcR_{\sharp})^{-1} \ast_{\bPhi} \bcG_{\star} \ast_{\bPhi} (\bcR_{\sharp}^H \ast_{\bPhi} \bcR_{\sharp})^{-1} \ast_{\bPhi} \bcR_{\sharp}^H \|_F  \\
& \leq c_1 \sqrt{\frac{\mu s_r (n_1 + n_2) \log( (n_1 \vee n_2) n_3 )}{p n_1 n_2 n_3} } \| \bcL_{\star} \ast_{\bPhi} \bcR_{\triangle}^H \|_F^2 \| \bcR_{\sharp} \ast_{\bPhi} (\bcR_{\sharp}^H \ast_{\bPhi} \bcR_{\sharp})^{-1} \ast_{\bPhi} \bcG_{\star}^{\frac{1}{2}} \|^2  \\
& \leq \frac{c_1}{(1 - \epsilon)^2} \sqrt{\frac{\mu s_r (n_1 + n_2) \log( (n_1 \vee n_2) n_3 )}{p n_1 n_2 n_3} } \| \bcR_{\triangle} \ast_{\bPhi} \bcG_{\star}^{\frac{1}{2}} \|_F^2,
\end{align*}
where the last inequality uses \eqref{eqn:TC-consequences-1}. Combining the bounds for $\mathfrak{P}_{3,1}$ and $\mathfrak{P}_{3,2}$ to arrive at
\begin{align*}
|\mathfrak{P}_3| & \leq 2 \eta^2 \Big( \frac{c_1}{(1 - \epsilon)^2} \sqrt{\frac{\mu s_r (n_1 + n_2) \log( (n_1 \vee n_2) n_3 )}{p n_1 n_2 n_3} } \| \bcR_{\triangle} \ast_{\bPhi} \bcG_{\star}^{\frac{1}{2}} \|_F^2  \\
&\qquad\qquad\qquad + \frac{c_2}{(1 - \epsilon)^3} \sqrt{\frac{\ell \log( (n_1 \vee n_2) n_3 )}{p (n_1 \wedge n_2)} } \frac{\mu s_r c_{\varsigma} \kappa}{n_3} \| \bcL_{\triangle} \ast_{\bPhi} \bcG_{\star}^{\frac{1}{2}} \|_F \| \bcR_{\triangle} \ast_{\bPhi} \bcG_{\star}^{\frac{1}{2}} \|_F \Big)  \\
& \leq 2 \eta^2 \Big( \nu_1 \| \bcR_{\triangle} \ast_{\bPhi} \bcG_{\star}^{\frac{1}{2}} \|_F^2 + \nu_2 \| \bcL_{\triangle} \ast_{\bPhi} \bcG_{\star}^{\frac{1}{2}} \|_F \| \bcR_{\triangle} \ast_{\bPhi} \bcG_{\star}^{\frac{1}{2}} \|_F \Big)  \\
& \leq \eta^2 \Big( \nu_2 \| \bcL_{\triangle} \ast_{\bPhi} \bcG_{\star}^{\frac{1}{2}} \|_F^2 + (2 \nu_1 + \nu_2) \| \bcR_{\triangle} \ast_{\bPhi} \bcG_{\star}^{\frac{1}{2}} \|_F^2 \Big).
\end{align*}

\paragraph{Bound of $\mathfrak{P}_4$.} Moving to the term $\mathfrak{P}_4$, we have
\begin{align*}
\sqrt{\mathfrak{P}_4} & = \eta \| (p^{-1} \bcP_{\bOmega} - \bcI_{n_1})(\bcL_{\sharp} \ast_{\bPhi} \bcR_{\sharp}^H - \bcX_{\star}) \ast_{\bPhi} \bcR_{\sharp} \ast_{\bPhi} (\bcR_{\sharp}^H \ast_{\bPhi} \bcR_{\sharp})^{-1} \ast_{\bPhi} \bcG_{\star}^{\frac{1}{2}} \|_F  \\
& = \eta \max_{\widetilde{\bcL} \in \R^{n_1 \times r \times n_3} : \| \widetilde{\bcL} \|_F \leq 1} \langle (p^{-1} \bcP_{\bOmega} - \bcI_{n_1})(\bcL_{\sharp} \ast_{\bPhi} \bcR_{\sharp}^H - \bcX_{\star}) \ast_{\bPhi} \bcR_{\sharp} \ast_{\bPhi} (\bcR_{\sharp}^H \ast_{\bPhi} \bcR_{\sharp})^{-1} \ast_{\bPhi} \bcG_{\star}^{\frac{1}{2}} , \widetilde{\bcL} \rangle  \\
& = \eta \max_{\widetilde{\bcL} \in \R^{n_1 \times r \times n_3} : \| \widetilde{\bcL} \|_F \leq 1} \langle (p^{-1} \bcP_{\bOmega} - \bcI_{n_1})(\bcL_{\sharp} \ast_{\bPhi} \bcR_{\sharp}^H - \bcX_{\star}) , \widetilde{\bcL} \ast_{\bPhi} \bcG_{\star}^{\frac{1}{2}} \ast_{\bPhi} (\bcR_{\sharp}^H \ast_{\bPhi} \bcR_{\sharp})^{-1} \ast_{\bPhi} \bcR_{\sharp}^H \rangle  \\
& \leq \eta \Big( \Big| \langle (p^{-1} \bcP_{\bOmega} - \bcI_{n_1})(\bcL_{\triangle} \ast_{\bPhi} \bcR_{\star}^H) , \widetilde{\bcL} \ast_{\bPhi} \bcG_{\star}^{\frac{1}{2}} \ast_{\bPhi} (\bcR_{\sharp}^H \ast_{\bPhi} \bcR_{\sharp})^{-1} \ast_{\bPhi} \bcR_{\star}^H \rangle \Big|  \\
&\qquad + \Big| \langle (p^{-1} \bcP_{\bOmega} - \bcI_{n_1})(\bcL_{\triangle} \ast_{\bPhi} \bcR_{\star}^H) , \widetilde{\bcL} \ast_{\bPhi} \bcG_{\star}^{\frac{1}{2}} \ast_{\bPhi} (\bcR_{\sharp}^H \ast_{\bPhi} \bcR_{\sharp})^{-1} \ast_{\bPhi} \bcR_{\triangle}^H \rangle \Big|  \\
&\qquad + \Big| \langle (p^{-1} \bcP_{\bOmega} - \bcI_{n_1})(\bcL_{\sharp} \ast_{\bPhi} \bcR_{\triangle}^H) , \widetilde{\bcL} \ast_{\bPhi} \bcG_{\star}^{\frac{1}{2}} \ast_{\bPhi} (\bcR_{\sharp}^H \ast_{\bPhi} \bcR_{\sharp})^{-1} \ast_{\bPhi} \bcR_{\sharp}^H \rangle \Big|  \\
& \coloneq \eta (\mathfrak{P}_{4,1} + \mathfrak{P}_{4,2} + \mathfrak{P}_{4,3}).
\end{align*}
Note that the decomposition of $\sqrt{\mathfrak{P}_4}$ is extremely similar to that of $\mathfrak{P}_2$, thus we can follow a similar argument to control these terms as
\begin{align*}
\mathfrak{P}_{4,1} & \leq \frac{c_1}{(1 - \epsilon)^2} \sqrt{\frac{\mu s_r (n_1 + n_2) \log( (n_1 \vee n_2) n_3 )}{p n_1 n_2 n_3} } \| \bcL_{\triangle} \ast_{\bPhi} \bcG_{\star}^{\frac{1}{2}} \|_F;  \\
\mathfrak{P}_{4,2} & \leq \frac{c_2}{(1 - \epsilon)^2} \Big(1 + \frac{c_{\varsigma}}{1 - \epsilon} \Big) \sqrt{\frac{\ell \log( (n_1 \vee n_2) n_3 )}{p (n_1 \wedge n_2)} } \frac{\mu s_r \kappa}{n_3} \| \bcR_{\triangle} \ast_{\bPhi} \bcG_{\star}^{\frac{1}{2}} \|_F;  \\
\mathfrak{P}_{4,3} & \leq \frac{c_2}{(1 - \epsilon)^4} \sqrt{\frac{\ell \log( (n_1 \vee n_2) n_3 )}{p (n_1 \wedge n_2)} } \frac{\mu s_r c_{\varsigma}^2 \kappa^2}{n_3} \| \bcR_{\triangle} \ast_{\bPhi} \bcG_{\star}^{\frac{1}{2}} \|_F.
\end{align*}
Hence,
\begin{align*}
\sqrt{\mathfrak{P}_4} \leq \eta (\nu_1 \| \bcL_{\triangle} \ast_{\bPhi} \bcG_{\star}^{\frac{1}{2}} \|_F + \nu_2 \| \bcR_{\triangle} \ast_{\bPhi} \bcG_{\star}^{\frac{1}{2}} \|_F).
\end{align*}
Taking the square on both sides to obtain the upper bound
\begin{align*}
|\mathfrak{P}_4| \leq \eta^2 \Big( \nu_1 (\nu_1 + \nu_2) \| \bcL_{\triangle} \ast_{\bPhi} \bcG_{\star}^{\frac{1}{2}} \|_F^2 + \nu_2 (\nu_1 + \nu_2) \| \bcR_{\triangle} \ast_{\bPhi} \bcG_{\star}^{\frac{1}{2}} \|_F^2 \Big).
\end{align*}
Taking the bounds for $\mathfrak{P}_1$, $\mathfrak{P}_2$, $\mathfrak{P}_3$ and $\mathfrak{P}_4$ collectively yields
\begin{align*}
& \| (\widetilde{\bcL}_{t+1} \ast_{\bPhi} \bcQ_t - \bcL_{\star}) \ast_{\bPhi} \bcG_{\star}^{\frac{1}{2}} \|_F^2  \\
& \leq \Big( (1 - \eta)^2 + \frac{2 \epsilon}{1 - \epsilon} \eta (1 - \eta) \Big) \| \bcL_{\triangle} \ast_{\bPhi} \bcG_{\star}^{\frac{1}{2}} \|_F^2 + \frac{2 \epsilon + \epsilon^2}{(1 - \epsilon)^2} \eta^2 \| \bcR_{\triangle} \ast_{\bPhi} \bcG_{\star}^{\frac{1}{2}} \|_F^2  \\
&\quad + \eta (1 - \eta) \Big( (2 \nu_1 + \nu_2) \| \bcL_{\triangle} \ast_{\bPhi} \bcG_{\star}^{\frac{1}{2}} \|_F^2 + \nu_2 \| \bcR_{\triangle} \ast_{\bPhi} \bcG_{\star}^{\frac{1}{2}} \|_F^2 \Big)  \\
&\quad + \eta^2 \Big( \nu_2 \| \bcL_{\triangle} \ast_{\bPhi} \bcG_{\star}^{\frac{1}{2}} \|_F^2 + (2 \nu_1 + \nu_2) \| \bcR_{\triangle} \ast_{\bPhi} \bcG_{\star}^{\frac{1}{2}} \|_F^2 \Big)  \\
&\quad + \eta^2 \Big( \nu_1 (\nu_1 + \nu_2) \| \bcL_{\triangle} \ast_{\bPhi} \bcG_{\star}^{\frac{1}{2}} \|_F^2 + \nu_2 (\nu_1 + \nu_2) \| \bcR_{\triangle} \ast_{\bPhi} \bcG_{\star}^{\frac{1}{2}} \|_F^2 \Big).
\end{align*}
A similar upper bound also holds for the second term in \eqref{eqn:TCexpand}. It then turns out that 
\begin{align*}
\| (\widetilde{\bcL}_{t+1} \ast_{\bPhi} \bcQ_t - \bcL_{\star}) \ast_{\bPhi} \bcG_{\star}^{\frac{1}{2}} \|_F^2 + \| (\widetilde{\bcR}_{t+1} \ast_{\bPhi} \bcQ_t^{-H} - \bcR_{\star}) \ast_{\bPhi} \bcG_{\star}^{\frac{1}{2}} \|_F^2 \leq \hbar^2(\eta;\epsilon,\nu_1,\nu_2) \dist^2(\bcF_t, \bcF_{\star}),
\end{align*}
where the contraction rate $\hbar^2(\eta;\epsilon,\nu_1,\nu_2)$ is given by 
\begin{align*}
\hbar^2 (\eta;\epsilon,\nu_1,\nu_2) \coloneq (1 - \eta)^2 + \Big( \frac{2 \epsilon}{1 - \epsilon} + 2(\nu_1 + \nu_2) \Big) \eta (1 - \eta) + \Big( \frac{2 \epsilon + \epsilon^2}{(1 - \epsilon)^2} + 2 (\nu_1 + \nu_2) + (\nu_1 + \nu_2)^2 \Big) \eta^2.
\end{align*}
As long as $p \geq c \Big( \frac{\mu s_r (n_1 + n_2) \log( (n_1 \vee n_2) n_3 )}{n_1 n_2 n_3} \vee \frac{\mu^2 s_r^2 \kappa^4 \ell \log( (n_1 \vee n_2) n_3 )}{(n_1 \wedge n_2) n_3^2} \Big)$ for some sufficiently large constant $c$, we have $\nu_1 + \nu_2 \leq 0.1$ under the setting $\epsilon = 0.02$. When $0 < \eta \leq \frac{2}{3}$, we further have $\hbar(\eta;\epsilon,\nu_1,\nu_2) \leq 1 - 0.6 \eta$. Thus we conclude that
\begin{align*}
\dist(\widetilde{\bcF}_{t+1}, \bcF_{\star}) & \leq \sqrt{\| (\widetilde{\bcL}_{t+1} \ast_{\bPhi} \bcQ_t - \bcL_{\star}) \ast_{\bPhi} \bcG_{\star}^{\frac{1}{2}} \|_F^2 + \| (\widetilde{\bcR}_{t+1} \ast_{\bPhi} \bcQ_t^{-H} - \bcR_{\star}) \ast_{\bPhi} \bcG_{\star}^{\frac{1}{2}} \|_F^2}  \\
& \leq (1 - 0.6 \eta) \dist(\bcF_t, \bcF_{\star}).
\end{align*}
This finishes the proof.

\subsubsection{Proof of Lemma~\ref{lemma:POmegaspectral}}

Denote the tensor $\bcH_{ijk} = (p^{-1} \delta_{ijk} - 1) \bcZ_{i,j,k} \bar{\boldsymbol{\ce}}_{ijk}$. Then we have
\begin{align*}
(p^{-1} \bcP_{\bOmega} - \bcI_{n_1}) (\bcZ) = \sum_{i,j,k} \bcH_{ijk}.
\end{align*}
Note that $\delta_{ijk}$'s are independent random scalars. Thus, $\bcH_{ijk}$'s are independent random tensors and $\widebar{\bH}_{ijk}$'s are independent random matrices. Observe that $\mathbb{E} [\widebar{\bH}_{ijk}] = \bzero$ and $\| \widebar{\bH}_{ijk} \| \leq p^{-1} \sqrt{\ell} \| \bcZ \|_{\infty}$ by using \eqref{eqn:eijkbound1}. According to the proof in Lemma~\ref{lemma:eijkbounds}, we know that $\sum_{i,j,k} \bcZ_{i,j,k}^2 L(\bar{\boldsymbol{\ce}}_{ijk}^H \ast_{\bPhi} \bar{\boldsymbol{\ce}}_{ijk})$ is f-diagonal and the $k'$-th entry of its $(j,j)$-th mode-3 tube is $\sum_{i,k} \bcZ_{i,j,k}^2 |\bPhi_{k',k}|^2$. We have
\begin{align*}
\sum_{i,k} \bcZ_{i,j,k}^2 |\bPhi_{k',k}|^2 \leq \sum_i \Big( \sum_k \bcZ_{i,j,k}^2 \Big) \Big( \max_k |\bPhi_{k',k}|^2 \Big) \leq \sum_{i,k} \bcZ_{i,j,k}^2 \| \bPhi \|_{\infty}^2 \leq \ell \sum_{i,k} \bcZ_{i,j,k}^2.
\end{align*}
Hence, the tensor spectral norm of $\sum_{i,j,k} \bcZ_{i,j,k}^2 (\bar{\boldsymbol{\ce}}_{ijk}^H \ast_{\bPhi} \bar{\boldsymbol{\ce}}_{ijk})$ has the bound
\begin{align*}
\Big\| \sum_{i,j,k} \bcZ_{i,j,k}^2 (\bar{\boldsymbol{\ce}}_{ijk}^H \ast_{\bPhi} \bar{\boldsymbol{\ce}}_{ijk}) \Big\| \leq \max_j \ell \sum_{i,k} \bcZ_{i,j,k}^2 \leq \ell \| \bcZ \|_{\infty,2}^2.
\end{align*}
Then we have
\begin{align*}
\Big\| \sum_{i,j,k} \mathbb{E} [\widebar{\bH}_{ijk}^H \widebar{\bH}_{ijk}] \Big\| & = \Big\| \sum_{i,j,k} \mathbb{E} [\bcH_{ijk}^H \ast_{\bPhi} \bcH_{ijk}] \Big\|  \\
& = \Big\| \sum_{i,j,k} \mathbb{E} [(1 - p^{-1} \delta_{ijk})^2] \bcZ_{i,j,k} (\bar{\boldsymbol{\ce}}_{ijk}^H \ast_{\bPhi} \bar{\boldsymbol{\ce}}_{ijk}) \Big\|  \\
& \leq \frac{(1 - p) \ell}{p} \| \bcZ \|_{\infty,2}^2  \\
& \leq p^{-1} \ell \| \bcZ \|_{\infty,2}^2.
\end{align*}
We can also have a similar calculation that yields $\| \sum_{i,j,k} \mathbb{E} [\widebar{\bH}_{ijk} \widebar{\bH}_{ijk}^H] \| \leq p^{-1} \ell \| \bcZ \|_{\infty,2}^2$. The proof is completed by applying the matrix Bernstein inequality in Lemma~\ref{lemma:Bernstein}.

\subsubsection{Proof of Lemma~\ref{lemma:POmegaFnorm}}

By Lemma~\ref{lemma:projerrorbound}, with high probability, for any $\bcA \in \bT$, it holds that
\begin{align}
p(1 - \epsilon) \| \bcA \|_F \leq \| \bcP_{\bT} \bcP_{\bOmega} \bcP_{\bT} (\bcA) \|_F \leq p(1 + \epsilon) \| \bcA \|_F.
\end{align}
Rewriting $\| \bcP_{\bOmega}(\bcA) \|_F = \langle \bcP_{\bOmega} \bcP_{\bT} (\bcA), \bcP_{\bOmega} \bcP_{\bT} (\bcA) \rangle = \langle \bcA, \bcP_{\bT} \bcP_{\bOmega} \bcP_{\bT} (\bcA) \rangle$, and using the Cauchy-Schwarz inequality, we can bound
\begin{align}\label{eqn:POmegaFnorm1-1}
\| \bcP_{\bOmega}(\bcA) \|_F^2 \leq p(1 + \epsilon) \| \bcA \|_F^2.
\end{align}
In addition, we have
\begin{align}\label{eqn:POmegaFnorm1-2}
\| \bcP_{\bOmega}(\bcA) \|_F^2 & = \langle \bcA, \bcP_{\bT} \bcP_{\bOmega} \bcP_{\bT} (\bcA) \rangle = \langle \bcA, \bcP_{\bT} \bcP_{\bOmega} \bcP_{\bT} (\bcA) - p \bcP_{\bT} (\bcA) + p \bcP_{\bT} (\bcA) \rangle  \nonumber \\
& \geq - \| \bcA \|_F \| \bcP_{\bT} \bcP_{\bOmega} \bcP_{\bT} (\bcA) - p \bcP_{\bT} (\bcA) \|_F + p \| \bcA \|_F^2  \nonumber \\
& \geq p(1 - \epsilon) \| \bcA \|_F^2,
\end{align}
where the last inequality follows from Lemma~\ref{lemma:projerrorbound}. Combining \eqref{eqn:POmegaFnorm1-1} and \eqref{eqn:POmegaFnorm1-2} proves \eqref{eqn:POmegaFnorm1}. To show \eqref{eqn:POmegaFnorm2}, let $\bcA' = \frac{\bcA}{\| \bcA \|_F}$ and $\bcB' = \frac{\bcB}{\| \bcB \|_F}$. Both $\bcA' + \bcB'$ and $\bcA' - \bcB'$ are in $\bT$. We have
\begin{align*}
\langle \bcP_{\bOmega}(\bcA'), \bcP_{\bOmega}(\bcB') \rangle & = \frac{1}{4} \Big( \| \bcP_{\bOmega} (\bcA' + \bcB') \|_F^2 - \| \bcP_{\bOmega} (\bcA' - \bcB') \|_F^2 \Big)  \\
& \leq \frac{1}{4} \Big( p(1 + \epsilon) \| \bcA' + \bcB' \|_F^2 - p(1 - \epsilon) \| \bcA' - \bcB' \|_F^2 \Big)  \\
& = \frac{1}{4} \Big( 2 p \epsilon ( \| \bcA' \|_F^2 + \| \bcB' \|_F^2 ) + 4p \langle \bcA', \bcB' \rangle \Big)  \\
& = p \epsilon + p \langle \bcA', \bcB' \rangle,
\end{align*}
where the first inequality follows from \eqref{eqn:POmegaFnorm1}. Thus, we have
\begin{align*}
p^{-1} \langle \bcP_{\bOmega}(\bcA), \bcP_{\bOmega}(\bcB) \rangle = p^{-1} \| \bcA \|_F \| \bcB \|_F \langle \bcP_{\bOmega}(\bcA'), \bcP_{\bOmega}(\bcB') \rangle \leq \epsilon \| \bcA \|_F \| \bcB \|_F + \langle \bcA, \bcB \rangle.
\end{align*}
Similarly, we can show
\begin{align*}
p^{-1} \langle \bcP_{\bOmega}(\bcA), \bcP_{\bOmega}(\bcB) \rangle \geq - \epsilon \| \bcA \|_F \| \bcB \|_F + \langle \bcA, \bcB \rangle.
\end{align*}
The proof is completed.

\subsubsection{Proof of Lemma~\ref{lemma:POmeganonconvex}}

First, using the definition of tensor inner product, we have
\begin{align*}
|\langle \bcA, \bcB \rangle| = \frac{1}{\ell} |\langle \widebar{\bA}, \widebar{\bB} \rangle| \leq \frac{1}{\ell} \| \widebar{\bA} \| \| \widebar{\bB} \|_{\ast} = \| \bcA \| \| \bcB \|_{\ast},
\end{align*}
for any $\bcA, \bcB$, where the inequality holds by matrix H\"{o}lder's inequality. Hence,
\begin{align}\label{eqn:specbound}
& | p^{-1} \langle \bcP_{\bOmega}(\bcL_A \ast_{\bPhi} \bcR_A^H), \bcP_{\bOmega}(\bcL_B \ast_{\bPhi} \bcR_B^H) \rangle - \langle \bcL_A \ast_{\bPhi} \bcR_A^H, \bcL_B \ast_{\bPhi} \bcR_B^H \rangle|  \nonumber \\
= & |\langle (p^{-1} \bcP_{\bOmega} - \bcI_{n_1}) (\bcJ), \big( (\bcL_A \ast_{\bPhi} \bcR_A^H) \circ (\bcL_B \ast_{\bPhi} \bcR_B^H) \big) \rangle|  \nonumber \\
\leq & \| (p^{-1} \bcP_{\bOmega} - \bcI_{n_1}) (\bcJ) \| \| (\bcL_A \ast_{\bPhi} \bcR_A^H) \circ (\bcL_B \ast_{\bPhi} \bcR_B^H) \|_{\ast},
\end{align}
where $\bcJ$ denotes tensor with all-one entries and $\circ$ denotes the Hadamard (elementwise) product. To bound $\| (p^{-1} \bcP_{\bOmega} - \bcI_{n_1}) (\bcJ) \|$, we again denote the tensor $\bcH_{ijk} = (p^{-1} \delta_{ijk} - 1) \bar{\boldsymbol{\ce}}_{ijk}$. Then we have
\begin{align*}
(p^{-1} \bcP_{\bOmega} - \bcI_{n_1}) (\bcJ) = \sum_{i,j,k} \bcH_{ijk}.
\end{align*}
Note that $\delta_{ijk}$'s are independent random scalars. Thus, $\bcH_{ijk}$'s are independent random tensors and $\widebar{\bH}_{ijk}$'s are independent random matrices. Observe that $\mathbb{E} [\widebar{\bH}_{ijk}] = \bzero$ and $\| \widebar{\bH}_{ijk} \| \leq p^{-1} \sqrt{\ell}$. Then we have
\begin{align*}
\Big\| \sum_{i,j,k} \mathbb{E} [\widebar{\bH}_{ijk}^H \widebar{\bH}_{ijk}] \Big\| & = \Big\| \sum_{i,j,k} \mathbb{E} [(1 - p^{-1} \delta_{ijk})^2] (\bar{\boldsymbol{\ce}}_{ijk}^H \ast_{\bPhi} \bar{\boldsymbol{\ce}}_{ijk}) \Big\|  \\
& = \Big\| \frac{1 - p}{p} \sum_{i,j,k} (\bar{\boldsymbol{\ce}}_{ijk}^H \ast_{\bPhi} \bar{\boldsymbol{\ce}}_{ijk}) \Big\|  \\
& \leq \frac{1 - p}{p} \Big\| \sum_{i,j,k} (\bar{\boldsymbol{\ce}}_{ijk}^H \ast_{\bPhi} \bar{\boldsymbol{\ce}}_{ijk}) \Big\|  \\
& \leq \frac{n_1 \ell}{p},
\end{align*}
where the last step uses \eqref{eqn:eijkbound2}. A similar calculation yields $\| \sum_{i,j,k} \mathbb{E} [\widebar{\bH}_{ijk} \widebar{\bH}_{ijk}^H] \| \leq \frac{n_2 \ell}{p}$. Using Lemma~\ref{lemma:Bernstein}, let $\epsilon = \sqrt{c (n_1 \vee n_2) \ell \log( (n_1 \vee n_2) n_3 ) / p}$, then with high probability,
\begin{align*}
\| (p^{-1} \bcP_{\bOmega} - \bcI_{n_1}) (\bcJ) \| \leq \epsilon,
\end{align*}
provided that $p \geq c \log( (n_1 \vee n_2) n_3 ) / (n_1 \wedge n_2)$. To give a bound of $\| (\bcL_A \ast_{\bPhi} \bcR_A^H) \circ (\bcL_B \ast_{\bPhi} \bcR_B^H) \|_{\ast}$, notice that it is equal to
\begin{align*}
\| (\bcL_A \ast_{\bPhi} \bcR_A^H) \circ (\bcL_B \ast_{\bPhi} \bcR_B^H) \|_{\ast} = \sum_{k=1}^{n_3} \frac{1}{\ell} \| (\xoverline{\bcL}_A(:,:,k) \xoverline{\bcR}_A(:,:,k)^H) \circ (\xoverline{\bcL}_B(:,:,k) \xoverline{\bcR}_B(:,:,k)^H) \|_{\ast},
\end{align*}
and we can decompose each $(\xoverline{\bcL}_A(:,:,k) \xoverline{\bcR}_A(:,:,k)^H) \circ (\xoverline{\bcL}_B(:,:,k) \xoverline{\bcR}_B(:,:,k)^H)$ into sum of rank one matrices as follows
\begin{align*}
& (\xoverline{\bcL}_A(:,:,k) \xoverline{\bcR}_A(:,:,k)^H) \circ (\xoverline{\bcL}_B(:,:,k) \xoverline{\bcR}_B(:,:,k)^H)  \\
= & \Big(\sum_{j=1}^{r_1} \xoverline{\bcL}_A(:,j,k) \xoverline{\bcR}_A(:,j,k)^H \Big) \circ \Big(\sum_{j=1}^{r_2} \xoverline{\bcL}_B(:,j,k) \xoverline{\bcR}_B(:,j,k)^H \Big)  \\
= & \sum_{j=1}^{r_1} \sum_{j'=1}^{r_2} \big( \xoverline{\bcL}_A(:,j,k) \circ \xoverline{\bcL}_B(:,j',k) \big) \cdot \big( \xoverline{\bcR}_A(:,j,k) \circ \xoverline{\bcR}_B(:,j',k) \big)^H.
\end{align*}
We can then find an upper bound of $\| (\bcL_A \ast_{\bPhi} \bcR_A^H) \circ (\bcL_B \ast_{\bPhi} \bcR_B^H) \|_{\ast}$ via
\begin{align*}
& \| (\bcL_A \ast_{\bPhi} \bcR_A^H) \circ (\bcL_B \ast_{\bPhi} \bcR_B^H) \|_{\ast}  \\
\leq & \sum_{k=1}^{n_3} \sum_{j=1}^{r_1} \sum_{j'=1}^{r_2} \frac{1}{\ell} \Big\| \big( \xoverline{\bcL}_A(:,j,k) \circ \xoverline{\bcL}_B(:,j',k) \big) \cdot \big( \xoverline{\bcR}_A(:,j,k) \circ \xoverline{\bcR}_B(:,j',k) \big)^H \Big\|_{\ast}  \\
= & \sum_{k=1}^{n_3} \sum_{j=1}^{r_1} \sum_{j'=1}^{r_2} \frac{1}{\ell} \Big\| \xoverline{\bcL}_A(:,j,k) \circ \xoverline{\bcL}_B(:,j',k) \Big\|_2 \Big\| \xoverline{\bcR}_A(:,j,k) \circ \xoverline{\bcR}_B(:,j',k) \Big\|_2  \\
= & \sum_{k=1}^{n_3} \sum_{j=1}^{r_1} \sum_{j'=1}^{r_2} \frac{1}{\ell} \sqrt{\sum_{i=1}^{n_1} |(\xoverline{\bcL}_A)_{i,j,k}|^2 |(\xoverline{\bcL}_B)_{i,j',k}|^2} \sqrt{\sum_{i=1}^{n_2} |(\xoverline{\bcR}_A)_{i,j,k}|^2 |(\xoverline{\bcR}_B)_{i,j',k}|^2},
\end{align*}
where we replace nuclear norm by vector $\ell_2$ norms in the second last line because the summands are all rank one matrices. Applying Cauchy-Schwarz inequality twice, we have
\begin{align}\label{eqn:nuclearbound}
& \| (\bcL_A \ast_{\bPhi} \bcR_A^H) \circ (\bcL_B \ast_{\bPhi} \bcR_B^H) \|_{\ast}  \nonumber \\
\leq & \sum_{k=1}^{n_3} \frac{1}{\ell} \sqrt{\sum_{j=1}^{r_1} \sum_{j'=1}^{r_2} \sum_{i=1}^{n_1} |(\xoverline{\bcL}_A)_{i,j,k}|^2 |(\xoverline{\bcL}_B)_{i,j',k}|^2} \sqrt{\sum_{j=1}^{r_1} \sum_{j'=1}^{r_2} \sum_{i=1}^{n_2} |(\xoverline{\bcR}_A)_{i,j,k}|^2 |(\xoverline{\bcR}_B)_{i,j',k}|^2}  \nonumber \\
= & \sum_{k=1}^{n_3} \frac{1}{\ell} \sqrt{\sum_{i=1}^{n_1} \|\xoverline{\bcL}_A(i,:,k)\|_2^2 \|\xoverline{\bcL}_B(i,:,k)\|_2^2} \sqrt{\sum_{i=1}^{n_2} \|\xoverline{\bcR}_A(i,:,k)\|_2^2 \|\xoverline{\bcR}_B(i,:,k)\|_2^2}  \nonumber \\
\leq & \frac{1}{\ell} \sqrt{\sum_{k=1}^{n_3} \sum_{i=1}^{n_1} \|\xoverline{\bcL}_A(i,:,k)\|_2^2 \|\xoverline{\bcL}_B(i,:,k)\|_2^2} \sqrt{\sum_{k=1}^{n_3} \sum_{i=1}^{n_2} \|\xoverline{\bcR}_A(i,:,k)\|_2^2 \|\xoverline{\bcR}_B(i,:,k)\|_2^2}  \nonumber \\
= & \ell \sqrt{\sum_{k=1}^{n_3} \sum_{i=1}^{n_1} \|\bcL_A(i,:,k)\|_2^2 \|\bcL_B(i,:,k)\|_2^2} \sqrt{\sum_{k=1}^{n_3} \sum_{i=1}^{n_2} \|\bcR_A(i,:,k)\|_2^2 \|\bcR_B(i,:,k)\|_2^2}  \nonumber \\
\leq & \ell \Big( \|\bcL_A\|_{2,2,\infty} \|\bcL_B\|_F \wedge \|\bcL_A\|_F \|\bcL_B\|_{2,2,\infty} \Big) \Big( \|\bcR_A\|_{2,2,\infty} \|\bcR_B\|_F \wedge \|\bcR_A\|_F \|\bcR_B\|_{2,2,\infty} \Big).
\end{align}
Putting \eqref{eqn:specbound} and \eqref{eqn:nuclearbound} together, we have
\begin{align*}
& | p^{-1} \langle \bcP_{\bOmega}(\bcL_A \ast_{\bPhi} \bcR_A^H), \bcP_{\bOmega}(\bcL_B \ast_{\bPhi} \bcR_B^H) \rangle - \langle \bcL_A \ast_{\bPhi} \bcR_A^H, \bcL_B \ast_{\bPhi} \bcR_B^H \rangle|  \\
\leq & c \ell^{\frac{3}{2}} \sqrt{\frac{(n_1 \vee n_2) \log( (n_1 \vee n_2) n_3 )}{p} }  \\
&\qquad \Big( \|\bcL_A\|_{2,2,\infty} \|\bcL_B\|_F \wedge \|\bcL_A\|_F \|\bcL_B\|_{2,2,\infty} \Big) \Big( \|\bcR_A\|_{2,2,\infty} \|\bcR_B\|_F \wedge \|\bcR_A\|_F \|\bcR_B\|_{2,2,\infty} \Big).
\end{align*}

\subsection{Proof of Lemma~\ref{lemma:TCinitial}}

In view of Lemma~\ref{lemma:Procrustes}, we have
\begin{align*}
\dist(\widetilde{\bcF}_0, \bcF_{\star}) & \leq \sqrt{\sqrt{2}+1} \| \bcU_0 \ast_{\bPhi} \bcG_0 \ast_{\bPhi} \bcV_0^H - \bcX_{\star} \|_F = \sqrt{\frac{\sqrt{2}+1}{\ell}} \| \widebar{\bU}_0 \widebar{\bG}_0 \widebar{\bV}_0^H - \widebar{\bX}_{\star} \|_F  \\
& \leq \sqrt{\frac{(\sqrt{2}+1) 2 s_r}{\ell}} \| \widebar{\bU}_0 \widebar{\bG}_0 \widebar{\bV}_0^H - \widebar{\bX}_{\star} \| = \sqrt{\frac{(\sqrt{2}+1) 2 s_r}{\ell}} \| \bcU_0 \ast_{\bPhi} \bcG_0 \ast_{\bPhi} \bcV_0^H - \bcX_{\star} \|,
\end{align*}
where we use the fact that $\widebar{\bU}_0 \widebar{\bG}_0 \widebar{\bV}_0^H - \widebar{\bX}_{\star}$ has rank at most $2 s_r$. Applying the triangle inequality, we obtain 
\begin{align*}
\| \bcU_0 \ast_{\bPhi} \bcG_0 \ast_{\bPhi} \bcV_0^H - \bcX_{\star} \| & \leq \| p^{-1} \bcP_{\bOmega}(\bcX_{\star}) - \bcU_0 \ast_{\bPhi} \bcG_0 \ast_{\bPhi} \bcV_0^H \| + \| p^{-1} \bcP_{\bOmega}(\bcX_{\star}) - \bcX_{\star} \|  \\
& \leq 2 \| p^{-1} \bcP_{\bOmega}(\bcX_{\star}) - \bcX_{\star} \|,
\end{align*}
where the second inequality relies on the fact that $\bcU_0 \ast_{\bPhi} \bcG_0 \ast_{\bPhi} \bcV_0^H$ is the best tubal rank-$r$ approximation to $p^{-1} \bcP_{\bOmega}(\bcX_{\star})$, i.e., $\| p^{-1} \bcP_{\bOmega}(\bcX_{\star}) - \bcU_0 \ast_{\bPhi} \bcG_0 \ast_{\bPhi} \bcV_0^H \| \leq \| p^{-1} \bcP_{\bOmega}(\bcX_{\star}) - \bcX_{\star} \|$. Combining the above two inequalities yields
\begin{align*}
\dist(\widetilde{\bcF}_0, \bcF_{\star}) \leq 2 \sqrt{\frac{(\sqrt{2}+1) 2 s_r}{\ell}} \| p^{-1} \bcP_{\bOmega}(\bcX_{\star}) - \bcX_{\star} \| \leq 5 \sqrt{\frac{s_r}{\ell}} \| p^{-1} \bcP_{\bOmega}(\bcX_{\star}) - \bcX_{\star} \|.
\end{align*}
Using Lemma~\ref{lemma:POmegaspectral}, we know that
\begin{align*}
\| (p^{-1} \bcP_{\bOmega} - \bcI_{n_1}) (\bcX_{\star}) \| \leq c \Big( \frac{\sqrt{\ell} \log( (n_1 \vee n_2) n_3 )}{p} \| \bcX_{\star} \|_{\infty} + \sqrt{\frac{\ell \log( (n_1 \vee n_2) n_3 )}{p}}\| \bcX_{\star} \|_{\infty,2} \Big),
\end{align*}
holds with high probability. The proof is finished by applying Lemma~\ref{lemma:infnormbound} and Lemma~\ref{lemma:2infbound} and plugging the following bounds from incoherence assumption of $\bX_{\star}$:
\begin{align*}
\| \bcX_{\star} \|_{\infty} & \leq \sqrt{\ell} \| \bcU_{\star} \|_{2,\infty} \| \bcG_{\star} \| \| \bcV_{\star} \|_{2,\infty} \leq \frac{\mu s_r}{\sqrt{n_1 n_2 \ell} n_3} \kappa \bar{\sigma}_{s_r} (\bcX_{\star});  \\
\| \bcX_{\star} \|_{\infty,2} & \leq \| \bcU_{\star} \|_{2,\infty} \| \bcG_{\star} \| \| \bcV_{\star} \| \vee \| \bcU_{\star} \| \| \bcG_{\star} \| \| \bcV_{\star} \|_{2,\infty} \leq \sqrt{\frac{\mu s_r}{(n_1 \wedge n_2) n_3 \ell}} \kappa \bar{\sigma}_{s_r} (\bcX_{\star}).
\end{align*}

\section{Proof for Robust Tensor Completion}
\label{sec:RobustCompletionproof}

\begin{lemma}\label{lemma:POmegarownorm}
Suppose that $\bcZ \in \R^{n_1 \times n_2 \times n_3}$ is a fixed tensor and $\bOmega \sim \mathrm{Ber}(p)$. Then with high probability, 
\begin{align*}
\| (p^{-1} \bcP_{\bOmega} - \bcI_{n_1}) (\bcZ) \|_{2,\infty} \leq 2 \sqrt{\frac{c \log(n_2 n_3)}{p}} \| \bcZ \|_{2,\infty} + \frac{c \log(n_2 n_3)}{p} \| \bcZ \|_{\infty},
\end{align*}
for some numerical constant $c > 0$.
\end{lemma}

\begin{proof}
For any fixed $b \in [n_1]$, the $b$-th horizontal slice of the tensor $(p^{-1} \bcP_{\bOmega} - \bcI_{n_1}) (\bcZ)$ can be written as
\begin{align*}
[(p^{-1} \bcP_{\bOmega} - \bcI_{n_1}) (\bcZ)]^H \ast_{\bPhi} \mathring{\boldsymbol{\ce}}_b = \sum_{i,j,k} (p^{-1} \delta_{ijk} - 1) \bcZ_{i,j,k} \bar{\boldsymbol{\ce}}_{ijk}^H \ast_{\bPhi} \mathring{\boldsymbol{\ce}}_b \coloneq \sum_{i,j,k} \bcH_{ijk},
\end{align*}
where $\bcH_{ijk}$'s are independent row tensors of size $n_2 \times 1 \times n_3$ and $\mathbb{E} [\widebar{\bH}_{ijk}] = \bzero$. Let $\bh_{ijk} \in \R^{n_2 n_3}$ be the column vector obtained by vectorizing $\bcH_{ijk}$. Then we have
\begin{align*}
\| \bh_{ijk} \|_2 \leq p^{-1} \bcZ_{i,j,k} \| \bar{\boldsymbol{\ce}}_{ijk}^H \ast_{\bPhi} \mathring{\boldsymbol{\ce}}_b \|_F \leq p^{-1} \| \bcZ \|_{\infty}.
\end{align*}
We also have
\begin{align*}
\Big| \sum_{i,j,k} \mathbb{E} [\bh_{ijk}^H \bh_{ijk}] \Big| = \Big| \sum_{i,j,k} \mathbb{E} [\| \bcH_{ijk} \|_F^2] \Big| = \frac{1 - p}{p} \sum_{i,j,k} \bcZ_{i,j,k}^2 \| \bar{\boldsymbol{\ce}}_{ijk}^H \ast_{\bPhi} \mathring{\boldsymbol{\ce}}_b \|_F^2.
\end{align*}
Note that $\bar{\boldsymbol{\ce}}_{ijk}^H \ast_{\bPhi} \mathring{\boldsymbol{\ce}}_b = \bzero$ if $i \neq b$. We further have
\begin{align*}
\Big| \sum_{i,j,k} \mathbb{E} [\bh_{ijk}^H \bh_{ijk}] \Big| = \frac{1 - p}{p} \sum_{i,j,k} \bcZ_{i,j,k}^2 \| \bar{\boldsymbol{\ce}}_{ijk}^H \ast_{\bPhi} \mathring{\boldsymbol{\ce}}_b \|_F^2 = \frac{1 - p}{p} \sum_{j,k} \bcZ_{b,j,k}^2 \| \bar{\boldsymbol{\ce}}_{bjk}^H \|_F^2 \leq p^{-1} \| \bcZ \|_{2,\infty}^2.
\end{align*}
We can bound $\| \sum_{i,j,k} \mathbb{E} [\bh_{ijk} \bh_{ijk}^H] \|$ by the same quantity in a similar manner. Treating $\bh_{ijk}$'s as $n_2 n_3 \times 1$ matrices and applying the matrix Bernstein inequality in Lemma~\ref{lemma:Bernstein} gives that with high probability,
\begin{align*}
\| [(p^{-1} \bcP_{\bOmega} - \bcI_{n_1}) (\bcZ)]^H \ast_{\bPhi} \mathring{\boldsymbol{\ce}}_b \|_F & = \| \bcH_{ijk} \|_F = \| \bh_{ijk} \|_F  \\
& \leq 2 \sqrt{\frac{c \log(n_2 n_3)}{p}} \| \bcZ \|_{2,\infty} + \frac{c \log(n_2 n_3)}{p} \| \bcZ \|_{\infty}
\end{align*}
for some numerical constant $c > 0$.
\end{proof}

\subsection{Proof of Lemma~\ref{lemma:RTCcontraction}}

We prove Lemma~\ref{lemma:RTCcontraction} again by induction.

\paragraph{Base case.} Since $\rho^0 = 1$, the assumed initial conditions satisfy the base case at $t = 0$.

\paragraph{Induction step.} At the $t$-th iteration, we assume the conditions
\begin{align*}
\dist(\bcF_t, \bcF_{\star}) & \leq \frac{\epsilon}{\sqrt{\ell}} \rho^t \bar{\sigma}_{s_r} (\bcX_{\star}),  \\
\sqrt{n_1} \| (\bcL_t \ast_{\bPhi} \bcQ_t - \bcL_{\star}) \ast_{\bPhi} \bcG_{\star}^{\frac{1}{2}} \|_{2,\infty} & \vee \sqrt{n_2} \| (\bcR_t \ast_{\bPhi} \bcQ_t^{-H} - \bcR_{\star}) \ast_{\bPhi} \bcG_{\star}^{\frac{1}{2}} \|_{2,\infty} \leq \sqrt{\frac{\mu s_r}{n_3 \ell}} \rho^t \bar{\sigma}_{s_r} (\bcX_{\star})
\end{align*}
hold. In view of the condition $\dist(\bcF_t, \bcF_{\star}) \leq \frac{0.02}{\sqrt{\ell}} \rho^t \bar{\sigma}_{s_r} (\bcX_{\star})$ and Lemma~\ref{lemma:Qexistence}, one knows that $\bcQ_t$ exists and $\epsilon \coloneq 0.02$.

\emph{Distance contraction:} By the definition of $\dist^2(\bcF_{t+1},\bcF_{\star})$, we have
\begin{align*}
\dist^2(\bcF_{t+1},\bcF_{\star}) \leq \| (\bcL_{t+1} \ast_{\bPhi} \bcQ_t - \bcL_{\star}) \ast_{\bPhi} \bcG_{\star}^{\frac{1}{2}} \|_F^2 + \| (\bcR_{t+1} \ast_{\bPhi} \bcQ_t^{-H} - \bcR_{\star}) \ast_{\bPhi} \bcG_{\star}^{\frac{1}{2}} \|_F^2.
\end{align*}
According to the update rule \eqref{eqn:RTCupdate}, we have
\begin{align}\label{eqn:RTCexpand}
& (\bcL_{t+1} \ast_{\bPhi} \bcQ_t - \bcL_{\star}) \ast_{\bPhi} \bcG_{\star}^{\frac{1}{2}}  \nonumber \\
= & \Big( \bcL_{\sharp} - \eta p^{-1} \bcP_{\bOmega} (\bcL_{\sharp} \ast_{\bPhi} \bcR_{\sharp}^H + \bcS_{t+1} - \bcX_{\star} - \bcS_{\star}) \ast_{\bPhi} \bcR_{\sharp} \ast_{\bPhi} (\bcR_{\sharp}^H \ast_{\bPhi} \bcR_{\sharp})^{-1} - \bcL_{\star} \Big) \ast_{\bPhi} \bcG_{\star}^{\frac{1}{2}}  \nonumber \\
= & \Big( \bcL_{\triangle} - \eta p^{-1} \bcP_{\bOmega} (\bcL_{\sharp} \ast_{\bPhi} \bcR_{\sharp}^H - \bcX_{\star}) \ast_{\bPhi} \bcR_{\sharp} \ast_{\bPhi} (\bcR_{\sharp}^H \ast_{\bPhi} \bcR_{\sharp})^{-1}  \nonumber \\
&\qquad\qquad - \eta p^{-1} \bcP_{\bOmega} (\bcS_{t+1} - \bcS_{\star}) \ast_{\bPhi} \bcR_{\sharp} \ast_{\bPhi} (\bcR_{\sharp}^H \ast_{\bPhi} \bcR_{\sharp})^{-1} \Big) \ast_{\bPhi} \bcG_{\star}^{\frac{1}{2}}  \nonumber \\
= & \bcL_{\triangle} \ast_{\bPhi} \bcG_{\star}^{\frac{1}{2}} - \eta (\bcL_{\sharp} \ast_{\bPhi} \bcR_{\sharp}^H - \bcX_{\star}) \ast_{\bPhi} \bcR_{\sharp} \ast_{\bPhi} (\bcR_{\sharp}^H \ast_{\bPhi} \bcR_{\sharp})^{-1} \ast_{\bPhi} \bcG_{\star}^{\frac{1}{2}}  \nonumber \\
&\qquad\qquad - \eta (p^{-1} \bcP_{\bOmega} - \bcI_{n_1})(\bcL_{\sharp} \ast_{\bPhi} \bcR_{\sharp}^H - \bcX_{\star}) \ast_{\bPhi} \bcR_{\sharp} \ast_{\bPhi} (\bcR_{\sharp}^H \ast_{\bPhi} \bcR_{\sharp})^{-1} \ast_{\bPhi} \bcG_{\star}^{\frac{1}{2}}  \nonumber \\
&\qquad\qquad - \eta p^{-1} \bcP_{\bOmega} (\bcS_{\triangle}) \ast_{\bPhi} \bcR_{\sharp} \ast_{\bPhi} (\bcR_{\sharp}^H \ast_{\bPhi} \bcR_{\sharp})^{-1} \ast_{\bPhi} \bcG_{\star}^{\frac{1}{2}}  \nonumber \\
= & (1 - \eta) \bcL_{\triangle} \ast_{\bPhi} \bcG_{\star}^{\frac{1}{2}} - \eta \bcL_{\star} \ast_{\bPhi} \bcR_{\triangle}^H \ast_{\bPhi} \bcR_{\sharp} \ast_{\bPhi} (\bcR_{\sharp}^H \ast_{\bPhi} \bcR_{\sharp})^{-1} \ast_{\bPhi} \bcG_{\star}^{\frac{1}{2}}  \nonumber \\
&\qquad\qquad - \eta (p^{-1} \bcP_{\bOmega} - \bcI_{n_1})(\bcL_{\sharp} \ast_{\bPhi} \bcR_{\sharp}^H - \bcX_{\star}) \ast_{\bPhi} \bcR_{\sharp} \ast_{\bPhi} (\bcR_{\sharp}^H \ast_{\bPhi} \bcR_{\sharp})^{-1} \ast_{\bPhi} \bcG_{\star}^{\frac{1}{2}}  \nonumber \\
&\qquad\qquad - \eta p^{-1} \bcP_{\bOmega} (\bcS_{\triangle}) \ast_{\bPhi} \bcR_{\sharp} \ast_{\bPhi} (\bcR_{\sharp}^H \ast_{\bPhi} \bcR_{\sharp})^{-1} \ast_{\bPhi} \bcG_{\star}^{\frac{1}{2}}.
\end{align}
Taking the squared Frobenius norm of both sides of \eqref{eqn:RTCexpand} to obtain
\begin{align*}
& \| (\bcL_{t+1} \ast_{\bPhi} \bcQ_t - \bcL_{\star}) \ast_{\bPhi} \bcG_{\star}^{\frac{1}{2}} \|_F^2  \\
= & \| (1 - \eta) \bcL_{\triangle} \ast_{\bPhi} \bcG_{\star}^{\frac{1}{2}} - \eta \bcL_{\star} \ast_{\bPhi} \bcR_{\triangle}^H \ast_{\bPhi} \bcR_{\sharp} \ast_{\bPhi} (\bcR_{\sharp}^H \ast_{\bPhi} \bcR_{\sharp})^{-1} \ast_{\bPhi} \bcG_{\star}^{\frac{1}{2}} \|_F^2  \\
\qquad & - 2 \eta (1 - \eta) \langle \bcL_{\triangle} \ast_{\bPhi} \bcG_{\star}^{\frac{1}{2}} , (p^{-1} \bcP_{\bOmega} - \bcI_{n_1})(\bcL_{\sharp} \ast_{\bPhi} \bcR_{\sharp}^H - \bcX_{\star}) \ast_{\bPhi} \bcR_{\sharp} \ast_{\bPhi} (\bcR_{\sharp}^H \ast_{\bPhi} \bcR_{\sharp})^{-1} \ast_{\bPhi} \bcG_{\star}^{\frac{1}{2}} \rangle  \\
\qquad & + 2 \eta^2 \langle \bcL_{\star} \ast_{\bPhi} \bcR_{\triangle}^H \ast_{\bPhi} \bcR_{\sharp} \ast_{\bPhi} (\bcR_{\sharp}^H \ast_{\bPhi} \bcR_{\sharp})^{-1} \ast_{\bPhi} \bcG_{\star}^{\frac{1}{2}} ,  \\
&\qquad\qquad\qquad (p^{-1} \bcP_{\bOmega} - \bcI_{n_1})(\bcL_{\sharp} \ast_{\bPhi} \bcR_{\sharp}^H - \bcX_{\star}) \ast_{\bPhi} \bcR_{\sharp} \ast_{\bPhi} (\bcR_{\sharp}^H \ast_{\bPhi} \bcR_{\sharp})^{-1} \ast_{\bPhi} \bcG_{\star}^{\frac{1}{2}} \rangle  \\
\qquad & + \eta^2 \| (p^{-1} \bcP_{\bOmega} - \bcI_{n_1})(\bcL_{\sharp} \ast_{\bPhi} \bcR_{\sharp}^H - \bcX_{\star}) \ast_{\bPhi} \bcR_{\sharp} \ast_{\bPhi} (\bcR_{\sharp}^H \ast_{\bPhi} \bcR_{\sharp})^{-1} \ast_{\bPhi} \bcG_{\star}^{\frac{1}{2}} \|_F^2  \\
\qquad & - 2 \eta (1 - \eta) \langle \bcL_{\triangle} \ast_{\bPhi} \bcG_{\star}^{\frac{1}{2}} , p^{-1} \bcP_{\bOmega} (\bcS_{\triangle}) \ast_{\bPhi} \bcR_{\sharp} \ast_{\bPhi} (\bcR_{\sharp}^H \ast_{\bPhi} \bcR_{\sharp})^{-1} \ast_{\bPhi} \bcG_{\star}^{\frac{1}{2}} \rangle  \\
\qquad & + 2 \eta^2 \langle \bcL_{\star} \ast_{\bPhi} \bcR_{\triangle}^H \ast_{\bPhi} \bcR_{\sharp} \ast_{\bPhi} (\bcR_{\sharp}^H \ast_{\bPhi} \bcR_{\sharp})^{-1} \ast_{\bPhi} \bcG_{\star}^{\frac{1}{2}} ,  \\
&\qquad\qquad\qquad p^{-1} \bcP_{\bOmega} (\bcS_{\triangle}) \ast_{\bPhi} \bcR_{\sharp} \ast_{\bPhi} (\bcR_{\sharp}^H \ast_{\bPhi} \bcR_{\sharp})^{-1} \ast_{\bPhi} \bcG_{\star}^{\frac{1}{2}} \rangle  \\
\qquad & + 2 \eta^2 \langle (p^{-1} \bcP_{\bOmega} - \bcI_{n_1})(\bcL_{\sharp} \ast_{\bPhi} \bcR_{\sharp}^H - \bcX_{\star}) \ast_{\bPhi} \bcR_{\sharp} \ast_{\bPhi} (\bcR_{\sharp}^H \ast_{\bPhi} \bcR_{\sharp})^{-1} \ast_{\bPhi} \bcG_{\star}^{\frac{1}{2}} ,  \\
&\qquad\qquad\qquad p^{-1} \bcP_{\bOmega} (\bcS_{\triangle}) \ast_{\bPhi} \bcR_{\sharp} \ast_{\bPhi} (\bcR_{\sharp}^H \ast_{\bPhi} \bcR_{\sharp})^{-1} \ast_{\bPhi} \bcG_{\star}^{\frac{1}{2}} \rangle  \\
\qquad & + \eta^2 \| p^{-1} \bcP_{\bOmega} (\bcS_{\triangle}) \ast_{\bPhi} \bcR_{\sharp} \ast_{\bPhi} (\bcR_{\sharp}^H \ast_{\bPhi} \bcR_{\sharp})^{-1} \ast_{\bPhi} \bcG_{\star}^{\frac{1}{2}} \|_F^2  \\
\coloneq & \mathfrak{T}_1 - \mathfrak{T}_2 + \mathfrak{T}_3 + \mathfrak{T}_4 - \mathfrak{T}_5 + \mathfrak{T}_6 + \mathfrak{T}_7 + \mathfrak{T}_8.
\end{align*}

\paragraph{Bound of $\mathfrak{T}_1$, $\mathfrak{T}_2$, $\mathfrak{T}_3$ and $\mathfrak{T}_4$.} Repeating the same steps as in the proof of Lemma~\ref{lemma:TCcontraction} and utilizing Lemma~\ref{lemma:Deltanorm}, we have the following:
\begin{align*}
\mathfrak{T}_1 & \leq \Big( (1 - \eta)^2 + \frac{2 \epsilon}{1 - \epsilon} \eta (1 - \eta) \Big) \| \bcL_{\triangle} \ast_{\bPhi} \bcG_{\star}^{\frac{1}{2}} \|_F^2 + \frac{2 \epsilon + \epsilon^2}{(1 - \epsilon)^2} \eta^2 \| \bcR_{\triangle} \ast_{\bPhi} \bcG_{\star}^{\frac{1}{2}} \|_F^2,  \\
|\mathfrak{T}_2| & \leq 2 \eta (1 - \eta) \Big( \frac{c_1}{(1 - \epsilon)^2} \sqrt{\frac{\mu s_r (n_1 + n_2) \log( (n_1 \vee n_2) n_3 )}{p n_1 n_2 n_3} } \| \bcL_{\triangle} \ast_{\bPhi} \bcG_{\star}^{\frac{1}{2}} \|_F^2  \\
&\qquad + \frac{c_2}{(1 - \epsilon)^2} \sqrt{\frac{\ell \log( (n_1 \vee n_2) n_3 )}{p (n_1 \wedge n_2)} } \frac{\mu s_r}{n_3} (\rho^t + 4) \| \bcL_{\triangle} \ast_{\bPhi} \bcG_{\star}^{\frac{1}{2}} \|_F \| \bcR_{\triangle} \ast_{\bPhi} \bcG_{\star}^{\frac{1}{2}} \|_F \Big),  \\
|\mathfrak{T}_3| & \leq 2 \eta^2 \Big( \frac{c_1}{(1 - \epsilon)^2} \sqrt{\frac{\mu s_r (n_1 + n_2) \log( (n_1 \vee n_2) n_3 )}{p n_1 n_2 n_3} } \| \bcR_{\triangle} \ast_{\bPhi} \bcG_{\star}^{\frac{1}{2}} \|_F^2  \\
&\qquad + \frac{c_2}{(1 - \epsilon)^2} \sqrt{\frac{\ell \log( (n_1 \vee n_2) n_3 )}{p (n_1 \wedge n_2)} } \frac{2 \mu s_r}{n_3} \| \bcL_{\triangle} \ast_{\bPhi} \bcG_{\star}^{\frac{1}{2}} \|_F \| \bcR_{\triangle} \ast_{\bPhi} \bcG_{\star}^{\frac{1}{2}} \|_F \Big),  \\
\mathrm{and} \quad \sqrt{\mathfrak{T}_4} & \leq \eta \Big( \frac{c_1}{(1 - \epsilon)^2} \sqrt{\frac{\mu s_r (n_1 + n_2) \log( (n_1 \vee n_2) n_3 )}{p n_1 n_2 n_3} } \| \bcL_{\triangle} \ast_{\bPhi} \bcG_{\star}^{\frac{1}{2}} \|_F  \\
&\qquad + \frac{c_2}{(1 - \epsilon)^2} \sqrt{\frac{\ell \log( (n_1 \vee n_2) n_3 )}{p (n_1 \wedge n_2)} } \frac{\mu s_r}{n_3} (\rho^t + 4) \| \bcR_{\triangle} \ast_{\bPhi} \bcG_{\star}^{\frac{1}{2}} \|_F \Big).
\end{align*}

\paragraph{Bound of $\mathfrak{T}_5$.} Lemma~\ref{lemma:sparity} implies $\hat{\bcS} \coloneq \bcP_{\bOmega} (\bcS_{\triangle})$ is an $\alpha p$-sparse tensor, hence
\begin{align*}
& 2 \eta (1 - \eta) \langle \bcL_{\triangle} \ast_{\bPhi} \bcG_{\star}^{\frac{1}{2}} , p^{-1} \bcP_{\bOmega} (\bcS_{\triangle}) \ast_{\bPhi} \bcR_{\sharp} \ast_{\bPhi} (\bcR_{\sharp}^H \ast_{\bPhi} \bcR_{\sharp})^{-1} \ast_{\bPhi} \bcG_{\star}^{\frac{1}{2}} \rangle  \\
= & 2 \frac{\eta (1 - \eta)}{p \ell} \Big| \tr \Big( \widebar{\hat{\bS}} \widebar{\bR}_{\sharp} (\widebar{\bR}_{\sharp}^H \widebar{\bR}_{\sharp})^{-1} \widebar{\bG}_{\star} \widebar{\bL}_{\triangle}^H \Big) \Big|  \\
\leq & 2 \frac{\eta (1 - \eta)}{p \ell} \| \widebar{\hat{\bS}} \| \| \widebar{\bR}_{\sharp} (\widebar{\bR}_{\sharp}^H \widebar{\bR}_{\sharp})^{-1} \widebar{\bG}_{\star} \widebar{\bL}_{\triangle}^H \|_{\ast}  \\
\leq & \frac{\alpha \eta (1 - \eta)}{\sqrt{\ell}} (n_1 + n_2 n_3) \| \hat{\bS} \|_{\infty} \sqrt{s_r} \| \widebar{\bR}_{\sharp} (\widebar{\bR}_{\sharp}^H \widebar{\bR}_{\sharp})^{-1} \widebar{\bG}_{\star} \widebar{\bL}_{\triangle}^H \|_F  \\
\leq & \frac{\alpha \eta (1 - \eta)}{\sqrt{\ell}} (n_1 + n_2 n_3) \| \hat{\bS} \|_{\infty} \sqrt{s_r} \| \widebar{\bR}_{\sharp} (\widebar{\bR}_{\sharp}^H \widebar{\bR}_{\sharp})^{-1} \widebar{\bG}_{\star}^{\frac{1}{2}} \| \| \widebar{\bL}_{\triangle} \widebar{\bG}_{\star}^{\frac{1}{2}} \|_F  \\
\leq & 2 \frac{\alpha \sqrt{s_r} \eta (1 - \eta)}{1 - \epsilon} (n_1 + n_2 n_3) \| \bcX_t - \bcX_{\star} \|_{\infty} \| \bcL_{\triangle} \ast_{\bPhi} \bcG_{\star}^{\frac{1}{2}} \|_F  \\
\leq & 6 \frac{\alpha \mu s_r^{1.5} \eta (1 - \eta)}{\ell n_3 \sqrt{n_1 n_2}} \frac{\epsilon}{1 - \epsilon} (n_1 + n_2 n_3) \rho^{2t} \bar{\sigma}_{s_r}^2 (\bcX_{\star}),
\end{align*}
where the last step uses the results from Lemma~\ref{lemma:Deltanorm} and Lemma~\ref{lemma:Xinfnorm}.

\paragraph{Bound of $\mathfrak{T}_6$.} Similar to $\mathfrak{T}_5$, we have
\begin{align*}
& 2 \eta^2 \langle \bcL_{\star} \ast_{\bPhi} \bcR_{\triangle}^H \ast_{\bPhi} \bcR_{\sharp} \ast_{\bPhi} (\bcR_{\sharp}^H \ast_{\bPhi} \bcR_{\sharp})^{-1} \ast_{\bPhi} \bcG_{\star}^{\frac{1}{2}} ,  \\
&\qquad\qquad\qquad p^{-1} \bcP_{\bOmega} (\bcS_{\triangle}) \ast_{\bPhi} \bcR_{\sharp} \ast_{\bPhi} (\bcR_{\sharp}^H \ast_{\bPhi} \bcR_{\sharp})^{-1} \ast_{\bPhi} \bcG_{\star}^{\frac{1}{2}} \rangle  \\
= & 2 \frac{\eta^2}{p \ell} \Big| \tr \Big( \widebar{\hat{\bS}} \widebar{\bR}_{\sharp} (\widebar{\bR}_{\sharp}^H \widebar{\bR}_{\sharp})^{-1} \widebar{\bG}_{\star} (\widebar{\bR}_{\sharp}^H \widebar{\bR}_{\sharp})^{-1} \widebar{\bR}_{\sharp}^H \widebar{\bR}_{\triangle} \widebar{\bL}_{\star}^H \Big) \Big|  \\
\leq & 2 \frac{\eta^2}{p \ell} \| \widebar{\hat{\bS}} \| \| \widebar{\bR}_{\sharp} (\widebar{\bR}_{\sharp}^H \widebar{\bR}_{\sharp})^{-1} \widebar{\bG}_{\star} (\widebar{\bR}_{\sharp}^H \widebar{\bR}_{\sharp})^{-1} \widebar{\bR}_{\sharp}^H \widebar{\bR}_{\triangle} \widebar{\bL}_{\star}^H \|_{\ast}  \\
\leq & \frac{\alpha \eta^2}{\sqrt{\ell}} (n_1 + n_2 n_3) \| \hat{\bcS} \|_{\infty} \sqrt{s_r} \| \widebar{\bR}_{\sharp} (\widebar{\bR}_{\sharp}^H \widebar{\bR}_{\sharp})^{-1} \widebar{\bG}_{\star} (\widebar{\bR}_{\sharp}^H \widebar{\bR}_{\sharp})^{-1} \widebar{\bR}_{\sharp}^H \widebar{\bR}_{\triangle} \widebar{\bL}_{\star}^H \|_F  \\
\leq & \frac{\alpha \eta^2}{\sqrt{\ell}} (n_1 + n_2 n_3) \| \hat{\bcS} \|_{\infty} \sqrt{s_r} \| \widebar{\bR}_{\sharp} (\widebar{\bR}_{\sharp}^H \widebar{\bR}_{\sharp})^{-1} \widebar{\bG}_{\star}^{\frac{1}{2}} \|^2 \| \widebar{\bR}_{\triangle} \widebar{\bG}_{\star}^{\frac{1}{2}} \|_F \| \widebar{\bU}_{\star} \|  \\
\leq & 2 \alpha \eta^2 (n_1 + n_2 n_3) \sqrt{s_r} \frac{1}{(1 - \epsilon)^2} \| \bcX_t - \bcX_{\star} \|_{\infty} \| \bcR_{\triangle} \ast_{\bPhi} \bcG_{\star}^{\frac{1}{2}} \|_F  \\
\leq & 6 \frac{\alpha \mu s_r^{1.5} \eta^2}{\ell n_3 \sqrt{n_1 n_2}} (n_1 + n_2 n_3) \frac{\epsilon}{(1 - \epsilon)^2} \rho^{2t} \bar{\sigma}_{s_r}^2 (\bcX_{\star}).
\end{align*}

\paragraph{Bound of $\mathfrak{T}_7$.} Using the decomposition $\bcL_{\sharp} \ast_{\bPhi} \bcR_{\sharp}^H - \bcX_{\star} = \bcL_{\triangle} \ast_{\bPhi} \bcR_{\star}^H + \bcL_{\sharp} \ast_{\bPhi} \bcR_{\triangle}^H$ and applying the triangle inequality to obtain
\begin{align*}
|\mathfrak{T}_7| & = 2 \eta^2 \Big| \langle (p^{-1} \bcP_{\bOmega} - \bcI_{n_1})(\bcL_{\sharp} \ast_{\bPhi} \bcR_{\sharp}^H - \bcX_{\star}) \ast_{\bPhi} \bcR_{\sharp} \ast_{\bPhi} (\bcR_{\sharp}^H \ast_{\bPhi} \bcR_{\sharp})^{-1} \ast_{\bPhi} \bcG_{\star}^{\frac{1}{2}} ,  \\
&\qquad\qquad\qquad p^{-1} \bcP_{\bOmega} (\bcS_{\triangle}) \ast_{\bPhi} \bcR_{\sharp} \ast_{\bPhi} (\bcR_{\sharp}^H \ast_{\bPhi} \bcR_{\sharp})^{-1} \ast_{\bPhi} \bcG_{\star}^{\frac{1}{2}} \rangle \Big|  \\
& \leq 2 \eta^2 \Big( \Big| \langle p^{-1} \bcP_{\bOmega} (\bcS_{\triangle}) \ast_{\bPhi} \bcR_{\sharp} \ast_{\bPhi} (\bcR_{\sharp}^H \ast_{\bPhi} \bcR_{\sharp})^{-1} \ast_{\bPhi} \bcG_{\star}^{\frac{1}{2}} ,  \\
&\qquad\qquad\qquad (p^{-1} \bcP_{\bOmega} - \bcI_{n_1})(\bcL_{\triangle} \ast_{\bPhi} \bcR_{\star}^H) \ast_{\bPhi} \bcR_{\star} \ast_{\bPhi} (\bcR_{\sharp}^H \ast_{\bPhi} \bcR_{\sharp})^{-1} \ast_{\bPhi} \bcG_{\star}^{\frac{1}{2}} \rangle \Big|  \\
&\quad + \Big| \langle p^{-1} \bcP_{\bOmega} (\bcS_{\triangle}) \ast_{\bPhi} \bcR_{\sharp} \ast_{\bPhi} (\bcR_{\sharp}^H \ast_{\bPhi} \bcR_{\sharp})^{-1} \ast_{\bPhi} \bcG_{\star}^{\frac{1}{2}} ,  \\
&\qquad\qquad\qquad (p^{-1} \bcP_{\bOmega} - \bcI_{n_1})(\bcL_{\triangle} \ast_{\bPhi} \bcR_{\star}^H) \ast_{\bPhi} \bcR_{\triangle} \ast_{\bPhi} (\bcR_{\sharp}^H \ast_{\bPhi} \bcR_{\sharp})^{-1} \ast_{\bPhi} \bcG_{\star}^{\frac{1}{2}} \rangle \Big|  \\
&\quad + \Big| \langle p^{-1} \bcP_{\bOmega} (\bcS_{\triangle}) \ast_{\bPhi} \bcR_{\sharp} \ast_{\bPhi} (\bcR_{\sharp}^H \ast_{\bPhi} \bcR_{\sharp})^{-1} \ast_{\bPhi} \bcG_{\star}^{\frac{1}{2}} ,  \\
&\qquad\qquad\qquad (p^{-1} \bcP_{\bOmega} - \bcI_{n_1})(\bcL_{\sharp} \ast_{\bPhi} \bcR_{\triangle}^H) \ast_{\bPhi} \bcR_{\sharp} \ast_{\bPhi} (\bcR_{\sharp}^H \ast_{\bPhi} \bcR_{\sharp})^{-1} \ast_{\bPhi} \bcG_{\star}^{\frac{1}{2}} \rangle \Big| \Big)  \\
& \coloneq 2 \eta^2 (\mathfrak{T}_{7,1} + \mathfrak{T}_{7,2} + \mathfrak{T}_{7,3}).
\end{align*}
For the first term $\mathfrak{T}_{7,1}$, we can invoke Corollary~\ref{coro:POmegatangent} to obtain
\begin{align*}
\mathfrak{T}_{7,1} & \leq c_1 \sqrt{\frac{\mu s_r (n_1 + n_2) \log( (n_1 \vee n_2) n_3 )}{p n_1 n_2 n_3} } \| \bcL_{\triangle} \ast_{\bPhi} \bcR_{\star}^H \|_F  \\
&\qquad \| p^{-1} \bcP_{\bOmega} (\bcS_{\triangle}) \ast_{\bPhi} \bcR_{\sharp} \ast_{\bPhi} (\bcR_{\sharp}^H \ast_{\bPhi} \bcR_{\sharp})^{-1} \ast_{\bPhi} \bcG_{\star} \ast_{\bPhi} (\bcR_{\sharp}^H \ast_{\bPhi} \bcR_{\sharp})^{-1} \ast_{\bPhi} \bcR_{\star}^H \|_F  \\
& \leq c_1 \sqrt{\frac{\mu s_r (n_1 + n_2) \log( (n_1 \vee n_2) n_3 )}{p n_1 n_2 n_3} } \| \bcL_{\triangle} \ast_{\bPhi} \bcG_{\star}^{\frac{1}{2}} \|_F \| \bcG_{\star}^{-\frac{1}{2}} \ast_{\bPhi} \bcR_{\star}^H \|  \\
&\qquad \| p^{-1} \bcP_{\bOmega} (\bcS_{\triangle}) \|_F \| \bcR_{\sharp} \ast_{\bPhi} (\bcR_{\sharp}^H \ast_{\bPhi} \bcR_{\sharp})^{-1} \ast_{\bPhi} \bcG_{\star}^{\frac{1}{2}} \| \| \bcG_{\star}^{\frac{1}{2}} \ast_{\bPhi} (\bcR_{\sharp}^H \ast_{\bPhi} \bcR_{\sharp})^{-1} \ast_{\bPhi} \bcG_{\star}^{\frac{1}{2}} \| \| \bcG_{\star}^{-\frac{1}{2}} \ast_{\bPhi} \bcR_{\star}^H \|  \\
& = c_1 \sqrt{\frac{\mu s_r (n_1 + n_2) \log( (n_1 \vee n_2) n_3 )}{p^3 n_1 n_2 n_3} } \| \bcL_{\triangle} \ast_{\bPhi} \bcG_{\star}^{\frac{1}{2}} \|_F \sqrt{\alpha p n_1 n_2 n_3} \| \hat{\bcS} \|_{\infty}  \\
&\qquad \| \bcR_{\sharp} \ast_{\bPhi} (\bcR_{\sharp}^H \ast_{\bPhi} \bcR_{\sharp})^{-1} \ast_{\bPhi} \bcG_{\star}^{\frac{1}{2}} \| \| \bcG_{\star}^{\frac{1}{2}} \ast_{\bPhi} (\bcR_{\sharp}^H \ast_{\bPhi} \bcR_{\sharp})^{-1} \ast_{\bPhi} \bcG_{\star}^{\frac{1}{2}} \|  \\
& \leq \frac{c_1}{(1 - \epsilon)^3} \sqrt{\frac{\mu \alpha s_r (n_1 + n_2) \log( (n_1 \vee n_2) n_3 )}{p^2} } \| \bcL_{\triangle} \ast_{\bPhi} \bcG_{\star}^{\frac{1}{2}} \|_F \| \hat{\bcS} \|_{\infty}  \\
& \leq 2 \frac{c_1}{(1 - \epsilon)^3} \sqrt{\frac{\mu \alpha s_r (n_1 + n_2) \log( (n_1 \vee n_2) n_3 )}{p^2} } \| \bcL_{\triangle} \ast_{\bPhi} \bcG_{\star}^{\frac{1}{2}} \|_F \| \bcX_t - \bcX_{\star} \|_{\infty}  \\
& \leq 6 \frac{c_1 \mu^{1.5} s_r^{1.5} \epsilon}{p \ell n_3 (1 - \epsilon)^3} \sqrt{\frac{\alpha (n_1 + n_2) \log( (n_1 \vee n_2) n_3 )}{n_1 n_2}} \rho^{2t} \bar{\sigma}_{s_r}^2 (\bcX_{\star}).
\end{align*}
In regard to $\mathfrak{T}_{7,2}$, we can invoke Lemma~\ref{lemma:POmeganonconvex} with $\bcL_A \coloneq \bcL_{\triangle} \ast_{\bPhi} \bcG_{\star}^{\frac{1}{2}}$, $\bcR_A \coloneq \bcR_{\star} \ast_{\bPhi} \bcG_{\star}^{-\frac{1}{2}}$, $\bcL_B \coloneq p^{-1} \bcP_{\bOmega} (\bcS_{\triangle})$, $\bcR_B \coloneq \bcR_{\triangle} \ast_{\bPhi} (\bcR_{\sharp}^H \ast_{\bPhi} \bcR_{\sharp})^{-1} \ast_{\bPhi} \bcG_{\star} \ast_{\bPhi} (\bcR_{\sharp}^H \ast_{\bPhi} \bcR_{\sharp})^{-1} \ast_{\bPhi} \bcR_{\sharp}^H$ and use the consequences in Lemma~\ref{lemma:TC-cond} to obtain
\begin{align*}
\mathfrak{T}_{7,2} & \leq c_2 \ell^{\frac{3}{2}} \sqrt{\frac{(n_1 \vee n_2) \log( (n_1 \vee n_2) n_3 )}{p} } \| \bcL_{\triangle} \ast_{\bPhi} \bcG_{\star}^{\frac{1}{2}} \|_F \| p^{-1} \bcP_{\bOmega} (\bcS_{\triangle}) \|_{2,2,\infty}  \\
&\qquad \| \bcR_{\star} \ast_{\bPhi} \bcG_{\star}^{-\frac{1}{2}} \|_{2,2,\infty} \| \bcR_{\triangle} \ast_{\bPhi} (\bcR_{\sharp}^H \ast_{\bPhi} \bcR_{\sharp})^{-1} \ast_{\bPhi} \bcG_{\star} \ast_{\bPhi} (\bcR_{\sharp}^H \ast_{\bPhi} \bcR_{\sharp})^{-1} \ast_{\bPhi} \bcR_{\sharp}^H \|_F  \\
& \leq c_2 \ell^{\frac{3}{2}} \sqrt{\frac{(n_1 \vee n_2) \log( (n_1 \vee n_2) n_3 )}{p^3} } \| \bcL_{\triangle} \ast_{\bPhi} \bcG_{\star}^{\frac{1}{2}} \|_F \| \bcP_{\bOmega} (\bcS_{\triangle}) \|_{2,2,\infty}  \\
&\qquad \| \bcR_{\star} \ast_{\bPhi} \bcG_{\star}^{-\frac{1}{2}} \|_{2,\infty} \| \bcR_{\triangle} \ast_{\bPhi} \bcG_{\star}^{-\frac{1}{2}} \|_F \| \bcG_{\star}^{\frac{1}{2}} \ast_{\bPhi} (\bcR_{\sharp}^H \ast_{\bPhi} \bcR_{\sharp})^{-1} \ast_{\bPhi} \bcG_{\star}^{\frac{1}{2}} \| \| \bcG_{\star}^{\frac{1}{2}} \ast_{\bPhi} (\bcR_{\sharp}^H \ast_{\bPhi} \bcR_{\sharp})^{-1} \ast_{\bPhi} \bcR_{\sharp}^H \|  \\
& \leq c_2 \ell^{\frac{3}{2}} \sqrt{\frac{(n_1 \vee n_2) \log( (n_1 \vee n_2) n_3 )}{p^3} } \| \bcL_{\triangle} \ast_{\bPhi} \bcG_{\star}^{\frac{1}{2}} \|_F \sqrt{\alpha p n_2} \| \hat{\bcS} \|_{\infty}  \\
&\qquad \| \bcR_{\star} \ast_{\bPhi} \bcG_{\star}^{-\frac{1}{2}} \|_{2,\infty} \| \bcR_{\triangle} \ast_{\bPhi} \bcG_{\star}^{-\frac{1}{2}} \|_F \| \bcG_{\star}^{\frac{1}{2}} \ast_{\bPhi} (\bcR_{\sharp}^H \ast_{\bPhi} \bcR_{\sharp})^{-1} \ast_{\bPhi} \bcG_{\star}^{\frac{1}{2}} \| \| \bcG_{\star}^{\frac{1}{2}} \ast_{\bPhi} (\bcR_{\sharp}^H \ast_{\bPhi} \bcR_{\sharp})^{-1} \ast_{\bPhi} \bcR_{\sharp}^H \|  \\
& \leq 2 c_2 \ell^{\frac{3}{2}} \sqrt{\frac{(n_1 \vee n_2) \log( (n_1 \vee n_2) n_3 )}{p^3} } \| \bcL_{\triangle} \ast_{\bPhi} \bcG_{\star}^{\frac{1}{2}} \|_F \sqrt{\alpha p n_2} \| \bcX_t - \bcX_{\star} \|_{\infty}  \\
&\qquad \| \bcR_{\star} \ast_{\bPhi} \bcG_{\star}^{-\frac{1}{2}} \|_{2,\infty} \| \bcR_{\triangle} \ast_{\bPhi} \bcG_{\star}^{-\frac{1}{2}} \|_F \| \bcG_{\star}^{\frac{1}{2}} \ast_{\bPhi} (\bcR_{\sharp}^H \ast_{\bPhi} \bcR_{\sharp})^{-1} \ast_{\bPhi} \bcG_{\star}^{\frac{1}{2}} \| \| \bcG_{\star}^{\frac{1}{2}} \ast_{\bPhi} (\bcR_{\sharp}^H \ast_{\bPhi} \bcR_{\sharp})^{-1} \ast_{\bPhi} \bcR_{\sharp}^H \|  \\
& \leq 6 \frac{c_2 \mu^{1.5} s_r^{1.5}}{p n_3 (1 - \epsilon)^3} \sqrt{\frac{\alpha \ell \log( (n_1 \vee n_2) n_3 )}{(n_1 \wedge n_2) n_3} } \rho^t \| \bcL_{\triangle} \ast_{\bPhi} \bcG_{\star}^{\frac{1}{2}} \|_F \| \bcR_{\triangle} \ast_{\bPhi} \bcG_{\star}^{\frac{1}{2}} \|_F.
\end{align*}
Similarly, we can bound $\mathfrak{T}_{7,3}$ as
\begin{align*}
\mathfrak{T}_{7,3} & \leq c_2 \ell^{\frac{3}{2}} \sqrt{\frac{(n_1 \vee n_2) \log( (n_1 \vee n_2) n_3 )}{p} } \| \bcL_{\sharp} \ast_{\bPhi} \bcG_{\star}^{-\frac{1}{2}} \|_{2,\infty} \| p^{-1} \bcP_{\bOmega} (\bcS_{\triangle}) \|_F  \\
&\qquad \| \bcR_{\triangle} \ast_{\bPhi} \bcG_{\star}^{\frac{1}{2}} \|_F \| \bcR_{\sharp} \ast_{\bPhi} \bcG_{\star}^{-\frac{1}{2}} \|_{2,\infty} \| \bcG_{\star}^{\frac{1}{2}} \ast_{\bPhi} (\bcR_{\sharp}^H \ast_{\bPhi} \bcR_{\sharp})^{-1} \ast_{\bPhi} \bcG_{\star}^{\frac{1}{2}} \| \| \bcG_{\star}^{\frac{1}{2}} \ast_{\bPhi} (\bcR_{\sharp}^H \ast_{\bPhi} \bcR_{\sharp})^{-1} \ast_{\bPhi} \bcR_{\sharp}^H \|  \\
& \leq c_2 \ell^{\frac{3}{2}} \sqrt{\frac{(n_1 \vee n_2) \log( (n_1 \vee n_2) n_3 )}{p^3} } \| \bcL_{\sharp} \ast_{\bPhi} \bcG_{\star}^{-\frac{1}{2}} \|_{2,\infty} \sqrt{\alpha p n_1 n_2 n_3} \| \hat{\bcS} \|_{\infty}  \\
&\qquad \| \bcR_{\triangle} \ast_{\bPhi} \bcG_{\star}^{\frac{1}{2}} \|_F \| \bcR_{\sharp} \ast_{\bPhi} \bcG_{\star}^{-\frac{1}{2}} \|_{2,\infty} \| \bcG_{\star}^{\frac{1}{2}} \ast_{\bPhi} (\bcR_{\sharp}^H \ast_{\bPhi} \bcR_{\sharp})^{-1} \ast_{\bPhi} \bcG_{\star}^{\frac{1}{2}} \| \| \bcG_{\star}^{\frac{1}{2}} \ast_{\bPhi} (\bcR_{\sharp}^H \ast_{\bPhi} \bcR_{\sharp})^{-1} \ast_{\bPhi} \bcR_{\sharp}^H \|  \\
& \leq 24 \frac{c_2 \mu^2 s_r^2 \epsilon}{p n_3 (1 - \epsilon)^3} \sqrt{\frac{\alpha \log( (n_1 \vee n_2) n_3 )}{\ell (n_1 \wedge n_2) n_3} } \rho^{2t} \bar{\sigma}_{s_r}^2 (\bcX_{\star}).
\end{align*}
We then combine the bounds for $\mathfrak{T}_{7,1}$, $\mathfrak{T}_{7,2}$ and $\mathfrak{T}_{7,3}$ to arrive at
\begin{align*}
|\mathfrak{T}_7| & \leq 2 \eta^2 \Big( 6 \frac{c_1 \mu^{1.5} s_r^{1.5} \epsilon}{p \ell n_3 (1 - \epsilon)^3} \sqrt{\frac{\alpha (n_1 + n_2) \log( (n_1 \vee n_2) n_3 )}{n_1 n_2}} \rho^{2t} \bar{\sigma}_{s_r}^2 (\bcX_{\star})  \\
&\qquad + 6 \frac{c_2 \mu^{1.5} s_r^{1.5}}{p n_3 (1 - \epsilon)^3} \sqrt{\frac{\alpha \ell \log( (n_1 \vee n_2) n_3 )}{(n_1 \wedge n_2) n_3} } \rho^t \| \bcL_{\triangle} \ast_{\bPhi} \bcG_{\star}^{\frac{1}{2}} \|_F \| \bcR_{\triangle} \ast_{\bPhi} \bcG_{\star}^{\frac{1}{2}} \|_F  \\
&\qquad + 24 \frac{c_2 \mu^2 s_r^2 \epsilon}{p n_3 (1 - \epsilon)^3} \sqrt{\frac{\alpha \log( (n_1 \vee n_2) n_3 )}{\ell (n_1 \wedge n_2) n_3} } \rho^{2t} \bar{\sigma}_{s_r}^2 (\bcX_{\star}) \Big).
\end{align*}

\paragraph{Bound of $\mathfrak{T}_8$.} Similar to $\mathfrak{T}_5$, we have
\begin{align*}
& \eta^2 \| p^{-1} \bcP_{\bOmega} (\bcS_{\triangle}) \ast_{\bPhi} \bcR_{\sharp} \ast_{\bPhi} (\bcR_{\sharp}^H \ast_{\bPhi} \bcR_{\sharp})^{-1} \ast_{\bPhi} \bcG_{\star}^{\frac{1}{2}} \|_F^2  \\
= & \frac{1}{\ell p^2} \eta^2 \| \widebar{\hat{\bS}} \widebar{\bR}_{\sharp} (\widebar{\bR}_{\sharp}^H \widebar{\bR}_{\sharp})^{-1} \widebar{\bG}_{\star}^{\frac{1}{2}} \|_F^2  \\
\leq & \frac{1}{\ell p^2} \eta^2 s_r \| \widebar{\hat{\bS}} \widebar{\bR}_{\sharp} (\widebar{\bR}_{\sharp}^H \widebar{\bR}_{\sharp})^{-1} \widebar{\bG}_{\star}^{\frac{1}{2}} \|^2  \\
\leq & \frac{1}{\ell p^2} \eta^2 s_r \| \widebar{\hat{\bS}} \|^2 \| \widebar{\bR}_{\sharp} (\widebar{\bR}_{\sharp}^H \widebar{\bR}_{\sharp})^{-1} \widebar{\bG}_{\star}^{\frac{1}{2}} \|^2  \\
\leq & \frac{\alpha^2 \eta^2 s_r}{4} (n_1 + n_2 n_3)^2 \| \hat{\bcS} \|_{\infty}^2 \| \widebar{\bR}_{\sharp} (\widebar{\bR}_{\sharp}^H \widebar{\bR}_{\sharp})^{-1} \widebar{\bG}_{\star}^{\frac{1}{2}} \|^2  \\
\leq & \alpha^2 \eta^2 s_r (n_1 + n_2 n_3)^2 \frac{1}{(1 - \epsilon)^2} \| \bcX_t - \bcX_{\star} \|_{\infty}^2  \\
\leq & 9 \frac{\alpha^2 \mu^2 s_r^3 \eta^2}{\ell n_1 n_2 n_3^2} (n_1 + n_2 n_3)^2 \frac{1}{(1 - \epsilon)^2} \rho^{2t} \bar{\sigma}_{s_r}^2 (\bcX_{\star}).
\end{align*}
Taking collectively the bounds for $\mathfrak{T}_1$, $\mathfrak{T}_2$, $\mathfrak{T}_3$, $\mathfrak{T}_4$, $\mathfrak{T}_5$, $\mathfrak{T}_6$, $\mathfrak{T}_7$ and $\mathfrak{T}_8$ together, we have
\begin{align*}
& \| (\bcL_{t+1} \ast_{\bPhi} \bcQ_t - \bcL_{\star}) \ast_{\bPhi} \bcG_{\star}^{\frac{1}{2}} \|_F^2  \\
& \leq \Big( (1 - \eta)^2 + \frac{2 \epsilon}{1 - \epsilon} \eta (1 - \eta) \Big) \| \bcL_{\triangle} \ast_{\bPhi} \bcG_{\star}^{\frac{1}{2}} \|_F^2 + \frac{2 \epsilon + \epsilon^2}{(1 - \epsilon)^2} \eta^2 \| \bcR_{\triangle} \ast_{\bPhi} \bcG_{\star}^{\frac{1}{2}} \|_F^2  \\
&\quad + \eta (1 - \eta) \Big( (2 \nu_1 + \nu_3) \| \bcL_{\triangle} \ast_{\bPhi} \bcG_{\star}^{\frac{1}{2}} \|_F^2 + \nu_3 \| \bcR_{\triangle} \ast_{\bPhi} \bcG_{\star}^{\frac{1}{2}} \|_F^2 \Big)  \\
&\quad + \eta^2 \Big( \nu_3 \| \bcL_{\triangle} \ast_{\bPhi} \bcG_{\star}^{\frac{1}{2}} \|_F^2 + (2 \nu_1 + \nu_3) \| \bcR_{\triangle} \ast_{\bPhi} \bcG_{\star}^{\frac{1}{2}} \|_F^2 \Big)  \\
&\quad + \eta^2 \Big( \nu_1 (\nu_1 + \nu_3) \| \bcL_{\triangle} \ast_{\bPhi} \bcG_{\star}^{\frac{1}{2}} \|_F^2 + \nu_3 (\nu_1 + \nu_3) \| \bcR_{\triangle} \ast_{\bPhi} \bcG_{\star}^{\frac{1}{2}} \|_F^2 \Big)  \\
&\quad + \eta^2 \Big( \nu_4 \| \bcL_{\triangle} \ast_{\bPhi} \bcG_{\star}^{\frac{1}{2}} \|_F^2 + \nu_4 \| \bcR_{\triangle} \ast_{\bPhi} \bcG_{\star}^{\frac{1}{2}} \|_F^2 \Big)  \\
&\quad + 6 \eta (1 - \eta) \frac{\alpha \mu s_r^{1.5}}{\ell n_3 \sqrt{n_1 n_2}} \frac{\epsilon}{1 - \epsilon} (n_1 + n_2 n_3) \rho^{2t} \bar{\sigma}_{s_r}^2 (\bcX_{\star}) + 6 \eta^2 \frac{\alpha \mu s_r^{1.5}}{\ell n_3 \sqrt{n_1 n_2}} (n_1 + n_2 n_3) \frac{\epsilon}{(1 - \epsilon)^2} \rho^{2t} \bar{\sigma}_{s_r}^2 (\bcX_{\star})  \\
&\quad + 12 \eta^2 \frac{c_1 \mu^{1.5} s_r^{1.5} \epsilon}{p \ell n_3 (1 - \epsilon)^3} \sqrt{\frac{\alpha (n_1 + n_2) \log( (n_1 \vee n_2) n_3 )}{n_1 n_2}} \rho^{2t} \bar{\sigma}_{s_r}^2 (\bcX_{\star})  \\
&\quad + 48 \eta^2 \frac{c_2 \mu^2 s_r^2 \epsilon}{p n_3 (1 - \epsilon)^3} \sqrt{\frac{\alpha \log( (n_1 \vee n_2) n_3 )}{\ell (n_1 \wedge n_2) n_3} } \rho^{2t} \bar{\sigma}_{s_r}^2 (\bcX_{\star}) + 9 \eta^2 \frac{\alpha^2 \mu^2 s_r^3}{\ell n_1 n_2 n_3^2} (n_1 + n_2 n_3)^2 \frac{1}{(1 - \epsilon)^2} \rho^{2t} \bar{\sigma}_{s_r}^2 (\bcX_{\star}),
\end{align*}
where $\nu_1$ is defined in \eqref{eqn:defnu1nu2} and we denote
\begin{align*}
\nu_3 & \coloneq \frac{c_2}{(1 - \epsilon)^2} \sqrt{\frac{\ell \log( (n_1 \vee n_2) n_3 )}{p (n_1 \wedge n_2)} } \frac{\mu s_r}{n_3} (\rho^t + 4),  \\
\mathrm{and} \quad \nu_4 & \coloneq 6 \frac{c_2 \mu^{1.5} s_r^{1.5}}{p n_3 (1 - \epsilon)^3} \sqrt{\frac{\alpha \ell \log( (n_1 \vee n_2) n_3 )}{(n_1 \wedge n_2) n_3} }.
\end{align*}
A similar bound can be computed for $\| (\bcR_{t+1} \ast_{\bPhi} \bcQ_t^{-H} - \bcR_{\star}) \ast_{\bPhi} \bcG_{\star}^{\frac{1}{2}} \|_F^2$. As a result, we reach the conclusion that
\begin{align*}
& \dist^2(\bcF_{t+1}, \bcF_{\star})  \nonumber \\
\leq & \hbar^2(\eta;\epsilon,\nu_1,\nu_3,\nu_4) \dist^2(\bcF_t, \bcF_{\star}) + \frac{1}{\ell} \Big( 12 \eta (1 - \eta) \alpha \mu s_r^{1.5} \frac{n_1 + n_2 n_3}{n_3 \sqrt{n_1 n_2}} \frac{\epsilon}{1 - \epsilon} \rho^{2t} \bar{\sigma}_{s_r}^2 (\bcX_{\star})  \nonumber \\
&\quad + 12 \eta^2 \alpha \mu s_r^{1.5} \frac{n_1 + n_2 n_3}{n_3 \sqrt{n_1 n_2}} \frac{\epsilon}{(1 - \epsilon)^2} \rho^{2t} \bar{\sigma}_{s_r}^2 (\bcX_{\star}) + 18 \eta^2 \alpha^2 \mu^2 s_r^3 \frac{(n_1 + n_2 n_3)^2}{n_1 n_2 n_3^2} \frac{1}{(1 - \epsilon)^2} \rho^{2t} \bar{\sigma}_{s_r}^2 (\bcX_{\star})  \nonumber \\
&\quad + 24 \eta^2 \frac{c_1 \mu^{1.5} s_r^{1.5} \epsilon}{p n_3 (1 - \epsilon)^3} \sqrt{\frac{\alpha (n_1 + n_2) \log( (n_1 \vee n_2) n_3 )}{n_1 n_2}} \rho^{2t} \bar{\sigma}_{s_r}^2 (\bcX_{\star})  \\
&\quad + 96 \eta^2 \frac{c_2 \mu^2 s_r^2 \epsilon}{p n_3 (1 - \epsilon)^3} \sqrt{\frac{\alpha \ell \log( (n_1 \vee n_2) n_3 )}{(n_1 \wedge n_2) n_3} } \rho^{2t} \bar{\sigma}_{s_r}^2 (\bcX_{\star}) \Big),
\end{align*}
where the contraction rate $\hbar^2(\eta;\epsilon,\nu_1,\nu_3,\nu_4)$ is given by 
\begin{align*}
& \hbar^2 (\eta;\epsilon,\nu_1,\nu_3,\nu_4)  \\
\coloneq & (1 - \eta)^2 + \Big( \frac{2 \epsilon}{1 - \epsilon} + 2(\nu_1 + \nu_3) \Big) \eta (1 - \eta) + \Big( \frac{2 \epsilon + \epsilon^2}{(1 - \epsilon)^2} + 2 (\nu_1 + \nu_3 + \nu_4) + (\nu_1 + \nu_3)^2 \Big) \eta^2.
\end{align*}
As long as $p \geq c \Big( \frac{\mu^{1.5} s_r^{1.5} (n_1 + n_2) \log( (n_1 \vee n_2) n_3 )}{n_3 \sqrt{n_1 n_2}} \vee \frac{\mu^2 s_r^2 \ell \log( (n_1 \vee n_2) n_3 )}{n_3 \sqrt{(n_1 \wedge n_2) n_3}} \Big)$ for some sufficiently large constant $c$ and $\alpha \mu s_r^{1.5} \frac{n_1 + n_2 n_3}{n_3 \sqrt{n_1 n_2}} \leq \alpha \mu s_r^{1.5} \frac{n_1 + n_2 n_3}{n_3 \sqrt{n_{(2)} n_2}} \leq 10^{-4}$, we have $\nu_1 + \nu_2 + \nu_4 \leq 0.1$ and
\begin{align*}
\frac{c_1 \mu^{1.5} s_r^{1.5}}{p n_3} \sqrt{\frac{\alpha (n_1 + n_2) \log( (n_1 \vee n_2) n_3 )}{n_1 n_2}} + 4 \frac{c_2 \mu^2 s_r^2}{p n_3} \sqrt{\frac{\alpha \ell \log( (n_1 \vee n_2) n_3 )}{(n_1 \wedge n_2) n_3} } \leq 0.0005
\end{align*}
under the setting $\epsilon = 0.02$. With $\frac{1}{5} < \eta \leq \frac{1}{2}$, we can have
\begin{align*}
& \dist^2(\bcF_{t+1}, \bcF_{\star})  \\
\leq & (1 - 0.6 \eta)^2 \dist^2(\bcF_t, \bcF_{\star}) + \Big( 12 \eta (1 - \eta) \alpha \mu s_r^{1.5} \frac{n_1 + n_2 n_3}{n_3 \sqrt{n_1 n_2}} \frac{1}{\epsilon (1 - \epsilon)} + 0.012 \eta^2 \frac{1}{\epsilon (1 - \epsilon)^3}  \\
& + 12 \eta^2 \alpha \mu s_r^{1.5} \frac{n_1 + n_2 n_3}{n_3 \sqrt{n_1 n_2}} \frac{1}{\epsilon (1 - \epsilon)^2} + 18 \eta^2 \alpha^2 \mu^2 s_r^3 \frac{(n_1 + n_2 n_3)^2}{n_1 n_2 n_3^2} \frac{1}{\epsilon^2 (1 - \epsilon)^2} \Big) \frac{1}{\ell} \epsilon^2 \rho^{2t} \bar{\sigma}_{s_r}^2 (\bcX_{\star})  \\
\leq & (1 - 0.3 \eta)^2 \frac{1}{\ell} \epsilon^2 \rho^{2t} \bar{\sigma}_{s_r}^2 (\bcX_{\star}).
\end{align*}
Thus we conclude that
\begin{align*}
\dist(\bcF_{t+1}, \bcF_{\star}) & \leq \frac{\epsilon}{\sqrt{\ell}} \rho^{t+1} \bar{\sigma}_{s_r} (\bcX_{\star})
\end{align*}
by setting $\rho = 1 - 0.3 \eta$.

\emph{Incoherence condition:} We first use \eqref{eqn:RTCexpand} again to obtain
\begin{align*}
& \| (\bcL_{t+1} \ast_{\bPhi} \bcQ_t - \bcL_{\star}) \ast_{\bPhi} \bcG_{\star}^{\frac{1}{2}} \|_{2,\infty}  \\
\leq & (1 - \eta) \| \bcL_{\triangle} \ast_{\bPhi} \bcG_{\star}^{\frac{1}{2}} \|_{2,\infty} + \eta \| \bcL_{\star} \ast_{\bPhi} \bcR_{\triangle}^H \ast_{\bPhi} \bcR_{\sharp} \ast_{\bPhi} (\bcR_{\sharp}^H \ast_{\bPhi} \bcR_{\sharp})^{-1} \ast_{\bPhi} \bcG_{\star}^{\frac{1}{2}} \|_{2,\infty}  \\
&\quad + \eta \| (p^{-1} \bcP_{\bOmega} - \bcI_{n_1})(\bcL_{\sharp} \ast_{\bPhi} \bcR_{\sharp}^H - \bcX_{\star}) \ast_{\bPhi} \bcR_{\sharp} \ast_{\bPhi} (\bcR_{\sharp}^H \ast_{\bPhi} \bcR_{\sharp})^{-1} \ast_{\bPhi} \bcG_{\star}^{\frac{1}{2}} \|_{2,\infty}  \\
&\quad + \eta \| p^{-1} \bcP_{\bOmega} (\bcS_{\triangle}) \ast_{\bPhi} \bcR_{\sharp} \ast_{\bPhi} (\bcR_{\sharp}^H \ast_{\bPhi} \bcR_{\sharp})^{-1} \ast_{\bPhi} \bcG_{\star}^{\frac{1}{2}} \|_{2,\infty}  \\
\coloneq & \mathfrak{W}_1 + \mathfrak{W}_2 + \mathfrak{W}_3 + \mathfrak{W}_4.
\end{align*}

\paragraph{Bound of $\mathfrak{W}_1$ and $\mathfrak{W}_2$.} The bounds of $\mathfrak{W}_1$ and $\mathfrak{W}_2$ are identical to those of $\mathfrak{B}_1$ and $\mathfrak{B}_2$ in the proof of Lemma~\ref{lemma:TRPCAcontraction}, respectively, i.e., $\mathfrak{W}_1 \leq (1 - \eta) \sqrt{\frac{\mu s_r}{n_1 n_3 \ell}} \rho^t \bar{\sigma}_{s_r} (\bcX_{\star})$ and $\mathfrak{W}_2 \leq \eta \frac{\epsilon}{1 - \epsilon} \sqrt{\frac{\mu s_r}{n_1 n_3 \ell}} \rho^t \bar{\sigma}_{s_r} (\bcX_{\star})$.

\paragraph{Bound of $\mathfrak{W}_3$.} Following Lemma~\ref{lemma:2infbound} and Lemma~\ref{lemma:POmegarownorm}, we have
\begin{align*}
\mathfrak{W}_3 & \leq \eta \| (p^{-1} \bcP_{\bOmega} - \bcI_{n_1})(\bcL_{\sharp} \ast_{\bPhi} \bcR_{\sharp}^H - \bcX_{\star}) \|_{2,\infty} \| \bcR_{\sharp} \ast_{\bPhi} (\bcR_{\sharp}^H \ast_{\bPhi} \bcR_{\sharp})^{-1} \ast_{\bPhi} \bcG_{\star}^{\frac{1}{2}} \|  \\
& \leq \frac{\eta}{1 - \epsilon} \| (p^{-1} \bcP_{\bOmega} - \bcI_{n_1})(\bcL_{\sharp} \ast_{\bPhi} \bcR_{\sharp}^H - \bcX_{\star}) \|_{2,\infty}  \\
& \leq \frac{\eta}{1 - \epsilon} \Big( 2 \sqrt{\frac{c \log(n_2 n_3)}{p}} \| \bcL_{\sharp} \ast_{\bPhi} \bcR_{\sharp}^H - \bcX_{\star} \|_{2,\infty} + \frac{c \log(n_2 n_3)}{p} \| \bcL_{\sharp} \ast_{\bPhi} \bcR_{\sharp}^H - \bcX_{\star} \|_{\infty} \Big)  \\
& = \frac{\eta}{1 - \epsilon} \Big( 2 \sqrt{\frac{c \log(n_2 n_3)}{p}} \| \bcL_{\triangle} \ast_{\bPhi} \bcR_{\sharp}^H + \bcL_{\star} \ast_{\bPhi} \bcR_{\triangle}^H \|_{2,\infty} + \frac{c \log(n_2 n_3)}{p} \| \bcL_{\triangle} \ast_{\bPhi} \bcR_{\sharp}^H + \bcL_{\star} \ast_{\bPhi} \bcR_{\triangle}^H \|_{\infty} \Big),
\end{align*}
where
\begin{align*}
& \| \bcL_{\triangle} \ast_{\bPhi} \bcR_{\sharp}^H + \bcL_{\star} \ast_{\bPhi} \bcR_{\triangle}^H \|_{2,\infty}  \\
\leq & \| \bcL_{\triangle} \ast_{\bPhi} \bcR_{\sharp}^H \|_{2,\infty} + \| \bcL_{\star} \ast_{\bPhi} \bcR_{\triangle}^H \|_{2,\infty}  \\
\leq & \| \bcL_{\triangle} \ast_{\bPhi} \bcG_{\star}^{\frac{1}{2}} \|_{2,\infty} \| \bcR_{\sharp} \ast_{\bPhi} \bcG_{\star}^{-\frac{1}{2}} \| + \| \bcL_{\star} \ast_{\bPhi} \bcG_{\star}^{-\frac{1}{2}} \|_{2,\infty} \| \bcR_{\triangle} \ast_{\bPhi} \bcG_{\star}^{\frac{1}{2}} \|  \\
\leq & \| \bcL_{\triangle} \ast_{\bPhi} \bcG_{\star}^{\frac{1}{2}} \|_{2,\infty} (\| \bcR_{\triangle} \ast_{\bPhi} \bcG_{\star}^{\frac{1}{2}} \| \| \bcG_{\star}^{-1} \| + \| \bcR_{\star} \ast_{\bPhi} \bcG_{\star}^{-\frac{1}{2}} \|) + \| \bcL_{\star} \ast_{\bPhi} \bcG_{\star}^{-\frac{1}{2}} \|_{2,\infty} \| \bcR_{\triangle} \ast_{\bPhi} \bcG_{\star}^{\frac{1}{2}} \|  \\
\leq & \sqrt{\frac{\mu s_r}{n_1 n_3 \ell}} \rho^t \bar{\sigma}_{s_r} (\bcX_{\star}) (\epsilon \rho^t + 1) + \sqrt{\frac{\mu s_r}{n_1 n_3 \ell}} \epsilon \rho^t \bar{\sigma}_{s_r} (\bcX_{\star})  \\
\leq & \sqrt{\frac{\mu s_r}{n_1 n_3 \ell}} (2 \epsilon + 1) \rho^t \bar{\sigma}_{s_r} (\bcX_{\star}),
\end{align*}
and $\| \bcL_{\triangle} \ast_{\bPhi} \bcR_{\sharp}^H + \bcL_{\star} \ast_{\bPhi} \bcR_{\triangle}^H \|_{\infty} \leq 3 \frac{\mu s_r}{n_3 \sqrt{n_1 n_2 \ell}} \rho^t \bar{\sigma}_{s_r} (\bcX_{\star})$. Thus
\begin{align*}
\mathfrak{W}_3 & \leq \frac{\eta}{1 - \epsilon} \Big( 2 \sqrt{\frac{c \log(n_2 n_3)}{p}} \sqrt{\frac{\mu s_r}{n_1 n_3 \ell}} (2 \epsilon + 1) \rho^t \bar{\sigma}_{s_r} (\bcX_{\star}) + \frac{c \log(n_2 n_3)}{p} 3 \frac{\mu s_r}{n_3 \sqrt{n_1 n_2 \ell}} \rho^t \bar{\sigma}_{s_r} (\bcX_{\star}) \Big).
\end{align*}

\paragraph{Bound of $\mathfrak{W}_4$.}
\begin{align*}
\mathfrak{W}_4 & \leq \eta \| p^{-1} \bcP_{\bOmega} (\bcS_{\triangle}) \|_{2,\infty} \| \bcR_{\sharp} \ast_{\bPhi} (\bcR_{\sharp}^H \ast_{\bPhi} \bcR_{\sharp})^{-1} \ast_{\bPhi} \bcG_{\star}^{\frac{1}{2}} \| \leq \frac{\eta}{p} \frac{\sqrt{\alpha p n_2 n_3}}{1 - \epsilon} \| \bcP_{\bOmega} (\bcS_{\triangle}) \|_{\infty}  \\
& \leq 6 \eta \frac{\sqrt{\alpha \mu s_r}}{1 - \epsilon} \sqrt{\frac{\mu s_r}{p n_1 n_3 \ell}} \rho^t \bar{\sigma}_{s_r} (\bcX_{\star}).
\end{align*}
Putting all the bounds together, we obtain
\begin{align*}
& \| (\bcL_{t+1} \ast_{\bPhi} \bcQ_t - \bcL_{\star}) \ast_{\bPhi} \bcG_{\star}^{\frac{1}{2}} \|_{2,\infty}  \\
\leq & \Big( (1 - \eta) + \eta \frac{\epsilon}{1 - \epsilon} + \frac{\eta}{1 - \epsilon} \big( 2 \sqrt{\frac{c \log(n_2 n_3)}{p}} (2 \epsilon + 1) + 3 \frac{c \log(n_2 n_3)}{p} \sqrt{\frac{\mu s_r}{n_2 n_3}} + 6 \sqrt{\frac{\alpha \mu s_r}{p}} \big) \Big) \sqrt{\frac{\mu s_r}{n_1 n_3 \ell}} \rho^t \bar{\sigma}_{s_r} (\bcX_{\star}).
\end{align*}
We can then have
\begin{align*}
& \| (\bcL_{t+1} \ast_{\bPhi} \bcQ_t - \bcL_{\star}) \ast_{\bPhi} \bcG_{\star}^{-\frac{1}{2}} \|_{2,\infty}  \\
\leq & \Big( (1 - \eta) + \eta \frac{\epsilon}{1 - \epsilon} + \frac{\eta}{1 - \epsilon} \big( 2 \sqrt{\frac{c \log(n_2 n_3)}{p}} (2 \epsilon + 1) + 3 \frac{c \log(n_2 n_3)}{p} \sqrt{\frac{\mu s_r}{n_2 n_3}} + 6 \sqrt{\frac{\alpha \mu s_r}{p}} \big) \Big) \sqrt{\frac{\mu s_r}{n_1 n_3 \ell}} \rho^t.
\end{align*}
Applying the triangle inequality and Lemma~\ref{lemma:Qperturbation}, we have
\begin{align*}
& \| (\bcL_{t+1} \ast_{\bPhi} \bcQ_{t+1} - \bcL_{\star}) \ast_{\bPhi} \bcG_{\star}^{\frac{1}{2}} \|_{2,\infty}  \\
\leq & \| (\bcL_{t+1} \ast_{\bPhi} \bcQ_t - \bcL_{\star}) \ast_{\bPhi} \bcG_{\star}^{\frac{1}{2}} \|_{2,\infty} + \Big( \| (\bcL_{t+1} \ast_{\bPhi} \bcQ_t - \bcL_{\star}) \ast_{\bPhi} \bcG_{\star}^{-\frac{1}{2}} \|_{2,\infty}  \\
&\qquad\qquad + \| \bcL_{\star} \ast_{\bPhi} \bcG_{\star}^{-\frac{1}{2}} \|_{2,\infty} \Big) \| \bcG_{\star}^{\frac{1}{2}} \ast_{\bPhi} \bcQ_t^{-1} \ast_{\bPhi} (\bcQ_{t+1} - \bcQ_t) \ast_{\bPhi} \bcG_{\star}^{\frac{1}{2}} \|  \\
\leq & \Big( (1 - \eta) + \eta \frac{\epsilon}{1 - \epsilon} + \frac{\eta}{1 - \epsilon} \big( 2 \sqrt{\frac{c \log(n_2 n_3)}{p}} (2 \epsilon + 1) + 3 \frac{c \log(n_2 n_3)}{p} \sqrt{\frac{\mu s_r}{n_2 n_3}}  \\
&\qquad\qquad + 6 \sqrt{\frac{\alpha \mu s_r}{p}} \big) \Big) \sqrt{\frac{\mu s_r}{n_1 n_3 \ell}} \rho^t \bar{\sigma}_{s_r} (\bcX_{\star}) + \Big( \Big( (1 - \eta) + \eta \frac{\epsilon}{1 - \epsilon} + \frac{\eta}{1 - \epsilon} \big( 2 \sqrt{\frac{c \log(n_2 n_3)}{p}} (2 \epsilon + 1)  \\
&\qquad\qquad + 3 \frac{c \log(n_2 n_3)}{p} \sqrt{\frac{\mu s_r}{n_2 n_3}} + 6 \sqrt{\frac{\alpha \mu s_r}{p}} \big) \Big) \sqrt{\frac{\mu s_r}{n_1 n_3 \ell}} \rho^t + \sqrt{\frac{\mu s_r}{n_1 n_3 \ell}} \Big) \frac{2 \epsilon}{1 - \epsilon} \rho^{t+1} \bar{\sigma}_{s_r} (\bcX_{\star}) \\
\leq & \Big( (1 - \eta) + \eta \frac{\epsilon}{1 - \epsilon} + \frac{\eta}{1 - \epsilon} \big( 2 \sqrt{\frac{c \log(n_2 n_3)}{p}} (2 \epsilon + 1) + 3 \frac{c \log(n_2 n_3)}{p} \sqrt{\frac{\mu s_r}{n_2 n_3}} + 6 \sqrt{\frac{\alpha \mu s_r}{p}} \big)  \\
&\qquad\qquad + \frac{2 \epsilon}{1 - \epsilon} \Big( (2 - \eta) + \eta \frac{\epsilon}{1 - \epsilon} + \frac{\eta}{1 - \epsilon} \big( 2 \sqrt{\frac{c \log(n_2 n_3)}{p}} (2 \epsilon + 1) + 3 \frac{c \log(n_2 n_3)}{p} \sqrt{\frac{\mu s_r}{n_2 n_3}}  \\
&\qquad\qquad + 6 \sqrt{\frac{\alpha \mu s_r}{p}} \big) \Big) \Big) \sqrt{\frac{\mu s_r}{n_1 n_3 \ell}} \rho^t \bar{\sigma}_{s_r} (\bcX_{\star}).
\end{align*}
Since $p \geq c \frac{\mu^2 s_r^2 \ell \log( (n_1 \vee n_2) n_3 )}{n_3 \sqrt{(n_1 \wedge n_2) n_3}} \geq c \frac{\mu^2 s_r \ell \log( (n_1 \vee n_2) n_3 )}{\sqrt{(n_1 \wedge n_2) n_3}}$ for some sufficiently large constant $c$, we can have $2 \sqrt{\frac{c \log(n_2 n_3)}{p}} (2 \epsilon + 1) + 3 \frac{c \log(n_2 n_3)}{p} \sqrt{\frac{\mu s_r}{n_2 n_3}} \leq 0.1$ with $\epsilon = 0.02$. With $\alpha \mu s_r / p < \alpha \mu s_r^{1.5} \frac{n_1 + n_2 n_3}{p n_3 \sqrt{n_{(2)} n_2}} \leq 10^{-4}$, and $\frac{1}{5} \leq \eta \leq \frac{1}{2}$, we can have the conclusion that $\| (\bcL_{t+1} \ast_{\bPhi} \bcQ_{t+1} - \bcL_{\star}) \ast_{\bPhi} \bcG_{\star}^{\frac{1}{2}} \|_{2,\infty} \leq (1 - 0.3 \eta) \sqrt{\frac{\mu s_r}{n_1 n_3 \ell}} \rho^t \bar{\sigma}_{s_r} (\bcX_{\star})$. We can also obtain a similar bound for $\| (\bcR_{t+1} \ast_{\bPhi} \bcQ_{t+1}^{-H} - \bcR_{\star}) \ast_{\bPhi} \bcG_{\star}^{\frac{1}{2}} \|_{2,\infty}$. Thus,
\begin{align*}
& \sqrt{n_1} \| (\bcL_{t+1} \ast_{\bPhi} \bcQ_{t+1} - \bcL_{\star}) \ast_{\bPhi} \bcG_{\star}^{\frac{1}{2}} \|_{2,\infty} \vee \sqrt{n_2} \| (\bcR_{t+1} \ast_{\bPhi} \bcQ_{t+1}^{-H} - \bcR_{\star}) \ast_{\bPhi} \bcG_{\star}^{\frac{1}{2}} \|_{2,\infty}  \\
& \leq \sqrt{\frac{\mu s_r}{n_3 \ell}} \rho^{t+1} \bar{\sigma}_{s_r} (\bcX_{\star})
\end{align*}
can be achieved by setting $\rho = 1 - 0.3 \eta$. This finishes the proof.

\subsection{Proof of Lemma~\ref{lemma:RTCinitial}}

Invoking Lemma~\ref{lemma:sparity} with $\bcX_{-1} = \bzero$, we have
\begin{align*}
\| \bcP_{\bOmega} (\bcS_{\star}) - \bcS_0 \|_{\infty} \leq 2 \frac{\mu s_r}{n_3 \sqrt{n_1 n_2 \ell}} \bar{\sigma}_1 (\bcX_{\star}) \quad \mathrm{and} \quad \supp(\bcS_0) \subseteq \supp(\bcP_{\bOmega} (\bcS_{\star})),
\end{align*}
which implies $\bcP_{\bOmega} (\bcS_{\star}) - \bcS_0$ is an $\alpha p$-sparse tensor. Let $\bcM \coloneq p^{-1} \bcP_{\bOmega} (\bcY - \bcS_0)$, we have
\begin{align*}
\| \bcX_{\star} - \bcM \| \leq \| \bcM - p^{-1} \bcP_{\bOmega} (\bcX_{\star}) \| + \| p^{-1} \bcP_{\bOmega} (\bcX_{\star}) - \bcX_{\star} \|.
\end{align*}
For the first term, notice that $\bcM - p^{-1} \bcP_{\bOmega} (\bcX_{\star}) = p^{-1} (\bcP_{\bOmega} (\bcS_{\star}) - \bcS_0)$. Applying Lemma~\ref{lemma:normbound}, we have
\begin{align*}
\| \bcM - p^{-1} \bcP_{\bOmega} (\bcX_{\star}) \| \leq \frac{\alpha p \sqrt{\ell}}{2} (n_1 + n_2 n_3) \| \bcP_{\bOmega} (\bcS_{\star}) - \bcS_0 \|_{\infty} \leq \frac{\alpha p \mu s_r \kappa}{n_3 \sqrt{n_1 n_2}} (n_1 + n_2 n_3) \bar{\sigma}_{s_r} (\bcX_{\star}).
\end{align*}
For the second term, we again have
\begin{align*}
\| p^{-1} \bcP_{\bOmega} (\bcX_{\star}) - \bcX_{\star} \| \leq c \Big( \frac{\mu s_r \log( (n_1 \vee n_2) n_3 ) }{p n_3 \sqrt{n_1 n_2}} + \sqrt{\frac{\mu s_r \log( (n_1 \vee n_2) n_3 )}{p (n_1 \wedge n_2) n_3}} \Big) \kappa \bar{\sigma}_{s_r} (\bcX_{\star}).
\end{align*}
Since $\bcX_0 = \bcL_0 \ast_{\bPhi} \bcR_0^H$ is the best approximation of $\bcM$ with tubal rank $r$, we obtain
\begin{align*}
\| \bcX_{\star} - \bcX_0 \| & \leq \| \bcX_{\star} - \bcM \| + \| \bcM - \bcX_0 \| \leq 2 \| \bcX_{\star} - \bcM \|  \\
& \leq 2 \Big( \frac{\alpha p \mu s_r}{n_3 \sqrt{n_1 n_2}} (n_1 + n_2 n_3)  \\
&\qquad + c \Big( \frac{\mu s_r \log( (n_1 \vee n_2) n_3 ) }{p n_3 \sqrt{n_1 n_2}} + \sqrt{\frac{\mu s_r \log( (n_1 \vee n_2) n_3 )}{p (n_1 \wedge n_2) n_3}} \Big) \Big) \kappa \bar{\sigma}_{s_r} (\bcX_{\star}).
\end{align*}
Applying Lemma~\ref{lemma:Procrustes} and using the fact that $\widebar{\bX}_{\star} - \widebar{\bX}_0$ has at most rank-$2 s_r$, we obtain
\begin{align*}
\dist(\bcF_0, \bcF_{\star}) & \leq (\sqrt{2}+1)^{\frac{1}{2}} \frac{1}{\sqrt{\ell}} \| \widebar{\bX}_{\star} - \widebar{\bX}_0 \|_F \leq \sqrt{\frac{2(\sqrt{2}+1) s_r}{\ell}} \| \widebar{\bX}_{\star} - \widebar{\bX}_0 \|  \\
& \leq 5 \sqrt{\frac{s_r}{\ell}} \Big( \frac{\alpha p \mu s_r}{n_3 \sqrt{n_1 n_2}} (n_1 + n_2 n_3)  \\
&\qquad + c \Big( \frac{\mu s_r \log( (n_1 \vee n_2) n_3 ) }{p n_3 \sqrt{n_1 n_2}} + \sqrt{\frac{\mu s_r \log( (n_1 \vee n_2) n_3 )}{p (n_1 \wedge n_2) n_3}} \Big) \Big) \kappa \bar{\sigma}_{s_r} (\bcX_{\star}).
\end{align*}
As long as $p \geq c \mu s_r \kappa^2 \log( (n_1 \vee n_2) n_3 ) / \sqrt{(n_1 \wedge n_2) n_3}$ for some sufficiently large constant $c$, we can have $5 \sqrt{s_r} \kappa c \Big( \frac{\mu s_r \log( (n_1 \vee n_2) n_3 ) }{p n_3 \sqrt{n_1 n_2}} + \sqrt{\frac{\mu s_r \log( (n_1 \vee n_2) n_3 )}{p (n_1 \wedge n_2) n_3}} \Big) \leq 0.01$. Given $\epsilon = 5 c_0 + 0.01$ and $\alpha / p \leq \frac{c_0}{\mu s_r^{1.5} \kappa \frac{n_1 + n_2 n_3}{n_3 \sqrt{n_1 n_2}}}$, our first claim
\begin{align}\label{eqn:initdist2}
\dist(\bcF_0, \bcF_{\star}) \leq \frac{5 c_0 + 0.01}{\sqrt{\ell}} \bar{\sigma}_{s_r} (\bcX_{\star})
\end{align}
is proved.

To prove the second claim, we again define $\bcL_{\sharp} \coloneq \bcL_0 \ast_{\bPhi} \bcQ_0$, $\bcR_{\sharp} \coloneq \bcR_0 \ast_{\bPhi} \bcQ_0^{-H}$, $\bcL_{\triangle} \coloneq \bcL_{\sharp} - \bcL_{\star}$, $\bcR_{\triangle} \coloneq \bcR_{\sharp} - \bcR_{\star}$, $\bcS_{\triangle} \coloneq \bcS_0 - \bcS_{\star}$. Notice that $\bcU_0 \ast_{\bPhi} \bcG_0 \ast_{\bPhi} \bcV_0^H = \text{t-SVD}_r (p^{-1} \bcP_{\bOmega} (\bcY - \bcS_0)) = \text{t-SVD}_r (p^{-1} \bcP_{\bOmega} (\bcX_{\star} - \bcS_{\triangle}))$, thus 
\begin{align*}
\bcL_0 & = \bcU_0 \ast_{\bPhi} \bcG_0^{\frac{1}{2}} = p^{-1} \bcP_{\bOmega} (\bcX_{\star} - \bcS_{\triangle}) \ast_{\bPhi} \bcV_0 \ast_{\bPhi} \bcG_0^{-\frac{1}{2}} = p^{-1} \bcP_{\bOmega} (\bcX_{\star} - \bcS_{\triangle}) \ast_{\bPhi} \bcR_0 \ast_{\bPhi} \bcG_0^{-1}  \\
& = p^{-1} \bcP_{\bOmega} (\bcX_{\star} - \bcS_{\triangle}) \ast_{\bPhi} \bcR_0 \ast_{\bPhi} (\bcR_0^H \ast_{\bPhi} \bcR_0)^{-1}.
\end{align*}
Multiplying $\bcQ_0 \ast_{\bPhi} \bcG_{\star}^{\frac{1}{2}}$ on both sides, we have
\begin{align*}
\bcL_{\sharp} \ast_{\bPhi} \bcG_{\star}^{\frac{1}{2}} & = \bcL_0 \ast_{\bPhi} \bcQ_0 \ast_{\bPhi} \bcG_{\star}^{\frac{1}{2}} = p^{-1} \bcP_{\bOmega} (\bcX_{\star} - \bcS_{\triangle}) \ast_{\bPhi} \bcR_0 \ast_{\bPhi} (\bcR_0^H \ast_{\bPhi} \bcR_0)^{-1} \ast_{\bPhi} \bcQ_0 \ast_{\bPhi} \bcG_{\star}^{\frac{1}{2}}  \\
& = p^{-1} \bcP_{\bOmega} (\bcX_{\star} - \bcS_{\triangle}) \ast_{\bPhi} \bcR_{\sharp} \ast_{\bPhi} (\bcR_{\sharp}^H \ast_{\bPhi} \bcR_{\sharp})^{-1} \ast_{\bPhi} \bcG_{\star}^{\frac{1}{2}}.
\end{align*}
Subtracting $\bcX_{\star} \ast_{\bPhi} \bcR_{\sharp} \ast_{\bPhi} (\bcR_{\sharp}^H \ast_{\bPhi} \bcR_{\sharp})^{-1} \ast_{\bPhi} \bcG_{\star}^{\frac{1}{2}}$ on both sides, we have
\begin{align*}
& \bcL_{\triangle} \ast_{\bPhi} \bcG_{\star}^{\frac{1}{2}} + \bcL_{\star} \ast_{\bPhi} \bcR_{\triangle}^H \ast_{\bPhi} \bcR_{\sharp} \ast_{\bPhi} (\bcR_{\sharp}^H \ast_{\bPhi} \bcR_{\sharp})^{-1} \ast_{\bPhi} \bcG_{\star}^{\frac{1}{2}}  \\
= & (p^{-1} \bcP_{\bOmega} (\bcX_{\star}) - \bcX_{\star}) \ast_{\bPhi} \bcR_{\sharp} \ast_{\bPhi} (\bcR_{\sharp}^H \ast_{\bPhi} \bcR_{\sharp})^{-1} \ast_{\bPhi} \bcG_{\star}^{\frac{1}{2}} - p^{-1} \bcP_{\bOmega} (\bcS_{\triangle}) \ast_{\bPhi} \bcR_{\sharp} \ast_{\bPhi} (\bcR_{\sharp}^H \ast_{\bPhi} \bcR_{\sharp})^{-1} \ast_{\bPhi} \bcG_{\star}^{\frac{1}{2}},
\end{align*}
where we use the fact $\bcL_{\sharp} \ast_{\bPhi} \bcG_{\star}^{\frac{1}{2}} = \bcL_{\sharp} \ast_{\bPhi} \bcR_{\sharp}^H \ast_{\bPhi} \bcR_{\sharp} \ast_{\bPhi} (\bcR_{\sharp}^H \ast_{\bPhi} \bcR_{\sharp})^{-1} \ast_{\bPhi} \bcG_{\star}^{\frac{1}{2}}$. Thus, 
\begin{align*}
\| \bcL_{\triangle} \ast_{\bPhi} \bcG_{\star}^{\frac{1}{2}} \|_{2,\infty} & \leq \| \bcL_{\star} \ast_{\bPhi} \bcR_{\triangle}^H \ast_{\bPhi} \bcR_{\sharp} \ast_{\bPhi} (\bcR_{\sharp}^H \ast_{\bPhi} \bcR_{\sharp})^{-1} \ast_{\bPhi} \bcG_{\star}^{\frac{1}{2}} \|_{2,\infty}  \\
&\qquad + \| (p^{-1} \bcP_{\bOmega} (\bcX_{\star}) - \bcX_{\star}) \ast_{\bPhi} \bcR_{\sharp} \ast_{\bPhi} (\bcR_{\sharp}^H \ast_{\bPhi} \bcR_{\sharp})^{-1} \ast_{\bPhi} \bcG_{\star}^{\frac{1}{2}} \|_{2,\infty}  \\
&\qquad + \| p^{-1} \bcP_{\bOmega} (\bcS_{\triangle}) \ast_{\bPhi} \bcR_{\sharp} \ast_{\bPhi} (\bcR_{\sharp}^H \ast_{\bPhi} \bcR_{\sharp})^{-1} \ast_{\bPhi} \bcG_{\star}^{\frac{1}{2}} \|_{2,\infty}  \\
& \coloneq \mathfrak{G}_1 + \mathfrak{G}_2 + \mathfrak{G}_3.
\end{align*}

\paragraph{Bound of $\mathfrak{G}_1$.} This term has been controlled as in \eqref{eqn:J1bound} by $\mathfrak{G}_1 \leq \sqrt{\frac{\mu s_r}{n_1 n_3 \ell}} \frac{\epsilon}{1 - \epsilon} \bar{\sigma}_{s_r} (\bcX_{\star})$.

\paragraph{Bound of $\mathfrak{G}_2$.} By Lemma~\ref{lemma:2infbound} and Lemma~\ref{lemma:POmegarownorm}, we have
\begin{align*}
\mathfrak{G}_2 & \leq \| (p^{-1} \bcP_{\bOmega} (\bcX_{\star}) - \bcX_{\star}) \|_{2,\infty} \| \bcR_{\sharp} \ast_{\bPhi} (\bcR_{\sharp}^H \ast_{\bPhi} \bcR_{\sharp})^{-1} \ast_{\bPhi} \bcG_{\star}^{\frac{1}{2}} \|  \\
& \leq \frac{1}{1 - \epsilon} \Big( 2 \sqrt{\frac{c \log(n_2 n_3)}{p}} \| \bcX_{\star} \|_{2,\infty} + \frac{c \log(n_2 n_3)}{p} \| \bcX_{\star} \|_{\infty} \Big)  \\
& \leq \frac{\kappa}{1 - \epsilon} \Big( 2 \sqrt{\frac{c \log(n_2 n_3)}{p}} \sqrt{\frac{\mu s_r}{n_1 n_3 \ell}} + \frac{c \log(n_2 n_3)}{p} \frac{\mu s_r}{\sqrt{n_1 n_2 \ell} n_3} \Big) \bar{\sigma}_{s_r} (\bcX_{\star})
\end{align*}
holds with high probability.

\paragraph{Bound of $\mathfrak{G}_3$.} By Lemmas~\ref{lemma:normbound}, ~\ref{lemma:2infbound} and~\ref{lemma:2infprodbound}, we have
\begin{align*}
\mathfrak{G}_3 & \leq \sqrt{n_2 \ell} \| p^{-1} \bcP_{\bOmega} (\bcS_{\triangle}) \|_{2,\infty} \| \bcR_{\sharp} \ast_{\bPhi} (\bcR_{\sharp}^H \ast_{\bPhi} \bcR_{\sharp})^{-1} \ast_{\bPhi} \bcG_{\star}^{\frac{1}{2}} \|_{2,\infty}  \\
& \leq p^{-1} \sqrt{n_2 \ell} \| \bcP_{\bOmega} (\bcS_{\triangle}) \|_{2,\infty} \| \bcR_{\sharp} \ast_{\bPhi} \bcG_{\star}^{-\frac{1}{2}} \|_{2,\infty} \| \bcG_{\star}^{\frac{1}{2}} \ast_{\bPhi} (\bcR_{\sharp}^H \ast_{\bPhi} \bcR_{\sharp})^{-1} \ast_{\bPhi} \bcG_{\star}^{\frac{1}{2}} \|  \\
& \leq p^{-1} \sqrt{n_2 \ell} \sqrt{\alpha p n_2 n_3} \| \bcP_{\bOmega} (\bcS_{\triangle}) \|_{\infty} \| \bcR_{\sharp} \ast_{\bPhi} \bcG_{\star}^{-\frac{1}{2}} \|_{2,\infty} \| \bcR_{\sharp} \ast_{\bPhi} (\bcR_{\sharp}^H \ast_{\bPhi} \bcR_{\sharp})^{-1} \ast_{\bPhi} \bcG_{\star}^{\frac{1}{2}} \|^2  \\
& \leq 2 n_2 \sqrt{\alpha n_3 \ell} \frac{\mu s_r}{n_3 \sqrt{p n_1 n_2 \ell}} \bar{\sigma}_1 (\bcX_{\star}) \| \bcR_{\sharp} \ast_{\bPhi} \bcG_{\star}^{-\frac{1}{2}} \|_{2,\infty} \| \bcR_{\sharp} \ast_{\bPhi} (\bcR_{\sharp}^H \ast_{\bPhi} \bcR_{\sharp})^{-1} \ast_{\bPhi} \bcG_{\star}^{\frac{1}{2}} \|^2  \\
& \leq 2 \frac{\mu s_r \kappa}{(1 - \epsilon)^2} \sqrt{\frac{\alpha n_2}{p n_1 n_3}} \bar{\sigma}_{s_r} (\bcX_{\star}) \| \bcR_{\sharp} \ast_{\bPhi} \bcG_{\star}^{-\frac{1}{2}} \|_{2,\infty}  \\
& \leq 2 \frac{\mu s_r \kappa}{(1 - \epsilon)^2} \sqrt{\frac{\alpha n_2}{p n_1 n_3}} \Big( \sqrt{\frac{\mu s_r}{n_2 n_3 \ell}} + \| \bcR_{\triangle} \ast_{\bPhi} \bcG_{\star}^{-\frac{1}{2}} \|_{2,\infty} \Big) \bar{\sigma}_{s_r} (\bcX_{\star}).
\end{align*}
Note that $\| \bcR_{\triangle} \ast_{\bPhi} \bcG_{\star}^{-\frac{1}{2}} \|_{2,\infty} \leq \frac{\| \bcR_{\triangle} \ast_{\bPhi} \bcG_{\star}^{\frac{1}{2}} \|_{2,\infty}}{\bar{\sigma}_{s_r} (\bcX_{\star})}$. Thus,
\begin{align*}
\sqrt{n_1} \| \bcL_{\triangle} \ast_{\bPhi} \bcG_{\star}^{\frac{1}{2}} \|_{2,\infty} & \leq \sqrt{\frac{\mu s_r}{n_3 \ell}} \Big( \frac{\epsilon}{1 - \epsilon} + \frac{2 \kappa}{1 - \epsilon} \sqrt{\frac{c \log(n_2 n_3)}{p}} + \frac{c \kappa \log(n_2 n_3)}{p (1 - \epsilon)} \sqrt{\frac{\mu s_r}{n_2 n_3}}  \\
&\qquad\qquad + 2 \frac{\mu s_r \kappa}{(1 - \epsilon)^2} \sqrt{\frac{\alpha}{p n_3}} \Big) \bar{\sigma}_{s_r} (\bcX_{\star}) + 2 \frac{\mu s_r \kappa}{(1 - \epsilon)^2} \sqrt{\frac{\alpha}{p n_3}} \sqrt{n_2} \| \bcR_{\triangle} \ast_{\bPhi} \bcG_{\star}^{\frac{1}{2}} \|_{2,\infty},
\end{align*}
and similarly one can see
\begin{align*}
\sqrt{n_2} \| \bcR_{\triangle} \ast_{\bPhi} \bcG_{\star}^{\frac{1}{2}} \|_{2,\infty} & \leq \sqrt{\frac{\mu s_r}{n_3 \ell}} \Big( \frac{\epsilon}{1 - \epsilon} + \frac{2 \kappa}{1 - \epsilon} \sqrt{\frac{c \log(n_1 n_3)}{p}} + \frac{c \kappa \log(n_1 n_3)}{p (1 - \epsilon)} \sqrt{\frac{\mu s_r}{n_1 n_3}}  \\
&\qquad\qquad + 2 \frac{\mu s_r \kappa}{(1 - \epsilon)^2} \sqrt{\frac{\alpha}{p n_3}} \Big) \bar{\sigma}_{s_r} (\bcX_{\star}) + 2 \frac{\mu s_r \kappa}{(1 - \epsilon)^2} \sqrt{\frac{\alpha}{p n_3}} \sqrt{n_1} \| \bcL_{\triangle} \ast_{\bPhi} \bcG_{\star}^{\frac{1}{2}} \|_{2,\infty}.
\end{align*}
Since $p \geq c \mu s_r \kappa^2 \log( (n_1 \vee n_2) n_3 ) / \sqrt{(n_1 \wedge n_2) n_3}$ for some sufficiently large constant $c$, we can have $2 \kappa \sqrt{\frac{c \log( (n_1 \vee n_2) n_3 )}{p}} + \frac{c \kappa \log( (n_1 \vee n_2) n_3 )}{p} \sqrt{\frac{\mu s_r}{(n_1 \wedge n_2) n_3}} \leq 0.1$. Therefore,
\begin{align*}
& \sqrt{n_1} \| \bcL_{\triangle} \ast_{\bPhi} \bcG_{\star}^{\frac{1}{2}} \|_{2,\infty} \vee \sqrt{n_2} \| \bcR_{\triangle} \ast_{\bPhi} \bcG_{\star}^{\frac{1}{2}} \|_{2,\infty}  \\
\leq & \frac{\frac{\epsilon}{1 - \epsilon} + \frac{2 \kappa}{1 - \epsilon} \sqrt{\frac{c \log( (n_1 \vee n_2) n_3 )}{p}} + \frac{c \kappa \log( (n_1 \vee n_2) n_3 )}{p (1 - \epsilon)} \sqrt{\frac{\mu s_r}{(n_1 \wedge n_2) n_3}} + 2 \frac{\mu s_r \kappa}{(1 - \epsilon)^2} \sqrt{\frac{\alpha}{p n_3}}}{1 - 2 \frac{\mu s_r \kappa}{(1 - \epsilon)^2} \sqrt{\frac{\alpha}{p n_3}}} \sqrt{\frac{\mu s_r}{n_3 \ell}} \bar{\sigma}_{s_r} (\bcX_{\star})  \\
\leq & \frac{\frac{\epsilon}{1 - \epsilon} + \frac{0.1}{1 - \epsilon} + 2 \frac{c_0}{(1 - \epsilon)^2}}{1 - 2 \frac{c_0}{(1 - \epsilon)^2}} \sqrt{\frac{\mu s_r}{n_3 \ell}} \bar{\sigma}_{s_r} (\bcX_{\star})  \\
= & \frac{\frac{5 c_0}{1 - 5 c_0} + \frac{0.1}{1 - 5 c_0} + 2 \frac{c_0}{(1 - 5 c_0)^2}}{1 - 2 \frac{c_0}{(1 - 5 c_0)^2}} \sqrt{\frac{\mu s_r}{n_3 \ell}} \bar{\sigma}_{s_r} (\bcX_{\star})  \\
\leq & \sqrt{\frac{\mu s_r}{n_3 \ell}} \bar{\sigma}_{s_r} (\bcX_{\star})
\end{align*}
as long as $c_0 \leq 0.05$. This finishes the proof.

\section{Proof for Tensor Regression}
\label{sec:Regressionproof}

We begin with a useful lemma regarding TRIP, which can be proved in the same way as in Lemma E.7 of \citet{HanWZ.AOS2022}.
\begin{lemma}\label{lemma:TRIPdistance}
Suppose that $\cA(\cdot)$ obeys the $2r$-TRIP with a constant $\delta_{2r}$. Then for any $\bcX_1, \bcX_2 \in \R^{n_1 \times n_2 \times n_3}$ of tubal rank at most $r$, we have
\begin{align*}
| \langle \cA(\bcX_1), \cA(\bcX_2) \rangle - \langle \bcX_1, \bcX_2 \rangle | \leq \delta_{2r} \| \bcX_1 \|_F \| \bcX_2 \|_F,
\end{align*}
or equivalently,
\begin{align}\label{eqn:2rTRIP}
| \langle (\cA^{\ast}\cA - \bcI_{n_1})(\bcX_1), \bcX_2 \rangle | \leq \delta_{2r} \| \bcX_1 \|_F \| \bcX_2 \|_F.
\end{align}
\end{lemma}
Then following from the variational representation of the Frobenius norm, for any tensor $\bcR \in \R^{n_2 \times r \times n_3}$, we have
\begin{align}\label{eqn:coroTRIP}
\| (\cA^{\ast}\cA - \bcI_{n_1})(\bcX_1) \ast_{\bPhi} \bcR \|_F & = \max_{\widetilde{\bcL} : \| \widetilde{\bcL} \|_F \leq 1} \langle (\cA^{\ast}\cA - \bcI_{n_1})(\bcX_1) \ast_{\bPhi} \bcR, \widetilde{\bcL} \rangle  \nonumber \\
& \leq \max_{\widetilde{\bcL} : \| \widetilde{\bcL} \|_F \leq 1} \delta_{2r} \| \bcX_1 \|_F \| \widetilde{\bcL} \ast_{\bPhi} \bcR^H \|_F  \nonumber \\
& \leq \delta_{2r} \| \bcX_1 \|_F \| \bcR \|,
\end{align}
where the last line follows from the relation $\| \bcA \ast_{\bPhi} \bcB \|_F \leq \| \bcA \|_F \| \bcB \|$.

\subsection{Proof of Lemma~\ref{lemma:TRcontraction}}
 
Since $\dist(\bcF_t, \bcF_{\star}) \leq \frac{0.1}{\sqrt{\ell}} \bar{\sigma}_{s_r} (\bcX_{\star})$, Lemma~\ref{lemma:Qexistence} ensures that $\bcQ_t$, the optimal alignment tensor between $\bcF_t$ and $\bcF_{\star}$ exists, and $\epsilon \coloneq 0.1$. Again, repeating the derivation for Lemma~\ref{lemma:Qexistence}, we have
\begin{align}\label{eqn:TR-cond}
\| \bcL_{\triangle} \ast_{\bPhi} \bcG_{\star}^{-\frac{1}{2}} \| \vee \| \bcR_{\triangle} \ast_{\bPhi} \bcG_{\star}^{-\frac{1}{2}} \| & \leq \epsilon.
\end{align}
The conclusion $\| \bcL_t \ast_{\bPhi} \bcR_t^H - \bcX_{\star} \|_F \leq 1.5 \dist(\bcF_t, \bcF_{\star})$ is a simple consequence of Lemma~\ref{lemma:tensor2factor}; see \eqref{eqn:disttensor} for details. In the following, we focus on proving the distance contraction.

First of all, we have by the definition of $\dist(\bcF_{t+1},\bcF_{\star})$ that
\begin{align*}
\dist^2(\bcF_{t+1}, \bcF_{\star}) \leq \| (\bcL_{t+1} \ast_{\bPhi} \bcQ_t - \bcL_{\star}) \ast_{\bPhi} \bcG_{\star}^{\frac{1}{2}} \|_F^2 + \| (\bcR_{t+1} \ast_{\bPhi} \bcQ_t^{-H} - \bcR_{\star}) \ast_{\bPhi} \bcG_{\star}^{\frac{1}{2}} \|_F^2.
\end{align*}
According to the update rule \eqref{eqn:TRupdate} and the decomposition $\bcL_{\sharp} \ast_{\bPhi} \bcR_{\sharp}^H - \bcX_{\star} = \bcL_{\triangle} \ast_{\bPhi} \bcR_{\sharp}^H + \bcL_{\star} \ast_{\bPhi} \bcR_{\triangle}^H$, we have
\begin{align}\label{eqn:TRexpand}
& (\bcL_{t+1} \ast_{\bPhi} \bcQ_t - \bcL_{\star}) \ast_{\bPhi} \bcG_{\star}^{\frac{1}{2}}  \nonumber \\
= & \Big( \bcL_{\sharp} - \eta \cA^{\ast} \cA(\bcL_{\sharp} \ast_{\bPhi} \bcR_{\sharp}^H - \bcX_{\star}) \ast_{\bPhi} \bcR_{\sharp} \ast_{\bPhi} (\bcR_{\sharp}^H \ast_{\bPhi} \bcR_{\sharp})^{-1} - \bcL_{\star} \Big) \ast_{\bPhi} \bcG_{\star}^{\frac{1}{2}}  \nonumber \\
= & \Big( \bcL_{\triangle} - \eta (\bcL_{\sharp} \ast_{\bPhi} \bcR_{\sharp}^H - \bcX_{\star}) \ast_{\bPhi} \bcR_{\sharp} \ast_{\bPhi} (\bcR_{\sharp}^H \ast_{\bPhi} \bcR_{\sharp})^{-1}  \nonumber \\
&\qquad\qquad - \eta (\cA^{\ast} \cA - \bcI_{n_1}) (\bcL_{\sharp} \ast_{\bPhi} \bcR_{\sharp}^H - \bcX_{\star}) \ast_{\bPhi} \bcR_{\sharp} \ast_{\bPhi} (\bcR_{\sharp}^H \ast_{\bPhi} \bcR_{\sharp})^{-1} \Big) \ast_{\bPhi} \bcG_{\star}^{\frac{1}{2}}  \nonumber \\
= & (1 - \eta) \bcL_{\triangle} \ast_{\bPhi} \bcG_{\star}^{\frac{1}{2}} - \eta \bcL_{\star} \ast_{\bPhi} \bcR_{\triangle}^H \ast_{\bPhi} \bcR_{\sharp} \ast_{\bPhi} (\bcR_{\sharp}^H \ast_{\bPhi} \bcR_{\sharp})^{-1} \ast_{\bPhi} \bcG_{\star}^{\frac{1}{2}}  \nonumber \\
&\qquad\qquad - \eta (\cA^{\ast} \cA - \bcI_{n_1}) (\bcL_{\sharp} \ast_{\bPhi} \bcR_{\sharp}^H - \bcX_{\star}) \ast_{\bPhi} \bcR_{\sharp} \ast_{\bPhi} (\bcR_{\sharp}^H \ast_{\bPhi} \bcR_{\sharp})^{-1} \ast_{\bPhi} \bcG_{\star}^{\frac{1}{2}}.
\end{align}
Taking the squared Frobenius norm of both sides of \eqref{eqn:TRexpand} to obtain
\begin{align*}
& \| (\bcL_{t+1} \ast_{\bPhi} \bcQ_t - \bcL_{\star}) \ast_{\bPhi} \bcG_{\star}^{\frac{1}{2}} \|_F^2 \\
= & \| (1 - \eta) \bcL_{\triangle} \ast_{\bPhi} \bcG_{\star}^{\frac{1}{2}} - \eta \bcL_{\star} \ast_{\bPhi} \bcR_{\triangle}^H \ast_{\bPhi} \bcR_{\sharp} \ast_{\bPhi} (\bcR_{\sharp}^H \ast_{\bPhi} \bcR_{\sharp})^{-1} \ast_{\bPhi} \bcG_{\star}^{\frac{1}{2}} \|_F^2  \\
\qquad & - 2 \eta (1 - \eta) \langle \bcL_{\triangle} \ast_{\bPhi} \bcG_{\star}^{\frac{1}{2}} , (\cA^{\ast} \cA - \bcI_{n_1}) (\bcL_{\sharp} \ast_{\bPhi} \bcR_{\sharp}^H - \bcX_{\star}) \ast_{\bPhi} \bcR_{\sharp} \ast_{\bPhi} (\bcR_{\sharp}^H \ast_{\bPhi} \bcR_{\sharp})^{-1} \ast_{\bPhi} \bcG_{\star}^{\frac{1}{2}} \rangle  \\
\qquad & + 2 \eta^2 \langle \bcL_{\star} \ast_{\bPhi} \bcR_{\triangle}^H \ast_{\bPhi} \bcR_{\sharp} \ast_{\bPhi} (\bcR_{\sharp}^H \ast_{\bPhi} \bcR_{\sharp})^{-1} \ast_{\bPhi} \bcG_{\star}^{\frac{1}{2}} ,  \\
&\qquad\qquad (\cA^{\ast} \cA - \bcI_{n_1}) (\bcL_{\sharp} \ast_{\bPhi} \bcR_{\sharp}^H - \bcX_{\star}) \ast_{\bPhi} \bcR_{\sharp} \ast_{\bPhi} (\bcR_{\sharp}^H \ast_{\bPhi} \bcR_{\sharp})^{-1} \ast_{\bPhi} \bcG_{\star}^{\frac{1}{2}} \rangle  \\
\qquad & + \eta^2 \| (\cA^{\ast} \cA - \bcI_{n_1}) (\bcL_{\sharp} \ast_{\bPhi} \bcR_{\sharp}^H - \bcX_{\star}) \ast_{\bPhi} \bcR_{\sharp} \ast_{\bPhi} (\bcR_{\sharp}^H \ast_{\bPhi} \bcR_{\sharp})^{-1} \ast_{\bPhi} \bcG_{\star}^{\frac{1}{2}} \|_F^2  \\
\coloneq & \mathfrak{S}_1 + \mathfrak{S}_2 + \mathfrak{S}_3 + \mathfrak{S}_4.
\end{align*}

\paragraph{Bound of $\mathfrak{S}_1$.} The first term $\mathfrak{S}_1$ has already been controlled in \eqref{eqn:TF-Lt-bound} as follows.
\begin{align*}
\mathfrak{S}_1 \leq \Big( (1 - \eta)^2 + \frac{2 \epsilon}{1 - \epsilon} \eta (1 - \eta) \Big) \| \bcL_{\triangle} \ast_{\bPhi} \bcG_{\star}^{\frac{1}{2}} \|_F^2 + \frac{2 \epsilon + \epsilon^2}{(1 - \epsilon)^2} \eta^2 \| \bcR_{\triangle} \ast_{\bPhi} \bcG_{\star}^{\frac{1}{2}} \|_F^2.
\end{align*}

\paragraph{Bound of $\mathfrak{S}_2$.} Using the decomposition $\bcL_{\sharp} \ast_{\bPhi} \bcR_{\sharp}^H - \bcX_{\star} = \bcL_{\triangle} \ast_{\bPhi} \bcR_{\star}^H + \bcL_{\star} \ast_{\bPhi} \bcR_{\triangle}^H + \bcL_{\triangle} \ast_{\bPhi} \bcR_{\triangle}^H$ and applying the triangle inequality to obtain
\begin{align*}
|\mathfrak{S}_2| & = 2 \eta (1 - \eta) \Big| \langle \bcL_{\triangle} \ast_{\bPhi} \bcG_{\star}^{\frac{1}{2}} , (\cA^{\ast} \cA - \bcI_{n_1}) (\bcL_{\sharp} \ast_{\bPhi} \bcR_{\sharp}^H - \bcX_{\star}) \ast_{\bPhi} \bcR_{\sharp} \ast_{\bPhi} (\bcR_{\sharp}^H \ast_{\bPhi} \bcR_{\sharp})^{-1} \ast_{\bPhi} \bcG_{\star}^{\frac{1}{2}} \rangle \Big|  \\
& \leq 2 \eta (1 - \eta) \Big( \Big| \langle \bcL_{\triangle} \ast_{\bPhi} \bcG_{\star}^{\frac{1}{2}} , (\cA^{\ast} \cA - \bcI_{n_1}) (\bcL_{\triangle} \ast_{\bPhi} \bcR_{\star}^H) \ast_{\bPhi} \bcR_{\sharp} \ast_{\bPhi} (\bcR_{\sharp}^H \ast_{\bPhi} \bcR_{\sharp})^{-1} \ast_{\bPhi} \bcG_{\star}^{\frac{1}{2}} \rangle \Big|  \\
&\quad + \Big| \langle \bcL_{\triangle} \ast_{\bPhi} \bcG_{\star}^{\frac{1}{2}} , (\cA^{\ast} \cA - \bcI_{n_1}) (\bcL_{\star} \ast_{\bPhi} \bcR_{\triangle}^H) \ast_{\bPhi} \bcR_{\sharp} \ast_{\bPhi} (\bcR_{\sharp}^H \ast_{\bPhi} \bcR_{\sharp})^{-1} \ast_{\bPhi} \bcG_{\star}^{\frac{1}{2}} \rangle \Big|  \\
&\quad + \Big| \langle \bcL_{\triangle} \ast_{\bPhi} \bcG_{\star}^{\frac{1}{2}} , (\cA^{\ast} \cA - \bcI_{n_1}) (\bcL_{\triangle} \ast_{\bPhi} \bcR_{\triangle}^H) \ast_{\bPhi} \bcR_{\sharp} \ast_{\bPhi} (\bcR_{\sharp}^H \ast_{\bPhi} \bcR_{\sharp})^{-1} \ast_{\bPhi} \bcG_{\star}^{\frac{1}{2}} \rangle \Big| \Big).
\end{align*}
Applying Lemma~\ref{lemma:TRIPdistance} to further obtain 
\begin{align*}
|\mathfrak{S}_2| & \leq 2 \eta (1 - \eta) \delta_{2r} \Big( \| \bcL_{\triangle} \ast_{\bPhi} \bcR_{\star}^H \|_F + \| \bcL_{\star} \ast_{\bPhi} \bcR_{\triangle}^H \|_F + \| \bcL_{\triangle} \ast_{\bPhi} \bcR_{\triangle}^H \|_F \Big)  \\
&\qquad\qquad \| \bcR_{\sharp} \ast_{\bPhi} (\bcR_{\sharp}^H \ast_{\bPhi} \bcR_{\sharp})^{-1} \ast_{\bPhi} \bcG_{\star} \ast_{\bPhi} \bcL_{\triangle}^H \|_F  \\
& \leq 2 \eta (1 - \eta) \delta_{2r} \Big( \| \bcL_{\triangle} \ast_{\bPhi} \bcR_{\star}^H \|_F + \| \bcL_{\star} \ast_{\bPhi} \bcR_{\triangle}^H \|_F + \| \bcL_{\triangle} \ast_{\bPhi} \bcR_{\triangle}^H \|_F \Big)  \\
&\qquad\qquad \| \bcR_{\sharp} \ast_{\bPhi} (\bcR_{\sharp}^H \ast_{\bPhi} \bcR_{\sharp})^{-1} \ast_{\bPhi} \bcG_{\star}^{\frac{1}{2}} \| \| \bcL_{\triangle} \ast_{\bPhi} \bcG_{\star}^{\frac{1}{2}} \|_F.
\end{align*}
Taking the condition \eqref{eqn:TR-cond} and Lemma~\ref{lemma:Weyl} and Lemma~\ref{lemma:tensor2factor} together to obtain 
\begin{align*}
\| \bcR_{\sharp} \ast_{\bPhi} (\bcR_{\sharp}^H \ast_{\bPhi} \bcR_{\sharp})^{-1} \ast_{\bPhi} \bcG_{\star}^{\frac{1}{2}} \| & \leq \frac{1}{1 - \epsilon};  \\
\| \bcL_{\triangle} \ast_{\bPhi} \bcR_{\star}^H \|_F + \| \bcL_{\star} \ast_{\bPhi} \bcR_{\triangle}^H \|_F + \| \bcL_{\triangle} \ast_{\bPhi} \bcR_{\triangle}^H \|_F & \leq (1 + \frac{\epsilon}{2}) \Big( \| \bcL_{\triangle} \ast_{\bPhi} \bcG_{\star}^{\frac{1}{2}} \|_F + \| \bcR_{\triangle} \ast_{\bPhi} \bcG_{\star}^{\frac{1}{2}} \|_F \Big).
\end{align*}
Hence, we have
\begin{align*}
|\mathfrak{S}_2| & \leq \eta (1 - \eta) \frac{\delta_{2r}(2 + \epsilon)}{1 - \epsilon} \Big( \| \bcL_{\triangle} \ast_{\bPhi} \bcG_{\star}^{\frac{1}{2}} \|_F + \| \bcR_{\triangle} \ast_{\bPhi} \bcG_{\star}^{\frac{1}{2}} \|_F \Big) \| \bcL_{\triangle} \ast_{\bPhi} \bcG_{\star}^{\frac{1}{2}} \|_F  \\
& = \eta (1 - \eta) \frac{\delta_{2r}(2 + \epsilon)}{1 - \epsilon} \Big( \| \bcL_{\triangle} \ast_{\bPhi} \bcG_{\star}^{\frac{1}{2}} \|_F^2 + \| \bcL_{\triangle} \ast_{\bPhi} \bcG_{\star}^{\frac{1}{2}} \|_F \| \bcR_{\triangle} \ast_{\bPhi} \bcG_{\star}^{\frac{1}{2}} \|_F \Big)  \\
& \leq \eta (1 - \eta) \frac{\delta_{2r}(2 + \epsilon)}{1 - \epsilon} \Big( \frac{3}{2} \| \bcL_{\triangle} \ast_{\bPhi} \bcG_{\star}^{\frac{1}{2}} \|_F^2 + \frac{1}{2} \| \bcR_{\triangle} \ast_{\bPhi} \bcG_{\star}^{\frac{1}{2}} \|_F \Big),
\end{align*}
where we use the elementary inequality $2ab \leq a^2 + b^2$ in the last line.

\paragraph{Bound of $\mathfrak{S}_3$.} The third term $\mathfrak{S}_3$ can be similarly bounded as 
\begin{align*}
|\mathfrak{S}_3| & \leq 2 \eta^2 \delta_{2r} \Big( \| \bcL_{\triangle} \ast_{\bPhi} \bcR_{\star}^H \|_F + \| \bcL_{\star} \ast_{\bPhi} \bcR_{\triangle}^H \|_F + \| \bcL_{\triangle} \ast_{\bPhi} \bcR_{\triangle}^H \|_F \Big)  \\
&\qquad\qquad \| \bcR_{\sharp} \ast_{\bPhi} (\bcR_{\sharp}^H \ast_{\bPhi} \bcR_{\sharp})^{-1} \ast_{\bPhi} \bcG_{\star} \ast_{\bPhi} (\bcR_{\sharp}^H \ast_{\bPhi} \bcR_{\sharp})^{-1} \ast_{\bPhi} \bcR_{\sharp}^H \ast_{\bPhi} \bcR_{\triangle} \ast_{\bPhi} \bcL_{\star}^H \|_F  \\
& \leq 2 \eta^2 \delta_{2r} \Big( \| \bcL_{\triangle} \ast_{\bPhi} \bcR_{\star}^H \|_F + \| \bcL_{\star} \ast_{\bPhi} \bcR_{\triangle}^H \|_F + \| \bcL_{\triangle} \ast_{\bPhi} \bcR_{\triangle}^H \|_F \Big)  \\
&\qquad\qquad \| \bcR_{\sharp} \ast_{\bPhi} (\bcR_{\sharp}^H \ast_{\bPhi} \bcR_{\sharp})^{-1} \ast_{\bPhi} \bcG_{\star}^{\frac{1}{2}} \|^2 \| \bcR_{\triangle} \ast_{\bPhi} \bcL_{\star}^H \|_F  \\
& \leq \eta^2 \frac{\delta_{2r}(2 + \epsilon)}{(1 - \epsilon)^2} \Big( \| \bcL_{\triangle} \ast_{\bPhi} \bcG_{\star}^{\frac{1}{2}} \|_F + \| \bcR_{\triangle} \ast_{\bPhi} \bcG_{\star}^{\frac{1}{2}} \|_F \Big) \| \bcR_{\triangle} \ast_{\bPhi} \bcG_{\star}^{\frac{1}{2}} \|_F  \\
& \leq \eta^2 \frac{\delta_{2r}(2 + \epsilon)}{(1 - \epsilon)^2} \Big( \frac{1}{2} \| \bcL_{\triangle} \ast_{\bPhi} \bcG_{\star}^{\frac{1}{2}} \|_F^2 + \frac{3}{2} \| \bcR_{\triangle} \ast_{\bPhi} \bcG_{\star}^{\frac{1}{2}} \|_F \Big).
\end{align*}

\paragraph{Bound of $\mathfrak{S}_4$.} For the last term $\mathfrak{S}_4$, we have 
\begin{align*}
\sqrt{\mathfrak{S}_4} & = \eta \| (\cA^{\ast} \cA - \bcI_{n_1}) (\bcL_{\sharp} \ast_{\bPhi} \bcR_{\sharp}^H - \bcX_{\star}) \ast_{\bPhi} \bcR_{\sharp} \ast_{\bPhi} (\bcR_{\sharp}^H \ast_{\bPhi} \bcR_{\sharp})^{-1} \ast_{\bPhi} \bcG_{\star}^{\frac{1}{2}} \|_F  \\
& \leq \eta \Big( \| (\cA^{\ast} \cA - \bcI_{n_1}) (\bcL_{\triangle} \ast_{\bPhi} \bcR_{\star}^H) \ast_{\bPhi} \bcR_{\sharp} \ast_{\bPhi} (\bcR_{\sharp}^H \ast_{\bPhi} \bcR_{\sharp})^{-1} \ast_{\bPhi} \bcG_{\star}^{\frac{1}{2}} \|_F  \\
&\qquad + \| (\cA^{\ast} \cA - \bcI_{n_1}) (\bcL_{\star} \ast_{\bPhi} \bcR_{\triangle}^H) \ast_{\bPhi} \bcR_{\sharp} \ast_{\bPhi} (\bcR_{\sharp}^H \ast_{\bPhi} \bcR_{\sharp})^{-1} \ast_{\bPhi} \bcG_{\star}^{\frac{1}{2}} \|_F  \\
&\qquad + \| (\cA^{\ast} \cA - \bcI_{n_1}) (\bcL_{\triangle} \ast_{\bPhi} \bcR_{\triangle}^H) \ast_{\bPhi} \bcR_{\sharp} \ast_{\bPhi} (\bcR_{\sharp}^H \ast_{\bPhi} \bcR_{\sharp})^{-1} \ast_{\bPhi} \bcG_{\star}^{\frac{1}{2}} \|_F \Big).
\end{align*}
Following \eqref{eqn:coroTRIP}, we obtain 
\begin{align*}
\sqrt{\mathfrak{S}_4} & \leq \eta \delta_{2r} \Big( \| \bcL_{\triangle} \ast_{\bPhi} \bcR_{\star}^H \|_F + \| \bcL_{\star} \ast_{\bPhi} \bcR_{\triangle}^H \|_F + \| \bcL_{\triangle} \ast_{\bPhi} \bcR_{\triangle}^H \|_F \Big) \| \bcR_{\sharp} \ast_{\bPhi} (\bcR_{\sharp}^H \ast_{\bPhi} \bcR_{\sharp})^{-1} \ast_{\bPhi} \bcG_{\star}^{\frac{1}{2}} \|  \\
& \leq \eta \frac{\delta_{2r} (2 + \epsilon)}{2 (1 - \epsilon)} \Big( \| \bcL_{\triangle} \ast_{\bPhi} \bcG_{\star}^{\frac{1}{2}} \|_F + \| \bcR_{\triangle} \ast_{\bPhi} \bcG_{\star}^{\frac{1}{2}} \|_F \Big).
\end{align*}
We can then take the squares of both sides and use $(a+b)^2 \leq 2 a^2 + 2 b^2$ to reach
\begin{align*}
\mathfrak{S}_4 \leq \eta^2 \frac{\delta_{2r}^2 (2 + \epsilon)^2}{2 (1 - \epsilon)^2} \Big( \| \bcL_{\triangle} \ast_{\bPhi} \bcG_{\star}^{\frac{1}{2}} \|_F^2 + \| \bcR_{\triangle} \ast_{\bPhi} \bcG_{\star}^{\frac{1}{2}} \|_F^2 \Big).
\end{align*}
Taking the bounds for $\mathfrak{S}_1, \mathfrak{S}_2, \mathfrak{S}_3, \mathfrak{S}_4$ collectively yields
\begin{align*}
\| (\bcL_{t+1} \ast_{\bPhi} \bcQ_t - \bcL_{\star}) \ast_{\bPhi} \bcG_{\star}^{\frac{1}{2}} \|_F^2 & \leq \Big( (1 - \eta)^2 + \frac{2 \epsilon}{1 - \epsilon} \eta (1 - \eta) \Big) \| \bcL_{\triangle} \ast_{\bPhi} \bcG_{\star}^{\frac{1}{2}} \|_F^2 + \frac{2 \epsilon + \epsilon^2}{(1 - \epsilon)^2} \eta^2 \| \bcR_{\triangle} \ast_{\bPhi} \bcG_{\star}^{\frac{1}{2}} \|_F^2  \\
&\quad + \eta (1 - \eta) \frac{\delta_{2r}(2 + \epsilon)}{1 - \epsilon} \Big( \frac{3}{2} \| \bcL_{\triangle} \ast_{\bPhi} \bcG_{\star}^{\frac{1}{2}} \|_F^2 + \frac{1}{2} \| \bcR_{\triangle} \ast_{\bPhi} \bcG_{\star}^{\frac{1}{2}} \|_F \Big)  \\
&\quad + \eta^2 \frac{\delta_{2r}(2 + \epsilon)}{(1 - \epsilon)^2} \Big( \frac{1}{2} \| \bcL_{\triangle} \ast_{\bPhi} \bcG_{\star}^{\frac{1}{2}} \|_F^2 + \frac{3}{2} \| \bcR_{\triangle} \ast_{\bPhi} \bcG_{\star}^{\frac{1}{2}} \|_F \Big)  \\
&\quad + \eta^2 \frac{\delta_{2r}^2 (2 + \epsilon)^2}{2 (1 - \epsilon)^2} \Big( \| \bcL_{\triangle} \ast_{\bPhi} \bcG_{\star}^{\frac{1}{2}} \|_F^2 + \| \bcR_{\triangle} \ast_{\bPhi} \bcG_{\star}^{\frac{1}{2}} \|_F^2 \Big).
\end{align*}
A similar bound holds for the second term in \eqref{eqn:TRexpand}. Therefore we obtain 
\begin{align*}
\| (\bcL_{t+1} \ast_{\bPhi} \bcQ_t - \bcL_{\star}) \ast_{\bPhi} \bcG_{\star}^{\frac{1}{2}} \|_F^2 + \| (\bcR_{t+1} \ast_{\bPhi} \bcQ_t^{-H} - \bcR_{\star}) \ast_{\bPhi} \bcG_{\star}^{\frac{1}{2}} \|_F^2 & \leq \hbar^2 (\eta;\epsilon,\delta_{2r}) \dist^2(\bcF_t, \bcF_{\star}),
\end{align*}
where the contraction rate is given by
\begin{align*}
\hbar^2 (\eta;\epsilon,\delta_{2r}) \coloneq (1 - \eta)^2 + \frac{2 \epsilon + \delta_{2r} (4 + 2 \epsilon)}{1 - \epsilon} \eta (1 - \eta) + \frac{2 \epsilon + \epsilon^2 + \delta_{2r} (4 + 2 \epsilon) + \delta_{2r}^2 (2 + \epsilon)^2}{(1 - \epsilon)^2} \eta^2.
\end{align*}
With $\epsilon = 0.1$, $\delta_{2r} \leq 0.02$, and $0 < \eta \leq \frac{2}{3}$, we have $\hbar (\eta;\epsilon,\delta_{2r}) \leq 1 - 0.6 \eta$. Thus we conclude that 
\begin{align*}
\dist(\bcF_{t+1}, \bcF_{\star}) & \leq \sqrt{\| (\bcL_{t+1} \ast_{\bPhi} \bcQ_t - \bcL_{\star}) \ast_{\bPhi} \bcG_{\star}^{\frac{1}{2}} \|_F^2 + \| (\bcR_{t+1} \ast_{\bPhi} \bcQ_t^{-H} - \bcR_{\star}) \ast_{\bPhi} \bcG_{\star}^{\frac{1}{2}} \|_F^2}  \\
& \leq (1 - 0.6 \eta) \dist(\bcF_t, \bcF_{\star}).
\end{align*}

\subsection{Proof of Lemma~\ref{lemma:TRinitial}}

We first invoke Lemma~\ref{lemma:Procrustes} and use the fact that $\bcL_0 \ast_{\bPhi} \bcR_0^H - \bcX_{\star}$ has tubal rank at most $2r$ to obtain
\begin{align*}
\dist(\bcF_0, \bcF_{\star}) \leq \sqrt{\sqrt{2}+1} \| \bcL_0 \ast_{\bPhi} \bcR_0^H - \bcX_{\star} \|_F \leq \sqrt{2(\sqrt{2}+1)} \| \bcL_0 \ast_{\bPhi} \bcR_0^H - \bcX_{\star} \|_{F,r}.
\end{align*}
On the other hand, combining Lemma~\ref{lemma:parFnormvariation} and \eqref{eqn:coroTRIP} to reach at 
\begin{align*}
\| (\cA^{\ast} \cA - \bcI_{n_1}) (\bcX_{\star}) \|_{F,r} = \max_{\widetilde{\bcR} \in \R^{n_2 \times r \times n_3} : \| \widetilde{\bcR} \| \le 1} \| (\cA^{\ast} \cA - \bcI_{n_1}) (\bcX_{\star}) \ast_{\bPhi} \widetilde{\bcR} \|_F \leq \delta_{2r} \| \bcX_{\star} \|_F.
\end{align*}
Note that $\bcL_0 \ast_{\bPhi} \bcR_0^H$ is the best tubal rank-$r$ approximation of $\cA^{\ast} \cA(\bcX_{\star})$, we can apply the triangle inequality combined with Lemma~\ref{lemma:Fnorm-Eckart-Yang} to obtain
\begin{align*}
\| \bcL_0 \ast_{\bPhi} \bcR_0^H - \bcX_{\star} \|_{F,r} & \leq \| \cA^{\ast} \cA (\bcX_{\star}) - \bcL_0 \ast_{\bPhi} \bcR_0^H \|_{F,r} + \| \cA^{\ast} \cA (\bcX_{\star}) - \bcX_{\star} \|_{F,r}  \\
& \leq 2 \| (\cA^{\ast} \cA - \bcI_{n_1}) (\bcX_{\star}) \|_{F,r} \leq 2 \delta_{2r} \| \bcX_{\star} \|_F.
\end{align*}
As a result, we have 
\begin{align*}
\dist(\bF_{0},\bF_{\star}) \leq 2 \sqrt{2(\sqrt{2}+1)} \delta_{2r} \| \bcX_{\star} \|_F \leq 5 \delta_{2r} \sqrt{\frac{s_r}{\ell}} \kappa \bar{\sigma}_{s_r} (\bcX_{\star}).
\end{align*}

\section{Computational Complexity of Algorithm~\ref{algo:ScaledGDTRPCA}}
\label{sec:complexity}

We provide the breakdown of ScaledGD's computational complexity for tensor RPCA in Algorithm~\ref{algo:ScaledGDTRPCA}:
\begin{enumerate}
  \item Compute $\bcL_t \ast_{\bPhi} \bcR_t^H$ ($\bcL_t \in \R^{n_1 \times r \times n_3}$ and $\bcR_t \in \R^{n_2 \times r \times n_3}$): $n_1 n_2 n_3 r + n_1 n_2 n_3^2$ flops.
  \item Compute $\bcY - \bcL_t \ast_{\bPhi} \bcR_t^H$: $n_1 n_2 n_3$ flops.
  \item Soft-thresholding on $\bcY - \bcL_t \ast_{\bPhi} \bcR_t^H$: $n_1 n_2 n_3$ flops.
  \item Compute $\bcL_t \ast_{\bPhi} \bcR_t^H + \bcS_{t+1} - \bcY = \bcS_{t+1} - (\bcY - \bcL_t \ast_{\bPhi} \bcR_t^H)$ and transform the result into the spectral domain: $n_1 n_2 n_3 + n_1 n_2 n_3^2$ flops.
  \item Compute $\widebar{\bR}_t^{(k)^H} \widebar{\bR}_t^{(k)}$ ($\widebar{\bR}_t^{(k)} \in \C^{n_2 \times r}$): $n_2 r^2$ flops.
  \item Compute $(\widebar{\bR}_t^{(k)^H} \widebar{\bR}_t^{(k)})^{-1}$: $\cO(r^3)$ flops.
  \item Compute $\widebar{\bR}_t^{(k)} (\widebar{\bR}_t^{(k)^H} \widebar{\bR}_t^{(k)})^{-1}$: $n_2 r^2$ flops.
  \item Compute $(\widebar{\bL}_t^{(k)} \widebar{\bR}_t^{(k)^H} + \widebar{\bS}_{t+1}^{(k)} - \widebar{\bY}^{(k)}) \cdot \widebar{\bR}_t^{(k)} (\widebar{\bR}_t^{(k)^H} \widebar{\bR}_t^{(k)})^{-1}$: $n_1 n_2 r$ flops.
  \item Compute $\widebar{\bL}_{t+1}^{(k)} = \widebar{\bL}_t^{(k)} - \eta (\widebar{\bL}^{(k)} \widebar{\bR}^{(k)^H} + \widebar{\bS}^{(k)} - \widebar{\bY}^{(k)}) \cdot \widebar{\bR}^{(k)} (\widebar{\bR}^{(k)^H} \widebar{\bR}^{(k)})^{-1}$: $2 n_1 r$ flops.
  \item Repeat step 5 - 9 for computing $\widebar{\bR}_{t+1}^{(k)}$: $2 n_1 r^2 + \cO(r^3) + n_1 n_2 r + 2 n_2 r$ flops.
\end{enumerate}

In total, ScaledGD costs $\cO(n_1 n_2 n_3 r + n_1 n_2 n_3^2 + n_1 n_2 n_3 + (n_1 + n_2) n_3 r^2 + n_3 r^3 + (n_1 + n_2) n_3 r) = \cO(n_1 n_2 n_3 r + n_1 n_2 n_3^2 + (n_1 + n_2) n_3 r^2 + n_3 r^3)$ flops per iteration provided that $r \ll n_1 \wedge n_2$. We conclude by noting that for some special invertible linear transforms $L$, e.g., DFT, since the application of DFT on an $n_3$-dimensional vector requires $\cO(n_3 \log(n_3))$ operations, the per-iteration complexity is $\cO(n_1 n_2 n_3 r + n_1 n_2 n_3 \log(n_3) + (n_1 + n_2) n_3 r^2 + n_3 r^3)$.

\end{document}